\documentclass[twoside]{article}

\usepackage[accepted]{aistats2023}
\usepackage[utf8]{inputenc} % allow utf-8 input
\usepackage[T1]{fontenc}    % use 8-bit T1 fonts
\usepackage{hyperref}       % hyperlinks
\usepackage{url}            % simple URL typesetting
\usepackage{booktabs}       % professional-quality tables
\usepackage{amsfonts}       % blackboard math symbols
\usepackage{nicefrac}       % compact symbols for 1/2, etc.
\usepackage{microtype}      % microtypography
\usepackage{wrapfig,lipsum}
\usepackage{amsthm, mathrsfs}
\usepackage{amssymb}
\usepackage{paralist}
\usepackage{tcolorbox}
\usepackage{pifont}
\usepackage{sectsty}
\usepackage{natbib}
\usepackage{xargs}
\usepackage{dsfont}
\usepackage{caption}
\usepackage{subcaption} % graphicx,
\usepackage{upgreek}
\usepackage{mathrsfs}
\usepackage[noend]{algpseudocode}  % Eviter les endif, endfor
\usepackage{algorithm}
\usepackage{algorithmicx}
\hypersetup{colorlinks,citecolor = blue}
\graphicspath{{images/}}

\usepackage{enumitem}
\usepackage{aliascnt}
\usepackage{cleveref}
\usepackage{todonotes}
\usepackage{autonum}

\newtheorem{theorem}{Theorem}
\crefname{theorem}{theorem}{Theorems}
\Crefname{Theorem}{Theorem}{Theorems}

\newtheorem*{lemma_nonumber*}{Lemma}

\newaliascnt{lemma}{theorem}
\newtheorem{lemma}[lemma]{Lemma}
\aliascntresetthe{lemma}
\crefname{lemma}{lemma}{lemmas}
\Crefname{Lemma}{Lemma}{Lemmas}

\newaliascnt{corollary}{theorem}
\newtheorem{corollary}[corollary]{Corollary}
\aliascntresetthe{corollary}
\crefname{corollary}{corollary}{corollaries}
\Crefname{Corollary}{Corollary}{Corollaries}

\newaliascnt{proposition}{theorem}
\newtheorem{proposition}[proposition]{Proposition}
\aliascntresetthe{proposition}
\crefname{proposition}{proposition}{propositions}
\Crefname{Proposition}{Proposition}{Propositions}

\newaliascnt{definition}{theorem}

\aliascntresetthe{definition}
\crefname{definition}{definition}{definitions}
\Crefname{Definition}{Definition}{Definitions}

\newaliascnt{remark}{theorem}
\newtheorem{remark}[remark]{Remark}
\aliascntresetthe{remark}
\crefname{remark}{remark}{remarks}
\Crefname{Remark}{Remark}{Remarks}

\crefname{example}{example}{examples}
\Crefname{Example}{Example}{Examples}

\crefname{figure}{figure}{figures}
\Crefname{Figure}{Figure}{Figures}

\newtheorem{assumption}{\textbf{H}\hspace{-3pt}}
\crefformat{assumption}{{\textbf{H}}#2#1#3}

\newtheorem{assA}{\textbf{A}\hspace{-3pt}}
\crefformat{assA}{{\textbf{A}}#2#1#3}

\newtheorem{HX}{\textbf{HX}\hspace{-3pt}}
\crefformat{HX}{{\textbf{HX}}#2#1#3}

\definecolor{aurometalsaurus}{rgb}{0.43, 0.5, 0.5}
\definecolor{britishracinggreen}{rgb}{0.0, 0.26, 0.15}
\definecolor{burntumber}{rgb}{0.54, 0.2, 0.14}
\definecolor{cobalt}{rgb}{0.0, 0.28, 0.67}
\definecolor{bulgarianrose}{rgb}{0.28, 0.02, 0.03}
\definecolor{ceruleanblue}{rgb}{0.16, 0.32, 0.75}

\definecolor{darkgreen}{RGB}{0,128,0}
\definecolor{darkorange}{RGB}{255,140,0}
\definecolor{darkblue}{RGB}{0,0,139}

\newcommand{\PP}{\mathbb{P}}

\def\msb{\mathsf{B}}
\def\msc{\mathsf{C}}
\def\mse{\mathsf{E}}

\def\mcc{\mathcal{C}}
\def\mce{\mathcal{E}}

\def\rset{\mathbb{R}}

\def\nset{\mathbb{N}}
\def\nsets{\mathbb{N}^*}

\def\rmd{\mathrm{d}}

\def\trace{\operatorname{Tr}}
\newcommand{\argmax}{\operatorname*{arg\,max}}
\newcommand{\LeftEqNo}{\let\veqno\@@leqno}

\newcommand{\normLigne}[1]{\Vert #1 \Vert}
\newcommand{\norm}[1]{\Vert #1 \Vert}

\newcommand{\plusinfty}{+\infty}

\def\ie{\textit{i.e.}}

\newcommand{\ocint}[1]{\left(#1\right]}
\newcommand{\ooint}[1]{\left(#1\right)}
\newcommand{\ccint}[1]{\left[#1\right]}

\newcommand{\tcr}[1]{\textcolor{red}{#1}}

\newcommandx\sequence[3][2=,3=]
{\ifthenelse{\equal{#3}{}}{\ensuremath{\{ #1_{#2}\}}}{\ensuremath{\{ #1_{#2}, \eqsp #2 \in #3 \}}}}
\newcommandx\sequenceD[3][2=,3=]
{\ifthenelse{\equal{#3}{}}{\ensuremath{\{ #1_{#2}\}}}{\ensuremath{( #1)_{ #2 \in #3} }}}
\newcommandx\seq[3][2=,3=]
{\ifthenelse{\equal{#3}{}}{\ensuremath{\{ #1 \, : \, #2\}}}{\ensuremath{\{ #1_{#2}, \eqsp  #3 \}}}}

\newcommand{\wrt}{w.r.t.}

\def\iid{\text{i.i.d.}}

\def\eg{e.g.}

\newcommand{\opnorm}[1]{{\left\vert\kern-0.25ex\left\vert\kern-0.25ex\left\vert #1
    \right\vert\kern-0.25ex\right\vert\kern-0.25ex\right\vert}}

\newcommand\coupling[2]{\Gamma(\mu,\nu)}

\newcommandx{\wasserstein}[2][1=2]{\ifthenelse{\equal{#1}{}}{\mathbf{W}}{\mathbf{W}_{#1}}}
 \newcommandx{\wassersteinLigne}[3][1=\distance,3=]{\mathbf{W}_{#1}^{#3}(#2)}
 \newcommandx{\wassersteinD}[1][1=\distance]{\mathbf{W}_{#1}}
 \newcommandx{\wassersteinDLigne}[1][1=\distance]{\mathbf{W}_{#1}}

\def\bgamma{\bar{\gamma}}

\def\bgammavrs{\bar{\gamma}^{\mathrm{Vr}\star}}

\def\tP{\tilde{P}}

\def\txts{\textstyle}

\newcommand{\note}[1]{\textcolor{red}{\newline[\textbf{note:} {#1}]\hrule}}

\newcommand{\cobalt}[1]{\textcolor{cobalt}{#1}}

\newcommand{\pause}[1]{\bgroup\color{gray}{#1}\egroup}
\newcommand{\new}[1]{\bgroup\color{cobalt}{#1}\egroup}
\newcommand{\N}{\mathbb{N}}

\newcommand{\R}{\mathbb{R}}
\newcommand{\E}{\mathbb{E}}

\newcommand{\msy}{\mathsf{Y}}
\newcommand{\mcy}{\mathcal{Y}}

\newcommand{\Rd}{\mathbb{R}^{d}}

\newcommand{\pr}[1]{\left({#1}\right)}
\newcommand{\prn}[1]{({\textstyle{#1}})}
\newcommand{\prbig}[1]{\big({#1}\big)}

\newcommand{\br}[1]{\left[{#1}\right]}
\newcommand{\bbr}[1]{\left\{{#1}\right\}}
\newcommand{\brn}[1]{[{\textstyle{#1}}]}

\newcommand{\ac}[1]{\left\{{#1}\right\}}
\newcommand{\acn}[1]{\{{\textstyle{#1}}\}}

\newcommand{\acBig}[1]{\Big\{{#1}\Big\}}

\newcommand{\normlr}[1]{\left\|{#1}\right\|}
\newcommand{\normn}[1]{\|{\textstyle{#1}}\|}

\newcommand{\abs}[1]{\left\lvert{#1}\right\rvert}
\newcommand{\absn}[1]{\lvert{#1}\rvert}
\newcommand{\ps}[2]{\left\langle{#1},{#2}\right\rangle}  % produit scalaire

\newcommand{\argmin}{\operatornamewithlimits{\arg\min}}

\newcommand{\Oh}{\operatorname{\mathrm{O}}}
\newcommand{\prob}{\mathbb{P}}

\newcommand{\1}{\mathds{1}}

\newcommand{\nofrac}[2]{{#1}/{#2}}

\newcommand{\eqsp}{\,}
\newcommand{\gauss}{\mathbf{N}}
\newcommand{\wass}{\wassersteinD[2]}

\newcommand{\var}{\operatorname{Var}}
\newcommand{\potential}{U}

\def\pc{p_{\mathrm{c}}}
\def\qc{q_{\mathrm{c}}}

\newcommand{\funC}{\mathscr{C}}
\newcommand{\funD}{\mathscr{D}}
\newcommand{\funY}{\mathscr{Y}}
\newcommand{\funG}{\mathscr{G}}

\newcommand{\conv}{m}
\newcommand{\gradsto}{H}

\newcommand{\cte}{\mathrm{C}^{\gamma}}
\newcommand{\ctesigma}{\mathrm{C}_{\sigma}^{\gamma}}
\newcommand{\ctedzero}{\mathrm{C}_{d}^{\gamma}}

\newcommand{\cterate}{\mathrm{C}_{r}^{\gamma}}
\newcommand{\ctev}{\mathrm{C}_{V}^{\gamma}}
\newcommand{\ctedelta}{\mathrm{C}_{\delta}^{\gamma}}
\newcommand{\cteeps}{\mathrm{C}_{\epsilon}^{\gamma}}

\newcommand{\betaempty}{}
\newcommand{\betaemptysquared}{}
\newcommand{\betaemptyinv}{}
\newcommand{\betaemptyinvtwo}{}
\newcommand{\barpotential}{\bar{\potential}}
\newcommand{\newrate}{\alpha}
\def\algoun{\textsc{FALD}}
\def\FALD{\textsc{FALD}}
\def\VRFALDs{\textsc{VR-FALD}$^\star$}
\def\Scaffold{\textsc{Scaffold}}  % {SCAFFOLD} {\textsc{Scaffold}}
\def\SGLD{\textsc{SGLD}}
\def\SGD{\textsc{SGD}}
\def\langevinrest{\mathsf{J}}
\def\heterogeneity{\mathsf{H}}
\def\initconst{\mathsf{I}}
\def\initvrsconst{\mathsf{I}^{\mathrm{Vr}{\star}}}
\def\varconst{\mathsf{V}}

\newcommand{\algoquatre}{\textsc{VR-FALD}$^\star$}
\newcommandx{\maingradi}[1][1=k]{\widehat{\nabla}\potential^i_{#1}}
\newcommandx{\maingrad}[2][1=k,2=i]{\widehat{\nabla}\potential^{#2}_{#1}}

\newcommandx{\maingradiD}[1][1=k]{\widehat{\nabla}\potential^i_{#1}}
\newcommandx{\maingradD}[2][1=k,2=i]{\widehat{\nabla}\potential^{#2}_{#1}}

\newcommand{\constvr}[1]{\tcr{REMOVE}}  % {\mathrm{C}_{#1}^{\textsc{Vr}}}

\newcommand{\consts}[1]{\tcr{REMOVE}}  % {\mathrm{C}_{#1}^{\star}}
\def\mug{\mu^{(\mathrm{F})}}
\def\nug{\mu^{(\gamma)}}
\def\mugvrs{\mu^{(\mathrm{Vr}\star)}}

\def\dist{d}
\def\constsvrg{\omega}
\def\constgradsto{\tilde{\omega}}
\def\Xcontinuous{\mathsf{X}}
\def\Xavg{X}
\def\Xlocal{X}
\def\Xupdate{\tilde{X}}
\def\pce{p_{\mathrm{c},\epsilon}}

\def\bgc{\bgroup\color{cobalt}}
\def\bgg{\bgroup\color{gray}}
\def\egp{\egroup}

\newcommand{\mat}[1]{\mathrm{#1}}

\begin{document}

\twocolumn[

\aistatstitle{Federated Averaging Langevin Dynamics: \\ Toward a unified theory and new algorithms}

\aistatsauthor{
  Vincent Plassier
  \And
  Alain Durmus
  \And
  \'Eric Moulines
}

\aistatsaddress{
  CMAP, \'Ecole Polytechnique \\
  Lagrange Mathematics and  \\
  Computing Research Center
  \And
  CMAP, École Polytechnique\\
  Institut Polytechnique de Paris
  \And
  CMAP, École Polytechnique\\
  Institut Polytechnique de Paris
}]

\allowdisplaybreaks

\begin{abstract}
  This paper focuses on Bayesian inference in a federated learning context (FL). While several distributed MCMC algorithms have been proposed, few consider the specific limitations of FL such as communication bottlenecks and statistical heterogeneity. Recently, Federated Averaging Langevin Dynamics (\FALD{}) was introduced, which extends the Federated Averaging algorithm to Bayesian inference. We obtain a novel tight non-asymptotic upper bound on the Wasserstein distance to the global posterior for \FALD{}. This bound highlights the effects of statistical heterogeneity, which causes a drift in the local updates that negatively affects convergence. We propose a new algorithm \VRFALDs{} that uses control variates to correct the client drift. We establish non-asymptotic bounds showing that \VRFALDs{} is not affected by statistical heterogeneity.
  Finally, we illustrate our results on several FL benchmarks for Bayesian inference.
\end{abstract}

\addtocontents{toc}{\protect\setcounter{tocdepth}{0}}  % to remove section from the TOC

% !TEX root = main.tex

\section{Introduction}\label{sec:intro}

The paradigm of fully centralized machine learning is increasingly at odds with real-world use cases. Centralized machine learning leads to (a) data processing bottlenecks, (b) inefficient use of communication resources and (c)  risks exposing individuals' private data.
As storage and computational capacity increases at the agent level, it becomes increasingly attractive to decentralize computational tasks whenever possible. The term \emph{federated} learning (FL) was recently coined to capture some aspects of this grand challenge~\citep{mcmahan2017communication,kairouz2019advances,yang2019federated,alistarh2017qsgd,horvath2022stochastic,wang2021field}.

% There is now an extensive literature on supervised learning in the context of FL~\citep{}.
Reducing communication costs has been identified as one of the major challenges of FL~\citep{kairouz2019advances}.
Two main approaches have been proposed to achieve this goal. In the former, agents perform multiple local optimization steps before sending a model update to the central node \citep{mcmahan2017communication}. The latter consists in compressing the messages exchanged \citep{alistarh2017qsgd,horvath2022stochastic}.
In this paper, we focus on the first approach which is widely used in practice.
% Local updates are widely used in practice. but raise many practical and theoretical problems.
However, due to statistical heterogeneity, performing multiple steps can hinder convergence, as model updates target each agent's local minimizer \citep{li2019convergence, ro2021communication}. This results in a trade-off between communication cost and convergence \citep{wang2020tackling} and a need for algorithms that mitigate \emph{client drift}~\citep{scaffold20}.  % (\Scaffold{}).  % ,li2020federated, \textsc{Fed-Prox}

Most of existing FL algorithms minimize a training loss.
However, their results do not provide reliable uncertainty quantification, a strong requirement in safety-critical applications \citep{coglianese2016regulating,fatima2017survey}.
% such as medical diagnostics \citep{fatima2017survey} or autonomous decision-making \citep{coglianese2016regulating}
We address this problem by considering the federated version of Bayesian inference \citep{Welling11,yurochkin2019bayesian,chen2020fedbe,izmailov2021bayesian,wilson2021evaluating}.
The objective is to compute the predictive distribution, highest posterior density regions (HPD).
To this end, it is required to sample the posterior distribution $\pi\propto\exp(-\potential)$ associated with the model at hand. %Boltzmann-Gibbs distribution $\pi\propto\exp(-\potential)$.
This target posterior decomposes into the product of local posteriors $\pi=\prod_{i\in[b]}\pi^i$.
It is well known that sampling according to  product distributions \citep{neiswanger2014asymptotically,JMLRv14hoffman13a,minsker2014,Wang2015,alshedivat2021federated,dai2021bayesian} raises serious computational challenges even when sampling from each local posterior $\pi^{i}$ is reasonably easy. We tackle this question in our contributions which can be summarized as follows.

\textbf{Contributions.}
\begin{itemize} %[wide,nosep,leftmargin=!,labelindent=0pt,labelwidth=!,itemsep=.5ex]
    \item We study a random loop version of the \FALD{} algorithm proposed in \citet{deng2021convergence}, and we establish non-asymptotic upper bounds in Wasserstein distance for strongly convex potentials $\potential$.
    An analysis of \FALD{} was conducted in \citep[Theorem 5.7]{deng2021convergence}. However, the proof is plagued by an error; see \Cref{subsec:salad}.
    \item We give matching lower bounds to show that even with full batch gradients, \FALD{} can be slower than Stochastic Gradient Langevin Dynamics (\SGLD{}) due to client-drift.
    \item We propose a new method \VRFALDs{} that circumvents the shortcomings of \FALD{}. This algorithm extends the Shifted Local-SVRG of \citet{gorbunov2021local} to the Bayesian context. It combines Stochastic Variance Reduced Gradient (SVRG) Langevin Dynamics (LD) \citep{dubey2016variance} and adapts the bias reduction techniques from \Scaffold{} \citep{scaffold20}.
    \item We derive theoretical guarantees for \VRFALDs{} which highlight its gradient variance reduction effect and its ability to deal with data heterogeneity.
    \item The results are based on a general framework developed in the supplement, that encompasses a broad family of federated Bayes algorithms based on Langevin dynamics.
    This is the first unifying study among existing works on federated Bayesian inference.
    \item Finally, in \Cref{main:sec-numerical-exp} we illustrate our results using classical FL benchmarks and provide a thorough comparison with existing FL Bayesian methods.
\end{itemize}

\textbf{Related works.}
Many distributed MCMC algorithms have been proposed in the last decade and it is difficult to credit all the references. The first significant contributions in this direction are  the  Consensus Monte Carlo (CMC) approach and ``embarrassingly parallel'' MCMC algorithms; see, e.g.~\citet{neiswanger2014asymptotically,Wang2013,scott2016bayes}.
These methods require running separate MCMC chains on each client / computational node, with each chain targeting the local posterior $\pi^{i}$.
In the final stage, the algorithms recombine the samples from these chains to generate samples from the desired global posterior $\pi$ \citep{minsker2014}. The local posteriors may differ significantly from each other due to statistical heterogeneity, data imbalance, and / or inaccurate approximation. The effectiveness of the final combinations is either based on stringent assumptions on the local likelihoods \citep{liu2014distributed,nemeth2018merging,mesquita2020embarrassingly,chittoor2021coded}  or on ``fusion'' algorithms that are exact but scale badly with the dimension; see, e.g.~\citet{dai2021bayesian,de2022parallel}.

\citet{vono2020asymptotically,rendell2020global,plassier2021dg,vono2022efficient} introduced hierarchical Bayesian models to simulate separate MCMC chains on each machine. Inspired by the alternating direction method of multipliers \citep{boyd2011}, each client is assigned an auxiliary parameter that is conditionally independent given the server parameter. These authors developed MCMC schemes which alternate between sampling the clients parameters given the server parameter, and sampling the server parameter given the clients parameters. However, these approaches require tuning an additional hyperparameter to control the dispersion of the ``local parameters''. This parameter characterizes the trade-off between computational tractability and closeness to the original target distribution.

A competing approach to Federated Averaring, the quantized-\SGD\ scheme, has been proposed in \citep{alistarh2017qsgd} for non Bayesian FL. In this framework, the agents do not adapt parameters locally but a random subset of the agents compute at each iteration a new gradient estimator and transmit a compressed form---see \citet{haddadpour2021federated}, among many others, \citep{bernstein2018signsgd,tang20211} for scalar quantization or \citep{shlezinger2020uveqfed}, for vector quantization.
These approaches have been extended to the Bayesian inference context in \citet{lee2020bayesian,zhang:fl:22,vono2022qlsd}. Performance analysis is given in \citet{vono2022qlsd,sun2022federated}.

The Federated Gradient Stochastic Langevin Dynamics (FSGLD algorithm introduced by \citet{el2021federated} extends the distributed-SGLD (DSGLD) \citep{Ahn14} to the FL setting. Specifically, FSGLD operates passing a Markov chain between computing nodes and using only local data to estimate gradients at each step.

Methods with multiple local steps have been considered by several authors.
\citet{deng2021convergence} designed \FALD{} as a Bayesian version of \textsc{FedAvg}.
\citet{alshedivat2021federated} proposed \textsc{FedPa} as a generalization of \textsc{FedAvg}. This method performs several local steps to infer Gaussian approximations of the clients local parameters. These local parameters are then reweighted using the estimated local means and covariance matrices before being aggregated on the central server.

% \citep{liu2021bayesian} online Laplace approximation : très proche de FEDPA car ce sert également de
% la formule close pour le produit de densités gaussiennes. Cependant, l'algorithme d'optimisation
% sur le client est différent. Le prior gaussien évolue avec $\gauss(\theta_{\text{server}}, \Sigma_{\text{server}})$ (contrainte à ne pas
% trop s'éloigner du paramètre serveur) (“guide the local training”, et tacler l'hétérogénéité) (“global posterior probabilistic parameters distributed from the server as priors”)

\textbf{Notation and Convention.}
The Euclidean norm on $\mathbb{R}^d$ is denoted by $\|\cdot\|$, and we set $\nsets = \nset\setminus\{0\}$.
For $n \in \N^{\star}$, we refer to $\{1,\ldots,n\}$ with the notation $[n]$.
We denote by $\mathcal{P}_{2}(\Rd)$ the set of probability measures on $\rset^d$ with finite $2$-moment.
For any random variable $\xi$ with values in $\R^d$, we define $\var(\xi)=\E[\normn{\xi-\E \xi}^2]$.
Let $\mu,\nu$ be in $\mathcal{P}_{2}(\Rd)$, we define the Wasserstein distance of order $2$ by $\wass(\mu, \nu) = (\inf_{\zeta \in \boldsymbol{\Pi}(\mu,\nu)} \int_{\mathbb{R}^d \times \mathbb{R}^d}\|x-x'\|^2\mathrm{d}\zeta(x,x'))^{1/2}$, where $\boldsymbol{\Pi}(\mu, \nu)$ is the set of transference plans of $\mu$ and $\nu$.

\vspace{-.2cm}
\section{Algorithm derivation}
\label{sec:feder-bayes-infer}
\vspace{-.25cm}
We aim to sample a target probability density function $\pi$ defined for $x \in\rset^d$ by
\begin{align}\label{eq:def:pitarget}
  \txts \pi\pr{x} \propto \prod_{i=1}^{b} \pi^{i} \pr{x} \eqsp,&
  &\pi^{i}(x) \propto \exp(-\potential^{i}(x)) \eqsp,
\end{align}
where $b$ is the number of clients and the potential $\potential^i$ is a finite sum expressed by
\begin{equation}\label{eq:def:Uij}
  \begin{aligned}
    &\txts\potential^{i}(x) = \varpi^{i} \potential^{0}(x) + \sum_{j=1}^{N_{i}} \potential^{i}_{j}(x) \eqsp,
  \end{aligned}
\end{equation}
with $\{\varpi^{i}\}_{i\in[b]} \in \ccint{0,1}^{b}$ and $\sum_{i\in[b]} \varpi^{i} = 1$.
This setting encompasses the Bayesian federated learning as a particular case, in which $\pi$ stands for the global posterior distribution and $\{\pi^i\}_{i\in[b]}$ are referred to as local posteriors \citep{wu2017average,dai2021bayesian}.
In this case $\potential^{0}$ is the global negative log-prior, $N_i$ denotes the number of observations of client $i$, $\potential_j^i$ is the negative log-likelihood of the $j$-th data of client $i$, and $\varpi^{i} \potential^{0}$ is the fraction of the negative log-prior allocated to this client \citep{rendell2020global}.

\textbf{Federated Averaging Langevin Dynamics (\FALD).}
\FALD{}, proposed in \citet{deng2021convergence}, is an extension to the Bayesian setting of \textsc{FedAvg} \citep{mcmahan2017communication}.
The updates performed on the $i$th client define a sequence of local parameters $(\Xlocal_k^i)_{k\in\nset}$ which are transmitted according to some preset schedule (which is deterministic in \citet{deng2021convergence} and is random in this work) to a central server. The central server averages the local parameters to update the global parameter. This global parameter is finally transmitted back to each client, and is used as a starting point of a new round of local iteractions. Hence, each iteration $k \ge 0$ of \FALD{} can be decomposed into two steps:
\begin{enumerate}[label=(\arabic*),wide, labelwidth=!, labelindent=0pt]
    \item \label{step_1_fald} \textbf{Local iteration on each client.}  Each client $i$ performs one step of the Langevin Monte Carlo algorithm \citep{GrenanderMiller1994,Roberts1996} with a stochastic gradient associated with its local potential:
    \begin{equation}
      \label{eq:local-update-salad}
      \begin{aligned}
        % \label{eq:def:Gki:fald}
        &G_{k+1}^i=\maingradi[k+1](\Xlocal_k^{i})\eqsp,\\
        &\Xupdate_{k+1}^{i} = \Xlocal_{k}^{i} - \gamma G_{k+1}^i + \sqrt{2\gamma}\,Z^{i}_{k+1} \eqsp,
      \end{aligned}
    \end{equation}
    where  $\gamma > 0$ and for $x\in\R^d$, $\maingradi[k+1](x)$ is an unbiased estimator of $\nabla U^i(x)$ given by (see \citet{Welling11} -- general updates are considered in the supplement)
    \begin{equation}
        \label{eq:def_SG_batch}
        \txts \maingradi[k+1] = \varpi^i \nabla \potential^{0} + (\nofrac{N_i}{n_i}) \sum_{j \in S_{k+1}^i} \nabla U_j^i \eqsp,
    \end{equation}
    where $(S_k^i)_{k \in\nset^\star}$ is a sequence of \iid~uniform random subsets of $[N_i]$ of cardinal number $n_i$.
    Moreover, $({Z}_k^i)_{k \in\nsets}$, $i\in[b]$ are sequence of i.i.d Gaussian random variables which  might be correlated across the agents and the central server.
    More precisely, given independent sequences, $(\tilde{Z}_k^i)_{k \in\nsets}$, $i\in[b]$ and $(\tilde{Z}_k)_{k \in\nsets}$ of \iid{} $d$-dimensional standard Gaussian random variables, for $\tau\in \ccint{0,1}$ we set
    \begin{equation}
      \label{eq:def:Zlocal}
      \txts Z^{i}_{k} = \sqrt{\tau}\,\tilde{Z}_{k} + \sqrt{1-\tau}\,\tilde{Z}_{k}^{i} \eqsp.
    \end{equation}
    \item \label{eq:fald:step:2} \textbf{A local update.}
    With probability $\pc \in \ocint{0,1}$, %\textbf{aggregation of the local updates and broadcasting of the resulting parameter by the central server.}
    the $i$th client communicates its parameter $\Xupdate_{k+1}^i$, resulting from the first step, to the central server which in turns broadcasts the average $\Xavg_{k+1} = b^{-1} \sum_{i \in [b]} \Xupdate^i_{k+1}$. Finally, each client updates its parameter as $\Xlocal_{k+1}^i = \Xavg_{k+1}$.
    When no communication is performed, each client updates its parameter as $\Xlocal_{k+1}^i = \Xupdate_{k+1}^i$.
\end{enumerate}
The local recursions defined by \FALD{} can be written for $i\in [b]$ and $k \ge 0$ as
\begin{equation}
  \label{eq:fald:local-updates-concise}
  \txts \Xlocal_{k+1}^i = (1-B_{k+1}) \Xupdate_{k+1}^i + (B_{k+1}/b) \sum_{j \in [b]} \Xupdate_{k+1}^j \eqsp,
\end{equation}
where $(B_k)_{k \in\nsets}$ is a sequence of \iid{} Bernoulli random variables with parameter $\pc$.

For $k \ge 1$, denote by $\mug_k$ the distribution of the average parameter
\begin{equation}
    \label{eq:def:global-parameter}
    \txts \Xavg_{k}=(\nofrac{1}{b})\sum_{i\in[b]} \Xlocal_{k}^{i} \eqsp.
\end{equation}
Non-asymptotic Wasserstein bounds between $\mug_k$ and the target distribution $\pi$ are established in \Cref{main:thm:bound:wass:atlernative:algoun}  under the following assumptions.
\begin{assA}\label{ass:fi}
  For any $i\in[b]$, $\potential^{i}$ is continuously differentiable. In addition, there exist $\conv, L >0$ such that for any $i\in[b]$, the function $\potential^{i}$ is $L$-smooth and $\conv$-strongly convex, \textit{i.e.}, for any $x,x'\in\R^d$,
  \begin{multline}\label{ass:fi:lip}
      (\nofrac{\conv}{2})\norm{x'-x}^{2} \le
      \potential^{i}(x') - \potential^{i}(x) - \ps{\nabla\potential^{i}(x)}{x'-x} \\
      \le (\nofrac{L}{2})  \norm{x'-x}^{2}\eqsp.
  \end{multline}
\end{assA}
\begin{assA}\label{main:ass:gradsto:lip}
  For any $i \in [b]$, $(\{\maingradi\}_{i\in[b]})_{k\in\nset}$ are \iid~unbiased estimates of $\{\nabla U^{i}\}_{i\in[b]}$.
  In addition, there exists $\hat{L}\ge 0$ such that for any $x, x'\in\R^d$ we have
  \[
    \E\br{\normn{\maingradi(x') - \maingradi(x)}^{2}} \le \hat{L}^{2}\normn{x'-x}^{2}\eqsp.
  \]
\end{assA}
In the mini-batch scenario \eqref{eq:def_SG_batch}, \Cref{main:ass:gradsto:lip} is satisfied if for $i\in[b]$, $j\in[N_i]$ there exists $L_j^i\ge 0$ such that for any $x,x'\in\R^d$, $\normn{\nabla\potential_{j}^i(x')-\nabla\potential_{j}^i(x)}\le L_j^i\normn{x'-x}$.% holds (see \Cref{rem:mini-batch:1}).

Finally, we also consider the following optional smoothness condition on the potentials $\{U^i\}_{i\in[b]}$. This additional assumption, often satisfied in applications have been considered \eg{} in \citet{durmus2019high,dalalyan2019user}.
\begin{HX}\label{ass:fi:ctrois}
  There exists $\tilde{L} \ge 0$, such that for any $i \in [b]$, the function $\potential^i$ is three times continuously differentiable and  for any $x,x'\in\R^d$, $\normn{\nabla^{2}\potential^i(x) - \nabla^{2}\potential^i(x')} \le \tilde{L} \norm{x-x'}$.
\end{HX}

We introduce some key quantities appearing in the theoretical derivations below.
Denote by $x_\star$ the minimizer of $\sum_{i\in[b]} U^i$ which exists and is unique under \Cref{ass:fi}. We define
\begin{equation}\label{eq:def:Vmeasure}
    \begin{aligned}
    \textstyle\varconst_\pi&= \txts\int_{\R^d}\var\acn{b^{-1}{\textstyle \sum_{i\in[b]}\maingradi[1](x)}}\pi(\rmd x)  \eqsp, \\
        \textstyle \varconst_\star&= \txts\var\acn{b^{-1}{\textstyle \sum_{i\in[b]}\maingradi[1](x_\star)}} \eqsp,
    \end{aligned}
\end{equation}
the average of the stochastic gradient variance under the stationary distribution $\pi$ and at the minimum $x_{\star}$, respectively. Finally, the statistical heterogeneity between the clients is quantified by (see, e.g.~\citet{stich2018sparsified})
\begin{equation}\label{eq:def:Hheterogeneity}
    \textstyle\heterogeneity = b^{-1}\sum_{i\in[b]}\norm{\nabla\potential^{i} \pr{x_{\star}}}^2 \eqsp.
\end{equation}
For ease of presentation, for two sequences $(a_k)_{k\in\nset}$ and $(b_k)_{k\in\N}$ we write $a_k\lesssim b_k$ if there exists $C>0$ only depending on the constants introduced in \Cref{ass:fi}, \Cref{main:ass:gradsto:lip} and \Cref{ass:fi:ctrois} such that $a_k\le Cb_k$, for any $k\in\N$.
\begin{theorem}[Simplified]\label{main:thm:bound:wass:atlernative:algoun}
  Assume \Cref{ass:fi}, \Cref{main:ass:gradsto:lip} and suppose for any $i\in[b]$, $\Xlocal_0^{i}=X_0$.
  Then, there exist $\bgamma >0$, such that for any $\gamma \in \ocint{0,\bgamma}$, $k\in\N$, $X_0 \sim \mu_{0} \in \mathcal{P}_2(\rset^d)$, we have
  \begin{multline}\label{main:eq:bound:wass:atlernative:algoun}
    \wass^{2}\prn{\mug_k, \pi}
    \lesssim (1-\nofrac{\gamma\conv}{8})^k \, \initconst(\mu_0)
    + \frac{\gamma^{e}}{b} \langevinrest
    + \gamma \varconst_\pi  % + \frac{\gamma}{\conv} \varconst_\pi
    \\
    + \frac{\gamma^2 (1-\pc)}{\pc^2} \acBig{\heterogeneity + \pc \varconst_\star + \frac{d}{b}}
    + \frac{\gamma (1-\tau)(1-b^{-1}) d}{\pc}
    % + \frac{\gamma^2 (1-\pc)}{\pc^2} \const
    \eqsp,
  \end{multline}
  where $\langevinrest = d$, $e=1$ and
  $\initconst(\mu_0)< \infty$ is a function of the initial condition $\mu_0$.
  If \Cref{ass:fi:ctrois} holds, then $e=2$ and $\langevinrest=d(1 + \nofrac{d}{b})$.
\end{theorem}
Elements of proof are provided in \Cref{sec:proof-outline}; a precise statement is given in \Cref{thm:bound:wass:atlernative:salad} with detailed proofs. Note the step size upper bound $\bgamma$ is proportional to $\pc$. In the single user case ($b=\pc=\tau=1$), we recover up to numerical constants the results stated in \citet{durmus2019high,dalalyan2019user}.
Note that, under \Cref{ass:fi:ctrois} the leading term in the step size $\gamma$ is proportional to the stochastic gradient variance $\varconst_\pi$, in accordance with the bounds obtained for SGLD by \eg, \citet{dalalyan2019user}.
More discussions on these bounds are postponed after the statement of \Cref{main:thm:bound:wass:atlernative:vrsalads}.

\textbf{Lower bounding the effect of heterogeneity.}
Similar to \textsc{FedAvg}, the convergence of \FALD{} is impaired by data heterogeneity.
Multiple local \SGLD{} steps described in \eqref{eq:local-update-salad} cause $\Xlocal_k^i$ to target the local posteriors $\pi^{i} \propto \exp(U^{i})$.
We now provide lower bound on the Wasserstein distance between the distribution of the samples generated by \FALD{} and the target distribution $\pi$ which is proportional to the heterogeneity $\gamma^2\heterogeneity$.
\begin{proposition}
  \label{prop:fald:heterogeneity-result}
  There exist $\bgamma >0$, potentials $\{U^{i}\}_{i=1}^2$ on $\rset$ satisfying \Cref{ass:fi}, \Cref{ass:fi:ctrois} and an instance of \FALD{} satisfying \Cref{main:ass:gradsto:lip} such that for any $\gamma \in \ocint{0,\bgamma}$, we have
  \begin{equation}
    \label{eq:2}
    \liminf_{k \to \plusinfty} \wass^2\prn{\mug_k, \pi} \gtrsim \gamma^2 \heterogeneity\eqsp. % \sum_{i\in\{1,2\}}\norm{\nabla\potential^{i} \pr{x_{\star}}} \eqsp.
  \end{equation}
\end{proposition}
This proposition extends \citet[Theorem II]{scaffold20} to the Bayesian context and underlines the same limitation as \textsc{FedAvg}.
To circumvent this, various bias reduction techniques have been suggested in the stochastic optimization literature \citep{horvath2022stochastic,gorbunov2021local}.
In the next section, we adapt similar mechanisms to derive an alternative to \FALD{} satisfying better finite bounds.

\textbf{\FALD{} with control variates and bias reduction.}
\label{sec:impr-conv-salad}
To mitigate the impact of local stochastic gradients, we adapt variance-reduction techniques \citep{wang2013variance,kovalev2020don} and bias-reduction techniques \citep{horvath2022stochastic,gorbunov2021local}. % to extend the {\FALD} algorithm.
This new approach introduces a different recursion rule in step \ref{step_1_fald} of \FALD{}, while keeping step \ref{eq:fald:step:2} unchanged.
The local update rule is based on a reference point $Y_k\in\R^d$ common to all clients. This common point is updated with probability $\qc\in\ocint{0,1}$ and allows the inclusion of a local shift $C_k$ to recenter the local gradients. This mechanism eliminates the ``infamous non-stationarity of the local methods'' (paraphrasing \citet{gorbunov2021local}) and therefore avoids extra bias.
At each iteration $k$, the first step of the  {\VRFALDs} algorithm is divided into two parts:
\begin{enumerate}[label=(1.\arabic*),wide, labelwidth=!, labelindent=0pt]
    \item \textbf{Update of the reference parameter and control variate.}
    The variance reduced gradient requires a sporadic computation of the full local gradient.
    Let $(B_k^Y)_{k \in\nsets}$ be a sequence of \iid{} Bernoulli random variables with parameter $\qc \in \ocint{0,1}$.
    If $B_{k+1}^Y=1$, then the client reference point $Y_{k}$ is updated: the clients transmit their local parameter $\{\Xlocal_k^i\}_{i\in[b]}$ to the central server which computes their average $Y_{k+1}=b^{-1}\sum_{i\in[b]} \Xlocal_k^i$; which is sent back to the clients. The clients then compute the full gradients $\acn{\nabla\potential^i(Y_{k+1})}_{i\in[b]}$ and transmit them to the central server which updates the shift $C_{k+1}=b^{-1}\sum_{i\in[b]} \nabla\potential^i(Y_{k+1})$. To summarize, the reference point and the shift are updated according to
    \begin{align}
      % \begin{aligned}
        \label{eq:vrsalads:YC}
        &\txts Y_{k+1} = (1-B_{k+1}^Y) Y_k + (B^Y_{k+1}/b)\sum_{i\in[b]}  \Xlocal_k^i \eqsp, \\
        &\txts C_{k+1} = (1-B_{k+1}^Y) C_k + (B^Y_{k+1}/b) \sum_{i\in[b]} \nabla U^{i} (Y_{k+1}) \eqsp.
      % \end{aligned}
    \end{align}
\item \textbf{Local iteration on each client.} This step is similar to \FALD{}, upon replacing the local updates \ref{eq:fald:step:2} by the variance-reduced version
\begin{align}
  \label{eq:def:tilde-X-vr-salads}
  \txts  &G_{k+1}^{i} = \maingradi[k+1](\Xlocal_k^{i})-\maingradi[k+1](Y_k)+C_k\eqsp, \\
  \label{eq:def:local-update-salad-vrs}
  \txts &\Xupdate_{k+1}^{i} =  \Xlocal_{k}^{i} - \gamma G_{k+1}^{i}  + \sqrt{2\gamma} Z^{i}_{k+1} \eqsp.
\end{align}
\end{enumerate}

The \VRFALDs{} analysis relies on the following additional assumption.
\begin{assA}\label{ass:gradsto:meandiff}
There exists $\constsvrg\ge 0$ such that for any $i\in[b]$, $k\in\N^{\star}$ and $x, y\in\R^d$, the following inequality holds
\begin{multline}
	\E\br{\norm{\maingradi[k](x) - \maingradi[k](y) - \nabla\potential^{i}(x) + \nabla\potential^{i}(y)}^{2}} \\
	\le \constsvrg \normlr{x-y}^{2}\eqsp.
\end{multline}
\end{assA}
Under \Cref{ass:fi} and \Cref{main:ass:gradsto:lip}, \Cref{ass:gradsto:meandiff} is satisfied with $\constsvrg = 2L^2 + 2\hat{L}^2$.
However, using this result leads to some discrepancy in previous existing analysis, since $\constsvrg = 0$ in the non-stochastic gradient case while $2L^2 + 2\hat{L}^2 \neq 0$ in general. Finally, in the mini-batch scenario \eqref{eq:def_SG_batch}, if $\acn{ \nabla \potential_j^i}_{j\in[N_i]}$ are $L_i$-Lipschitz, then \Cref{ass:gradsto:meandiff} holds with $\constsvrg=\max_{i\in[b]}\{N_i L_i^2/n_i\}$; see \Cref{rem:mini-batch:gradsto:difflip}.

For $k \ge 0$, denote by $\mugvrs_k$ the distribution of the average $\Xavg_{k}=b^{-1}\sum_{i\in[b]} \Xlocal_k^i$ where $\Xlocal_{k}^i$ is defined as in \eqref{eq:fald:local-updates-concise} with $\Xupdate_k^i$ given in \eqref{eq:def:local-update-salad-vrs}.
With these notations, we obtain the following theoretical guarantee on \VRFALDs{}.
\begin{theorem}[Simplified]\label{main:thm:bound:wass:atlernative:vrsalads}
  Assume \Cref{ass:fi}, \Cref{main:ass:gradsto:lip}, \Cref{ass:gradsto:meandiff} and suppose for $i\in[b]$, $\Xlocal_0^{i} = Y_0 = X_0$.
  Then, there exist $\bgammavrs >0$, such that for any $\qc\le \pc$, $\gamma \in \ocint{0,\bgammavrs}$, $k\in\N$, $X_0 \sim \mu_{0} \in \mathcal{P}_2(\rset^d)$, we have
  \begin{multline}\label{main:eq:bound:wass:atlernative:vrsalads}
    \wass^{2}\prn{\mugvrs_k, \pi}
    \lesssim (1-\gamma\conv/8)^k \, \initvrsconst(\mu_0) + \frac{\gamma^{e}}{b} \langevinrest + \frac{\gamma^2 d}{b \qc} \constsvrg \\
    % + \pr{1+\frac{\qc}{\pc}}\br{
    + \frac{\gamma (1-\tau)(1-b^{-1}) d}{\pc}
    + \frac{\gamma^2 (1-\pc)}{\pc^2} \acBig{\gamma\varconst_\star + \frac{d}{b}}
    % }
    \eqsp,
  \end{multline}
  where $\langevinrest = d$, $e=1$, $\varconst_\star$ is defined in \eqref{eq:def:Vmeasure}, $\initvrsconst(\mu_0)< \infty$ is a function of the initial condition $\mu_0$. If \Cref{ass:fi:ctrois} holds, then $e=2$ and $\langevinrest=d(1 + \nofrac{d}{b})$.
\end{theorem}
The proof is postponed to \Cref{subsec:vrsaladstar}. Compared to  \Cref{main:thm:bound:wass:atlernative:algoun}, the client-drift term does no longer appear, highlighting the advantage of \algoquatre{} in dealing with data heterogeneity between agents.

Further, the variance of the stochastic gradients of  {\VRFALDs} only appear in the factor $\gamma^2\constsvrg$. This result agrees with \citet{chatterji2018theory} for SVRG-LD, which might be seen as a particular instance of \VRFALDs{} with $b=1$, $\pc=1$.
Nevertheless, a close inspection of the proof in \citet{chatterji2018theory} reveals a gap---see \Cref{rem:vrsalad}, which is corrected in the proof of \Cref{thm:bound:wass:svrgLangevin}.

\textbf{Complexity and Communication costs.}
We now discuss the complexity and communication costs of \FALD{} and \VRFALDs{}.
We study two extreme cases: (A) the local computation cost is negligible and only the communication cost matters, which is typical in cross-device applications. (B) the communication cost is negligible and only the local computation cost (complexity) matters.  More general scenarios are discussed in the supplement~\Cref{sec:analysis-costs}.
In this discussion, it is assumed that \Cref{ass:fi:ctrois} is satisfied and $\tau=1$.
In both cases, for a target precision $\epsilon>0$, we optimize the hyperparameters (number of iterations $K_\epsilon$, learning rate $\gamma_\epsilon$, probability of communication $\pce$) to ensure $\wass\prn{\mug_{K_\epsilon}, \pi} \le \epsilon$ (\FALD{}) or $\wass\prn{\mugvrs_{K_\epsilon}, \pi} \le \epsilon$ (\VRFALDs{}).
\begin{enumerate}[label=(Scenario \Alph*),wide,nosep,leftmargin=!,labelwidth=!,labelindent=0pt,itemsep=.1ex,noitemsep,topsep=-1ex]
  \item The objective is to minimize the number of communications $\pce K_\epsilon$.
  As $\gamma$ can be arbitrarily small, we set $K_\epsilon = \gamma^{-1}\lambda_\epsilon$, $\pce = \rho_\epsilon \gamma$, where $\lambda_\epsilon,\rho_\epsilon>0$.
  Hence, the optimization problem becomes $\min\{\lambda_\epsilon \rho_\epsilon\}$ subject to $\initconst(\mu_0)\exp(-\lambda_\epsilon\conv/8) + \rho_\epsilon^{-2}(\heterogeneity + d/b)\le \epsilon^2$.
  As $\epsilon \downarrow 0^+$, the minimum number of communications $\pce K_\epsilon$ scales as $\tilde{O}(\epsilon^{-1}\sqrt{\heterogeneity+b^{-1}d})$ for \FALD{} and $\tilde{O}(\epsilon^{-1}\sqrt{b^{-1}d})$ for \VRFALDs{}.
  \item We take $\pce=1$ and seek to minimize the total number of iterations $K_\epsilon$. As $\epsilon \downarrow 0^+$, $K_\epsilon$ scales as $\tilde{O}(\epsilon^{-2}(\varconst_{\pi} + \epsilon \sqrt{b^{-1}\langevinrest}))$ for \FALD{} and $\tilde{O}(\epsilon^{-1}\sqrt{b^{-1}\langevinrest+b^{-1}\omega d})$ for \VRFALDs{}.
\end{enumerate}
\begin{figure}[]
  \captionsetup[subfigure]{labelformat=empty}
  \centering
  \begin{subfigure}{0.5\textwidth}
    \caption{(Scenario A) Numerical results optimizing $\pce K_\epsilon$.}
    \label{fig:cost0}
    \mbox{\includegraphics[width=\textwidth]{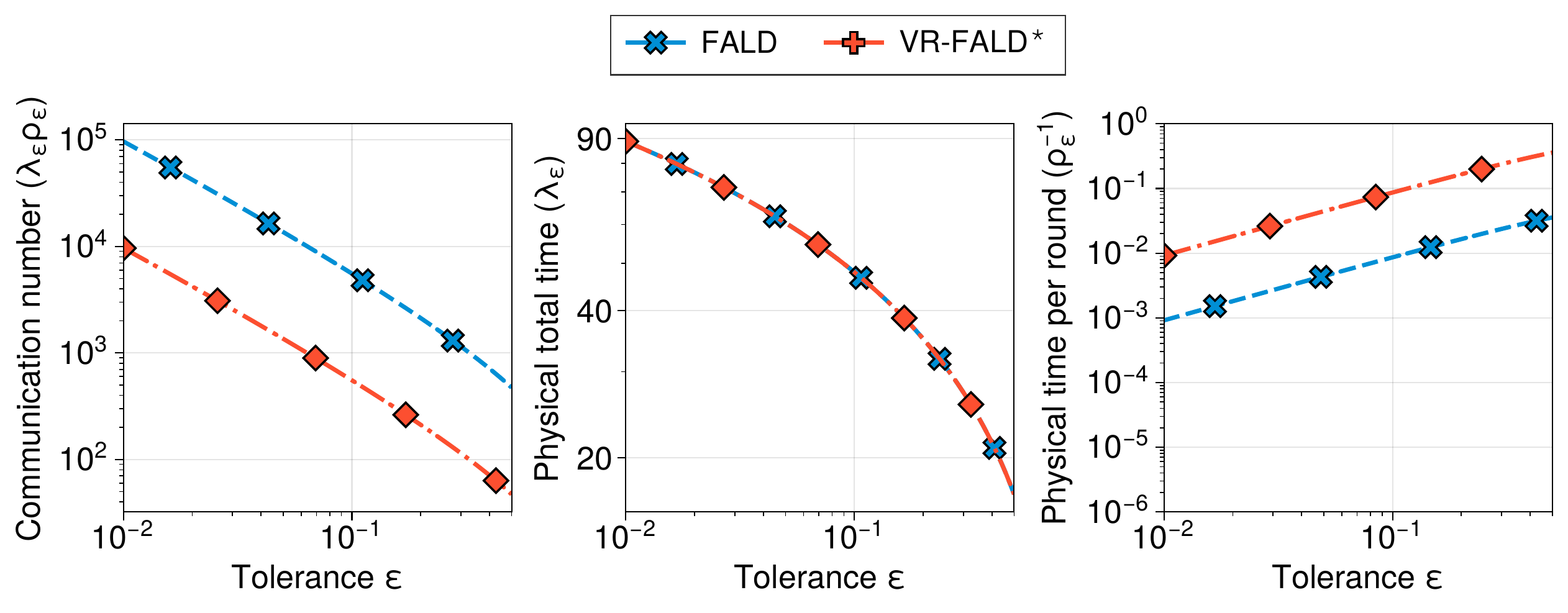}}
  \end{subfigure}
  \hfill
  \begin{subfigure}{0.5\textwidth}
    \caption{(Scenario B) Numerical results optimizing $K_\epsilon$.}
    \label{fig:cost1}
    \mbox{\includegraphics[width=\textwidth]{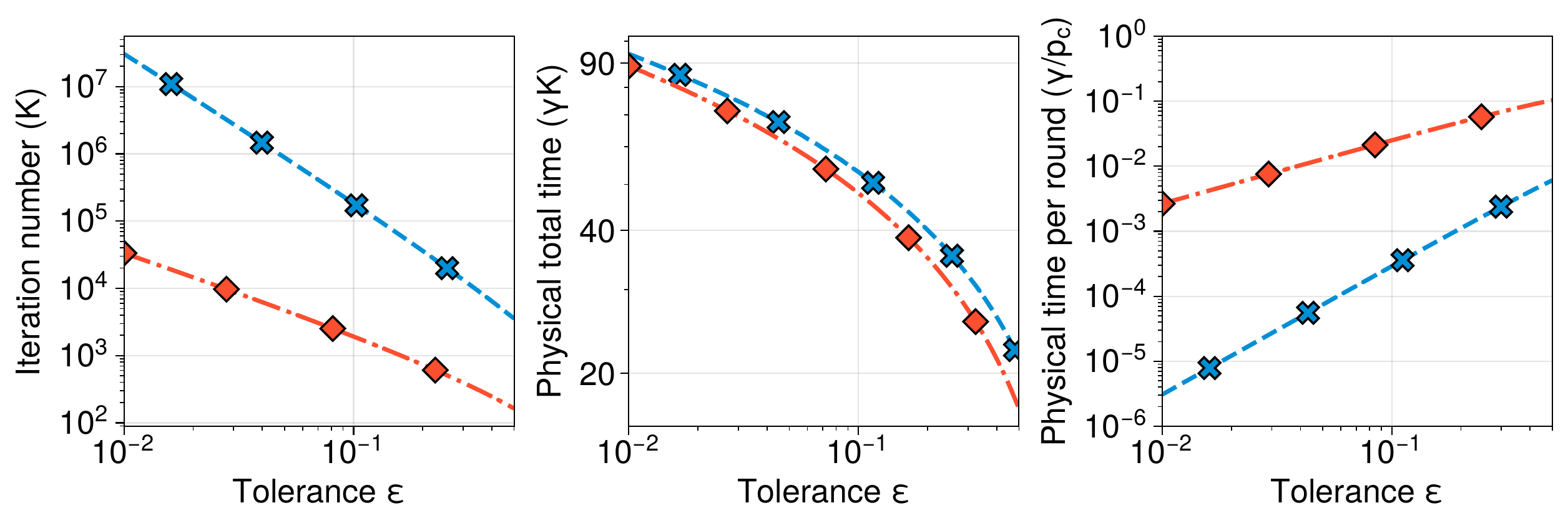}}
  \end{subfigure}
\end{figure}
In Figures~\ref{fig:cost0}-\ref{fig:cost1}, we display the optimal number of communications $\pce K_\epsilon$ as a function of $\epsilon$ (left panels~Figures~\ref{fig:cost0}-\ref{fig:cost1}), the total ``physical'' time ($\lambda_\epsilon$ for (A) and $\gamma_\epsilon K_\epsilon$ for (B)---middle panels~Figures~\ref{fig:cost0}-\ref{fig:cost1}), the average physical time between two consecutive communications ($\rho_\epsilon^{-1}$ for (A) and $\gamma/\pce$ for (B)---right panels~Figures~\ref{fig:cost0}-\ref{fig:cost1}).
The values of $\conv$, $d$, $\heterogeneity$, $\varconst_\star$, $\varconst_\pi$, $\initconst$, $\initvrsconst$ and $\langevinrest$ are given in Appendix.
The total physical time is (almost) the same for \FALD{}, \VRFALDs{}, in scenarios (A) and (B).
\VRFALDs{} significantly reduces the number of communications $\pce K_\epsilon$ in scenario (A) (top panel) and number of rounds $K_\epsilon$ (B) (bottom panel) \wrt\ \FALD{}.

Figures~\ref{fig:cost0}-\ref{fig:cost1} also illustrate that the ``embarrassingly parallel'' approach of \citep{neiswanger2014asymptotically} is far from optimal.
Indeed, our results show the importance of making multiple interactions (rather than a single consensus step) and using correlated noises between clients.
In scenario (A), the optimal number of communications scales inversely proportional to $1/\epsilon$ which improve the bounds $\tilde{O}(1/\epsilon^2)$ derived in \citet[Section 5.3.1]{deng2021convergence}. For scenario (B), \FALD{} has the same complexity as QLSD \citet{vono2022qlsd} under similar assumptions; see also \citet{sun2022federated}. \VRFALDs{} has the lowest complexity ($\tilde{O}(1/\epsilon)$) among the Bayesian Federated algorithms reported earlier. This bound matches the one obtained by \citet{chatterji2018theory} for the fully centralized SVRG-LD (corresponding to $b=1$).

\vspace{-2.5pt}
\section{Proofs outline}\label{sec:proof-outline}
\vspace{-2.5pt}

We briefly outline the main steps of the proof of \Cref{main:thm:bound:wass:atlernative:algoun,main:thm:bound:wass:atlernative:vrsalads}.
Details of the proofs can be found in the supplementary paper, where we analyze the two algorithms under a common unifying framework.
For both algorithms, the local parameters $(X_k^i)_{i\in[b]}$, $k\ge 0$, are given by \eqref{eq:fald:local-updates-concise}, where $(\tilde{X}_k^i)_{i \in [b]}$ stands for local iterations, which are given in \eqref{eq:local-update-salad} for \FALD{} and
\eqref{eq:def:tilde-X-vr-salads} for \VRFALDs{}.
Then, we bound the Wasserstein distance between the target distribution $\pi$ and the distribution of $\Xavg_{k}=b^{-1}\sum_{i\in[b]}\Xlocal_k^i$ which is denoted by $(\nug_{k})_{k\in\nset}$. The Wasserstein distance is defined as the infimum over the coupling. We use below the synchronous coupling construction used in \citep{durmus2019high,dalalyan2019user} for the analysis of Stochastic Gradient Langevin algorithms.

\textbf{Synchronous coupling.}
We first construct a Brownian motion $(\mathsf{W}_t)_{t \ge 0}$ by $\mathsf{W}_t = \sqrt{\tau}\,\tilde{\mathsf{W}}_{t} + \sqrt{\nofrac{(1-\tau)}{b}}\,\sum_{i\in[b]}\tilde{\mathsf{W}}_{t}^{i}$, starting from $b+1$ independent $d$-dimensional standard Brownian motions $(\tilde{\mathsf{W}}_t^{i})_{t \ge 0}$, $i\in[b]$, and $(\tilde{\mathsf{W}}_t)_{t \ge 0}$.
Second, we define the following standard Gaussian random variables $\tilde{Z}^{i}_{k+1} = \gamma^{-1/2}(\tilde{\mathsf{W}}_{(k+1)\gamma}^{i}-\tilde{\mathsf{W}}_{k\gamma}^{i})$, $\tilde{Z}_{k+1} = \gamma^{-1/2}(\tilde{\mathsf{W}}_{(k+1)\gamma}-\tilde{\mathsf{W}}_{k\gamma})$, and we set $Z_{k}^{i}$ as in \eqref{eq:def:Zlocal}.
For $k\in\N$, it holds that $\sqrt{\gamma}\sum_{i\in[b]}Z_{k+1}^{i} = \sqrt{b}(\mathsf{W}_{(k+1)\gamma}-\mathsf{W}_{k\gamma})$.
Finally, we consider $(\Xcontinuous_t)_{t \ge 0}$ the strong solution of the Langevin diffusion associated with $\pi$ and starting from  $\Xcontinuous_0 \sim \pi$ (see \eqref{eq:def:pitarget}) and driven by  $(\mathsf{W}_t)_{t \ge 0}$:
\begin{equation}\label{eq:def:langevin-continuous}
	\txts \rmd\Xcontinuous_{t}
	= - (\nofrac{1}{b})\sum_{i\in[b]}\nabla\potential^i(\Xcontinuous_t) \,\rmd t + \sqrt{\nofrac{2}{b}} \,\rmd \mathsf{W}_t\eqsp.
\end{equation}
Under \Cref{ass:fi} and \Cref{main:ass:gradsto:lip}, $\pi$ is the unique stationary distribution for the Langevin diffusion, hence the distribution of $\Xcontinuous_t$ is $\pi$ for all $t \ge 0$; see \eg{} \citet{Roberts1996}. Hence, $(\Xavg_{k},\Xcontinuous_{k\gamma})$ defines a coupling between $\nug_k$ and $\pi$, thus for any $k \in\nset$ we get
\begin{equation}
    \wass^{2}\prn{\nug_k, \pi} \le \E\br{\normLigne{\Xavg_{k} - \Xcontinuous_{k\gamma}}^2}\eqsp.
\end{equation}
The rest of the proof then consists in bounding the right-hand side.
It is worth noting that in contrast to most analysis on Langevin dynamics, we consider a Langevin diffusion \eqref{eq:def:langevin-continuous} we scale the gradient term by $b^{-1}$  and  the Brownian motion by $b^{-1/2}$.
This scaling is adapted to the averaging procedure defining $(\Xavg_{k})_{k\in\nset}$. % which naturally discretized \eqref{eq:def:langevin-continuous}.

\textbf{Decomposition of $\E[\normLigne{\Xavg_{k} - \Xcontinuous_{k\gamma}}^2]$.}
Denote by $\mathcal{F}_k$ the filtration generated by $\Xcontinuous_{0}, (\mathsf{W}_t)_{t\le k\gamma}$ and $(\{X_l^{i}\}_{i=1}^{b})_{l\le k}$.
Using the definition \eqref{eq:def:global-parameter} of $(\Xavg_{k})_{k \in\nset}$ combined with \Cref{ass:fi}, we show in \Cref{lem:contraction:xkbis} that for any $\gamma \lesssim 1$
\begin{multline}\label{eq:bound:decomposition}
		\txts\E^{\mathcal{F}_{k}}\br{\normn{\Xcontinuous_{(k+1)\gamma} - \Xavg_{k+1}}^{2}}
		\lesssim \pr{1 - \gamma\conv/2} \normn{\Xcontinuous_{k\gamma} - \Xavg_{k}}^{2}
		\\
		\txts
    + E_k
    + \gamma^{2} S_k
		+ V_{k}
		\eqsp,
\end{multline}
where
$V_k = {b}^{-1}\sum_{i\in[b]}\normn{\Xlocal_{k}^{i}-\Xavg_{k}}^{2}$ and
\begin{align}
  &\txts S_k = \var^{\mathcal{F}_{k}}\prn{b^{-1}\sum_{i\in[b]} G_{k}^{i}}\eqsp, \\
  &\txts E_k = \gamma^{-1}\normn{\E^{\mathcal{F}_{k}}\brn{I_{k}}}^{2} + \E^{\mathcal{F}_{k}}\br{\normn{I_{k}}^{2}}\eqsp,
\end{align}
with $I_k=b^{-1}\sum_{i\in[b]}\int_{k\gamma}^{(k+1)\gamma}(\nabla\potential^i(\Xcontinuous_s)-\nabla\potential^i(\Xcontinuous_{k\gamma}))\rmd s$.

\textbf{Bounding $E_k$.}
The term $E_k$  accounts for the difference between the diffusion and its discretization; the bound is the same for \FALD{} and \VRFALDs{}. By adapting \citet[Lemma 21]{durmus2019high}, we establish in \Cref{lem:bound:Ik:unified} that
\begin{equation}\label{eq:bound:ek}
  \E\br{E_k} \lesssim \nofrac{\gamma^2 d}{b}
  \eqsp.
\end{equation}
Under \Cref{ass:fi:ctrois}, for $\gamma\lesssim 1$ the bound can be sharpened in
\begin{equation}\label{eq:bound:ek:HX1}
  \E\br{E_k} \lesssim (\nofrac{\gamma^3 d}{b}) \prn{1 + \nofrac{d}{b}}
  \eqsp.
\end{equation}
The right-hand side of \eqref{eq:bound:ek} has a higher order with respect to the step size $\gamma$ in comparison to \eqref{eq:bound:ek:HX1}.
This step is the reason why we consider the more restrictive assumption \Cref{ass:fi:ctrois}, which leads to different guarantees depending on whether this condition is met or not.

\textbf{Bounding $S_k$.}
 $S_k$ is the conditional variance of the stochastic gradient. This is the main difference between the two algorithms.
For \FALD{}, we show in \Cref{lem:salad-ass-sup-general-case} that
\begin{equation}
  \label{eq:bound:Sk:1}
  \E\br{S_k} \lesssim \E\br{\normn{\Xavg_{k} - \Xcontinuous_{k\gamma}}^2} + \E\br{V_{k}}
  + \varconst_\pi \eqsp.
\end{equation}
On the other hand, under \Cref{ass:gradsto:meandiff}, we establish in \Cref{lem:bound:diffXkYki:vrsaladstar} that for \VRFALDs{}, it holds that
\begin{multline}
  \label{eq:bound:Sk:2}
  \E\br{S_k} \lesssim \constsvrg\E\br{\normn{\Xavg_{k} - \Xcontinuous_{k\gamma}}^{2}}
  + \constsvrg\E\br{V_{k}}
  + \frac{\gamma \constsvrg d}{b \qc}
  \\
  + \constsvrg \qc\sum_{l=0}^{k-1}(1-\qc)^{k-l-1} \E\br{\normlr{\Xcontinuous_{l\gamma} - \Xavg_{l}}^{2}}
  \eqsp.
\end{multline}
Compared to the inequality \eqref{eq:bound:Sk:1}, which holds for \FALD{}, the variance term $\varconst_\pi$ for \VRFALDs{} is replaced by $\nofrac{\gamma \constsvrg d}{b \qc}$, which can be made arbitrarily small with $\gamma\to 0$.
Note that this term is inversely proportional to the update probability $\qc$ of the control variate.
Interestingly, the term $S_k$ vanishes when $\constsvrg=0$, i.e., when each client uses its full local gradient at each iteration.

\textbf{Bounding $V_k$.}
We show in \Cref{lem:bound:Vk:expec:salad} (\FALD{}) and \Cref{lem:bound:Vk:expec:vrsaladstar} (\VRFALDs{}), there exist $a_0,a_1\ge0$ satisfying
\begin{equation}
  \label{eq:bound:Vk:simplify}
  \E[V_k] \le (1-\gamma\conv/8)^k a_0 + a_1 \eqsp.
\end{equation}
To establish this result, we consider the sequence $(f_k)_{k\in\N}$ with general term given by
\begin{equation}
  \label{eq:7}
  f_{k} = V_k+\alpha_d \dist_k^2+\alpha_{\sigma}\sigma_k^2 \eqsp,
\end{equation}
where $\alpha_d,\alpha_\sigma\ge0$ are given in \eqref{def:eq:alphadsigma}; $\dist_{k} = \normn{\Xavg_{k} - x_{\star}}$ denotes the distance between the average parameter $\Xavg_{k}$ and the minimizer $x_\star$ of the global potential $\potential$;
$\sigma_k=0$ for \FALD{} and $\sigma_k^2={b}^{-1}\sum_{i\in[b]}\E^{\mathcal{F}_{k}}[\|\maingradi[k](Y_{k}) - \maingradi[k](x_{\star})\|^{2}]$ for \VRFALDs{} with $Y_k$ defined in \eqref{eq:vrsalads:YC}.
The weights $\alpha_d,\alpha_\sigma$ are tailored to prove a contraction; more precisely, we show the existence of $a_2>0$ whose expression is given in \Cref{lem:bound:Vk:new}, such that
\begin{equation}
  \label{eq:9}
  f_{k+1} \le \pr{1-\nofrac{\gamma \conv}{4}} f_k
  + \gamma^2 a_2
  + 2\gamma d\pr{1-\tau}(1-{b}^{-1})\eqsp.
\end{equation}
An immediate induction combines with $V_k\le f_k$ yields a first bound for $\E[V_k]$ of the form \eqref{eq:bound:Vk:simplify} with $a_1$ of order $\gamma$.
In a final step \Cref{lem:bound:Vk}, we refine this bound to obtain a term $a_1$ of order $\gamma^2$.

\textbf{Gathering all the bounds.}
The proof is concluded by plugging the upper bounds derived for $E_k$, $S_k$, $V_k$ into \eqref{eq:bound:decomposition}.

\vspace{-0.1em}
\section{Numerical experiments}\label{main:sec-numerical-exp}
\vspace{-0.1em}
To illustrate our findings, we perform three numerical experiments on both synthetic toy-examples and real datasets.
We compare \FALD{}, \VRFALDs{}  with Bayesian federated learning benchmarks: DG-LMC \citep{plassier2021dg}, the Federated Stochastic Langevin Dynamics {FSGLD}  \citep{el2021federated}, the Quantized Langevin Stochastic Dynamic QLSD and its variance-reduced version QLSDPP \citep{vono2022qlsd}. We also include in our benchmark state of the art (centralized MCMC) algorithms: HMC \citep{brooks2011handbook}, the Stochastic Gradient Langevin Dynamics (SGLD) \citep{Welling11} and the preconditioned SGLD (pSGLD) \citep{li2016preconditioned}.

\textbf{Gaussian posterior.} We consider $b=100$ clients associated to local Gaussian potentials with mean $\acn{\mu_i}_{i\in[b]}$ and covariance $\acn{\Sigma_i}_{i\in[b]}$, i.e., $\potential^{i}(x) = (1/2) (x - \mu_i)^{\top} \Sigma_i^{-1}(x - \mu_i)$.
For different values of the hyperparameters $(\pc, \gamma, \tau)$, we run $100$ chains with $k_1=10^{7}$ iterations $(X_k)_{k=1}^{k_1}$ and discard $10\%$ of the samples (more details are reported in \Cref{subsec:gaussian}).
For each chain, we estimate the posterior variance $\sigma_\star^2 = \int \norm{x-x_{\star}}^2 \rmd \pi(x)$ using \algoun{} and \algoquatre{}, where $\pi\propto\exp(-\sum_{i\in[b]}\potential^{i})$ and $x_{\star}=\argmax_{x \in \rset^d} \pi(x)$.
We compute a Monte-Carlo estimates (over $10^2$ independent replications) of the Mean Squared Error (MSE) given by $\{(k_1-k_0)^{-1} \sum_{k=k_0+1}^{k_1}\norm{\Xavg_{k}-x_{\star}}^2 - \sigma_\star^2\}^2$ where $k_1$ is the total number of samples and $k_0$ is the burn-in period. The values of the hyperparameters are reported in \Cref{subsec:gaussian}.
From \Cref{table:gaussian-comparison}, \VRFALDs{} always outperforms \FALD{} for any choices of $\pc,\gamma$.
This illustrates the impact of the heterogeneity and supports the theoretical findings given in \Cref{main:thm:bound:wass:atlernative:algoun,main:thm:bound:wass:atlernative:vrsalads}.
Furthermore, the asymptotic bias for \VRFALDs\ improves when $\tau=1$ as derived in the theoretical analysis.
\begin{table*}[t]
  \centering
  \begin{small}
      \begin{tabular}{l|ccc|ccc|ccc}
      \toprule
      \textsc{Probability} $\pc$ & \multicolumn{3}{|c|}{$\pc = 1/5$} & \multicolumn{3}{c|}{$\pc = 1/10$} & \multicolumn{3}{c}{$\pc = 1/20$} \\
      \textsc{Stepsize} $\gamma$ & $\frac{1}{2}\pc\bar{\gamma}$ & $\frac{1}{5}\pc\bar{\gamma}$ & $\frac{1}{10}\pc\bar{\gamma}$ & $\frac{1}{2}\pc\bar{\gamma}$ & $\frac{1}{5}\pc\bar{\gamma}$ & $\frac{1}{10}\pc\bar{\gamma}$ & $\frac{1}{2}\pc\bar{\gamma}$ & $\frac{1}{5}\pc\bar{\gamma}$ & $\frac{1}{10}\pc\bar{\gamma}$ \\
      \midrule
      \algoun{} ($\tau$ = 0) & 2.5E+01 & 9.5E-01 & 3.9E-02 & 3.6E+01 & 1.1E+00 & 8.2E-02 & 4.2E+01 & 2.0E+00 & 1.1E-01 \\
      \algoquatre{} ($\tau$ = 0) & 4.8E-02 & 2.6E-02 & 1.4E-02 & 5.0E-02 & 4.9E-02 & 3.7E-02 & 9.8E-02 & 5.3E-02 & 3.9E-02 \\
      % \midrule
      \algoquatre{} ($\tau$ = 1) & 2.8E-02 & 2.0E-02 & 1.3E-02 & 4.1E-02 & 3.7E-02 & 1.4E-02 & 8.6E-02 & 4.3E-02 & 2.1E-02 \\
      \bottomrule
      \end{tabular}
  \end{small}
    \caption{Asymptotic bias in function of $\tau$, $\pc$ and $\gamma$. \label{table:gaussian-comparison}}
\end{table*}
\begin{figure}[]
  \begin{center}
    \mbox{\includegraphics[width=.4\textwidth]{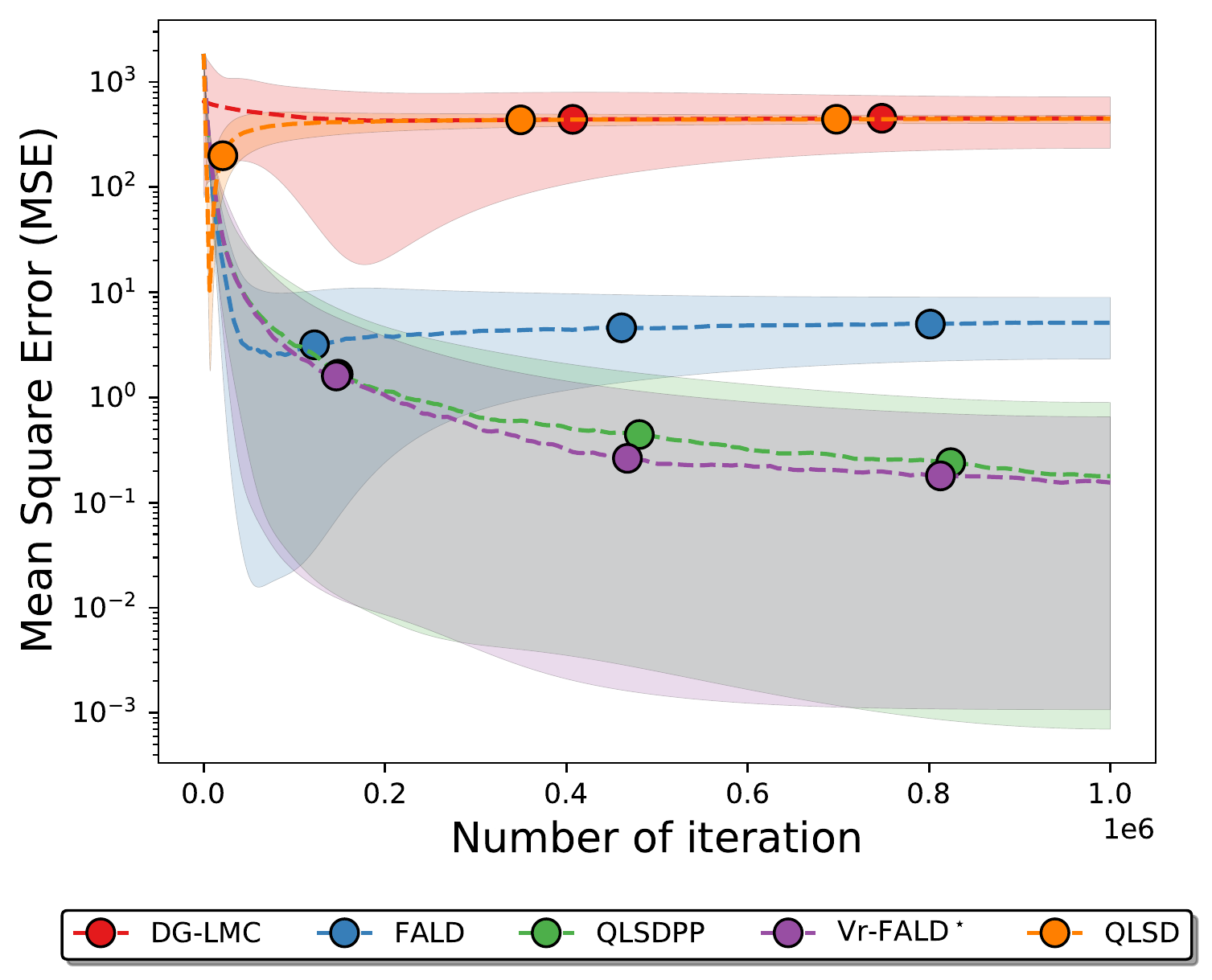}}
  \end{center}
  \caption{MSE comparison with $\pc=1/5$ and $\gamma=\bar{\gamma}/3$. \label{fig:mse}}
\end{figure}

\textbf{Bayesian Logistic Regression.}
We assess the performance of \FALD{} and \VRFALDs{} using calibration metrics---the expected calibration error (\texttt{ECE}), the Brier score (\texttt{BS}), and the negative log likelihood (\texttt{nNLL}); see \citet{guo2017calibration}---and predictive accuracy.
We consider Bayesian logistic regression applied to the Titanic dataset, which consists of $p=2$ classes with $N=2201$ samples in dimension $d=4$.
This dataset is allocated between $b=10$ clients in a very heterogeneous manner, as displayed in \Cref{fig:titanic:dataset-wass}.
We use an isotropic Gaussian prior with a mean of zero and variance $1$.
We also report the total variation distance between the predictive distribution obtained for \FALD{} and \VRFALDs{} to the predictive distribution approximated by  $100$ long runs of Langevin Stochastic Dynamics (LSD). These metrics are evaluated on a test data sets of $441$ samples, and the mean and standard deviation are reported in \Cref{table:titanic-comparison}.
Moreover, we illustrate the quality improvement of \VRFALDs{} over \FALD{} in \Cref{fig:titanic:hpd-nll}. We compared the Wasserstein distance using POT \citep{flamary2021pot} between the empirical distributions generated by \FALD{}, \VRFALDs{} to the estimated target distribution. Based on the same samples, we compute the relative highest posterior density (HPD) error; see \Cref{subsec:bayesian-logistic} for details.

\begin{table*}[t]
  \centering
  \begin{small}
  \begin{tabular}{lcccccc}
      \toprule
      \textsc{Method} & \texttt{Accuracy} & \texttt{Agreement} & $10^4\times$ \texttt{TV} & $10\times$\texttt{ECE} & $10\times$\texttt{BS} & $10\times$\texttt{nNLL} \\
      \midrule
      {LSD} & 72.4 $\pm$ 0.1 & 99.9 $\pm$ 0.1 & 5.53 $\pm$ 2.00 & 1.20 $\pm$ 0.01 & 3.44 $\pm$ 0.00 & 5.30 $\pm$ 0.00 \\
      \algoun{} & 77.0 $\pm$ 0.8 & 91.3 $\pm$ 0.9 & 533.32 $\pm$ 8.13 & 1.05 $\pm$ 0.09 & 3.37 $\pm$ 0.01 & 5.19 $\pm$ 0.00 \\
      \algoquatre{} & 74.9 $\pm$ 0.1 & 93.6 $\pm$ 0.1 & 287.81 $\pm$ 2.04 & 1.00 $\pm$ 0.05 & 3.51 $\pm$ 0.00 & 5.35 $\pm$ 0.00 \\
      \bottomrule
  \end{tabular}
  \end{small}
  \caption{Bayesian Logistic Regression on Titanic.}
  \label{table:titanic-comparison}
\end{table*}
\begin{figure}[!h]
  \begin{center}
    \addtolength{\leftskip} {-.5cm}
    \addtolength{\rightskip}{-.5cm}
    \includegraphics[width=.5\textwidth]{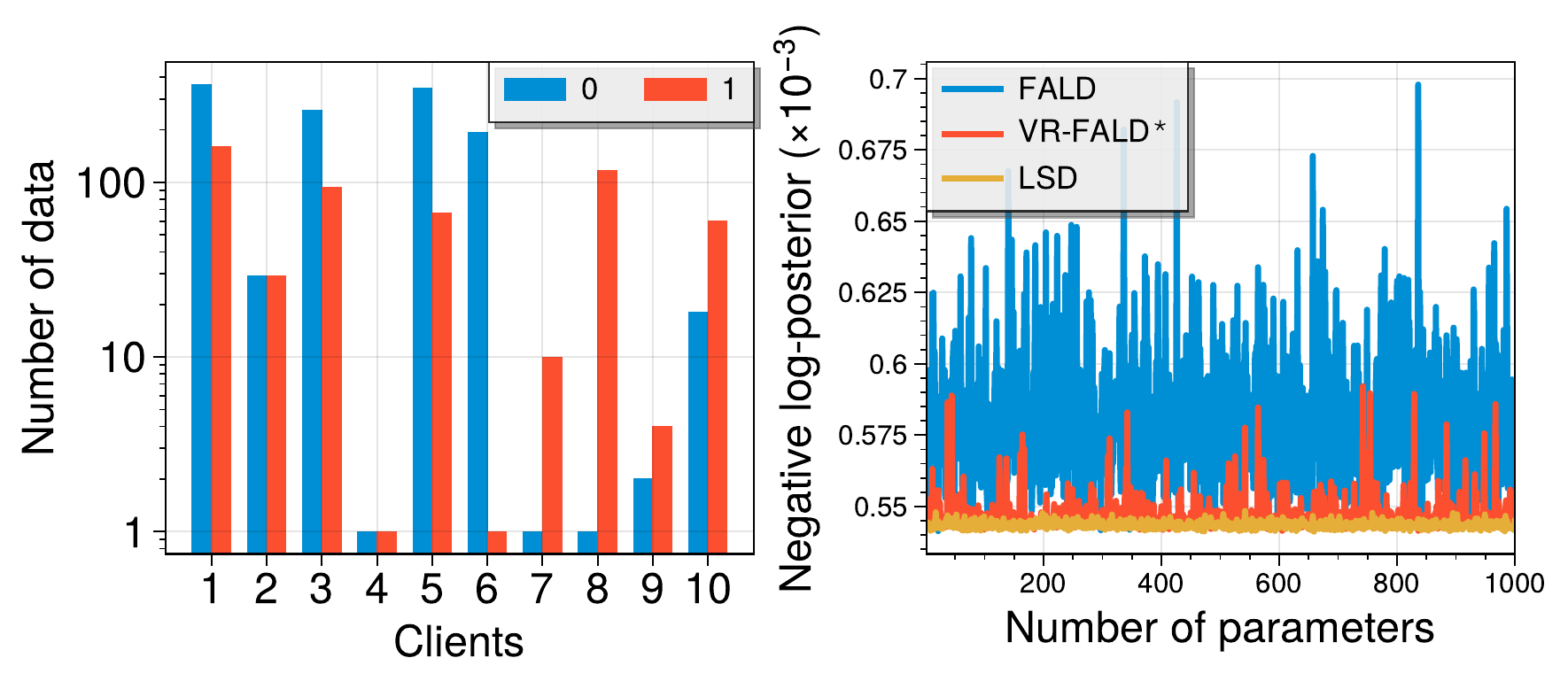}
  \end{center}
   \caption{Logistic regression -- dataset distribution (Log Scale) and negative log-posterior (right). \label{fig:titanic:dataset-wass}}
\end{figure}
\begin{figure}[!h]
  \centering
  \includegraphics[width=.5\textwidth]{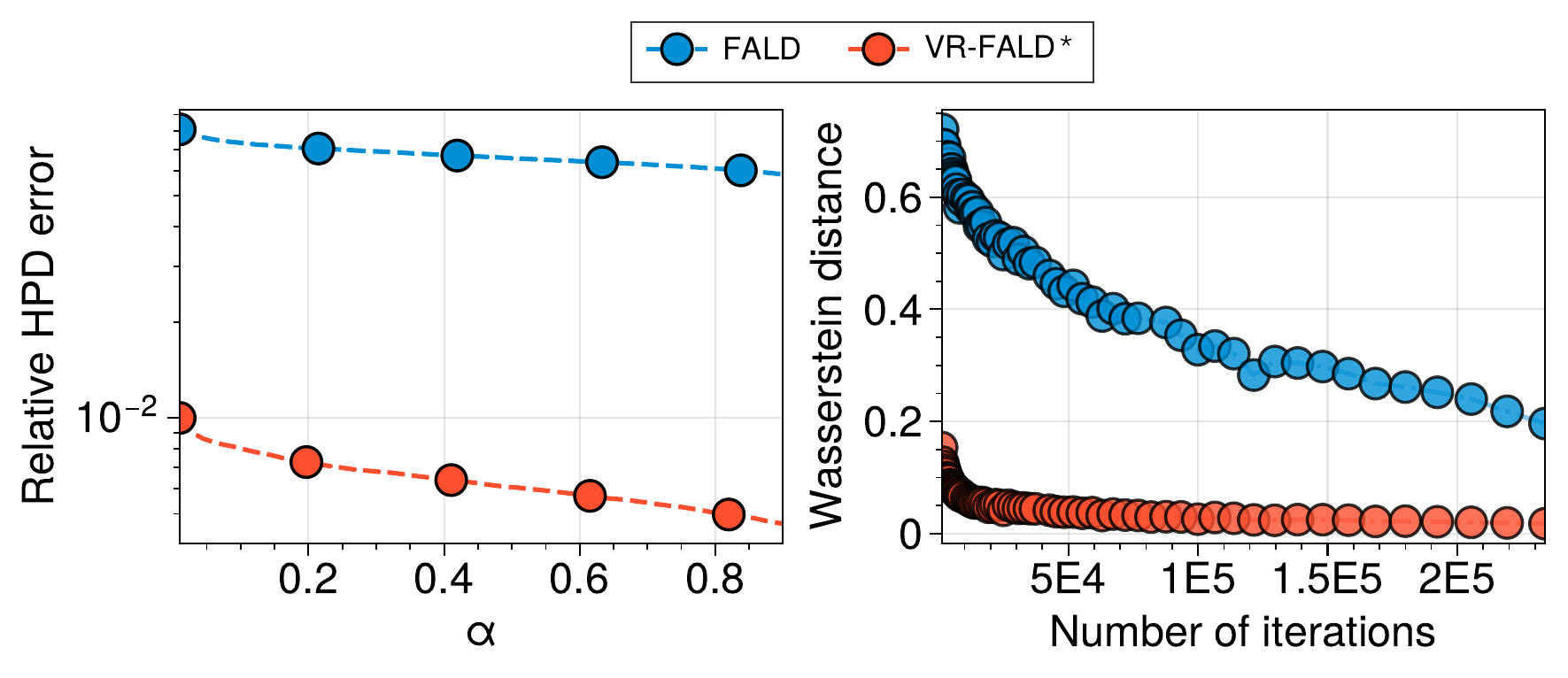} \hfill
  \caption{Logistic regression -- HPD relative error (left) and Wasserstein distance (right). \label{fig:titanic:hpd-nll}}
\end{figure}

\textbf{Bayesian Neural Network: MNIST.}
To illustrate the behavior of \FALD{} and \VRFALDs{} in a non-convex setting,
we perform Bayesian Neural Network (BNN) inference on the MNIST dataset \citep{deng2012mnist}. To this end, we distribute the dataset to $b=20$ clients as follows: $80\%$ of the data labeled $\mathrm{y} \in \{0,\ldots,9\}$ are equally allocated to clients $i = \mathrm{y}+1$ and $i = \mathrm{y} +10$; the remaining data are evenly distributed among the $b$ clients.
The likelihood of the observations is computed using LeNet5  neural network \citep{lecun1998gradient} with an isotropic Gaussian prior.
Finally, we implement \FALD{} and its variants with $\pc=1/b$ and $\qc = N_b/N_d$, where $N_b$ is the batch size used in the experiments and $N_d$ is the total number of data.
All standard deviations and the values of the other parameters are reported in \Cref{subsec:MNIST}.

In \Cref{table:mnist-comparison} we can observe that the best results are obtained by \VRFALDs: it achieves similar performance to the (fully centralized) \SGLD{} and p\SGLD{}. 
Alleviating client drift using control variates is still effective even in the highly non-convex BNN setting.
 \begin{table}[!h]
  \centering
  \addtolength{\leftskip} {-4cm}
  \addtolength{\rightskip}{-4cm}
  \addtolength{\tabcolsep}{-.1cm}
  \begin{small}
      \begin{tabular}{lccccccc}
        \toprule
        \textsc{Method} & {\SGLD} & {p\SGLD} & \algoun{} & \algoquatre{} & {FSGLD} \\
        \midrule
        \texttt{Accuracy}         & $99.1$ & $99.2$ & $99.1$ & $99.2$ & $98.5$ \\
        $10^3\times$\texttt{ECE}  & $6.88$ & $21.6$ & $4.07$ & $4.34$ & $6.34$ \\
        $10^2\times$\texttt{BS}   & $1.66$ & $1.45$ & $1.47$ & $1.39$ & $2.39$ \\
        $10^2\times$\texttt{nNLL} & $3.53$ & $4.24$ & $3.06$ & $3.43$ & $4.87$ \\
        \bottomrule
      \end{tabular}
  \end{small}
  \caption{Performance of Bayesian FL algorithms on MNIST. \label{table:mnist-comparison}}
\end{table}

\vspace{-1.5em}
\textbf{Bayesian Neural Network: CIFAR10.}
We consider the CIFAR10 dataset \citep{CIFAR10} and the ResNet-20 model \citep{he2016deep}. We split the data across 20 clients, similar to the previous example. Denote by $\mathsf{Y} = \{\mathrm{y}_1,\ldots,\mathrm{y}_{10}\}$ the set of labels. Then $80\%$ of the data associated with a label
$\mathrm{y}_j \in\mathsf{Y}$, $j \in [10]$, is distributed among clients $j$ and $j+10$, while the rest of the data is evenly distributed among clients. We assess the performance of {\FALD} and {\VRFALDs} against HMC, Deep Ensemble, and SGLD. We follow \citet{izmailov2021bayesian} by computing the \emph{accuracy}, \emph{agreement}, and total deviation distance between the predictive distribution. All of these quantities are defined in the Appendix; see \Cref{subsec:CIFAR}. We also report the calibration results and all resulting scores in \Cref{table:cifar10-comparison}; the results for {HMC} and {\SGLD} are from \citet[Table 6]{izmailov2021bayesian}.
Details on the implementation and choice of hyperparameters can be found in \Cref{subsec:CIFAR}.
We can see that \VRFALDs{} gives very similar results to \SGLD{} and performs favorably in terms of agreement. Finally, {\FALD} and {\VRFALDs} outperform Deep Ensembles.
\begin{table}[!h]
    \centering
    \addtolength{\leftskip} {-4cm}
    \addtolength{\rightskip}{-4cm}
    \addtolength{\tabcolsep}{-.1cm}
    \begin{scriptsize} % scriptsize footnotesize
    \begin{tabular}{lccccccc}
        \toprule
        \textsc{Method}         & {HMC} & {SGD} & \textsc{Deep Ens.} & {\SGLD} & \algoun{}& \algoquatre{}  \\
        \midrule
        \texttt{Accuracy}       & 89.6 & 91.57 & 91.68  & 89.96 & \textbf{92.54} & 92.03  \\
        \texttt{Agreement}      & 94.0 & 90.99 & 91.03  & \textbf{92.43} & 91.53 & 91.12 \\
        $10\times$ \texttt{TV}  & 0.74 & 1.45  & 1.49   & \textbf{1.03} & 1.42 & 1.39 \\
        $10^2\times$\texttt{ECE} & 5.9 & 4.71  & 5.44   & 4.41 & 3.79 & \textbf{3.26} \\
        $10\times$\texttt{BS}   & 1.4  & 1.69  & 1.45   & 1.53 & \textbf{1.16} & 1.20  \\
        $10\times$\texttt{nNLL} & 3.07 & 3.35  & 3.81   & 3.15 & 2.75 & \textbf{2.63} \\
        \bottomrule
    \end{tabular}
    \end{scriptsize}
    \caption{Performance of Bayesian FL algo. on CIFAR10.}
    \label{table:cifar10-comparison}
\end{table}
\vspace{-1.em}

\section{Conclusion}\label{main:sec-conclusion}
In this work, we propose \VRFALDs{} which extends the \FALD{} \citet{deng2021convergence} algorithm by introducing control variates to mitigate client drift and reducing stochastic gradient variance.
We develop a unifying framework for Bayesian FL combining ideas from Langevin Monte Carlo and Federated Averaging schemes.
The theory covers a wide range of local stochastic gradient algorithms; connections can be even be made with the global consensus Monte Carlo method \citep{rendell2020global,vono2022efficient}. Using this theoretical framework, we develop non-asymptotic bounds for the algorithms \FALD{} and \VRFALDs{}, and discuss the choice of hyperparameters (learning rate, communication probability, control variate update probability) to obtain optimal tradeoffs. Our analysis allows to correct some errors in the results obtained previously for \FALD{}. The results we obtain on both toy examples and applications to BNNs clearly show the importance of variance reduction and heterogeneity, even when the potential is non-convex.

\textbf{Acknowledgements.}
Alain Durmus and Eric Moulines gratefully acknowledge support from the Lagrange Mathematics and Computing Research Center.

\bibliography{biblio}

\begin{thebibliography}{}

\bibitem[Ahn et~al., 2014]{Ahn14}
Ahn, S., Shahbaba, B., and Welling, M. (2014).
\newblock {Distributed Stochastic Gradient {MCMC}}.
\newblock In {\em International Conference on Machine Learning}.

\bibitem[Al-Shedivat et~al., 2021]{alshedivat2021federated}
Al-Shedivat, M., Gillenwater, J., Xing, E., and Rostamizadeh, A. (2021).
\newblock {F}ederated {L}earning via posterior inference: A new perspective and
  practical algorithms.
\newblock In {\em ICLR 2021}.

\bibitem[Alistarh et~al., 2017]{alistarh2017qsgd}
Alistarh, D., Grubic, D., Li, J., Tomioka, R., and Vojnovic, M. (2017).
\newblock {QSGD}: Communication-efficient {SGD} via gradient quantization and
  encoding.
\newblock {\em Advances in Neural Information Processing Systems}, 30.

\bibitem[Bernstein et~al., 2018]{bernstein2018signsgd}
Bernstein, J., Wang, Y.-X., Azizzadenesheli, K., and Anandkumar, A. (2018).
\newblock {signSGD}: Compressed optimisation for non-convex problems.
\newblock In {\em International Conference on Machine Learning}, pages
  560--569. PMLR.

\bibitem[Boyd et~al., 2011]{boyd2011}
Boyd, S., Parikh, N., Chu, E., Peleato, B., and Eckstein, J. (2011).
\newblock Distributed optimization and statistical learning via the alternating
  direction method of multipliers.
\newblock {\em Foundations and Trends in Machine Learning}, 3(1):1--122.

\bibitem[Brooks et~al., 2011]{brooks2011handbook}
Brooks, S., Gelman, A., Jones, G., and Meng, X.-L. (2011).
\newblock {\em Handbook of markov chain monte carlo}.
\newblock CRC press.

\bibitem[Chatterji et~al., 2018]{chatterji2018theory}
Chatterji, N., Flammarion, N., Ma, Y., Bartlett, P., and Jordan, M. (2018).
\newblock On the theory of variance reduction for stochastic gradient {M}onte
  {C}arlo.
\newblock In {\em International Conference on Machine Learning}, pages
  764--773. PMLR.

\bibitem[Chen and Chao, 2021]{chen2020fedbe}
Chen, H.-Y. and Chao, W.-L. (2021).
\newblock Fedbe: Making {B}ayesian model ensemble applicable to {F}ederated
  {L}earning.
\newblock In {\em 9th International Conference on Learning Representations,
  ICLR 2021, Virtual Event, Austria, May 3-7, 2021}. OpenReview.net.

\bibitem[Chittoor and Simeone, 2021]{chittoor2021coded}
Chittoor, H. H.~S. and Simeone, O. (2021).
\newblock Coded consensus {M}onte {C}arlo: Robust one-shot distributed
  {B}ayesian learning with stragglers.
\newblock {\em arXiv preprint arXiv:2112.09794}.

\bibitem[Clark, 1987]{clark1987short}
Clark, D.~S. (1987).
\newblock Short proof of a discrete gronwall inequality.
\newblock {\em Discrete applied mathematics}, 16(3):279--281.

\bibitem[Coglianese and Lehr, 2016]{coglianese2016regulating}
Coglianese, C. and Lehr, D. (2016).
\newblock Regulating by robot: Administrative decision making in the
  machine-learning era.
\newblock {\em Geo. LJ}, 105:1147.

\bibitem[Dai et~al., 2021]{dai2021bayesian}
Dai, H., Pollock, M., and Roberts, G. (2021).
\newblock {B}ayesian fusion: Scalable unification of distributed statistical
  analyses.
\newblock {\em arXiv preprint arXiv:2102.02123}.

\bibitem[Dalalyan, 2017]{dalalyan2017further}
Dalalyan, A. (2017).
\newblock Further and stronger analogy between sampling and optimization:
  {L}angevin {M}onte {C}arlo and gradient descent.
\newblock In {\em Conference on Learning Theory}, pages 678--689. PMLR.

\bibitem[Dalalyan and Karagulyan, 2019]{dalalyan2019user}
Dalalyan, A.~S. and Karagulyan, A. (2019).
\newblock User-friendly guarantees for the {L}angevin {M}onte {C}arlo with
  inaccurate gradient.
\newblock {\em Stochastic Processes and their Applications},
  129(12):5278--5311.

\bibitem[Dawid and Musio, 2014]{dawid2014theory}
Dawid, A.~P. and Musio, M. (2014).
\newblock Theory and applications of proper scoring rules.
\newblock {\em Metron}, 72(2):169--183.

\bibitem[De~Souza et~al., 2022]{de2022parallel}
De~Souza, D.~A., Mesquita, D., Kaski, S., and Acerbi, L. (2022).
\newblock Parallel {MCMC} without embarrassing failures.
\newblock In {\em International Conference on Artificial Intelligence and
  Statistics}, pages 1786--1804. PMLR.

\bibitem[Deng, 2012]{deng2012mnist}
Deng, L. (2012).
\newblock The mnist database of handwritten digit images for machine learning
  research.
\newblock {\em IEEE Signal Processing Magazine}, 29(6):141--142.

\bibitem[Deng et~al., 2021]{deng2021convergence}
Deng, W., Ma, Y.-A., Song, Z., Zhang, Q., and Lin, G. (2021).
\newblock On convergence of federated averaging {L}angevin dynamics.
\newblock {\em arXiv preprint arXiv:2112.05120}.

\bibitem[Douc et~al., 2018]{douc2018markov}
Douc, R., Moulines, E., Priouret, P., and Soulier, P. (2018).
\newblock {\em Markov chains}.
\newblock Springer.

\bibitem[Dubey et~al., 2016]{dubey2016variance}
Dubey, K.~A., J~Reddi, S., Williamson, S.~A., Poczos, B., Smola, A.~J., and
  Xing, E.~P. (2016).
\newblock Variance reduction in stochastic gradient {L}angevin dynamics.
\newblock {\em Advances in neural information processing systems}, 29.

\bibitem[Durmus and Moulines, 2019]{durmus2019high}
Durmus, A. and Moulines, E. (2019).
\newblock High-dimensional {B}ayesian inference via the unadjusted {L}angevin
  algorithm.
\newblock {\em Bernoulli}, 25(4A):2854--2882.

\bibitem[El~Mekkaoui et~al., 2021]{el2021federated}
El~Mekkaoui, K., Mesquita, D., Blomstedt, P., and Kaski, S. (2021).
\newblock Federated stochastic gradient {L}angevin dynamics.
\newblock In {\em Uncertainty in Artificial Intelligence}, pages 1703--1712.
  PMLR.

\bibitem[Fatima et~al., 2017]{fatima2017survey}
Fatima, M., Pasha, M., et~al. (2017).
\newblock Survey of machine learning algorithms for disease diagnostic.
\newblock {\em Journal of Intelligent Learning Systems and Applications},
  9(01):1.

\bibitem[Flamary et~al., 2021]{flamary2021pot}
Flamary, R., Courty, N., Gramfort, A., Alaya, M.~Z., Boisbunon, A., Chambon,
  S., Chapel, L., Corenflos, A., Fatras, K., Fournier, N., Gautheron, L.,
  Gayraud, N.~T., Janati, H., Rakotomamonjy, A., Redko, I., Rolet, A., Schutz,
  A., Seguy, V., Sutherland, D.~J., Tavenard, R., Tong, A., and Vayer, T.
  (2021).
\newblock {POT}: {P}ython {O}ptimal {T}ransport.
\newblock {\em Journal of Machine Learning Research}, 22(78):1--8.

\bibitem[Gorbunov et~al., 2021]{gorbunov2021local}
Gorbunov, E., Hanzely, F., and Richt{\'a}rik, P. (2021).
\newblock Local sgd: Unified theory and new efficient methods.
\newblock In {\em International Conference on Artificial Intelligence and
  Statistics}, pages 3556--3564. PMLR.

\bibitem[Grenander and Miller, 1994]{GrenanderMiller1994}
Grenander, U. and Miller, M.~I. (1994).
\newblock Representations of knowledge in complex systems.
\newblock {\em Journal of the Royal Statistical Society, Series B},
  56(4):549--603.

\bibitem[Guo et~al., 2017]{guo2017calibration}
Guo, C., Pleiss, G., Sun, Y., and Weinberger, K.~Q. (2017).
\newblock On calibration of modern neural networks.
\newblock In {\em International Conference on Machine Learning}, pages
  1321--1330. PMLR.

\bibitem[Haddadpour et~al., 2021]{haddadpour2021federated}
Haddadpour, F., Kamani, M.~M., Mokhtari, A., and Mahdavi, M. (2021).
\newblock Federated learning with compression: Unified analysis and sharp
  guarantees.
\newblock In {\em International Conference on Artificial Intelligence and
  Statistics}, pages 2350--2358. PMLR.

\bibitem[He et~al., 2016]{he2016deep}
He, K., Zhang, X., Ren, S., and Sun, J. (2016).
\newblock Deep residual learning for image recognition.
\newblock In {\em IEEE Conference on Computer Vision and Pattern Recognition},
  pages 770--778.

\bibitem[Hoffman et~al., 2013]{JMLRv14hoffman13a}
Hoffman, M.~D., Blei, D.~M., Wang, C., and Paisley, J. (2013).
\newblock {Stochastic Variational Inference}.
\newblock {\em Journal of Machine Learning Research}, 14(4):1303--1347.

\bibitem[Holte, 2009]{holte2009discrete}
Holte, J.~M. (2009).
\newblock Discrete gronwall lemma and applications.
\newblock In {\em MAA-NCS meeting at the University of North Dakota},
  volume~24, pages 1--7.

\bibitem[Horv{\'a}th et~al., 2022]{horvath2022stochastic}
Horv{\'a}th, S., Kovalev, D., Mishchenko, K., Richt{\'a}rik, P., and Stich, S.
  (2022).
\newblock Stochastic distributed learning with gradient quantization and
  double-variance reduction.
\newblock {\em Optimization Methods and Software}, pages 1--16.

\bibitem[Izmailov et~al., 2021]{izmailov2021bayesian}
Izmailov, P., Vikram, S., Hoffman, M.~D., and Wilson, A. G.~G. (2021).
\newblock What are {B}ayesian neural network posteriors really like?
\newblock In {\em International Conference on Machine Learning}, pages
  4629--4640. PMLR.

\bibitem[Johnson and Zhang, 2013]{johnson2013accelerating}
Johnson, R. and Zhang, T. (2013).
\newblock Accelerating stochastic gradient descent using predictive variance
  reduction.
\newblock {\em Advances in neural information processing systems}, 26:315--323.

\bibitem[Kairouz et~al., 2021]{kairouz2019advances}
Kairouz, P., McMahan, H.~B., Avent, B., Bellet, A., Bennis, M., Bhagoji, A.~N.,
  Bonawitz, K., Charles, Z., Cormode, G., Cummings, R., et~al. (2021).
\newblock Advances and open problems in federated learning.
\newblock {\em Foundations and Trends{\textregistered} in Machine Learning},
  14(1--2):1--210.

\bibitem[Karimireddy et~al., 2020]{scaffold20}
Karimireddy, S.~P., Kale, S., Mohri, M., Reddi, S., Stich, S., and Suresh,
  A.~T. (2020).
\newblock {SCAFFOLD}: Stochastic controlled averaging for {F}ederated
  {L}earning.
\newblock In III, H.~D. and Singh, A., editors, {\em Proceedings of the 37th
  International Conference on Machine Learning}, volume 119 of {\em Proceedings
  of Machine Learning Research}, pages 5132--5143. PMLR.

\bibitem[Kovalev et~al., 2020]{kovalev2020don}
Kovalev, D., Horv{\'a}th, S., and Richt{\'a}rik, P. (2020).
\newblock Don’t jump through hoops and remove those loops: Svrg and katyusha
  are better without the outer loop.
\newblock In {\em Algorithmic Learning Theory}, pages 451--467. PMLR.

\bibitem[Krizhevsky, 2009]{CIFAR10}
Krizhevsky, A. (2009).
\newblock Learning multiple layers of features from tiny images.
\newblock Available at \url{http://www.cs.toronto.edu/~kriz/cifar.html}.

\bibitem[LeCun et~al., 1998]{lecun1998gradient}
LeCun, Y., Bottou, L., Bengio, Y., and Haffner, P. (1998).
\newblock Gradient-based learning applied to document recognition.
\newblock {\em Proceedings of the IEEE}, 86(11):2278--2324.

\bibitem[Lee et~al., 2020]{lee2020bayesian}
Lee, S., Park, C., Hong, S.-N., Eldar, Y.~C., and Lee, N. (2020).
\newblock {B}ayesian {F}ederated {L}earning over wireless networks.
\newblock {\em IEEE Journal on Selected Areas in Communications}.

\bibitem[Li et~al., 2016]{li2016preconditioned}
Li, C., Chen, C., Carlson, D., and Carin, L. (2016).
\newblock Preconditioned stochastic gradient {L}angevin dynamics for deep
  neural networks.
\newblock In {\em Thirtieth AAAI Conference on Artificial Intelligence}.

\bibitem[Li et~al., 2019]{li2019convergence}
Li, X., Huang, K., Yang, W., Wang, S., and Zhang, Z. (2019).
\newblock On the convergence of fedavg on non-iid data.
\newblock In {\em International Conference on Learning Representations}.

\bibitem[Liu and Ihler, 2014]{liu2014distributed}
Liu, Q. and Ihler, A.~T. (2014).
\newblock Distributed estimation, information loss and exponential families.
\newblock {\em Advances in neural information processing systems}, 27.

\bibitem[Maddox et~al., 2019]{maddox2019simple}
Maddox, W.~J., Izmailov, P., Garipov, T., Vetrov, D.~P., and Wilson, A.~G.
  (2019).
\newblock A simple baseline for bayesian uncertainty in deep learning.
\newblock {\em Advances in Neural Information Processing Systems}, 32.

\bibitem[McMahan et~al., 2017]{mcmahan2017communication}
McMahan, B., Moore, E., Ramage, D., Hampson, S., and y~Arcas, B.~A. (2017).
\newblock Communication-efficient learning of deep networks from decentralized
  data.
\newblock In {\em Artificial intelligence and statistics}, pages 1273--1282.
  PMLR.

\bibitem[Mesquita et~al., 2020]{mesquita2020embarrassingly}
Mesquita, D., Blomstedt, P., and Kaski, S. (2020).
\newblock Embarrassingly parallel {MCMC} using deep invertible transformations.
\newblock In {\em Uncertainty in Artificial Intelligence}, pages 1244--1252.
  PMLR.

\bibitem[Minsker et~al., 2014]{minsker2014}
Minsker, S., Srivastava, S., Lin, L., and Dunson, D. (2014).
\newblock Scalable and robust {B}ayesian inference via the median posterior.
\newblock In {\em Proceedings of the 31st International Conference on Machine
  Learning}.

\bibitem[Neiswanger et~al., 2014]{neiswanger2014asymptotically}
Neiswanger, W., Wang, C., and Xing, E.~P. (2014).
\newblock Asymptotically exact, embarrassingly parallel mcmc.
\newblock In {\em Proceedings of the Thirtieth Conference on Uncertainty in
  Artificial Intelligence}, pages 623--632.

\bibitem[Nemeth and Sherlock, 2018]{nemeth2018merging}
Nemeth, C. and Sherlock, C. (2018).
\newblock Merging {MCMC} subposteriors through {G}aussian-process
  approximations.
\newblock {\em {B}ayesian Analysis}, 13(2):507--530.

\bibitem[Nesterov, 2003]{nesterov2003introductory}
Nesterov, Y. (2003).
\newblock {\em Introductory lectures on convex optimization: A basic course},
  volume~87.
\newblock Springer Science \& Business Media.

\bibitem[Ovadia et~al., 2019]{ovadia2019can}
Ovadia, Y., Fertig, E., Ren, J., Nado, Z., Sculley, D., Nowozin, S., Dillon,
  J., Lakshminarayanan, B., and Snoek, J. (2019).
\newblock Can you trust your model's uncertainty? evaluating predictive
  uncertainty under dataset shift.
\newblock {\em Advances in neural information processing systems}, 32.

\bibitem[Plassier et~al., 2021]{plassier2021dg}
Plassier, V., Vono, M., Durmus, A., and Moulines, E. (2021).
\newblock {DG-LMC}: A turn-key and scalable synchronous distributed {MCMC}
  algorithm via {L}angevin {M}onte {C}arlo within gibbs.
\newblock In {\em International Conference on Machine Learning}, pages
  8577--8587. PMLR.

\bibitem[Rendell et~al., 2020]{rendell2020global}
Rendell, L.~J., Johansen, A.~M., Lee, A., and Whiteley, N. (2020).
\newblock Global consensus {M}onte {C}arlo.
\newblock {\em Journal of Computational and Graphical Statistics},
  30(2):249--259.

\bibitem[Ro et~al., 2021]{ro2021communication}
Ro, J., Chen, M., Mathews, R., Mohri, M., and Suresh, A.~T. (2021).
\newblock Communication-efficient agnostic federated averaging.
\newblock In {\em 22nd Annual Conference of the International Speech
  Communication Association, INTERSPEECH 2021}, pages 1753--1757. International
  Speech Communication Association.

\bibitem[Roberts and Tweedie, 1996]{Roberts1996}
Roberts, G.~O. and Tweedie, R.~L. (1996).
\newblock {Exponential convergence of {L}angevin distributions and their
  discrete approximations}.
\newblock {\em Bernoulli}, 2(4):341--363.

\bibitem[Scott et~al., 2016]{scott2016bayes}
Scott, S.~L., Blocker, A.~W., Bonassi, F.~V., Chipman, H.~A., George, E.~I.,
  and McCulloch, R.~E. (2016).
\newblock {B}ayes and big data: The consensus {M}onte {C}arlo algorithm.
\newblock {\em International Journal of Management Science and Engineering
  Management}, 11(2):78--88.

\bibitem[Shlezinger et~al., 2020]{shlezinger2020uveqfed}
Shlezinger, N., Chen, M., Eldar, Y.~C., Poor, H.~V., and Cui, S. (2020).
\newblock Uveqfed: Universal vector quantization for federated learning.
\newblock {\em IEEE Transactions on Signal Processing}, 69:500--514.

\bibitem[Smith and Topin, 2019]{smith2019super}
Smith, L.~N. and Topin, N. (2019).
\newblock Super-convergence: Very fast training of neural networks using large
  learning rates.
\newblock In {\em Artificial intelligence and machine learning for multi-domain
  operations applications}, volume 11006, page 1100612. International Society
  for Optics and Photonics.

\bibitem[Stich et~al., 2018]{stich2018sparsified}
Stich, S.~U., Cordonnier, J.-B., and Jaggi, M. (2018).
\newblock Sparsified sgd with memory.
\newblock {\em Advances in Neural Information Processing Systems}, 31.

\bibitem[Sun et~al., 2022]{sun2022federated}
Sun, L., Salim, A., and Richt{\'a}rik, P. (2022).
\newblock {F}ederated {L}earning with a sampling algorithm under isoperimetry.
\newblock {\em arXiv preprint arXiv:2206.00920}.

\bibitem[Tang et~al., 2021]{tang20211}
Tang, H., Gan, S., Awan, A.~A., Rajbhandari, S., Li, C., Lian, X., Liu, J.,
  Zhang, C., and He, Y. (2021).
\newblock 1-bit adam: Communication efficient large-scale training with
  adam’s convergence speed.
\newblock In {\em International Conference on Machine Learning}, pages
  10118--10129. PMLR.

\bibitem[Villani, 2009]{villani2009optimal}
Villani, C. (2009).
\newblock {\em Optimal transport: old and new}, volume 338.
\newblock Springer.

\bibitem[Vono et~al., 2020]{vono2020asymptotically}
Vono, M., Dobigeon, N., and Chainais, P. (2020).
\newblock Asymptotically exact data augmentation: Models, properties, and
  algorithms.
\newblock {\em Journal of Computational and Graphical Statistics},
  30(2):335--348.

\bibitem[Vono et~al., 2022a]{vono2022efficient}
Vono, M., Paulin, D., and Doucet, A. (2022a).
\newblock Efficient {MCMC} sampling with dimension-free convergence rate using
  {ADMM}-type splitting.
\newblock {\em Journal of Machine Learning Research}, 23(25).

\bibitem[Vono et~al., 2022b]{vono2022qlsd}
Vono, M., Plassier, V., Durmus, A., Dieuleveut, A., and Moulines, E. (2022b).
\newblock Qlsd: Quantised {L}angevin {S}tochastic {D}ynamics for {B}ayesian
  federated learning.
\newblock In {\em International Conference on Artificial Intelligence and
  Statistics}, pages 6459--6500. PMLR.

\bibitem[Wang et~al., 2013]{wang2013variance}
Wang, C., Chen, X., Smola, A.~J., and Xing, E.~P. (2013).
\newblock Variance reduction for stochastic gradient optimization.
\newblock {\em Advances in neural information processing systems}, 26.

\bibitem[Wang et~al., 2021]{wang2021field}
Wang, J., Charles, Z., Xu, Z., Joshi, G., McMahan, H.~B., Al-Shedivat, M.,
  Andrew, G., Avestimehr, S., Daly, K., Data, D., et~al. (2021).
\newblock A field guide to federated optimization.
\newblock {\em arXiv preprint arXiv:2107.06917}.

\bibitem[Wang et~al., 2020]{wang2020tackling}
Wang, J., Liu, Q., Liang, H., Joshi, G., and Poor, H.~V. (2020).
\newblock Tackling the objective inconsistency problem in heterogeneous
  federated optimization.
\newblock {\em Advances in neural information processing systems}.

\bibitem[Wang and Dunson, 2013]{Wang2013}
Wang, X. and Dunson, D.~B. (2013).
\newblock Parallelizing {{MCMC}} via {W}eierstrass sampler.
\newblock {\em arXiv preprint arXiv:1312.4605}.

\bibitem[Wang et~al., 2015]{Wang2015}
Wang, X., Guo, F., Heller, K.~A., and Dunson, D.~B. (2015).
\newblock Parallelizing {{MCMC}} with random partition trees.
\newblock In {\em Advances in Neural Information Processing Systems}.

\bibitem[Welling and Teh, 2011]{Welling11}
Welling, M. and Teh, Y.~W. (2011).
\newblock {B}ayesian learning via stochastic gradient {L}angevin dynamics.
\newblock In {\em International Conference on International Conference on
  Machine Learning}, page 681–688.
\newblock Available at
  \url{https://www.ics.uci.edu/~welling/publications/papers/stoclangevin_v6.pdf}.

\bibitem[Wilson et~al., 2021]{wilson2021evaluating}
Wilson, A.~G., Izmailov, P., Hoffman, M.~D., Gal, Y., Li, Y., Pradier, M.~F.,
  Vikram, S., Foong, A., Lotfi, S., and Farquhar, S. (2021).
\newblock Evaluating approximate inference in {B}ayesian deep learning.

\bibitem[Wu and Robert, 2017]{wu2017average}
Wu, C. and Robert, C.~P. (2017).
\newblock Average of recentered parallel mcmc for big data.
\newblock {\em arXiv preprint arXiv:1706.04780}.

\bibitem[Yang et~al., 2019]{yang2019federated}
Yang, Q., Liu, Y., Chen, T., and Tong, Y. (2019).
\newblock Federated machine learning: Concept and applications.
\newblock {\em ACM Transactions on Intelligent Systems and Technology (TIST)},
  10(2):1--19.

\bibitem[Yurochkin et~al., 2019]{yurochkin2019bayesian}
Yurochkin, M., Agarwal, M., Ghosh, S., Greenewald, K., Hoang, N., and Khazaeni,
  Y. (2019).
\newblock {B}ayesian nonparametric federated learning of neural networks.
\newblock In {\em International Conference on Machine Learning}, pages
  7252--7261. PMLR.

\bibitem[Zhang et~al., 2022]{zhang:fl:22}
Zhang, Y., Liu, D., and Simeone, O. (2022).
\newblock Leveraging channel noise for sampling and privacy via quantized
  federated {L}angevin {M}onte {C}arlo.

\end{thebibliography}

\addtocontents{toc}{\protect\setcounter{tocdepth}{2}}
\allowdisplaybreaks

\onecolumn
\aistatstitle{Federated Averaging Langevin Dynamics: \\ Toward a unified theory and new algorithms --- Supplementary Materials}

\tableofcontents

% !TEX root = main.tex

\paragraph{Notation and convention.}

The Euclidean norm and the scalar product on $\mathbb{R}^d$ are denoted by $\|\cdot\|$ and $\ps{\cdot}{\cdot}$ respectively. We set $\nsets = \nset\setminus\{0\}$ and denote by $\mathrm{N}(m,\Sigma)$ the Gaussian distribution with mean vector $m$ and covariance matrix $\Sigma$.
Finally, for any $f:\R^d\to\R$ twice continuously differentiable, we define the Laplacian $\Delta f$, which for all $x\in\R^d$ is given by $\Delta f(x)=\acn{\sum_{l=1}^{d}(\partial^{2} f_{j})(x)/\partial x_{l}^{2}}_{j=1}^{d}$.

\section{General scheme and technical results}\label{sec:generalscheme}

% !TEX root = main.tex

\paragraph{Problem statement.}\label{par:problemstatement}
We consider a general recursion that includes both \FALD{} and \VRFALDs{}. This general scheme is based on \iid~random variables $\seq{\xi_{k}}[k\in\nset]$ taking values in a measurable space $(\mse,\mce)$ and whose joint distribution is denoted by $\nu_{\xi}$. Moreover, we introduce a family of measurable functions $\{\funG^{i} : \rset^d \times \msy^{2} \times \msc^{2} \times \mse \to \rset^d\, , \,\funY^{i} : \rset^d\times \msy^{2} \times \mse \to \msy \, , \, \funC^{i} : \rset^d\times \msy \times \msc^{2} \times \mse \to \msc\}_{i=1}^{b}$, where $(\msy,\mcy)$ and $(\msc,\mcc)$ are measurable spaces.
For each $i \in [b]$, the functions $(\funG^{i},\funY^{i},\funC^{i})$ correspond to the update of the local parameter and control variate by the $i$th agent.
To define the global control variate update, we consider the function $\funD : \msy \times \msc^{b+1} \times (\rset^d)^{b+1} \times \mse \to \msy \times \msc$.
Starting from $\{G_{0}^{i}\}_{i=1}^{b}, \{\Xavg_{0}^{i}\}_{i=1}^{b} \in (\rset^d)^{b}$, $(C_{0}, \{C_{0}^{i}\}_{i=1}^{b})\in\msc^{b+1}$, $(Y_{0},\{Y_{0}^{i}\}_{i=1}^{b})\in\msy^{b+1}$ and set $\Xavg_{0} = b^{-1} \sum_{i=1}^{b}\Xavg_{0}^{i}$. For each $k\in\N$ the random variables are updated according to
\begin{align}
	\label{eq:def:Gik}
	&G_{k+1}^{i} = \funG^{i}\pr{\Xlocal_{k}^{i},Y_{k}^{i},Y_{k},C_{k}^{i},C_{k}, \xi_{k+1}}\eqsp,  \\
	\label{eq:def:tildeXik}
	& \Xupdate_{k+1}^{i} = \Xlocal_{k}^{i} - \gamma G^{i}_{k+1} + \sqrt{2\gamma}\pr{\sqrt{\tau/b}\, \tilde{Z}_{k+1} + \sqrt{1-\tau}\, Z_{k+1}^{i}} \eqsp,\\
	\label{eq:def:Yik}
	& Y_{k+1}^{i} = \funY^{i}\pr{\Xlocal_{k}^{i},Y_{k}^{i},Y_{k},\xi_{k+1}} \eqsp, \\
	\label{eq:def:Cik}
	& C_{k+1}^{i} = \funC^{i}\pr{\Xlocal_{k}^{i}, Y_{k}^{i},C_{k}^{i},C_{k},\xi_{k+1}}\eqsp, \\
	\label{eq:def:Xki}
	& \Xlocal_{k+1}^{i} = B_{k+1} \sum_{j=1}^{b}\Xupdate_{k+1}^{j} + (1-B_{k+1}) \Xupdate_{k+1}^{i}\eqsp,\\
	\label{eq:def:YkCk}
	& (Y_{k+1},C_{k+1}) =\funD(Y_{k},C_{k},\{C_{k}^{i}\}_{i=1}^{b},\{\Xlocal_{k}^{i}\}_{i=1}^{b},\xi_{k+1}) \eqsp,
\end{align}
where $\tau\in\ccint{0,1}$; $\gamma \in \ocint{0,\bgamma}$ is the stepsize; $\{(B_{k},\xi_{k},\tilde{Z}_{k},Z_{k}^1,\ldots,Z_{k}^{b}) \, : \, k\in\N^{\star}\}$ is a set of independent sequences of i.i.d. random variables such that for any $k\in\nsets$ $B_{k}$, is a Bernoulli random variable with parameter $\pc\in\ocint{0,1}$; and $(\tilde{Z}_{k},Z_{k}^1,\ldots,Z_{k}^{b})$ are $d$-dimensional standard Gaussian random variables. Recall that $\prn{\xi_{k}}_{k\ge 1}$ is a set of i.i.d. random variables distributed according to $\nu_{\xi}$ such that \Cref{ass:gunbiased} holds to ensure that the combination of functions $\{\funG^{i}\}_{i\in[b]}$ provides an unbiased estimate of $\nabla\potential$.

In iteration $k\ge 0$, the local parameter of the $i$th client is denoted by $\Xlocal_{k}^{i}$, and $G_{k}^{i}$ stands for its local gradient.
If $B_{k}=1$ (communication round), the local parameter $\Xlocal_{k}^{i}$ is set to the value of the global server parameter $\Xavg_{k}$.  If $B_k=0$, $\Xlocal_{k}^{i}$ is set to the local update $\Xupdate_{k}^{i}$. Moreover, we write $Y_{k}^{i}$ the reference point used to compute the control variate $C_{k}^{i}$.
The first step \eqref{eq:def:Gik} corresponds to the computation of a stochastic estimate of $\nabla\potential^{i}$ by the $i$th client. Then, the client updates the reference point $Y_{k}^{i}$ \eqref{eq:def:Yik} at which the local control variate is computed. The client also update its own local control variate $C_{k}^{i}$ in \eqref{eq:def:Cik}.
If $B_{k+1}=1$, then the server averages the parameter of each client, and broadcasts this average. If $B_{k+1}=0$, then each client keeps $\Xupdate_{k+1}^{i}$ as its new local parameter. Finally, the server updates the reference point $Y_{k}$ and the global control variate $C_{k}$ according to \eqref{eq:def:YkCk}.
\begin{algorithm} % H = here
	\caption{Stochastic Averaging Langevin Dynamics - {\algoun} and its variants}
	\label{algo:saladglobal}
	\begin{algorithmic}[]
		\State {\bfseries Input:} initial vectors $(\Xavg_{0}^{i})_{i\in[b]}$, noise parameter $\tau\in\ccint{0,1}$, number of communication rounds $K$, probability $\pc\in\ocint{0,1}$ of communication, probability $\qc\in\ccint{0,1}$ to update the control variates, and step-size $\gamma$
		\State {\bfseries Initialize:} $Y_{0}=(\nofrac{1}{b})\sum_{i=1}^b \Xavg_{0}^{i}$ and $C_{0}=(\nofrac{1}{b})\nabla\potential(Y_{0})$
		\For{$k=0$ {\bfseries to} $K-1$}
		% \State \hfill\Comment{On every client}
		\State Draw $B_{k+1}\sim\mathcal{B}(\pc), \tilde{Z}_{k+1}\sim \gauss(0_{d},\mathrm{I}_{d})$ \hfill\Comment{On every client}
		% \State \hfill\Comment{In parallel on the $b$ clients}
		\For{$i=1$ {\bfseries to} $b$} \hfill\Comment{In parallel on the $b$ clients}
		\State Draw $\xi_{k+1}^{i}\sim \nu_{\xi}^{i}$, $\tilde{Z}_{k+1}^{i}\sim \gauss(0_{d},\mathrm{I}_{d})$
		\State Compute {$G_{k}^{i}$} following \eqref{eq:def:Gik}
		\State Set
		$
			\Xupdate_{k+1}^{i} = \Xlocal_{k}^{i} - \gamma {G_{k}^{i}} + \sqrt{2\gamma}\,\prn{\sqrt{\tau/b}\,\tilde{Z}_{k+1} + \sqrt{1-\tau}\,\tilde{Z}_{k+1}^{i}}
		$
		\If{$B_{k+1}=1$}
		\State Broadcast $\Xupdate_{k+1}^{i}$ to the server \hfill\Comment{Communication round}
		\Else
		\State Update $\Xlocal_{k+1}^{i} \gets \Xupdate_{k+1}^{i}$ \hfill\Comment{Local step}
		\EndIf
		\If{$\tilde{B}_{k+1}=1$} \hfill\Comment{Control variate update round}
		\State Broadcast the necessary information to the server in order to update $(Y_{k}^{i}, {C_{k}^{i}}, Y_{k}, C_{k})$
		\Else
		\State Set $(Y_{k+1}^{i}, {C_{k+1}^{i}}, Y_{k+1}, C_{k+1}) \gets (Y_{k}^{i}, {C_{k}^{i}}, Y_{k}, C_{k})$ \hfill\Comment{No update}
		\EndIf
		\EndFor
		\If{$B_{k+1}=1$} \hfill\Comment{During communication round}
		% \State \Comment{On the central server}
		\State Update then broadcast $\Xavg_{k+1} \gets (\nofrac{1}{b})\sum_{i=1}^{b}\Xupdate_{k+1}^{i}$ \hfill\Comment{On the central server}
		% \State \Comment{On every client}
		\State Update the local parameter $\Xlocal_{k+1}^{i} \gets \Xavg_{k+1}$ \hfill\Comment{On every client}
		% \Else
		% \State $\Xs_{k+1} \gets \frac{B_{k+1}}{b} \sum_{i=1}^{b}\Xupdate_{k+1}^{i}
			% + (1-B_{k+1}) \pr{\Xs_{k} + \sqrt{\frac{2\gamma\tau}{b}}\tilde{Z}_{k+1} + \frac{\sqrt{2(1-\tau)\gamma}}{b}\sum_{i=1}^b Z_{k+1}^{i}}$
		\EndIf
		\If{$\tilde{B}_{k+1}=1$}  \hfill\Comment{During control variate update round}
		% \State \Comment{On the central server}
		\State If needed, update then broadcast $Y_{k+1} \gets (\nofrac{1}{b})\sum_{i=1}^{b} \Xlocal_{k}^{i}$ \hfill\Comment{On the central server}
		% \State \Comment{On every client}
		\State Update $(Y_{k}^{i}, \cobalt{C_{k}^{i}})$ using the parameters $(\Xlocal_{k}^{i}, Y_{k}^{i}, Y_{k}, Y_{k+1}, C_{k})$ \hfill\Comment{On every client}
		\State Update then broadcast $C_{k+1} \gets (\nofrac{1}{b})\sum_{i=1}^b {C_{k+1}^{i}}$ \hfill\Comment{On the central server}
		% \State Compute \cobalt{$C_{k+1}^{i}$} \qquad\Comment{on every client}
		% \State Update then broadcast $C_{k+1} \gets (\nofrac{1}{b})\sum_{i=1}^b C_{k+1}^{i}$ \qquad\Comment{on the central server}
		\EndIf
		\EndFor
		\State {\bfseries Output:} samples $\{\Xavg_{\ell}\}_{\{\ell\in[K]\,:\, B_{\ell}=1\}}$.
	\end{algorithmic}
\end{algorithm}
Denote the filtration $\acn{\mathcal{F}_{k}}_{k\in\N}$ defined for any $k\ge 0$, by
\begin{equation}\label{eq:def:Fk}
	\mathcal{F}_{k}=\sigma\pr{\Xcontinuous_{0}, \pr{B_{l}, C_l, Y_{l}, \tilde{Z}_{l}, \xi_{l}, \pr{C_{l}^{i},G_l^{i},X_l^{i},\Xupdate_l^{i}, Y_{l}^{i}, Z_{l}^{i}}_{i=1,\ldots,n}}_{0\le l\le k}}\eqsp
\end{equation}
and consider the conditional expectation and variance denoted by $\E^{\mathcal{F}_{k}}$, $\var^{\mathcal{F}_{k}}(\cdot) = \E^{\mathcal{F}_{k}}\brn{\normn{\cdot - \E^{\mathcal{F}_{k}}\brn{\cdot}}^{2}}$ respectively.
For $k\in\N$, we introduce $\Xavg_{k}$ the average of the local parameters given by
\begin{equation}\label{eq:def:Xk}
	\Xavg_{k} = \frac{1}{b}\sum_{i=1}^b \Xlocal_{k}^{i}
\end{equation}
and we set
\begin{equation}\label{eq:def:Vk}
	V_{k} = \frac{1}{b}\sum_{i=1}^{b}\normn{\Xlocal_{k}^{i}-\Xavg_{k}}^{2}\eqsp.
\end{equation}
Finally, to control the distance between the average parameter $\Xavg_{k}$ and the minimizer $x_{\star} = \argmin \potential$, we consider the parameter $\dist_{k}$, which for $k\ge 0$ is given by
\begin{equation}\label{eq:def:dk}
	\dist_{k} = \normn{\Xavg_{k} - x_{\star}}\eqsp.
\end{equation}
For each $k\in\nset$ and $\gamma\in\ocint{0,\bgamma}$, we denote by $\nug_{k}$ the distribution of $\Xavg_{k}$ defined by \eqref{eq:def:Xk}.
To ensure the quality of the samples generated by \Cref{algo:saladglobal}, we control the Wasserstein distance $\wass(\pi,\nug_{k})$. Recall that the Wasserstein distance is the infimum of $\E\brn{\normn{\Xcontinuous_{k\gamma} - \Xavg_{k}}^{2}}$ over all couplings $(\Xcontinuous_{k\gamma}, \Xavg_{k})$ such that $\Xcontinuous_{k\gamma}$ is distributed according to $\pi$.
Thus, to study the convergence of $(\nug_{k})_{k\in\N}$, we introduce a synchronous coupling $(\Xcontinuous_{k\gamma},\Xavg_{k})_{k\ge 0}$ with values in $(\R^d)^{2}$ between $\pi$ and $\nug_{k}$, starting from the couple $(\Xcontinuous_{0},\Xavg_{0})$ distributed according to $\zeta\in\mathcal{P}_{2}\prn{\Rd\times\Rd}$, \ie, $\zeta(\R^d,\cdot)=\nug_{0}\in\mathcal{P}_{2}(\R^d)$ and $\zeta(\cdot,\R^d)=\pi$. Since $\log \pi$ is supposed $m$-strongly concave by \Cref{ass:fi}, note that $\pi$ belongs in $\mathcal{P}_{2}(\R^d)$.
Based on independent $d$-dimensional standard Brownian motions $(\{\tilde{\mathsf{W}}_t,\{\tilde{\mathsf{W}}_t^{i}\}_{i=1}^b\})_{t\ge 0}$, we define $\mathsf{W}_t = \sqrt{\tau} \tilde{\mathsf{W}}_t + \sqrt{(1-\tau)/b}\sum_{i=1}^b \tilde{\mathsf{W}}_t^{i}$. For $k\in\N^{\star}$, we introduce $\tilde{Z}_{k} = \gamma^{-1/2}(\tilde{\mathsf{W}}_{k\gamma}-\tilde{\mathsf{W}}_{(k-1)\gamma})$, and for $i\in[b]$, we consider $\tilde{Z}_{k}^{i} = \gamma^{-1/2}(\tilde{\mathsf{W}}_{k\gamma}^{i}-\tilde{\mathsf{W}}_{(k-1)\gamma}^{i})$.
Therefore, for all $k\in\N^\star$ we can verify that $\mathsf{W}_{k\gamma}-\mathsf{W}_{(k-1)\gamma} = \sqrt{\gamma\tau}\tilde{Z}_{k} + \sqrt{\nofrac{\gamma(1-\tau)}{b}}\sum_{i=1}^{b}\tilde{Z}_{k}^{i}$.
Moreover, consider $(\Xcontinuous_t)_{t\ge 0}$ the strong solution of the Langevin stochastic differential equation (SDE) given by
\begin{equation}\label{eq:def:xcontinuous}
	\rmd\Xcontinuous_{t}
	% = - \frac{1}{b}\sum_{i=1}^{b}\nabla\potential^{i}(\Xcontinuous_t) \rmd t + \sqrt{\frac{2}{b}} \rmd \mathsf{W}_t\eqsp.
	= - \frac{1}{b}\nabla\potential(\Xcontinuous_t) \,\rmd t + \sqrt{\frac{2}{b}} \,\rmd \mathsf{W}_t\eqsp.
\end{equation}
The Langevin diffusion defines a Markov semigroup $(\tP_{t})_{t\ge 0}$ satisfying $\pi\tP_{t}=\pi$ for any $t\ge 0$, see for example \citet[Theorem 2.1]{Roberts1996}.
Note that $\Xcontinuous_t$ and $\Xavg_{k}$ are distributed according to  $\pi$ and $\nug_{k}$, respectively.
From the definition of the Wasserstein distance of order $2$ it follows that
\begin{equation}
	\wass\prn{\pi, \nug_{k}}
	\le \E\br{\normn{\Xcontinuous_{k\gamma}-\Xavg_{k}}^{2}}^{1/2}\eqsp.
\end{equation}
So the proof consists mainly of upper bounding the squared norm $\normn{\Xcontinuous_{k\gamma}-\Xavg_{k}}$, from which we derive an explicit bound on the Wasserstein distance by the previous inequality.

\paragraph{First upper bound on $\E^{\mathcal{F}_{k}}\brn{\normn{\Xcontinuous_{(k+1)\gamma} - \Xavg_{k+1}}^{2}}$.}

Under mild assumptions, we derive a first bound in \Cref{lem:contraction:xkbis} to control $\normn{\Xcontinuous_{(k+1)\gamma} - \Xavg_{k+1}}^{2}$ based on $\normn{\Xcontinuous_{k\gamma} - \Xavg_{k}}^{2}$, $(\nofrac{1}{b})\sum_{i=1}^{b} G_{k}^{i}$ and $V_k$. This decomposition highlights the different approximations brought by the discretization of the Langevin diffusion \eqref{eq:def:xcontinuous} between the averaged parameter $(\Xavg_{k})_{k\in\nset}$ defined in \eqref{eq:def:Xk} and $\acn{\Xcontinuous_{k\gamma}}_{k\in\N}$.
Recall that $x_{\star}=\argmin\potential$ and for all $k\in\N$, consider $I_{k}$ the approximation error defined by
\begin{equation}\label{eq:def:Ik}
	I_{k} = \int_{k\gamma}^{(k+1)\gamma}\left(\nabla\barpotential(\Xcontinuous_{s})-\nabla\barpotential(\Xcontinuous_{k \gamma})\right) \rmd s\eqsp.
\end{equation}
For $\bgamma>0$ small enough and $k \in \nset$, for all $\gamma \in \ocint{0,\bgamma}$ and under the following assumption \Cref{ass:gunbiased} we control the distance between the target distribution $\pi$ and $\nug_{k}$. 
\begin{assumption}\label{ass:gunbiased}
	For any $\{(x^{i},y^{i},c^{i})\}_{i=1}^{b} \in \rset^{3d}$, we have
	\[
			\sum_{i=1}^{b}\int_{\mse}\funG^{i}\pr{\ac{(x^{j},y^{j},c^{j})}_{j=1}^{b},\xi^{i}} \rmd \nu_{\xi}(\xi^{i})
			= \sum_{i=1}^{b}\nabla\potential^{i}(x^{i})\eqsp.
	\]
\end{assumption}

\begin{proposition}\label{lem:contraction:xkbis}
	Assume \Cref{ass:fi}, \Cref{ass:gunbiased} hold and let $\gamma\le 2(3L)^{-1}$.
	Then, for any $k\in\N$, we have
	\begin{multline}
		\E^{\mathcal{F}_{k}}\br{\normn{\Xcontinuous_{(k+1)\gamma} - \Xavg_{k+1}}^{2}}
		\le \br{1 - \gamma\conv\pr{1-3\gamma L}} \normn{\Xcontinuous_{k\gamma} - \Xavg_{k}}^{2}
		+ \gamma\pr{\frac{2 L^{2}}{\conv} + 3\gamma L^{2}} V_{k}\\
		+ \pr{\frac{2}{\betaempty\gamma\conv}\normlr{\E^{\mathcal{F}_{k}}\br{I_{k}}}^{2} + 3\E^{\mathcal{F}_{k}}\br{\normlr{I_{k}}^{2}}}
		+ \gamma^{2} \var^{\mathcal{F}_{k}}\pr{\frac{1}{b}\sum_{i=1}^{b} G_{k}^{i}} \eqsp,
	\end{multline}
	where $V_{k}, \mathcal{F}_{k}, \dist_{k}$ are defined in \eqref{eq:def:Vk}, \eqref{eq:def:Fk} and \eqref{eq:def:dk}.
\end{proposition}
\begin{proof}
	Let $k$ be in $\N$ and $\gamma$ in $\ocint{0,2(3L)^{-1}}$. Recall the stochastic processes $\Xavg_{k+1}, \Xcontinuous_{(k+1)\gamma}$ are defined in \eqref{eq:def:Xk} and \eqref{eq:def:xcontinuous} by
	\begin{equation}
		\begin{cases}
			\Xcontinuous_{(k+1)\gamma} = \Xcontinuous_{k\gamma} - \gamma \nabla\barpotential(\Xcontinuous_{k\gamma}) - I_{k} + \sqrt{2/b}\pr{\mathsf{W}_{(k+1)\gamma} - \mathsf{W}_{k\gamma}}\eqsp,\\
			\Xavg_{k+1} = \frac{1}{b}\sum_{i=1}^{b}\br{\Xlocal_{k}^{i} - \gamma G_{k}^{i} + \sqrt{2\gamma}\pr{\sqrt{\tau/b}\,\tilde{Z}_{k+1} + \sqrt{1-\tau}\,\tilde{Z}_{k+1}^{i}}}\eqsp,
		\end{cases}
	\end{equation}
	with $I_{k}$ defined in \eqref{eq:def:Ik}.
	Substracting the two above equations gives
	\begin{equation}
		\Xcontinuous_{(k+1)\gamma} - \Xavg_{k+1}
		= \pr{\Xcontinuous_{k\gamma} - \Xavg_{k}} - \betaemptyinv \pr{\int_{k\gamma}^{(k+1)\gamma}\nabla\barpotential(\Xcontinuous_s)\rmd s - \frac{\gamma}{b}\sum_{i=1}^{b}G_{k}^{i}}\eqsp.
	\end{equation}
	Taking the conditional expectation of the above equation and developing the squared norm, we obtain
	\begin{multline}\label{eq:lem:bound:contract:1}
		\E^{\mathcal{F}_{k}}\br{\normn{\Xcontinuous_{(k+1)\gamma} - \Xavg_{k+1}}^{2}}
		= \E^{\mathcal{F}_{k}}\br{\normn{\Xcontinuous_{k\gamma} - \Xavg_{k}}^{2}}
		- 2\betaemptyinv\gamma \ps{\Xcontinuous_{k\gamma} - \Xavg_{k}}{\nabla\barpotential(\Xcontinuous_{k\gamma}) - \nabla\barpotential(\Xavg_{k})}\\
		- 2\betaemptyinv\ps{\Xcontinuous_{k\gamma} - \Xavg_{k}}{\E^{\mathcal{F}_{k}}\br{I_{k}} + \gamma\nabla\barpotential(\Xavg_{k}) - \frac{\gamma}{b}\sum_{i=1}^{b}\E^{\mathcal{F}_{k}}\br{G_{k}^{i}}}
		+ \betaemptyinvtwo\E^{\mathcal{F}_{k}}\br{\normlr{I_{k} + \gamma\nabla\barpotential(\Xcontinuous_{k\gamma}) - \frac{\gamma}{b}\sum_{i=1}^{b} G_{k}^{i}}^{2}}\eqsp.
	\end{multline}
	Using that for all $\alpha>0, (a,b)\in(\R^d)^{2}$, $2\ps{a}{b}\le \alpha\normlr{a}^{2} + \prn{1/\alpha}\normlr{b}^{2}$ combined with \Cref{ass:gunbiased}, for any $\epsilon>0$ we have
	\begin{multline}\label{eq:lem:bound:contract:2}
		-2\ps{\Xcontinuous_{k\gamma} - \Xavg_{k}}{\E^{\mathcal{F}_{k}}\br{I_{k}} + \gamma\nabla\barpotential(\Xavg_{k}) - \frac{\gamma}{b}\sum_{i=1}^{b}\E^{\mathcal{F}_{k}}\br{G_{k}^{i}}}
		\le \epsilon\normn{\Xcontinuous_{k\gamma} - \Xavg_{k}}^{2}
		+ \frac{2}{\epsilon}\normlr{\E^{\mathcal{F}_{k}}\br{I_{k}}}^{2}\\
		+ \frac{2\gamma^{2}}{\epsilon}\normlr{\nabla\barpotential(\Xavg_{k}) - \frac{1}{b}\sum_{i=1}^{b}\nabla\potential^{i}(\Xlocal_{k}^{i})}^{2}\eqsp.
	\end{multline}
	In addition, the unbiased property \Cref{ass:gunbiased} implies that
	\begin{multline}\label{eq:lem:bound:contract:3}
		\E^{\mathcal{F}_{k}}\br{\normlr{I_{k} + \gamma\nabla\barpotential(\Xcontinuous_{k\gamma}) - \frac{\gamma}{b}\sum_{i=1}^{b} G_{k}^{i}}^{2}}
		= \gamma^{2}\var^{\mathcal{F}_{k}}\pr{\frac{1}{b}\sum_{i=1}^{b} G_{k}^{i}}\\
		+ \E^{\mathcal{F}_{k}}\br{\normlr{\gamma\pr{\nabla\barpotential(\Xcontinuous_{k\gamma}) - \nabla\barpotential(\Xavg_{k})} + I_{k} + \gamma\nabla\barpotential(\Xavg_{k}) - \frac{\gamma}{b}\sum_{i=1}^{b}\nabla\potential^{i}(\Xlocal_{k}^{i})}^{2}}\eqsp.
	\end{multline}
	% For any $i\in[b]$, the Lipschitz property \Cref{ass:fi} of $\nabla\potential^{i}$ combined with the Young inequality show that
	The Young inequality shows that
	\begin{multline}\label{eq:lem:bound:contract:4new}
		% \nonumber
		\E^{\mathcal{F}_{k}}\br{\normlr{\gamma\pr{\nabla\barpotential(\Xcontinuous_{k\gamma}) - \nabla\barpotential(\Xavg_{k})} + I_{k} + \gamma\nabla\barpotential(\Xavg_{k}) - \frac{\gamma}{b}\sum_{i=1}^{b}\nabla\potential^{i}(\Xlocal_{k}^{i})}^{2}}\\
		% \nonumber
		\le 3\gamma^{2} \normlr{\nabla\barpotential(\Xcontinuous_{k\gamma}) - \nabla\barpotential(\Xavg_{k})}^{2}
		+ 3\E^{\mathcal{F}_{k}}\br{\normlr{I_{k}}^{2}}
		+ 3\gamma^{2}\normlr{\nabla\barpotential(\Xavg_{k}) - \frac{1}{b}\sum_{i=1}^{b}\nabla\potential^{i}(\Xlocal_{k}^{i})}^{2}\eqsp.
		% &\le 3\gamma^{2} \normlr{\nabla\barpotential(\Xcontinuous_{k\gamma}) - \nabla\barpotential(\Xavg_{k})}^{2}
		% + 3\E^{\mathcal{F}_{k}}\br{\normlr{I_{k}}^{2}}
		% + 3\gamma^{2}L^{2} V_{k}\eqsp.
	\end{multline}
	By \Cref{ass:fi} we know that $\barpotential$ is $L$-smooth and convex which imply the co-coercivity of $\barpotential$ \citep[Theorem 2.1.5]{nesterov2003introductory}, that is for all $x,y\in\R^d$, $\normlr{\nabla\barpotential(y) - \nabla\barpotential(x)}^{2}\le L\ps{\nabla\barpotential(y) - \nabla\barpotential(x)}{y-x}$. Hence, we deduce that
	\begin{equation}\label{eq:lem:bound:contract:4}
		\normlr{\nabla\barpotential(\Xcontinuous_{k\gamma}) - \nabla\barpotential(\Xavg_{k})}^{2} \le L \ps{\Xcontinuous_{k\gamma} - \Xavg_{k}}{\nabla\barpotential(\Xcontinuous_{k\gamma}) - \nabla\barpotential(\Xavg_{k})}\eqsp.
	\end{equation}
	Setting $\epsilon=\gamma\conv$, we have $0<\epsilon\le 1$ and $1+1/\epsilon\le 2(\gamma\conv)^{-1}$. Therefore, \eqref{eq:lem:bound:contract:2}, \eqref{eq:lem:bound:contract:3} and \eqref{eq:lem:bound:contract:4} associated with \eqref{eq:lem:bound:contract:1} show that
	\begin{multline}\label{eq:lem:bound:contract:5}
		\E^{\mathcal{F}_{k}}\br{\normn{\Xcontinuous_{(k+1)\gamma} - \Xavg_{k+1}}^{2}}
		\le \pr{1 + \gamma\conv} \normn{\Xcontinuous_{k\gamma} - \Xavg_{k}}^{2}
		+ \pr{\frac{2}{\betaempty\gamma\conv}\normlr{\E^{\mathcal{F}_{k}}\br{I_{k}}}^{2} + 3\E^{\mathcal{F}_{k}}\br{\normlr{I_{k}}^{2}}}\\
		- \gamma \pr{2 - {3\gamma L}} \ps{\Xcontinuous_{k\gamma} - \Xavg_{k}}{\nabla\barpotential(\Xcontinuous_{k\gamma}) - \nabla\barpotential(\Xavg_{k})}\\
		+ \gamma^{2}\pr{3 + \frac{2}{\gamma\conv}} \normlr{\nabla\barpotential(\Xavg_{k}) - \frac{1}{b}\sum_{i=1}^{b}\nabla\potential^{i}(\Xlocal_{k}^{i})}^{2}
		+ \gamma^{2} \var^{\mathcal{F}_{k}}\pr{\frac{1}{b}\sum_{i=1}^{b} G_{k}^{i}}\eqsp.
	\end{multline}
	For any $i\in[b]$, by \Cref{ass:fi}, the $\conv$-convexity of $\barpotential$ gives that
	\begin{equation}\label{eq:lem:bound:contract:6}
		\ps{\Xcontinuous_{k\gamma} - \Xavg_{k}}{\nabla\barpotential(\Xcontinuous_{k\gamma}) - \nabla\barpotential(\Xavg_{k})}
		\ge \conv \normn{\Xcontinuous_{k\gamma} - \Xavg_{k}}^{2}
	\end{equation}
	In addition, under \Cref{ass:fi} the Jensen inequality implies
	\begin{equation}\label{eq:lem:bound:contract:7}
		\normlr{\nabla\barpotential(\Xavg_{k}) - \frac{1}{b}\sum_{i=1}^{b}\nabla\potential^{i}(\Xlocal_{k}^{i})}^{2}
		\le L^{2} V_{k}\eqsp,
	\end{equation}
	where $V_{k}$ is defined in \eqref{eq:def:Vk}.
	% Moreover, recall that under \Cref{ass:gradsto:lip} we have
	% \begin{equation}\label{eq:lem:bound:contract:8}
	% 	\var^{\mathcal{F}_{k}}\pr{\frac{1}{b}\sum_{i=1}^{b} G_{k}^{i}}
	% 	\le c_{d} \dist_{k}^{2} + c_{\sigma} \sigma_{k}^{2} + c_{V} V_{k} + c\eqsp.
	% \end{equation}
	% Therefore, plugging \eqref{eq:lem:bound:contract:6}, \eqref{eq:lem:bound:contract:7} and \eqref{eq:lem:bound:contract:8} in \eqref{eq:lem:bound:contract:5} yields the following inequality
	% \begin{multline}\label{eq:bound:contraction:xkbis:last}
	% 	\E^{\mathcal{F}_{k}}\br{\normn{\Xcontinuous_{(k+1)\gamma} - \Xavg_{k+1}}^{2}}
	% 	= \br{1 - \gamma\conv\pr{1-3\gamma L}} \normn{\Xcontinuous_{k\gamma} - \Xavg_{k}}^{2}
	% 	+ \gamma\pr{\frac{2 L^{2}}{\conv} + 3\gamma L^{2} + \gamma c_{V}} V_{k}\\
	% 	+ \gamma^{2} c_{d} \dist_{k}^{2}
	% 	+ \gamma^{2} c_{\sigma} \sigma_{k}^{2}
	% 	+ \pr{\frac{2}{\betaempty\gamma\conv}\normlr{\E^{\mathcal{F}_{k}}\br{I_{k}}}^{2} + 3\E^{\mathcal{F}_{k}}\br{\normlr{I_{k}}^{2}}}
	% 	+ \gamma^{2} c\eqsp.
	% \end{multline}
	Therefore, using the assumption on $\gamma$ and plugging \eqref{eq:lem:bound:contract:6} and \eqref{eq:lem:bound:contract:7} in \eqref{eq:lem:bound:contract:5} yields the expected inequality.
	% \begin{multline}\label{eq:bound:contraction:xkbis:last}
	% 	\E^{\mathcal{F}_{k}}\br{\normn{\Xcontinuous_{(k+1)\gamma} - \Xavg_{k+1}}^{2}}
	% 	\le \br{1 - \gamma\conv\pr{1-3\gamma L}} \normn{\Xcontinuous_{k\gamma} - \Xavg_{k}}^{2}
	% 	+ \gamma\pr{\frac{2 L^{2}}{\conv} + 3\gamma L^{2}} V_{k}\\
	% 	+ \pr{\frac{2}{\betaempty\gamma\conv}\normlr{\E^{\mathcal{F}_{k}}\br{I_{k}}}^{2} + 3\E^{\mathcal{F}_{k}}\br{\normlr{I_{k}}^{2}}}
	% 	+ \gamma^{2} \var^{\mathcal{F}_{k}}\pr{\frac{1}{b}\sum_{i=1}^{b} G_{k}^{i}} \eqsp.
	% \end{multline}
	% and using the assumption on $\gamma$ completes the proof.
\end{proof}

% !TEX root = main.tex

% \subsection{Continuous case}\label{sec:continuouscase}

\subsection{General supporting lemmas}\label{sub:supportinglemmas}

In this subsection, we consider the stochastic processes $(\Xavg_{k})_{k\in\N}$, $(\Xcontinuous_{k\gamma})_{k\in\N}$ defined in \eqref{eq:def:Xk} and \eqref{eq:def:xcontinuous}.
We derive several lemmas which allow us to derive a recursion on $\E\brn{\normn{\Xcontinuous_{k\gamma}-\Xavg_{k}}^{2}}$.\\
\begin{lemma}\label{lem:Ik:C2}
	Assume \Cref{ass:fi} holds.
	Then, for any $k\in\N$ and $\gamma>0$ we have
	\begin{equation}
		\E\br{\normlr{I_{k}}^{2}}
		\le \frac{d \gamma^3 L^{2}}{b}\pr{1 + \frac{\gamma L^{2}}{2\betaempty\conv} + \frac{\gamma^{2} L^{2}}{12\betaemptysquared}}\eqsp.
	\end{equation}
\end{lemma}
\begin{proof}
	Let $k$ be in $\N$. Using the Jensen inequality, we have
	\begin{align}
	\nonumber
	\E\br{\normlr{I_{k}}^{2}}
	&=\E\br{\normlr{\int_{k \gamma}^{(k+1) \gamma}\pr{\nabla\barpotential(\Xcontinuous_{s})-\nabla\barpotential(\Xcontinuous_{k \gamma})} \rmd s}^{2}} \\
	\nonumber
	&\le \gamma \int_{k\gamma}^{(k+1) \gamma}\E \br{\normlr{\nabla\barpotential(\Xcontinuous_{s})-\nabla\barpotential(\Xcontinuous_{k \gamma})}^{2}}\rmd s\\
	\label{eq:bound:Iq:first}
	&\le L^{2}\gamma \int_{k\gamma}^{(k+1) \gamma}\E \br{\normn{\Xcontinuous_{s}-\Xcontinuous_{k \gamma}}^{2}}\rmd s\eqsp.
	\end{align}
	Further, for any $s\in\R_+$, using \citet[Lemma 21]{durmus2019high} applied to $(\Xcontinuous_{bt})_{t\in\R_+}$ we obtain
	\begin{equation}\label{eq:bound:Iq:first:bis}
		\E^{\mathcal{F}_{k\gamma}}\br{\normn{\Xcontinuous_{s}-\Xcontinuous_{k \gamma}}^{2}}
		\le \frac{d(s-k\gamma)}{b}\pr{2 + (s-k\gamma)^{2}\frac{L^{2}}{3}}
		+ \frac{3}{2}(s-k\gamma)^{2}L^{2} \normn{\Xcontinuous_{k \gamma}-x_{\star}}^{2}\eqsp.
	\end{equation}
	Integrating the previous inequality on $[k\gamma,(k+1) \gamma]$, it implies
	\begin{equation}\label{eq:bound:Iq:intermediate}
		\int_{k\gamma}^{(k+1) \gamma}\E \br{\normn{\Xcontinuous_{s}-\Xcontinuous_{k \gamma}}^{2}}\rmd s
		\le \frac{\gamma^{2}}{b}\pr{d + \frac{bL^{2}\gamma}{2}\E\br{\normn{\Xcontinuous_{k \gamma}-x_{\star}}^{2}} + \frac{d L^{2}\gamma^{2}}{12\betaemptysquared}}\eqsp.
	\end{equation}
	Plugging \eqref{eq:bound:Iq:intermediate} in \eqref{eq:bound:Iq:first} gives
	\begin{equation}\label{eq:bound:Iq}
		\E\br{\normlr{I_{k}}^{2}}
		\le \frac{L^{2}\gamma^3}{b}\pr{d + \frac{bL^{2}\gamma}{2}\E\br{\normn{\Xcontinuous_{k\gamma}-x_{\star}}^{2}} + \frac{d L^{2}\gamma^{2}}{12}}\eqsp.
	\end{equation}
	Applying \citet[Proposition 1]{durmus2019high} to $(\Xcontinuous_{bt})_{t\in\R_+}$, we get
	\begin{equation}\label{eq:bound:contasympt}
		\E\br{\normn{\Xcontinuous_{k\gamma}-x_{\star}}^{2}}
		\le \betaempty\frac{d}{b\conv}\eqsp.
	\end{equation}
	Thus, combining \eqref{eq:bound:Iq} with \eqref{eq:bound:contasympt} completes the proof.
\end{proof}
%
% \begin{HX}\label{ass:fi:ctrois}
% 	Assume the function $\potential$ is three times continuously differentiable and there exists $\tilde{L}>0$ such that for all $x,y\in\R^d$,
% 	\[
% 		\normn{\nabla^{2}\potential(y) - \nabla^{2}\potential(x)} \le b\tilde{L} \normlr{y-x}\eqsp.
% 	\]
% \end{HX}
%
\begin{lemma}\label{lem:Ik:C3}
	Assume \Cref{ass:fi} and \Cref{ass:fi:ctrois} hold.
	Then, for any $k\in\N$ and $\gamma>0$ we have
	\begin{equation}
		% &\E\br{\normlr{I_{k}}^{2}}
		% \le {d\gamma^4}\br{2L^{2} + d\tilde{L}^{2} + {L^{4}}{(\betaempty\conv)^{-1}}}\eqsp,\\
		\E\br{\normlr{\E^{\mathcal{F}_{k}}\br{I_{k}}}^{2}}
		\le \frac{2\gamma^4 d}{3b} \pr{L^{3} + \frac{d\tilde{L}^{2}}{b}}\eqsp,
	\end{equation}
	where $I_{k}$ is defined in \eqref{eq:def:Ik}.
\end{lemma}
\begin{proof}
	Denote $\Delta$ the Laplacian defined, for all $x\in\R^d$, by $\Delta\potential(x)=\acn{\sum_{l=1}^{d}(\partial^{2} \potential_{j})(x)/\partial x_{l}^{2}}_{j=1}^{d}$, moreover let $k\in\N$ be a fixed integer and $\gamma>0$. Using the It\^o formula, we have for $s\in\ccint{k\gamma,(k+1)\gamma}$
	\begin{equation}\label{eq:lem:bound:Ik:base}
		\nabla\barpotential(\Xcontinuous_s) - \nabla\barpotential(\Xcontinuous_{k\gamma})
		= \int_{k\gamma}^{s}{\frac{1}{b}\Delta(\nabla\barpotential)(\Xcontinuous_u)
			- \nabla^{2}\barpotential(\Xcontinuous_u)\nabla\barpotential(\Xcontinuous_u)
				}\rmd u
		+ \sqrt{\frac{2}{b}}\int_{k\gamma}^s\nabla^{2}\barpotential(\Xcontinuous_u)\rmd B_u\eqsp.
	\end{equation}
	We will upper bound separately the three terms of the previous equality.
	First, the $L$-Lipschitz property of $\nabla\barpotential$ given by \Cref{ass:fi} implies for any $u\in\R_+$ that
	\begin{equation}
		\label{eq:lem:bound:Ik:1:1}
		\normlr{\nabla^{2}\barpotential(\Xcontinuous_u)\nabla\barpotential(\Xcontinuous_u)}
		\le L \normlr{\nabla\barpotential(\Xcontinuous_u) - \nabla\barpotential(x_{\star})}\eqsp.
		% &\le L^{4} \normlr{\Xcontinuous_u - x_{\star}}^{2}\eqsp.
	\end{equation}
	In addition, since for $u\in\R_+$, the random variable $\Xcontinuous_{u}$ is distributed according to the stationary distribution $\pi\propto \exp(-\potential)$, we know from \citet[Lemma 2]{dalalyan2017further} that
	\begin{equation}\label{eq:lem:bound:Ik:1:2}
		\E\br{\normlr{\nabla\barpotential(\Xcontinuous_u) - \nabla\barpotential(x_{\star})}^{2}}
		\le \frac{d L}{b}\eqsp. %\frac{d}{b\betaempty \conv}\eqsp.
	\end{equation}
	Therefore, we deduce from \eqref{eq:lem:bound:Ik:1:1} and \eqref{eq:lem:bound:Ik:1:2} the following bound
	\begin{equation}\label{eq:lem:bound:Ik:1:3}
		\E\br{\normlr{\nabla^{2}\barpotential(\Xcontinuous_u)\nabla\barpotential(\Xcontinuous_u)}^{2}}
		\le \frac{dL^{3}}{b\betaempty}\eqsp.
	\end{equation}
	Denote $(e_{i})_{i=1}^d$ the canonical basis of $\R^d$; using that U is three times continuously differentiable we can apply the Schwarz's theorem which combined with \Cref{ass:fi:ctrois}, immediately yield that
	\begin{align}
		\nonumber
		\normlr{\Delta(\nabla\barpotential)(x)}^{2}
		&= \sum_{i=1}^d\abs{\sum_{j=1}^d\partial_j^{2}\partial_{i} \barpotential(x)}^{2}
		\le d\sum_{i=1}^d\sum_{j=1}^d\abs{\partial_{i}\partial_j^{2} \barpotential(x)}^{2}\\
		\nonumber
		&= d\sum_{i=1}^d\lim_{\epsilon\to 0}\ac{\epsilon^{-2}\sum_{j=1}^d\abs{\partial_j^{2} \barpotential(x + \epsilon \cdot e_{i}) - \partial_j^{2} \barpotential(x)}^{2}}\\
		\label{eq:lem:bound:Ik:2:1}
		&\le d\sum_{i=1}^d\lim_{\epsilon\to 0}\ac{\epsilon^{-2}\pr{\tilde{L} \normn{(x+\epsilon \cdot e_{i}) - x}^{-1}}^{2}}
		\le \pr{d \tilde{L}}^{2}\eqsp.
	\end{align}
	Lastly, we upper bound the third term derived in \eqref{eq:lem:bound:Ik:base}.
	Since the potentials $\{\potential^{i}\}_{i\in[b]}$ are supposed $L$-smooth and $\barpotential$ twice continuously differentiable, for $s\in\ccint{k\gamma,(k+1)\gamma}$ we know that $\int_{k\gamma}^s\nabla^{2}\barpotential(\Xcontinuous_u)\rmd B_u$ is a $\mathcal{F}_s$-martingale.
	% , we have that $\sup_{x\in\R^d}\normn{\nabla^{2}\barpotential(x)}\le L$. Thus, the It\^o isometry implies, for $s\in\ccint{k\gamma,(k+1)\gamma}$ that
	% \begin{equation}\label{eq:lem:bound:Ik:3:1}
	% 	\E\br{\normlr{\int_{k\gamma}^s\nabla^{2}\barpotential(\Xcontinuous_u)\rmd B_u}^{2}}
	% 	\le L^{2}\E\br{\normlr{\int_{k\gamma}^s\rmd B_u}^{2}}
	% 	\le d L^{2}(s-k\gamma)\eqsp.
	% \end{equation}  \vincent{Reprendre ici.}
	Thus, for $k\ge 0$ we deduce that
	\begin{equation}\label{eq:lem:bound:Ik:3:1:2}
		\E^{\mathcal{F}_{k}}\br{\int_{k\gamma}^{(k+1)\gamma}\nabla^{2}\barpotential(\Xcontinuous_u)\,\rmd u} = 0\eqsp.
	\end{equation}
	Eventually, combining \eqref{eq:lem:bound:Ik:base}, \eqref{eq:lem:bound:Ik:1:3}, \eqref{eq:lem:bound:Ik:2:1} and \eqref{eq:lem:bound:Ik:3:1:2} with the Jensen and Young inequalities give
	\begin{align}%\label{eq:lem:bound:Ik:final}
		&\frac{1}{\gamma}\E\br{\normlr{\E^{\mathcal{F}_{k}}\br{I_{k}}}^{2}}
		= \frac{1}{\gamma}\E\br{\normlr{\int_{k\gamma}^{(k+1)\gamma}\E^{\mathcal{F}_{k}}\br{\nabla\barpotential(\Xcontinuous_{s})-\nabla\barpotential(\Xcontinuous_{k \gamma})} \rmd s}^{2}}\\
		&\le \int_{k\gamma}^{(k+1)\gamma} \E\br{\normlr{\E^{\mathcal{F}_{k}}\br{\nabla\barpotential(\Xcontinuous_s) - \nabla\barpotential(\Xcontinuous_{k\gamma})}}^{2}} \rmd s \\
		&= \int_{k\gamma}^{(k+1)\gamma} \E\br{\normlr{\E^{\mathcal{F}_{k}}\br{\int_{k\gamma}^{s}
			\frac{1}{b}\Delta(\nabla\barpotential)(\Xcontinuous_u)
			- \nabla^{2}\barpotential(\Xcontinuous_u)\nabla\barpotential(\Xcontinuous_u)
				\rmd u}}^{2}} \rmd s \\
		&\le 2\int_{k\gamma}^{(k+1)\gamma} (s-k\gamma)\int_{k\gamma}^{s}\E\br{\frac{1}{b^{2}}\normlr{\int_{k\gamma}^{s} \Delta(\nabla\barpotential)(\Xcontinuous_u)\rmd u}^{2}
			+ \normlr{\nabla^{2}\barpotential(\Xcontinuous_u)\nabla\barpotential(\Xcontinuous_u)\rmd u}^{2}} \rmd s\\
		&\le 2\int_{k\gamma}^{(k+1)\gamma} (s-k\gamma)^{2}\pr{\frac{dL^{3}}{b\betaempty} + \frac{(d\tilde{L})^{2}}{b^{2}}}\rmd s
		= \frac{2\gamma^3 d}{3b} \pr{L^{3} + \frac{d\tilde{L}^{2}}{b}}\eqsp.
	\end{align}
	% \begin{align}%\label{eq:lem:bound:Ik:final}
	% 	&\frac{1}{\gamma}\E\br{\normlr{I_{k}}^{2}}
	% 	= \frac{1}{\gamma}\E\br{\normlr{\int_{k\gamma}^{(k+1)\gamma}\left(\nabla\barpotential(\Xcontinuous_{s})-\nabla\barpotential(\Xcontinuous_{k \gamma})\right) \rmd s}^{2}}\\
	% 	&\le \int_{k\gamma}^{(k+1)\gamma} \E\br{\normlr{\nabla\barpotential(\Xcontinuous_s) - \nabla\barpotential(\Xcontinuous_{k\gamma})}^{2}} \rmd s \\
	% 	&\le 6\int_{k\gamma}^{(k+1)\gamma} \E\br{\normlr{\int_{k\gamma}^s\nabla^{2}\barpotential(\Xcontinuous_u)\rmd B_u}^{2}} \rmd s
	% 		+ 3\int_{k\gamma}^{(k+1)\gamma} \E\br{\normlr{\int_{k\gamma}^{s} \Delta(\nabla\barpotential)(\Xcontinuous_u)\rmd u}^{2}} \rmd s\\
	% 		&\qquad + 3\int_{k\gamma}^{(k+1)\gamma} \E\br{\normlr{\int_{k\gamma}^{s}\nabla^{2}\barpotential(\Xcontinuous_u)\nabla\barpotential(\Xcontinuous_u)\rmd u}^{2}} \rmd s\\
	% 	&\le \int_{k\gamma}^{(k+1)\gamma} (s-k\gamma) \pr{6d L^{2} + 3(s-k\gamma)\br{(d\tilde{L})^{2} + \frac{dL^{4}}{b\betaempty\conv}}}\rmd s
	% 	= {d\gamma^{2}}\br{3L^{2}
	% 		+ \gamma d\tilde{L}^{2}
	% 		+ \frac{\gamma L^{4}}{b\betaempty\conv}}\eqsp.
	% \end{align} \vincent{Le résultat m'a l'air faux.}
	Multiplying this last inequality by $\gamma>0$ proves the expected result.
\end{proof}
\begin{lemma}\label{lem:bound:Ik:unified}
	Assume \Cref{ass:fi} hold.
	  Then, for any $k\in\N$ and $\gamma\in\ocint{0,(3\conv)^{-1}}$ we have
	\begin{equation}
	  \frac{2}{\betaempty\gamma\conv}\E\br{\norm{\E^{\mathcal{F}_{k}}\br{I_{k}}}^{2}} + 3\E\br{\norm{I_{k}}^{2}}
	  \le
	  \begin{cases}
		\frac{3\gamma^2 d L^{2}}{b\conv} \pr{1 + \frac{19\gamma L^{2}}{36\conv}} \\
		\frac{\gamma^3d}{b\conv}\pr{5 L^3 + \frac{4d \tilde{L}^{2}}{3b}} \qquad\text{if \Cref{ass:fi:ctrois} holds and $\gamma\le L^{-1}$.}
	  \end{cases}
	\end{equation}
\end{lemma}

\begin{proof}
	Let $k$ be in $\N$ and $\gamma\in\ocint{0,(3\conv)^{-1}}$, using \Cref{lem:Ik:C2} we have
	\begin{equation}\label{eq:bound:4}
		\E\br{\normlr{I_{k}}^{2}}
		\le \frac{\gamma^3d L^{2}}{b}\pr{1 + \frac{\gamma L^{2}}{2\conv} + \frac{\gamma^{2} L^{2}}{12\betaemptysquared }}\eqsp.
		% \le 3d \gamma^3 L^{2}\eqsp.
	\end{equation}
	Therefore, we deduce
	\begin{equation}\label{eq:bound:6}
		{\frac{2}{\betaempty\gamma\conv}\E\br{\normlr{\E^{\mathcal{F}_{k}}\br{I_{k}}}^{2}} + 3\E\br{\normlr{I_{k}}^{2}}}
		\le \frac{3\gamma^2 d L^{2}}{b\conv} \pr{1 + \frac{\gamma L^{2}}{2\conv} + \frac{\gamma^{2} L^{2}}{12\betaemptysquared }}
		\eqsp.
	\end{equation}
	Moreover, if we additionally suppose the regularity of the Hessian of the potentials $(\potential^{i})_{i=1}^b$ as stated in \Cref{ass:fi:ctrois}, we sharpen the upper bound on $\E\brn{\normn{\E^{\mathcal{F}_{k}}\brn{I_{k}}}^{2}}$. Indeed, we show in \Cref{lem:Ik:C3} that
	\begin{equation}\label{eq:bound:5}
		\frac{2}{\betaempty\gamma\conv}\E\br{\normlr{\E^{\mathcal{F}_{k}}\br{I_{k}}}^{2}}
		\le \frac{4\gamma^3 d}{3b\conv} \pr{L^3 + \frac{d \tilde{L}^{2}}{b}}\eqsp.
	\end{equation}
	Hence, we deduce that
	\begin{align}\label{eq:bound:7}
		{\frac{2}{\betaempty\gamma\conv}\E\br{\normlr{\E^{\mathcal{F}_{k}}\br{I_{k}}}^{2}} + 3\E\br{\normlr{I_{k}}^{2}}}
		&\le \frac{3\gamma^3d L^{2}}{b}\pr{1 + \frac{\gamma L^{2}}{2\conv} + \frac{\gamma^{2} L^{2}}{12}}
		+ \frac{4\gamma^3 d}{3b\conv} \pr{L^3 + \frac{d \tilde{L}^{2}}{b}} \\
		&\le \frac{\gamma^3 d L^3}{b\conv}\pr{3 + \frac{4}{3} + \frac{19\gamma L}{36}} + \frac{4\gamma^3 d^2 \tilde{L}^{2}}{3b^2\conv}
		\eqsp.
	\end{align}
\end{proof}

\subsection{Derivation of the central theorem}\label{subsec:thm:contraction:vrsalad}

% We use the final result of this subsection to upper bound the variance reduction schemes {\algodeux} and {\algoquatre}.
% To this end, we introduce \Cref{ass:contraction:general:xk} and we prove in \Cref{lem:bound:diffXkYki:vrsalad}, \Cref{lem:bound:diffXkYki:vrsaladstar} that {\algodeux}, {\algoquatre} satisfy this assumption.
% \Cref{thm:bound:wassgeneral:svrgscheme} which implies

\begin{assumption}\label{ass:vk}
	There exist $\alpha_{v}\in\ooint{0,1}$ and $(v_{1}, v_{2})\in(\R_+)^{2}$ such that for any $k\in\N$, $V_{k}$ satisfies
	\begin{align}
		&\E\br{V_{k}}
		\le v_{1} \alpha_{v}^{k}
		+ v_{2}\eqsp,
	\end{align}
	where $V_{k}$ is defined in \eqref{eq:def:Vk}.
\end{assumption}

\begin{HX}\label{ass:contraction:general:xk}
	There exist $\qc\in(0,1)$ and $\alpha_{0},\alpha_{1}$, $\alpha_2$, $\alpha_3$, $\alpha_{4}\in\R_+$ satisfying $(1-\qc)\prn{1+\alpha_{0} + \sqrt{(\alpha_{0}-1)^{2} + 4\alpha_1}}<2$ such that for $k\ge 0$ the following inequality holds
	\begin{multline}
		% &(1-\qc)^{-1}\E\br{\normn{\Xcontinuous_{\gamma} - \Xavg_{1}}^{2}}
		% \le \alpha_{0} \E\br{\normn{\Xcontinuous_{0} - \Xavg_{0}}^{2}} + \alpha_2 \E\br{V_{0}} + \alpha_4 \eqsp,\\
		(1-\qc)^{-1} \E\br{\normn{\Xcontinuous_{(k+1)\gamma} - \Xavg_{k+1}}^{2}}
		\le \alpha_{0} \E\br{\normn{\Xcontinuous_{k\gamma} - \Xavg_{k}}^{2}}
		+ \alpha_1 \sum_{l=0}^{k-1}(1-\qc)^{k-l} \E\br{\normn{\Xcontinuous_{l\gamma} - \Xavg_{l}}^{2}}\\
		+ \alpha_2 \E\br{V_{k}} + \alpha_3 \sum_{l=0}^{k-1}(1-\qc)^{k-l} \E\br{V_l} + \alpha_4
		\eqsp.
	\end{multline}
\end{HX}

% \bgroup\color{cobalt}

With the notation introduced in \Cref{ass:contraction:general:xk}, consider
\begin{equation}\label{eq:def:delta:diffXkgeneral}
	\delta = \frac{-1-\alpha_{0} + \sqrt{(\alpha_{0}-1)^{2} + 4\alpha_1}}{2}\eqsp.
\end{equation}
At iteration $k\ge 0$, recall that $\nug_{k}$ denotes the distribution of the average parameter $\Xavg_{k}$ \eqref{eq:def:Xk}. The next result controls the Wasserstein distance between $\nug_{k}$ and the posterior distribution $\pi$.
\begin{theorem}\label{thm:bound:wasserstein:general:vrsalad}
	Assume \Cref{ass:contraction:general:xk} and \Cref{ass:vk} hold.
	Then, for any probability measure $\nug_{0}\in\mathcal{P}_{2}(\Rd)$, $k\in\N$, we have
	\begin{multline}\label{eq:bound:diffXkYki:diffXkgeneral:16}
		\wass^2\pr{\nug_{k},\pi}
		\le \pr{1+\alpha_{0}+\delta}^{k} \pr{1 - \qc}^{k} \wass^2\pr{\nug_{0},\pi}
		+ (1-\qc)v_1\pr{\alpha_2 + \frac{\alpha_3}{\alpha_{0}+\delta}} \frac{\alpha_{v}^k - \pr{1+\alpha_{0}+\delta}^{k} \pr{1 - \qc}^k}{\alpha_{v} - \pr{1+\alpha_{0}+\delta} \pr{1 - \qc}} \\
		+ \frac{1-\qc}{\qc - (1-\qc)(\alpha_{0}+\delta)} \br{\pr{\alpha_2 + \frac{\alpha_3}{\alpha_{0}+\delta}} v_{2} + \alpha_4} \eqsp.
	\end{multline}
\end{theorem}
\begin{proof}
	For any $n\in\N$, define
	\begin{equation}\label{eq:bound:diffXkYki:diffXkgeneral:0}
		\begin{aligned}
			&u_n = \pr{1 - \qc}^{-n} \E\br{\normn{\Xcontinuous_{n\gamma} - \Xavg_{n}}^{2}}\eqsp,&
			&S_n = \sum_{l=0}^n u_l\eqsp, \\
			&v_n = \pr{1 - \qc}^{-n} \pr{\alpha_2 \E\br{V_{n}} + \alpha_4} + \alpha_3 \sum_{l=0}^{n-1}(1-\qc)^{-l} \E\br{V_{l}} \eqsp.
		\end{aligned}
	\end{equation}
	With the above notations, \Cref{ass:contraction:general:xk} becomes
	\begin{equation}\label{eq:bound:diffXkYki:vrsalad:11}
		u_{k+1}
		\le \alpha_{0} u_{k} + \alpha_1 \sum_{l=0}^{k-1} u_l + v_{k} \eqsp,
	\end{equation}
	which can be rewritten as
	\begin{equation}\label{eq:bound:diffXkYki:diffXkgeneral:1}
		S_{k+1} - S_{k} \le \alpha_{0} \pr{S_{k} - S_{k-1}} + \alpha_1 S_{k-1} + v_{k}\eqsp.
	\end{equation}
	Since $\delta$ is solution of $\delta(1+\alpha_{0}+\delta) + \alpha_{0} - \alpha_{1} = 0$, adding $(1+\delta) S_{k}$ in \eqref{eq:bound:diffXkYki:diffXkgeneral:1} gives that
	\begin{align}\label{eq:bound:diffXkYki:diffXkgeneral:2}
		S_{k+1} + \delta S_{k}
		&\le \pr{1+\alpha_{0}+\delta} \pr{S_{k} - \frac{\alpha_{0} - \alpha_1}{1+\alpha_{0}+\delta} S_{k-1}} + v_{k}\\
		&= \pr{1+\alpha_{0}+\delta} \pr{S_{k} + \delta S_{k-1}} + v_{k}\eqsp.
	\end{align}
	Using the fact that $\alpha_{0}\le 1 + \sqrt{(\alpha_{0} - 1)^{2} + 4 \alpha_1}$, we obtain $2(1+\delta)=1-\alpha_{0} + \sqrt{(\alpha_{0}-1)^{2} + 4\alpha_1}\ge 0$.
	Hence $1+\delta>0$, which leads to the following upper bound
	\begin{equation}\label{eq:bound:diffXkYki:diffXkgeneral:3}
		u_{k+1} \le u_{k+1} + (1+\delta) \sum_{l=0}^{k} u_l = S_{k+1} + \delta S_{k}\eqsp.
	\end{equation}
	Thus, we obtain that
	\begin{equation}\label{eq:bound:diffXkYki:diffXkgeneral:4}
		u_{k} \le S_{k} + \delta S_{k-1} \le \pr{1+\alpha_{0}+\delta}^{k-1}\pr{u_{1} + (1+\delta) u_{k}} + \sum_{l=1}^{k-1} \pr{1+\alpha_{0}+\delta}^{k-l-1} v_l\eqsp.
	\end{equation}
	Plugging the definition~\eqref{eq:bound:diffXkYki:diffXkgeneral:0} of $u_{k}$ and $v_l$ inside the previous inequality, we get
	\begin{multline}\label{eq:bound:diffXkYki:diffXkgeneral:5}
		\pr{1 - \qc}^{-k} \E\br{\normn{\Xcontinuous_{k\gamma} - \Xavg_{k}}^{2}}
		\le \pr{1+\alpha_{0}+\delta}^{k-1}\pr{\pr{1 - \qc}^{-1} \E\br{\normn{\Xcontinuous_{\gamma} - \Xavg_{1}}^{2}} + (1+\delta) \E\br{\normn{\Xcontinuous_{0} - \Xavg_{0}}^{2}}} \\
		+ \sum_{l=1}^{k-1} \pr{1+\alpha_{0}+\delta}^{k-l-1} \br{\pr{1 - \qc}^{-l} \pr{\alpha_2 \E\br{V_{l}} + \alpha_4} + \alpha_3 \sum_{j=0}^{l-1}(1-\qc)^{-j} \E\br{V_{j}}} \eqsp.
	\end{multline}
	Moreover, using \Cref{ass:contraction:general:xk} we obtain that
	\begin{equation}\label{eq:bound:diffXkYki:diffXkgeneral:7}
		\E\br{\normn{\Xcontinuous_{\gamma} - \Xavg_{1}}^{2}}
		\le (1-\qc) \alpha_{0} \E\br{\normn{\Xcontinuous_{0} - \Xavg_{0}}^{2}}
		+ (1-\qc)\alpha_2 \E\br{V_{0}}
		+ \alpha_4 \eqsp,
	\end{equation}
	combining \eqref{eq:bound:diffXkYki:diffXkgeneral:5} with \eqref{eq:bound:diffXkYki:diffXkgeneral:7} yield
	\begin{multline}\label{eq:bound:diffXkYki:diffXkgeneral:8}
		\E\br{\normn{\Xcontinuous_{k\gamma} - \Xavg_{k}}^{2}}
		\le \pr{1+\alpha_{0}+\delta}^{k} \pr{1 - \qc}^{k} \E\br{\normn{\Xcontinuous_{0} - \Xavg_{0}}^{2}}
		+ \alpha_2 \sum_{l=0}^{k-1} \pr{1+\alpha_{0}+\delta}^{k-l-1} \pr{1 - \qc}^{k-l} \E\br{V_{l}} \\
		+ \alpha_3 \sum_{j=0}^{k-2} (1-\qc)^{k-j} \E\br{V_{j}} \sum_{l=j+1}^{k-1} \pr{1+\alpha_{0}+\delta}^{k-l-1}
		+ (1-\qc) \alpha_4 \sum_{l=0}^{k-1} \pr{1+\alpha_{0}+\delta}^{l} \pr{1 - \qc}^{l} \eqsp.
	\end{multline}
	Consider the function $f:a\in\R\to\R$ defined by $f(a) = a\prn{1+\alpha_{0}+a} + \alpha_{0} - \alpha_1$. Using the definition~\eqref{eq:def:delta:diffXkgeneral} of $\delta$ combined with the increasing property of $f$, we deduce from $f(\delta) = 0 > f(-\alpha_{0}) = -\alpha_1$ that $\delta>-\alpha_{0}$, and thus we get $1+\alpha_{0}+\delta>1$ which implies that
	\begin{align}\label{eq:bound:diffXkYki:diffXkgeneral:9}
		\sum_{l=j+1}^{k-1} \pr{1+\alpha_{0}+\delta}^{k-l-1}
		&\le \sum_{l=0}^{k-j-2} \pr{1+\alpha_{0}+\delta}^{k-j-l-2} \\
		&\le \frac{(1+\alpha_{0}+\delta)^{k-j-1}}{\alpha_{0}+\delta}\eqsp.
	\end{align}
	Therefore, plugging \eqref{eq:bound:diffXkYki:diffXkgeneral:9} in \eqref{eq:bound:diffXkYki:diffXkgeneral:8} gives
	\begin{equation}\label{eq:bound:diffXkYki:diffXkgeneral:10}
		\sum_{j=0}^{k-2} (1-\qc)^{k-j} \E\br{V_{j}} \sum_{l=j+1}^{k-1} \pr{1+\alpha_{0}+\delta}^{k-l-1}
		\le \sum_{l=0}^{k-2} \frac{\pr{1 - \qc}^{k-l} (1+\alpha_{0}+\delta)^{k-l-1}}{\alpha_{0}+\delta} \E\br{V_{l}}
		\eqsp.
	\end{equation}
	% In addition, using for $x\ge 0$, that $\sqrt{1+x} - 1 \le \sqrt{x}$ combined with \Cref{ass:contraction:general:xk} to ensure that $\alpha_{0}>1$, we obtain
	% \begin{align}\label{eq:bound:diffXkYki:diffXkgeneral:101}
	% 	2\delta = (\alpha_{0} - 1)\sqrt{1 + 4(\alpha_{0}-1)^{-2}\alpha_1} - 1 - \alpha_{0}
	% 	\le 2\sqrt{\alpha_1} - 2\eqsp.
	% \end{align} \vincent{Il y a un soucis, trop contraignant de montrer $(1-\qc)(\alpha_{0} + \sqrt{\alpha_1})<1$.}
	% Since $(1-\qc)(1+\alpha_{0}+\delta)\le (1-\qc)(\alpha_{0} + \sqrt{\alpha_1})<1$ by \Cref{ass:contraction:general:xk}, we have
	In addition, since \Cref{ass:contraction:general:xk} ensures that $(1-\qc)(1+\alpha_{0}+\delta)<1$, we have
	\begin{equation}\label{eq:bound:diffXkYki:diffXkgeneral:102}
		\sum_{l=0}^{k-1} \pr{1+\alpha_{0}+\delta}^{l} \pr{1 - \qc}^{l}
		\le \frac{1}{\qc - (1-\qc)(\alpha_{0}+\delta)}\eqsp.
	\end{equation}
	The last inequality combined with \eqref{eq:bound:diffXkYki:diffXkgeneral:8} and \eqref{eq:bound:diffXkYki:diffXkgeneral:10} show that
	\begin{multline}\label{eq:bound:diffXkYki:diffXkgeneral:11}
		\E\br{\normn{\Xcontinuous_{k\gamma} - \Xavg_{k}}^{2}}
		\le \pr{1+\alpha_{0}+\delta}^{k} \pr{1 - \qc}^{k} \E\br{\normn{\Xcontinuous_{0} - \Xavg_{0}}^{2}} \\
		+ \pr{\alpha_2 + \frac{\alpha_3}{\alpha_{0}+\delta}} \sum_{l=0}^{k-1} \pr{1+\alpha_{0}+\delta}^{k-l-1} \pr{1 - \qc}^{k-l} \E\br{V_{l}}
		+ \frac{(1-\qc) \alpha_4}{\qc - (1-\qc)(\alpha_{0}+\delta)} \eqsp.
	\end{multline}
	Further, since we assume \Cref{ass:vk}, we have
	\begin{multline}\label{eq:bound:diffXkYki:diffXkgeneral:12}
		\sum_{l=0}^{k-1} \pr{1+\alpha_{0}+\delta}^{k-l-1} \pr{1 - \qc}^{k-l} \E\br{V_{l}}
		\le v_1 \sum_{l=0}^{k-1} \pr{1+\alpha_{0}+\delta}^{k-l-1} \pr{1 - \qc}^{k-l} \alpha_{v}^l \\
		+ v_{2} \sum_{l=0}^{k-1} \pr{1+\alpha_{0}+\delta}^{k-l-1} \pr{1 - \qc}^{k-l}\eqsp.
	\end{multline}
	A calculation gives that
	\begin{equation}\label{eq:bound:diffXkYki:diffXkgeneral:13}
		\sum_{l=0}^{k-1} \pr{1+\alpha_{0}+\delta}^{k-l-1} \pr{1 - \qc}^{k-l} \alpha_{v}^l
		\le (1-\qc)\frac{\alpha_{v}^k - \pr{1+\alpha_{0}+\delta}^{k} \pr{1 - \qc}^k}{\alpha_{v} - \pr{1+\alpha_{0}+\delta} \pr{1 - \qc}}\eqsp
	\end{equation}
	and combining \eqref{eq:bound:diffXkYki:diffXkgeneral:102}, \eqref{eq:bound:diffXkYki:diffXkgeneral:12} with \eqref{eq:bound:diffXkYki:diffXkgeneral:13}, we find that
	\begin{equation}\label{eq:bound:diffXkYki:diffXkgeneral:14}
		\sum_{l=0}^{k-1} \pr{1+\alpha_{0}+\delta}^{k-l-1} \pr{1 - \qc}^{k-l} \E\br{V_{l}}
		\le (1-\qc)v_1\frac{\alpha_{v}^k - \pr{1+\alpha_{0}+\delta}^{k} \pr{1 - \qc}^k}{\alpha_{v} - \pr{1+\alpha_{0}+\delta} \pr{1 - \qc}}
		+ \frac{(1-\qc)v_{2}}{\qc - (1-\qc)(\alpha_{0}+\delta)}\eqsp.
	\end{equation}
	Therefore, plugging \eqref{eq:bound:diffXkYki:diffXkgeneral:14} inside \eqref{eq:bound:diffXkYki:diffXkgeneral:11} shows that
	\begin{multline}\label{eq:bound:diffXkYki:diffXkgeneral:15}
		\E\br{\normn{\Xcontinuous_{k\gamma} - \Xavg_{k}}^{2}}
		\le \pr{1+\alpha_{0}+\delta}^{k} \pr{1 - \qc}^{k} \E\br{\normn{\Xcontinuous_{0} - \Xavg_{0}}^{2}} \\
		+ (1-\qc)v_1\pr{\alpha_2 + \frac{\alpha_3}{\alpha_{0}+\delta}} \frac{\alpha_{v}^k - \pr{1+\alpha_{0}+\delta}^{k} \pr{1 - \qc}^k}{\alpha_{v} - \pr{1+\alpha_{0}+\delta} \pr{1 - \qc}} \\
		+ \frac{1-\qc}{\qc - (1-\qc)(\alpha_{0}+\delta)} \br{\pr{\alpha_2 + \frac{\alpha_3}{\alpha_{0}+\delta}} v_{2} + \alpha_4} \eqsp.
	\end{multline}	
	Eventually, since the Wasserstein distance $\wass(\pi,\nug_{k})$ is the infimum over all couplings, we obtain that $\wass^{2}(\pi,\nug_{k})\le \E\brn{\normn{\Xcontinuous_{k\gamma} - \Xavg_{k}}^{2}}$.
	Moreover, it follows from the strongly convex assumption~\Cref{ass:fi} that $\pi\in\mathcal{P}_{2}(\R^d)$. Thus, we can apply \citet[Theorem 4.1]{villani2009optimal} to prove the existence of an optimal coupling $\zeta$ such that taking $(\Xcontinuous_{0}, \Xavg_{0})$ distributed according to $\zeta$ implies that $\E\brn{\normn{\Xcontinuous_{0} - \Xavg_{0}}^{2}}^{1/2}=\wass(\pi,\nug_{0})$. Substituting these results into \eqref{eq:bound:diffXkYki:diffXkgeneral:15} completes the proof.
\end{proof}

\subsection{Upper bound on $V_{k}$}\label{sec:Vk}

% !TEX root = main.tex

The goal of this subsection is to prove the upper bound derived in \Cref{lem:bound:Vk:expec} for $(\E\br{V_{k}})_{k\in\N}$ to ensure that \Cref{ass:vk} holds. Recall that for $k\ge 0$, $V_k$ is defined in \eqref{eq:def:Vk}, $\dist_{k}$ in \eqref{eq:def:dk}, $G_k^{i}$ in \eqref{eq:def:Gik} and we introduce $\bar{G}_{k}^{i}=\E^{\mathcal{F}_k}\brn{G_k^{i}}$.
To prove the central lemma of this subsection, we also consider the assumptions \Cref{ass:dk:combination} and \Cref{ass:gradsto:g} given below.

\begin{HX}\label{ass:dk:combination}
	There exist $A_{d}, A_{\sigma} \in \ooint{0,1}, B_{d}, B_{\sigma}, C_{d}, C_{\sigma}, D_{d}, D_{\sigma} \in \R_{+}$, such that for any $k\in\N$, we have
	\begin{align}
		\label{eq:bound:ass:d}
		&\E\br{\dist_{k+1}^{2}}
		\le \pr{1-A_{d}}\E\br{\dist_{k}^{2}}
		+ B_{d}\E\br{\sigma_{k}^{2}}
		+ C_{d}\E\br{V_{k}}
		+ D_{d}\eqsp,\\
		\label{eq:bound:ass:sigma}
		&\E\br{\sigma_{k+1}^{2}}
		\le \pr{1-A_{\sigma}}\E\br{\sigma_{k}^{2}}
		+ B_{\sigma}\E\br{\dist_{k}^{2}}
		+ C_{\sigma}\E\br{V_{k}}
		+ D_{\sigma}\eqsp.
	\end{align}
\end{HX}

\begin{HX}\label{ass:gradsto:g}
	There exist $A,\bar{A},B,\bar{B},C,\bar{C},D,\bar{D}\ge 0$ such that for any $i\in[b], k\in\N$, we have
	\begin{align}
		\label{eq:bound:lemgradsto:1}
		&\frac{1}{b}\sum_{i=1}^{b}\E\br{\normlr{\bar{G}_{k}^{i}}^{2}}
		\le \bar{A}\E\br{V_{k}} + \bar{B}\E\br{\dist_{k}^{2}} + \bar{C}\E\br{\sigma_{k}^{2}} + \bar{D}\eqsp,\\
		\label{eq:bound:lemgradsto:2}
		&\frac{1}{b}\sum_{i=1}^{b}\E\br{\normlr{G_{k}^{i}-\bar{G}_{k}^{i}}^{2}}
		\le {A}\E\br{V_{k}} + {B}\E\br{\dist_{k}^{2}} + {C}\E\br{\sigma_{k}^{2}} + {D}\eqsp.
	\end{align}
        % where $G_{k}^{i}$ is given in \eqref{eq:def:Gik} and $\bar{G}_{k}^{i} = \nabla U^{i} (\Xlocal_{k}^{i})$. 
\end{HX}
%
% \begin{remark}
% 	Under \Cref{ass:gradsto:g}, the \Cref{lem:bound:gexplicit} provides constants $A,\bar{A},B,\bar{B},C,\bar{C},D,\bar{D}$ such that \eqref{eq:bound:lemgradsto:1} and \eqref{eq:bound:lemgradsto:2} hold.
% \end{remark}
%
% \begin{assumption}\label{ass:dk:combination}
% 	Assume there exists $A_{d}, A_{\sigma} \in \ooint{0,1}, B_{d},B_{\sigma},C_{d},C_{\sigma}, D_{d},D_{\sigma}\in \R_{+}$ such that for any $k\ge 1$,
% 	\begin{align}
% 		\label{eq:bound:ass:d}
% 		&\E\br{\dist_{k}^{2}}
% 		\le \pr{1-A_{d}}\E\br{\dist_{k-1}^{2}}
% 		+ C_{d}\E\br{V_{k-1}}
% 		+ D_{d}\eqsp,\\
% 		\label{eq:bound:ass:sigma}
% 		&\E\br{\sigma_{k}^{2}}
% 		\le \pr{1-A_{\sigma}}\E\br{\sigma_{k-1}^{2}}
% 		+ B_{\sigma}\E\br{\dist_{k-1}^{2}}
% 		+ C_{\sigma}\E\br{V_{k-1}}
% 		+ D_{\sigma}\eqsp,
% 	\end{align}
% 	where $\dist_{k}, \sigma_{k}$ are defined in \eqref{eq:def:dk} and \eqref{eq:def:sigmak:saladstar}, respectively.
% \end{assumption}
%
% \begin{remark}
% 	In \Cref{lem:bound:dk} and \Cref{lem:bound:sigmak} we derive some constants $A_{d}$, $A_{\sigma}$, $B_{d}$, $B_{\sigma}$, $C_{d}$, $C_{\sigma}$, $D_{d}$, $D_{\sigma}$ such that \eqref{eq:bound:ass:d} and \eqref{eq:bound:ass:sigma} are satisfied for all $k\ge 1$.
% \end{remark}
%
With the notation considered in \Cref{ass:dk:combination} and \Cref{ass:gradsto:g}, for any $\gamma>0$ we also introduce the following quantities:
\begin{equation}\label{eq:def:cte}
	\begin{aligned}
		% &\cte = (1-\pc)\gamma^{2}\br{\pr{B+\frac{2+\pc}{\pc}\bar{B}} + \frac{4B_{\sigma}}{\pc-4A_{d}}\pr{C+\frac{2+\pc}{\pc}\bar{C}}}\eqsp,\\
		% &\ctesigma = \frac{4(1-\pc)\gamma^{2}\pr{1-A_{d}}}{\pc - 4 A_{d}}\pr{C+\frac{2+\pc}{\pc}\bar{C}} + \frac{2B_{d}\cte}{A_{d}(A_{\sigma}-A_{d})}\eqsp,\\
		% &\cterate = \frac{4(1-\pc)\gamma^{2}C_{\sigma}}{\pc-4A_d}\pr{C+\frac{2+\pc}{\pc}\bar{C}} + \frac{2\cte}{A_{d}\pr{1-A_{d}}}\pr{C_{d} + \frac{B_{d}C_{\sigma}}{A_{\sigma} - A_{d}}}\eqsp,\\
		% &\ctedelta = \frac{4(1-\pc)\gamma^{2}}{\betaemptysquared \pc}\pr{D+\frac{2+\pc}{\pc}\bar{D}}
		% + \frac{4(1-\pc)\gamma^{2}D_{\sigma}}{\betaemptysquared A_{\sigma}(\pc-4A_{d})} \pr{C+\frac{2+\pc}{\pc}\bar{C}}
		% \\
		% &\qquad+ \frac{2\cte}{A_d^{2}}\pr{D_{d} + \frac{B_{d} D_{\sigma}}{A_{\sigma}}}
		% + \frac{8\pr{1-\tau}\pr{b-1}\gamma d}{b \pc}\eqsp.
		&\cte = \frac{4(1-\pc)\gamma^{2}}{\pc-4A_d}\br{B+\frac{2+\pc}{\pc}\bar{B}+ \frac{B_{\sigma}}{A_\sigma - A_d} \pr{C+\frac{2+\pc}{\pc}\bar{C}}}\eqsp, \\
		&\cterate = \frac{9\gamma^{2}\pr{1-\pc}C_{\sigma}}{\pc - 4A_d} \pr{C+\frac{2+\pc}{\pc}\bar{C}} + 3\cte\pr{C_{d} + \frac{B_{d}C_{\sigma}}{A_{\sigma} - A_{d}}} \eqsp,\\
		&\begin{aligned}
			\ctesigma = \frac{4(1-\pc)\gamma^{2}}{\pc - 4 A_{d}}\pr{C+\frac{2+\pc}{\pc}\bar{C}} + \cte B_{d}\pr{2 + \frac{3}{A_{\sigma}-A_{d}}}\eqsp,&
			&\ctedzero = 7 \cte\eqsp,&
			&\ctev = 1 + 2 \cte C_d \eqsp,
		\end{aligned}\\
		% \ctedone = \frac{\cte B_d B_\sigma}{(1-A_d)^{2}(A_\sigma - A_d)}\eqsp,\\
		&\ctedelta = \frac{4(1-\pc)\gamma^{2}D_{\sigma}}{\betaemptysquared A_{\sigma}(\pc-4A_{d})} \pr{C+\frac{2+\pc}{\pc}\bar{C}}
		+ \frac{4(1-\pc)\gamma^{2}}{\betaemptysquared \pc}\pr{D+\frac{2+\pc}{\pc}\bar{D}}
		\\
		&\qquad+ \frac{\cte}{A_d} \pr{1+ \frac{2B_d B_\sigma}{A_d(A_\sigma - A_d)}} \pr{D_{d} + \frac{B_{d} D_{\sigma}}{A_{\sigma}}}
		+ \frac{8\pr{1-\tau}\pr{b-1}\gamma d}{b \pc}\eqsp.
	\end{aligned}
\end{equation}
If $A_d\le A_\sigma/2$ and $A_d A_\sigma\ge 8B_dB_\sigma$, we also introduce a convergence rate (proved later in \Cref{lem:bound:Vk}) defined by
\begin{equation}\label{eq:def:newrate}
	\newrate
	= A_d - \frac{2(A_\sigma - A_d)^{-1}B_dB_\sigma}{1 + \sqrt{1 + 4 (1-A_d)^{-1}(A_\sigma - A_d)^{-1} B_d B_\sigma}} \eqsp.
\end{equation}
\begin{lemma}\label{lem:bound:alpha}
	Assume \Cref{ass:dk:combination} and also that $A_d\le A_\sigma/2$, $A_d A_\sigma\ge 8B_dB_\sigma$ hold. Then, we have
	\[
		A_d/2<\newrate\le A_d\eqsp.
	\]
\end{lemma}
\begin{proof}
	First, introduce $\delta_{\alpha}\in\R_+$ the unique non-negative solution of
	\[
		\delta_{\alpha}^{2} + \delta_{\alpha} = \frac{B_dB_\sigma}{(1-A_d)(A_\sigma - A_d)}\eqsp.
	\]
	Since we suppose $A_d\le A_\sigma/2$, thus we have $A_d\le 1/2$ which implies that $(1-A_d)\prn{A_d^{2}/4+A_d/2}\ge A_d/4$. In addition, using $A_d A_\sigma\ge 8B_dB_\sigma$, we get that
	\[
		\pr{1-A_d}\pr{\frac{A_d^{2}}{4} + \frac{A_d}{2}} \ge \frac{A_d}{4} \ge \frac{2 B_dB_\sigma}{A_\sigma} \ge (1-A_d)\pr{\delta_{\alpha}^{2} + \delta_{\alpha}}\eqsp.
	\]
	Hence, the increasing property of the function $x\in\R_+\mapsto x^{2} + x$ combined with the fact that $\delta_{\alpha}\ge 0$ prove that $A_d\ge 2\delta_{\alpha}$.
	Moreover, a calculation shows that $\newrate$ satisfies $\newrate=1-\prn{1-A_d}\prn{1+\delta_{\alpha}}$. Thus, using $0\le 2\delta_{\alpha}\le A_d$ implies that $\newrate\in\ocint{A_d/2,A_d}$.
\end{proof}
The random variable $V_{k}$ given in \eqref{eq:def:Vk} measures the averaged distance between the global parameter $\Xavg_{k}$ and the local ones $(\Xlocal_{k}^{i})_{i\in[b]}$. The first lines of the proof of the next lemma are based on \citet[Lemma E.3]{gorbunov2021local}, however their purpose was to upper bound $\sum_{l}w_{l}\E V_{l}$ for some weights $w_{l}>0$, while we prefer to control $\E V_{k}$ to combine this bound with that of \Cref{lem:contraction:xkbis}. Moreover, the assumptions considered in this work are different, so the proof requires the development of other techniques
\begin{lemma}\label{lem:bound:Vk}
	Assume \Cref{ass:dk:combination}, \Cref{ass:gradsto:g} hold with $A_{d}<\min(A_{\sigma}/2, \pc/4), A_d A_\sigma\ge 8B_dB_\sigma$ and consider $\gamma\le \betaempty{\pc^{1/2}}{(2-2\pc)^{-1/2}\brn{A+(1+2/\pc)\bar{A}}^{-1/2}}$.
	Then, for any $k\in\N$, we have
	% \begin{multline}
	% 	\E\br{V_{k}}
	% 	\le \pr{1-\frac{\pc}{4}}^{k}\E\br{V_{0}}
	% 	+ \ctesigma \pr{1-\frac{A_{d}}{2}}^{k-1} \E\br{\sigma_{0}^{2}}
	% 	+ \frac{4\cte}{A_{d}} \pr{1-\frac{A_{d}}{2}}^{k-1} \E\br{\dist_{0}^{2}}\\
	% 	+ \frac{\cte\pr{1-\nofrac{A_{d}}{2}}^{k-1}}{1-A_{d}}\E\br{\dist_{1}^{2}}
	% 	+ \cterate\sum_{i=0}^{k-2}\pr{1-\frac{A_{d}}{2}}^{k-i-1}\E\br{V_{i}}
	% 	+ \ctedelta
	% 	\eqsp.
	% \end{multline}
	\begin{equation}
		\E\br{V_{k}}
		\le
		% \pr{1-\frac{\pc}{4}}^{k}\E\br{V_{0}}
		\pr{1-\newrate}^{k}
			\pr{
				\ctev \E\br{V_0}
				+ \ctedzero \E\br{\dist_{0}^{2}}
				+ \ctesigma \E\br{\sigma_{0}^{2}}
				+ 2 D_d
			}
		+ \cterate
			\sum_{i=0}^{k-2}\pr{1-\newrate}^{k-i-1} \E\br{V_{i}}
		+ \ctedelta
		\eqsp,
	\end{equation}
	% \begin{align}
	% 	\E\br{V_{k}}
	% 	&\le \pr{1-A_{d}}^{k}\E\br{V_{0}}\\
	% 	&+ \frac{4(1-\pc)\gamma^{2}}{\pc -4A_{d}}\br{B+\frac{2+\pc}{\pc}\bar{B} + \pr{C+\frac{2+\pc}{\pc}\bar{C}}\frac{B_{\sigma}}{A_{\sigma}}}\pr{1-A_{d}}^{k} \E\br{\dist_{0}^{2}}\\
	% 	&+ \frac{4(1-\pc)\gamma^{2}}{\pc -4A_{d}}\pr{C+\frac{2+\pc}{\pc}\bar{C}}\pr{1-A_{d}}^{k} \E\br{\sigma_{0}^{2}}\\
	% 	&+ \frac{4(1-\pc)\gamma^{2}}{\pc -4A_{d}}\br{B+\frac{2+\pc}{\pc}\bar{B} + \pr{C+\frac{2+\pc}{\pc}\bar{C}}\pr{C_{\sigma}+\frac{B_{\sigma} C_{d}}{A_{\sigma}-A_{d}}}}\sum_{l=0}^{k-2}\pr{1-A_{d}}^{k-l-1}\E\br{V_{l}} \\
	% 	&+ \frac{4(1-\pc)\gamma^{2}}{\pc}\br{\pr{B+\frac{2+\pc}{\pc}\bar{B}}\frac{D_{d}}{A_{d}}+\pr{C+\frac{2+\pc}{\pc}\bar{C}}\pr{\frac{D_{\sigma}}{A_{\sigma}} + \frac{B_{\sigma} D_{d}}{A_{d} A_{\sigma}}} + D+\frac{2+\pc}{\pc}\bar{D}}\\
	% 	&+ 2\pr{1-\tau}\pr{1-1 / b}\gamma d\eqsp,
	% \end{align}
	where $V_{k}$ is defined in \eqref{eq:def:Vk}.
\end{lemma}
\begin{proof}
	Let $k\in\N^{\star}$, using for $i\in[b]$ the definitions \eqref{eq:def:Xki}, \eqref{eq:def:Xk} of $\Xlocal_{k}^{i}$ and $\Xavg_{k}$
	\begin{align}
		&\Xlocal_{k+1}^{i} = \Xlocal_{k}^{i} - \gamma G_{k}^{i} + \sqrt{2\gamma}\pr{\sqrt{\tau/b}\,\tilde{Z}_{k+1} + \sqrt{1-\tau}\,\tilde{Z}_{k+1}^{i}}\eqsp,\\
		&\Xavg_{k+1} = \Xavg_{k} -\frac{\gamma}{\betaempty b}\sum_{j=1}^{b} G_{k}^{i} + \sqrt{\frac{2\gamma\tau}{b}} \tilde{Z}_{k+1} + \frac{\sqrt{2(1-\tau)\gamma}}{b}\sum_{i=1}^{b} Z_{k+1}^{i}\eqsp.
	\end{align}
	\paragraph*{First upper bound on $\E\br{V_{k}}$.}
	Substracting the two above equations combined with the Jensen inequality give
	\begin{align}
		&\E\br{V_{k+1}}
		= \frac{1}{b}\sum_{i=1}^{b}\E\br{\normlr{\Xlocal_{k+1}^{i}-\Xavg_{k+1}}^{2}}\\
		&= \frac{1-\pc}{b}\sum_{i=1}^{b}\E\br{\normlr{(\Xlocal_{k}^{i}-\Xavg_{k}) - \gamma(G_{k}^{i}-G^{k})
			+ \sqrt{2(1-\tau)\gamma} Z_{k+1}^{i} - \frac{\sqrt{2(1-\tau)\gamma}}{b}\sum_{j=1}^{b} Z_{k+1}^{j}
			}^{2}}\\
		&= \frac{1-\pc}{b}\sum_{i=1}^{b}\E\br{\normlr{(\Xlocal_{k}^{i}-\Xavg_{k}) - \gamma(\bar{G}_{k}^{i}-\bar{G}^{k})}^{2}}
			+ \frac{(1-\pc)\gamma^{2}}{\betaemptysquared b}\sum_{i=1}^{b}\E\br{\normlr{(G_{k}^{i}-\bar{G}_{k}^{i}) - (G^{k}-\bar{G}^{k})}^{2}}\\
			&\qquad+ 2(1-\tau)\gamma\E\br{\normlr{Z_{k+1}^{i} - \frac{1}{b}\sum_{j=1}^{b} Z_{k+1}^{j}}^{2}}
\end{align}
Hence, we get 
\begin{align}
		&\E\br{V_{k+1}} \le \frac{1-\pc}{b}\sum_{i=1}^{b}\E\br{\normlr{(\Xlocal_{k}^{i}-\Xavg_{k}) - \gamma(\bar{G}_{k}^{i}-\bar{G}^{k})}^{2}}
			+ \frac{(1-\pc)\gamma^{2}}{\betaemptysquared b}\sum_{i=1}^{b}\E\br{\normlr{G_{k}^{i}-\bar{G}_{k}^{i}}^{2}}\\
			&\qquad+ 2\pr{1-\tau}\pr{1-1 / b}\gamma d\\
		&\le \frac{(1-\pc)(1+\pc/2)}{b}\sum_{i=1}^{b}\E\br{\normlr{\Xlocal_{k}^{i}-\Xavg_{k}}^{2}}
			+ \frac{(1-\pc)\gamma^{2}}{\betaemptysquared b}\sum_{i=1}^{b}\E\br{\normlr{G_{k}^{i}-\bar{G}_{k}^{i}}^{2}}\\
			&\qquad+ \frac{(1-\pc)(1+2/\pc)\gamma^{2}}{\betaemptysquared b}\sum_{i=1}^{b}\E\br{\normlr{\bar{G}_{k}^{i}-\bar{G}^{k}}^{2}}
				+ 2\pr{1-\tau}\pr{1-1 / b}\gamma d \eqsp.
\end{align}
Using $(1-\pc)(1+\pc/2)\le 1-\pc/2$, we finally obtain
\begin{align}
		&\E\br{V_{k+1}} \le \pr{1-\pc/2}\E\br{V_{k}}
			+ \frac{(1-\pc)(2+\pc)\gamma^{2}}{\betaemptysquared \pc b}\sum_{i=1}^{b}\E\br{\normlr{\bar{G}_{k}^{i}}^{2}}\\
			&\qquad+ \frac{(1-\pc)\gamma^{2}}{\betaemptysquared b}\sum_{i=1}^{b}\E\br{\normlr{G_{k}^{i}-\bar{G}_{k}^{i}}^{2}}
				+ 2\pr{1-\tau}\pr{1-1 / b}\gamma d\eqsp.
	\end{align}
	Combining the last inequality with \Cref{ass:gradsto:g}, it shows
	\begin{multline}
		\E\br{V_{k+1}}
		\le \pr{1-\frac{\pc}{2}+(1-\pc)\gamma^{2}\br{A+\frac{2+\pc}{\pc}\bar{A}}}\E\br{V_{k}}
		+ (1-\pc)\gamma^{2}\pr{D+\frac{2+\pc}{\pc}\bar{D}}\\
		+ (1-\pc)\gamma^{2}\pr{B+\frac{2+\pc}{\pc}\bar{B}}\E\br{\dist_{k}^{2}}
		+ (1-\pc)\gamma^{2}\pr{C+\frac{2+\pc}{\pc}\bar{C}}\E\br{\sigma_{k}^{2}}
		+ 2\pr{1-\tau}\pr{1-1 / b}\gamma d\eqsp.
	\end{multline}
	Since $\gamma\le \frac{\betaempty\pc^{1/2}}{2(1-\pc)^{1/2}\br{A+(1+2/\pc)\bar{A}}^{1/2}}$, the above inequality implies that
	\begin{multline}
		\E\br{V_{k+1}}
		\le \pr{1-\frac{\pc}{4}}\E\br{V_{k}}
		+ (1-\pc)\gamma^{2}\pr{D+\frac{2+\pc}{\pc}\bar{D}}
		+ 2\pr{1-\tau}\pr{1-1 / b}\gamma d\\
		+ (1-\pc)\gamma^{2}\pr{B+\frac{2+\pc}{\pc}\bar{B}}\E\br{\dist_{k}^{2}}
		+ (1-\pc)\gamma^{2}\pr{C+\frac{2+\pc}{\pc}\bar{C}}\E\br{\sigma_{k}^{2}}\eqsp.
	\end{multline}
	Using by convention that $\sum_{l=0}^{-1} =0$, an induction shows that
	\begin{multline}\label{eq:bound:Vk:recursion}
		\E\br{V_{k}}
		\le \pr{1-\frac{\pc}{4}}^{k}\E\br{V_{0}}
		+ \frac{4(1-\pc)\gamma^{2}}{\betaemptysquared \pc}\pr{D+\frac{2+\pc}{\pc}\bar{D}}
		+ \frac{8\pr{1-\tau}\pr{b-1}\gamma d}{b \pc}\\
		+ (1-\pc)\gamma^{2}\pr{B+\frac{2+\pc}{\pc}\bar{B}}\sum_{l=0}^{k-1}\pr{1-\frac{\pc}{4}}^{k-l-1}\E\br{\dist_{l}^{2}}\\
		+ (1-\pc)\gamma^{2}\pr{C+\frac{2+\pc}{\pc}\bar{C}}\sum_{l=0}^{k-1}\pr{1-\frac{\pc}{4}}^{k-l-1}\E\br{\sigma_{l}^{2}}\eqsp.
	\end{multline}
	Moreover, for any $l\in\N^{\star}$ the assumption \Cref{ass:dk:combination} implies that
	\[
		\E\br{\dist_{l}^{2}}
		\le \pr{1-A_{d}}\E\br{\dist_{l-1}^{2}}
		+ B_{d}\E\br{\sigma_{l-1}^{2}}
		+ C_{d}\E\br{V_{l-1}}
		+ D_{d}\eqsp,
	\]
	and unrolling the recursion gives that
	\begin{equation}\label{eq:bound:dl}
		\E\br{\dist_{l}^{2}}
		\le \pr{1-A_{d}}^{l}\E\br{\dist_{0}^{2}}
		+ \sum_{j=1}^{l}\pr{1-A_{d}}^{l-j}\pr{B_{d}\E\br{\sigma_{j-1}^{2}} + C_{d}\E\br{V_{j-1}}}
		+ \frac{D_{d}}{A_{d}}\eqsp.
	\end{equation}
	Similarly, we also have
	\begin{equation}\label{eq:bound:sigmal}
		\E\br{\sigma_{l}^{2}}
		\le \pr{1-A_{\sigma}}^{l}\E\br{\sigma_{0}^{2}}
		+ \sum_{j=1}^{l}\pr{1-A_{\sigma}}^{l-j}\pr{B_{\sigma}\E\br{\dist_{j-1}^{2}} + C_{\sigma}\E\br{V_{j-1}}}
		+ \frac{D_{\sigma}}{A_{\sigma}}\eqsp.
	\end{equation}
	Hence, by plugging \eqref{eq:bound:sigmal} in \eqref{eq:bound:Vk:recursion} we obtain that
	\begin{multline}\label{eq:bound:Vk:recursion:new:1}
		\E\br{V_{k}}
		\le \pr{1-\frac{\pc}{4}}^{k}\E\br{V_{0}}
		+ \frac{4(1-\pc)\gamma^{2}}{\betaemptysquared \pc}\pr{D+\frac{2+\pc}{\pc}\bar{D}}
		+ \frac{8\pr{1-\tau}\pr{b-1}\gamma d}{b \pc}\\
		+ (1-\pc)\gamma^{2}\pr{B+\frac{2+\pc}{\pc}\bar{B}}\sum_{l=0}^{k-1}\pr{1-\frac{\pc}{4}}^{k-l-1}\E\br{\dist_{l}^{2}}\\
		+ (1-\pc)\gamma^{2}\pr{C+\frac{2+\pc}{\pc}\bar{C}}\sum_{l=0}^{k-1}\pr{1-\frac{\pc}{4}}^{k-l-1}\pr{1-A_{\sigma}}^{l}\E\br{\sigma_{0}^{2}}\\
		+ B_{\sigma}(1-\pc)\gamma^{2}\pr{C+\frac{2+\pc}{\pc}\bar{C}}\sum_{l=0}^{k-1}\sum_{j=1}^{l}\pr{1-\frac{\pc}{4}}^{k-l-1}\pr{1-A_{\sigma}}^{l-j}\E\br{\dist_{j-1}^{2}}\\
		+ C_{\sigma}(1-\pc)\gamma^{2}\pr{C+\frac{2+\pc}{\pc}\bar{C}}\sum_{l=0}^{k-1}\sum_{j=1}^{l}\pr{1-\frac{\pc}{4}}^{k-l-1}\pr{1-A_{\sigma}}^{l-j}\E\br{V_{j-1}}\\
		+ \frac{4(1-\pc)\gamma^{2} D_{\sigma}}{\betaemptysquared A_{\sigma}(\pc-4A_{d})} \pr{C+\frac{2+\pc}{\pc}\bar{C}}
		\eqsp.
	\end{multline}
	In addition, interchanging the summations gives
	\begin{align}
		% &\sum_{l=0}^{k-1}\sum_{j=1}^{l}\pr{1-\frac{\pc}{4}}^{k-l-1}\pr{1-A_{\sigma}}^{l-j}\E\br{\dist_{j-1}^{2}}
		% 	= \sum_{i=0}^{k-2}\br{\sum_{l=0}^{k-i-2}\pr{1-\frac{\pc}{4}}^{k-i-2-l}\pr{1-A_{\sigma}}^{l}} \E\br{\dist_{i}^{2}} \eqsp,\\
		&\sum_{l=0}^{k-1}\sum_{j=1}^{l}\pr{1-\frac{\pc}{4}}^{k-l-1}\pr{1-A_{\sigma}}^{l-j}\E\br{V_{j-1}^{2}}
			= \sum_{i=0}^{k-2}\br{\sum_{l=0}^{k-i-2}\pr{1-\frac{\pc}{4}}^{k-i-2-l}\pr{1-A_{\sigma}}^{l}} \E\br{V_{i}} \eqsp.
	\end{align}
	Thus, using that $\sum_{l=0}^{k-i-2}\pr{1-\nofrac{\pc}{4}}^{k-i-2-l}\pr{1-A_{\sigma}}^{l}\le4\pr{1-A_{d}}^{k-i-1}\pr{\pc-4A_{d}}^{-1}$, we can simplify the upper bound of $\E\br{V_{k}}$ derived in \eqref{eq:bound:Vk:recursion:new:1}. Indeed, we can write
	\begin{multline}\label{eq:bound:Vk:recursion:new:2}
		\E\br{V_{k}}
		\le \pr{1-\frac{\pc}{4}}^{k}\E\br{V_{0}}
		+ \frac{4(1-\pc)\gamma^{2}\pr{1-A_{d}}^{k}}{\pc - 4 A_{d}}\pr{C+\frac{2+\pc}{\pc}\bar{C}}\E\br{\sigma_{0}^{2}}
		\\
		+ \frac{4(1-\pc)\gamma^{2}}{\betaemptysquared \pc}\pr{D+\frac{2+\pc}{\pc}\bar{D}}
		+ \frac{8\pr{1-\tau}\pr{b-1}\gamma d}{b \pc}
		+ \frac{4(1-\pc)\gamma^{2}D_{\sigma}}{\betaemptysquared A_{\sigma}(\pc-4A_{d})} \pr{C+\frac{2+\pc}{\pc}\bar{C}}
		\\
		+ (1-\pc)\gamma^{2}\pr{B+\frac{2+\pc}{\pc}\bar{B}}\sum_{l=0}^{k-1}\pr{1-\frac{\pc}{4}}^{k-l-1}\E\br{\dist_{l}^{2}}\\
		+ B_{\sigma}(1-\pc)\gamma^{2}\pr{C+\frac{2+\pc}{\pc}\bar{C}}\sum_{l=0}^{k-1}\pr{1-\frac{\pc}{4}}^{k-l-1}\sum_{j=0}^{l-1}\pr{1-A_{\sigma}}^{l-j-1}\E\br{\dist_{j}^{2}}\\
		+ \frac{4(1-\pc)\gamma^{2}C_{\sigma}}{\pc-4A_d}\pr{C+\frac{2+\pc}{\pc}\bar{C}}\sum_{l=0}^{k-2}\pr{1-A_{d}}^{k-l-1}\E\br{V_{l}}
		\eqsp.
	\end{multline}
	\paragraph*{Upper bound on $\E\br{\dist_{k}^{2}}$.}
	For $l\ge 1$, plugging \eqref{eq:bound:sigmal} into \eqref{eq:bound:dl} yields the following upper bound
	\begin{multline}\label{eq:bound:dl:2:1}
		\E\br{\dist_{l}^{2}}
		\le \pr{1-A_{d}}^{l}\E\br{\dist_{0}^{2}}
		+ C_{d}\sum_{j=1}^{l}\pr{1-A_{d}}^{l-j}\E\br{V_{j-1}}
		+ \frac{D_{d}}{A_{d}}\\
		+ B_{d}\sum_{j=1}^{l}\pr{1-A_{d}}^{l-j}\br{
			\pr{1-A_{\sigma}}^{j-1}\E\br{\sigma_{0}^{2}}
			+ \sum_{i=1}^{j-1}\pr{1-A_{\sigma}}^{j-i-1}\pr{B_{\sigma}\E\br{\dist_{i-1}^{2}} + C_{\sigma}\E\br{V_{i-1}}}
			+ \frac{D_{\sigma}}{A_{\sigma}}}
		\eqsp.
	\end{multline}
	The above inequality leads to the next inequality
	\begin{multline}\label{eq:bound:dl:2}
		\E\br{\dist_{l}^{2}}
		\le \pr{1-A_{d}}^{l}\E\br{\dist_{0}^{2}}
		+ B_{d} \sum_{j=1}^{l}\pr{1-A_{d}}^{l-j}\pr{1-A_{\sigma}}^{j-1}\E\br{\sigma_{0}^{2}}
		\\
		+ C_{d}\sum_{j=1}^{l}\pr{1-A_{d}}^{l-j}\E\br{V_{j-1}}
		+ B_{d}C_{\sigma}\sum_{j=1}^{l}\sum_{i=1}^{j-1}\pr{1-A_{\sigma}}^{j-i-1}\pr{1-A_{d}}^{l-j}\E\br{V_{i-1}}
		\\
		+ B_{d}B_{\sigma}\sum_{j=1}^{l}\sum_{i=1}^{j-1}\pr{1-A_{d}}^{l-j}\pr{1-A_{\sigma}}^{j-i-1}\E\br{\dist_{i-1}^{2}}
		+ \frac{D_{d}}{A_{d}}
		+ \frac{B_{d} D_{\sigma}}{A_{d} A_{\sigma}}
		\eqsp.
	\end{multline}
	By interchanging the double summations in \eqref{eq:bound:dl:2}, we obtain
	\begin{align}
		 &\sum_{j=1}^{l}\sum_{i=1}^{j-1}\pr{1-A_{d}}^{l-j}\pr{1-A_{\sigma}}^{j-i-1}\E\br{\dist_{i-1}^{2}}
		 = \sum_{i=1}^{l-1}\br{\sum_{j=i+1}^{l}\pr{1-A_{d}}^{l-j}\pr{1-A_{\sigma}}^{j-i-1}}\E\br{\dist_{i-1}^{2}}
		 \\
		 &= \sum_{i=0}^{l-2}\br{\sum_{j=0}^{l-i-2}\pr{1-A_{d}}^{l-i-2-j}\pr{1-A_{\sigma}}^{j}}\E\br{\dist_{i}^{2}}
		 \le \frac{1}{A_{\sigma} - A_{d}}\sum_{i=0}^{l-2}\pr{1-A_{d}}^{l-i-1}\E\br{\dist_{i}^{2}}
		 \eqsp. \label{eq:bound:dl:2:doubledk}
	\end{align}
	Similarly, we can also get that
	\begin{equation}\label{eq:bound:dl:2:doubleVk}
		\sum_{j=1}^{l}\sum_{i=1}^{j-1}\pr{1-A_{d}}^{l-j}\pr{1-A_{\sigma}}^{j-i-1}\E\br{V_{i-1}}
		\le \frac{1}{A_{\sigma} - A_{d}}\sum_{i=0}^{l-2}\pr{1-A_{d}}^{l-i-1}\E\br{V_{i}}\eqsp.
	\end{equation}
	Plugging back \eqref{eq:bound:dl:2:doubledk} and \eqref{eq:bound:dl:2:doubleVk} in \eqref{eq:bound:dl:2} shows
	\begin{multline}\label{eq:bound:dl:3}
		\E\br{\dist_{l}^{2}}
		\le \pr{1-A_{d}}^{l}\E\br{\dist_{0}^{2}}
		+ \frac{B_{d}\pr{1-A_{d}}^{l}}{A_{\sigma}-A_{d}} \E\br{\sigma_{0}^{2}}
		+ \frac{B_{d}B_{\sigma}}{A_{\sigma} - A_{d}}\sum_{i=0}^{l-2}\pr{1-A_{d}}^{l-i-1}\E\br{\dist_{i}^{2}}
		\\
		+ C_{d}\sum_{i=0}^{l-1}\pr{1-A_{d}}^{l-i-1}\E\br{V_{i}}
		+ \frac{B_{d}C_{\sigma}}{A_{\sigma} - A_{d}}\sum_{i=0}^{l-2}\pr{1-A_{d}}^{l-i-1}\E\br{V_{i}}
		+ \frac{D_{d}}{A_{d}}
		+ \frac{B_{d} D_{\sigma}}{A_{d} A_{\sigma}}
		\eqsp.
	\end{multline}
	Now, we want to control $\sum_{i=0}^{l-2}\pr{1-A_{d}}^{l-i-1}\E\br{\dist_{i}^{2}}$. For this, for any $l\in\N$ define
	\begin{multline}\label{eq:def:Ul}
		U_{l}
		= \E\br{\dist_{0}^{2}}
		+ \frac{B_{d}}{A_{\sigma}-A_{d}} \E\br{\sigma_{0}^{2}}
		+ \frac{D_{d}\pr{1-A_{d}}^{-l}}{A_{d}}
		+ \frac{B_{d} D_{\sigma}\pr{1-A_{d}}^{-l}}{A_{d} A_{\sigma}}
		\\
		+ C_{d}\sum_{i=0}^{l-1}\pr{1-A_{d}}^{-i-1}\E\br{V_{i}}
		+ \frac{B_{d}C_{\sigma}}{A_{\sigma} - A_{d}}\sum_{i=0}^{l-2}\pr{1-A_{d}}^{-i-1}\E\br{V_{i}}
	\end{multline}
	and consider
	\begin{equation}
		S_{l}
		= \sum_{i=0}^{l}\pr{1-A_{d}}^{-i}\E\br{\dist_{i}^{2}} \eqsp.
	\end{equation}
	With the above notation, \eqref{eq:bound:dl:3} can be rewritten as
	\begin{equation}\label{eq:bound:dl:4}
		S_{l} - S_{l-1}
		\le \frac{B_{d}B_{\sigma}}{\pr{1-A_{d}}(A_{\sigma} - A_{d})} S_{l-2} + U_{l}
		\eqsp.
	\end{equation}
	For $l\ge 2$, using the upper bound derived in \eqref{eq:bound:dl:4} gives
	\begin{equation}\label{eq:bound:dl:4:new}
		\E\br{\dist_{l}^{2}}
		= \pr{1 - A_d}^{l} \pr{S_{l} - S_{l-1}}
		\le \frac{B_{d}B_{\sigma} \pr{1 - A_d}^{l-1} S_{l-2}}{(A_{\sigma} - A_{d})} + \pr{1 - A_d}^{l} U_{l}
		\eqsp.
	\end{equation}
	Finally, we define
	\begin{equation}
		\delta_{\alpha} = \frac{-1 + \sqrt{1 + 4{\pr{1-A_{d}}^{-1}(A_{\sigma} - A_{d})^{-1}}{B_{d} B_{\sigma}}}}{2}\eqsp
	\end{equation}
	such that $\delta_{\alpha}$ is solution of the equation
	\begin{equation}\label{eq:eq:deltaeq}
		\delta_{\alpha}^{2} + \delta_{\alpha} = \frac{B_{d}B_{\sigma}}{\pr{1-A_{d}}(A_{\sigma} - A_{d})}
	\end{equation}
	Thus for $l\ge 2$, the definition of $\delta_{\alpha}$ combined with \eqref{eq:bound:dl:4} show
	\begin{equation}\label{eq:bound:dl:5}
		S_{l} + \delta_{\alpha} S_{l-1}
		\le \pr{1+\delta_{\alpha}} \pr{S_{l-1} + \delta_{\alpha} S_{l-2}} + U_{l}
		\eqsp.
	\end{equation}
	Unrolling this recursion gives
	\begin{equation}\label{eq:bound:dl:6}
		S_{k} + \delta_{\alpha} S_{k-1}
		\le \pr{1+\delta_{\alpha}}^{k-1} \pr{S_{1} + \delta_{\alpha} S_{0}} + \sum_{l=2}^{k}\pr{1 + \delta_{\alpha}}^{k-l}U_{l}
		\eqsp.
	\end{equation}

	\paragraph*{Upper bound on $\sum_{l=0}^{k-1}\pr{1-\tilde{\alpha}}^{l-j-1}\E\brn{\dist_{j}^{2}}$.}

	Let consider a fixed $\tilde{\alpha}\in\acn{\pc/4, A_\sigma}$, by assumption we have $A_d< \tilde{\alpha} <1$.
	Since we want to control $\sum_{l=0}^{k-1}\prn{1-\nofrac{\pc}{4}}^{k-l-1}\E\brn{\dist_{l}^{2}}$ and $\sum_{l=0}^{k-1}\prn{1-\nofrac{\pc}{4}}^{k-l-1}\sum_{j=0}^{l-1}\pr{1-A_{\sigma}}^{l-j-1}\E\brn{\dist_{j}^{2}}$ involved in the inequality \eqref{eq:bound:Vk:recursion:new:2}, we first study $\sum_{l=0}^{k-1}\prn{1-\tilde{\alpha}}^{k-l-1}\E\brn{\dist_{l}^{2}}$.
	From \eqref{eq:bound:dl:4:new}, we deduce that
	\begin{multline}\label{eq:bound:sumoneminusAsigmadlsquared}
		\sum_{l=0}^{k-1}\pr{1-\tilde{\alpha}}^{k-l-1}\E\br{\dist_{l}^{2}}
		\le \frac{B_d B_\sigma}{(1-A_d)(A_\sigma - A_d)}\sum_{l=0}^{k-1}\pr{1-A_d}^{l}\pr{1-\tilde{\alpha}}^{k-l-1}S_{l-2}\\
		+ \sum_{l=0}^{k-1}\pr{1-A_d}^{l}\pr{1-\tilde{\alpha}}^{k-l-1}U_l
		\eqsp.
	\end{multline}
	Since we suppose \Cref{ass:dk:combination} and $A_d\le A_\sigma/2$, $A_d A_\sigma\ge 8B_dB_\sigma$ we can apply \Cref{lem:bound:alpha} which shows that $1-\alpha=(1-A_d)(1+\delta_{\alpha})\in\ooint{0,1-\tilde{\alpha}}$ and leads to
	% Since we suppose $A_d\le A_\sigma/2$, we deduce that $A_d\le 1/2$ and thus $(1-A_d)\prn{A_d^{2}/4+A_d/2}\le A_d/4$. In addition, using $A_d A_\sigma\ge 8B_dB_\sigma$, we get that
	% \[
	% 	\pr{1-A_d}\pr{\frac{A_d^{2}}{4} + \frac{A_d}{2}} \ge \frac{A_d}{4} \ge \frac{2 B_dB_\sigma}{A_\sigma} \ge (1-A_d)\pr{\delta_{\alpha}^{2} + \delta_{\alpha}}\eqsp.
	% \]
	% Hence, we have proved that $A_d\ge 2\delta_{\alpha}$.
	% Let introduce $\newrate=1-\prn{1-A_d}\prn{1+\delta_{\alpha}}$. Using that $0\le 2\delta_{\alpha}\le A_d$, we get that $\newrate\in[A_d/2,A_d]$ which implies $0<\newrate<\tilde{\alpha}$ and leads to
	\begin{align}
		\sum_{l=0}^{k-1}\pr{1-A_d}^{l}\pr{1-\tilde{\alpha}}^{k-l-1}\pr{1+\delta_{\alpha}}^{l-3}
		&\le \pr{1+\delta_{\alpha}}^{-3}\sum_{l=0}^{k-1}\pr{1-\newrate}^l\pr{1-\tilde{\alpha}}^{k-l-1}\\
		\label{eq:bound:sumadasigmadelta}
		&\le \frac{\pr{1-\newrate}^{k}}{(\tilde{\alpha} - \newrate)(1+\delta_{\alpha})^3}\eqsp.
	\end{align}
	Moreover, for $l\ge 2$ applying the result given by \eqref{eq:bound:dl:6}, we have
	\begin{equation}\label{eq:bound:Skminustwo}
		S_{l-2}
		\le \pr{1+\delta_{\alpha}}^{l-3}\pr{S_1+\delta_{\alpha} S_{0}}
		+ \sum_{j=2}^{l-2}\pr{1 + \delta_{\alpha}}^{l-j-2}U_{j}
		\eqsp.
	\end{equation}
	Using the definition of $U_l$ given by \eqref{eq:def:Ul}, we can write the following equality
	\begin{multline}\label{eq:eq:sumweightlargeUl:1}
		\sum_{l=0}^{k-1}\pr{1-A_d}^l\pr{1-\tilde{\alpha}}^{k-l-1}\sum_{j=2}^{l-2}\pr{1+\delta_{\alpha}}^{l-j-2}U_j\\
		= \pr{\E\br{\dist_{0}^{2}} + \frac{B_{d}}{A_{\sigma}-A_{d}} \E\br{\sigma_{0}^{2}}} \sum_{l=0}^{k-1}\sum_{j=2}^{l-2}\pr{1-A_d}^l\pr{1-\tilde{\alpha}}^{k-l-1}\pr{1+\delta_{\alpha}}^{l-j-2}\\
		+ \pr{\frac{D_{d}}{A_{d}} + \frac{B_{d} D_{\sigma}}{A_{d} A_{\sigma}}} \sum_{l=0}^{k-1}\sum_{j=2}^{l-2}\pr{1-A_d}^{l-j}\pr{1-\tilde{\alpha}}^{k-l-1}\pr{1+\delta_{\alpha}}^{l-j-2} \\
		+ \pr{C_{d} + \frac{B_{d}C_{\sigma}}{A_{\sigma} - A_{d}}}
			\sum_{l=0}^{k-1}\sum_{j=2}^{l-2}\pr{1-A_d}^l\pr{1-\tilde{\alpha}}^{k-l-1}\pr{1+\delta_{\alpha}}^{l-j-2}\sum_{i=0}^{j-1}\pr{1-A_{d}}^{-i-1}\E\br{V_{i}}
	\end{multline}
	We now upper bound each quantity separately. Regarding the first double sum, since $\prn{1-A_d}\prn{1+\delta_{\alpha}}=1-\newrate$ we get
	\begin{align}
		&\sum_{l=0}^{k-1}\sum_{j=2}^{l-2}\pr{1-A_d}^l\pr{1-\tilde{\alpha}}^{k-l-1}\pr{1+\delta_{\alpha}}^{l-j-2}\\
		&= \sum_{j=2}^{k-3}\pr{1-A_d}^{j+2}\sum_{l=j+2}^{k-1}\pr{1-\tilde{\alpha}}^{k-l-1}\pr{1-\newrate}^{l-j-2}\\
		\label{eq:bound:sumweightlargeUl:2}
		&\le \frac{1}{\tilde{\alpha} - \newrate}\sum_{j=4}^{k-1}\pr{1-A_d}^{j}\pr{1-\newrate}^{k-j}
		\le \frac{\pr{1-A_d}^{4}\pr{1-\newrate}^{k-3}}{\prn{A_d - \newrate}(\tilde{\alpha} - \newrate)}\eqsp.
	\end{align}
	Using $\prn{1-A_d}\prn{1+\delta_{\alpha}}=1-\newrate$ combined with $\sum_{l=j+2}^{k-1}\prn{1-\newrate}^{l-j-2}\prn{1-\tilde{\alpha}}^{k-l-1}\le \prn{\tilde{\alpha} - \newrate}^{-1}\prn{1-\newrate}^{k-j-2}$ give
	\begin{align}
		&\sum_{l=0}^{k-1}\sum_{j=2}^{l-2}\pr{1-A_d}^{l-j}\pr{1-\tilde{\alpha}}^{k-l-1}\pr{1+\delta_{\alpha}}^{l-j-2}\\
		&= \pr{1-A_d}^{2}\sum_{l=0}^{k-1}\sum_{j=2}^{l-2}\pr{1-\newrate}^{l-j-2}\pr{1-\tilde{\alpha}}^{k-l-1}\\
		&= \pr{1-A_d}^{2}\sum_{j=2}^{k-3}\sum_{l=j+2}^{k-1}\pr{1-\newrate}^{l-j-2}\pr{1-\tilde{\alpha}}^{k-l-1}\\
		\label{eq:bound:sumweightlargeUl:3}
		&\le \frac{\pr{1-A_d}^{2}}{\tilde{\alpha} - \newrate}\sum_{j=2}^{k-3}\pr{1-\newrate}^{k-j-2}
		\le \frac{\prn{1-\newrate}\pr{1-A_d}^{2}}{\newrate(\tilde{\alpha} - \newrate)}\eqsp.
	\end{align}
	The same arguments show that
	\begin{align}
		&\sum_{l=0}^{k-1}\sum_{j=2}^{l-2}\pr{1-A_d}^l\pr{1-\tilde{\alpha}}^{k-l-1}\pr{1+\delta_{\alpha}}^{l-j-2}\sum_{i=0}^{j-1}\pr{1-A_{d}}^{-i-1}\E\br{V_{i}} \\
		&\le \sum_{i=0}^{k-4}\sum_{j=i+1}^{k-3}\sum_{l=j+2}^{k-1}\pr{1-A_d}^l\pr{1-\tilde{\alpha}}^{k-l-1}\pr{1+\delta_{\alpha}}^{l-j-2}\pr{1-A_{d}}^{-i-1}\E\br{V_{i}}\\
		&\le \sum_{i=0}^{k-4}\E\br{V_{i}}\sum_{j=i+1}^{k-3}\pr{1-A_d}^{j-i+1}\sum_{l=j+2}^{k-1}\pr{1-\tilde{\alpha}}^{k-l-1}\pr{1-\newrate}^{l-j-2}\\
		&\le \frac{1}{\tilde{\alpha} - \newrate}\sum_{i=0}^{k-4}\E\br{V_{i}}\sum_{j=i+1}^{k-3}\pr{1-A_d}^{j-i+1}\pr{1-\newrate}^{k-j-2}\\
		&= \frac{\pr{1-\newrate}\pr{1-A_d}^{2}}{\tilde{\alpha} - \newrate}\sum_{i=0}^{k-4}\E\br{V_{i}}\sum_{j=i+1}^{k-3}\pr{1-A_d}^{j-i-1}\pr{1-\newrate}^{k-j-3}\\
		\label{eq:bound:sumweightlargeUl:4}
		&\le \frac{\pr{1-\newrate}^{-1}\pr{1-A_d}^{2}}{(A_d-\newrate)(\tilde{\alpha} - \newrate)}\sum_{i=0}^{k-4}\pr{1-\newrate}^{k-i-1}\E\br{V_{i}} \eqsp.
	\end{align}
	Therefore, plugging \eqref{eq:bound:sumweightlargeUl:2}, \eqref{eq:bound:sumweightlargeUl:3}, \eqref{eq:bound:sumweightlargeUl:4} inside \eqref{eq:eq:sumweightlargeUl:1} implies
	\begin{multline}\label{eq:bound:sumweightlargeUl:5}
		\sum_{l=0}^{k-1}\sum_{j=2}^{l-2}\pr{1-A_d}^{l}\pr{1-\tilde{\alpha}}^{k-l-1}\pr{1 + \delta_{\alpha}}^{l-j-2}U_{j}
		\le \frac{\pr{1-\newrate}\pr{1-A_d}^{2}}{\newrate(\tilde{\alpha} - \newrate)} \pr{\frac{D_{d}}{A_{d}} + \frac{B_{d} D_{\sigma}}{A_{d} A_{\sigma}}} \\
		+ \frac{\pr{1-A_d}^{4}\pr{1-\newrate}^{k-3}}{\prn{A_d - \newrate}(\tilde{\alpha} - \newrate)} \pr{\E\br{\dist_{0}^{2}} + \frac{B_{d}}{A_{\sigma}-A_{d}} \E\br{\sigma_{0}^{2}}} \\
		+ \frac{\pr{1-\newrate}^{-1}\pr{1-A_d}^{2}}{(A_d-\newrate)(\tilde{\alpha} - \newrate)} \pr{C_{d} + \frac{B_{d}C_{\sigma}}{A_{\sigma} - A_{d}}}
			\sum_{i=0}^{k-4}\pr{1-\newrate}^{k-i-1}\E\br{V_{i}} \eqsp.
	\end{multline}
	In addition, by definition of $U_l$ provides in \eqref{eq:def:Ul} we have
	\begin{multline}
		\sum_{l=0}^{k-1}\pr{1-A_d}^{l}\pr{1-\tilde{\alpha}}^{k-l-1}U_l
		= \pr{\frac{D_{d}}{A_{d}} + \frac{B_{d} D_{\sigma}}{A_{d} A_{\sigma}}} \sum_{l=0}^{k-1}\pr{1-\tilde{\alpha}}^{k-l-1} \\
		+ \pr{\E\br{\dist_{0}^{2}} + \frac{B_{d}}{A_{\sigma}-A_{d}} \E\br{\sigma_{0}^{2}}} \sum_{l=0}^{k-1}\pr{1-A_d}^{l}\pr{1-\tilde{\alpha}}^{k-l-1}\\
		+ \pr{C_{d} + \frac{B_{d}C_{\sigma}}{A_{\sigma} - A_{d}}}
			\sum_{l=0}^{k-1}\pr{1-A_d}^{l}\pr{1-\tilde{\alpha}}^{k-l-1}\sum_{i=0}^{l-1}\pr{1-A_{d}}^{-i-1}\E\br{V_{i}}\eqsp.
	\end{multline}
	Thus, a calculation yields that
	\begin{multline}\label{eq:bound:sumweightUl}
		\sum_{l=0}^{k-1}\pr{1-A_d}^{l}\pr{1-\tilde{\alpha}}^{k-l-1}U_l
		\le \frac{\pr{1-A_d}^{k}}{\tilde{\alpha} - A_d}\pr{\E\br{\dist_{0}^{2}} + \frac{B_{d}}{A_{\sigma}-A_{d}} \E\br{\sigma_{0}^{2}}}\\
		+ \frac{1}{\tilde{\alpha}}\pr{\frac{D_{d}}{A_{d}} + \frac{B_{d} D_{\sigma}}{A_{d} A_{\sigma}}}
		+ \frac{1}{\tilde{\alpha} - A_d}\pr{C_{d} + \frac{B_{d}C_{\sigma}}{A_{\sigma} - A_{d}}}
			\sum_{i=0}^{k-2}\pr{1-A_{d}}^{k-i-1}\E\br{V_{i}}\eqsp.
	\end{multline}
	Plugging \eqref{eq:bound:Skminustwo} in \eqref{eq:bound:sumoneminusAsigmadlsquared} shows
	\begin{multline}\label{eq:bound:sumoneminusAsigmadlsquared:2}
		\sum_{l=0}^{k-1}\pr{1-\tilde{\alpha}}^{k-l-1}\E\br{\dist_{l}^{2}}
		\le \frac{B_d B_\sigma \pr{S_1+\delta_{\alpha} S_{0}}}{(1-A_d)(A_\sigma - A_d)}\sum_{l=0}^{k-1}\pr{1-A_d}^{l}\pr{1-\tilde{\alpha}}^{k-l-1}\pr{1+\delta_{\alpha}}^{l-3}
		\\
		+ \frac{B_d B_\sigma}{(1-A_d)(A_\sigma - A_d)}\sum_{l=0}^{k-1}\sum_{j=2}^{l-2}\pr{1-A_d}^{l}\pr{1-\tilde{\alpha}}^{k-l-1}\pr{1 + \delta_{\alpha}}^{l-j-2}U_{j}
		\\
		+ \sum_{l=0}^{k-1}\pr{1-A_d}^{l}\pr{1-\tilde{\alpha}}^{k-l-1}U_l\eqsp.
	\end{multline}
	Hence, by combining \eqref{eq:bound:sumadasigmadelta}, \eqref{eq:bound:sumweightlargeUl:5}, \eqref{eq:bound:sumweightUl} and \eqref{eq:bound:sumoneminusAsigmadlsquared:2} we obtain for $A_d>\newrate$, that
	\begin{multline}\label{eq:bound:sumoneminusAsigmadlsquared:general}
		\sum_{l=0}^{k-1}\pr{1-\tilde{\alpha}}^{k-l-1}\E\br{\dist_{l}^{2}}
		\le \frac{B_d B_\sigma \pr{S_1+\delta_{\alpha} S_{0}}\pr{1-\newrate}^{k}}{(1-A_d)(A_\sigma - A_d)(\tilde{\alpha} - \newrate)(1+\delta_{\alpha})^3}
		\\
		+ \pr{\frac{\pr{1-A_d}^{k}}{\tilde{\alpha} - A_d} + \frac{B_d B_\sigma\pr{1-\newrate}^{k}}{(A_\sigma - A_d)\prn{A_d - \newrate}(\tilde{\alpha} - \newrate)}}
			\pr{\E\br{\dist_{0}^{2}} + \frac{B_{d}}{A_{\sigma}-A_{d}} \E\br{\sigma_{0}^{2}}} \\
		+ \pr{\frac{1}{\tilde{\alpha}} + \frac{B_d B_\sigma}{\newrate(\tilde{\alpha} - \newrate)(A_\sigma - A_d)}}
			\pr{\frac{D_{d}}{A_{d}} + \frac{B_{d} D_{\sigma}}{A_{d} A_{\sigma}}}
		\\
		% + \frac{1}{\tilde{\alpha} - A_d}\pr{C_{d} + \frac{B_{d}C_{\sigma}}{A_{\sigma} - A_{d}}}
		% 	\sum_{i=0}^{k-2}\pr{1-A_{d}}^{k-i-1}\E\br{V_{i}}\\
		% + \frac{B_d B_\sigma}{(A_\sigma - A_d)(A_d-\newrate)(\tilde{\alpha} - \newrate)} \pr{C_{d} + \frac{B_{d}C_{\sigma}}{A_{\sigma} - A_{d}}}
		% 	\sum_{i=0}^{k-4}\pr{1-\newrate}^{k-i-1}\E\br{V_{i}}
		+ \pr{C_{d} + \frac{B_{d}C_{\sigma}}{A_{\sigma} - A_{d}}}\sum_{i=0}^{k-2}\pr{
			\frac{\pr{1-A_{d}}^{k-i-1}}{\tilde{\alpha} - A_d}
			+ \frac{B_d B_\sigma \pr{1-\newrate}^{k-i-1}}{(A_\sigma - A_d)(A_d-\newrate)(\tilde{\alpha} - \newrate)}
			} \E\br{V_{i}}		
		\eqsp.
	\end{multline}
	In addition, the above bound holds even if $A_d=\newrate$ by considering that $(A_d-\newrate)^{-1} B_d B_\sigma=0$.

	\paragraph*{Upper bound on $\sum_{l=0}^{k-1}\prn{1-\nofrac{\pc}{4}}^{k-l-1}\E\br{\dist_{l}^{2}}$.}

	Applying \eqref{eq:bound:sumoneminusAsigmadlsquared:general} with $\tilde{\alpha}=\pc/4$ gives
	\begin{multline}\label{eq:bound:sumoneminusAsigmadlsquared:pc}
		\sum_{l=0}^{k-1}\pr{1-\frac{\pc}{4}}^{k-l-1}\E\br{\dist_{l}^{2}}
		\le \frac{4B_d B_\sigma \pr{S_1+\delta_{\alpha} S_{0}}\pr{1-\newrate}^{k}}{(1-A_d)(A_\sigma - A_d)(\pc - 4\newrate)(1+\delta_{\alpha})^3}
		\\
		+ \pr{\frac{4\pr{1-A_d}^{k}}{\pc - 4A_d} + \frac{4B_d B_\sigma\pr{1-\newrate}^{k}}{(A_\sigma - A_d)\prn{A_d - \newrate}(\pc - 4\newrate)}}
			\pr{\E\br{\dist_{0}^{2}} + \frac{B_{d}}{A_{\sigma}-A_{d}} \E\br{\sigma_{0}^{2}}} \\
		+ \pr{\frac{4}{\pc} + \frac{4B_d B_\sigma}{\newrate(\pc - 4\newrate)(A_\sigma - A_d)}}
			\pr{\frac{D_{d}}{A_{d}} + \frac{B_{d} D_{\sigma}}{A_{d} A_{\sigma}}}
		\\
		+ 4\pr{C_{d} + \frac{B_{d}C_{\sigma}}{A_{\sigma} - A_{d}}}\sum_{i=0}^{k-2}\pr{
			\frac{\pr{1-A_{d}}^{k-i-1}}{\pc - 4A_d}
			+ \frac{B_d B_\sigma \pr{1-\newrate}^{k-i-1}}{(A_\sigma - A_d)(A_d-\newrate)(\pc - 4\newrate)}
			} \E\br{V_{i}}		
		\eqsp.
	\end{multline}

	\paragraph*{Upper bound on $\sum_{l=0}^{k-1}\prn{1-\nofrac{\pc}{4}}^{k-l-1}\sum_{j=0}^{l-1}\pr{1-A_{\sigma}}^{l-j-1}\E\brn{\dist_{j}^{2}}$.}

	Recall that we consider that $(A_d-\newrate)^{-1} B_d B_\sigma=0$ in the specific case where $A_d=\newrate$.
	This time, setting $\tilde{\alpha}=A_\sigma$ in \eqref{eq:bound:sumoneminusAsigmadlsquared:general} shows that
	\begin{multline}\label{eq:bound:sumoneminusAsigmadlsquared:Asigma}
		\sum_{j=0}^{l-1}\pr{1-A_\sigma}^{l-j-1}\E\br{\dist_{l}^{2}}
		\le \frac{B_d B_\sigma \pr{S_1+\delta_{\alpha} S_{0}}\pr{1-\newrate}^{l}}{(1-A_d)(A_\sigma - A_d)(A_\sigma - \newrate)(1+\delta_{\alpha})^3}
		\\
		+ \pr{\frac{\pr{1-A_d}^{l}}{A_\sigma - A_d} + \frac{B_d B_\sigma\pr{1-\newrate}^{l}}{(A_\sigma - A_d)\pr{A_d - \newrate}(A_\sigma - \newrate)}}
			\pr{\E\br{\dist_{0}^{2}} + \frac{B_{d}}{A_{\sigma}-A_{d}} \E\br{\sigma_{0}^{2}}} \\
		+ \pr{\frac{1}{A_\sigma} + \frac{B_d B_\sigma}{\newrate(A_\sigma - \newrate)(A_\sigma - A_d)}}
			\pr{\frac{D_{d}}{A_{d}} + \frac{B_{d} D_{\sigma}}{A_{d} A_{\sigma}}}
		\\
		+ \pr{C_{d} + \frac{B_{d}C_{\sigma}}{A_{\sigma} - A_{d}}}\sum_{i=0}^{l-2}\pr{
			\frac{\pr{1-A_{d}}^{l-i-1}}{A_\sigma - A_d}
			+ \frac{B_d B_\sigma \pr{1-\newrate}^{l-i-1}}{(A_\sigma - A_d)(A_d-\newrate)(A_\sigma - \newrate)}
			} \E\br{V_{i}}
		\eqsp.
	\end{multline}
	Moreover, we have the two following bounds
	\begin{equation}\label{eq:bound:Vksumconv:1}
		\begin{aligned}
			&\sum_{l=0}^{k-1}\pr{1-\frac{\pc}{4}}^{k-l-1}\pr{1-A_d}^{l}
			\le \frac{4\pr{1-A_d}^{k}}{\pc - 4A_d}\eqsp, \\
			&\sum_{l=0}^{k-1}\pr{1-\frac{\pc}{4}}^{k-l-1}\pr{1-\newrate}^{l}
			\le \frac{4\pr{1-\newrate}^{k}}{\pc - 4\newrate}\eqsp.
		\end{aligned}
	\end{equation}
	Therefore, permuting the summations implies
	\begin{align}
		\sum_{l=0}^{k-1}\pr{1-\frac{\pc}{4}}^{k-l-1}\sum_{i=0}^{l-2}\pr{1-A_d}^{l-i-1}\E\br{V_{i}}
		&\le \sum_{i=0}^{k-3}\E\br{V_{i}}\sum_{l=i+2}^{k-1} \pr{1-\frac{\pc}{4}}^{k-l-1} \pr{1-A_d}^{l-i-1} \\
		\label{eq:bound:Vksumconv:2}
		&\le \frac{4}{\pc - 4A_d}\sum_{i=0}^{k-3}\pr{1-A_d}^{k-i-1} \E\br{V_{i}} \eqsp.
	\end{align}
	In a similar way, we obtain
	\begin{equation}\label{eq:bound:Vksumconv:3}
		\sum_{l=0}^{k-1}\pr{1-\frac{\pc}{4}}^{k-l-1}\sum_{i=0}^{l-2}\pr{1-\newrate}^{l-i-1}\E\br{V_{i}}
		\le \frac{4}{\pc - 4\newrate}\sum_{i=0}^{k-3}\pr{1-\newrate}^{k-i-1} \E\br{V_{i}}\eqsp.
	\end{equation}
	Hence, the combination of \eqref{eq:bound:sumoneminusAsigmadlsquared:Asigma} with \eqref{eq:bound:Vksumconv:1}, \eqref{eq:bound:Vksumconv:2}, \eqref{eq:bound:Vksumconv:3} yields	
	\begin{multline}\label{eq:bound:sumoneminusAsigmadlsquared:pcAsigma}
		\sum_{l=0}^{k-1}\pr{1-\frac{\pc}{4}}^{k-l-1}\sum_{j=0}^{l-1}\pr{1-A_\sigma}^{l-j-1}\E\br{\dist_{l}^{2}}
		\le \frac{4B_d B_\sigma \pr{S_1+\delta_{\alpha} S_{0}}\pr{1-\newrate}^{k}}{(\pc - 4\newrate)(1-A_d)(A_\sigma - A_d)(A_\sigma - \newrate)(1+\delta_{\alpha})^3}
		\\
		+ \frac{4}{A_\sigma - A_d}\pr{\frac{\pr{1-A_d}^{k}}{\pc - 4A_{d}} + \frac{B_d B_\sigma\pr{1-\newrate}^{k}}{(\pc - 4\newrate)\prn{A_d - \newrate}(A_\sigma - \newrate)}}
			\pr{\E\br{\dist_{0}^{2}} + \frac{B_{d}}{A_{\sigma}-A_{d}} \E\br{\sigma_{0}^{2}}} \\
		+ \frac{4}{\pc}
			\pr{\frac{1}{A_\sigma} + \frac{B_d B_\sigma}{\newrate(A_\sigma - \newrate)(A_\sigma - A_d)}}
			\pr{\frac{D_{d}}{A_{d}} + \frac{B_{d} D_{\sigma}}{A_{d} A_{\sigma}}}
		\\
		+ \frac{4}{A_\sigma - A_d} \pr{C_{d} + \frac{B_{d}C_{\sigma}}{A_{\sigma} - A_{d}}}\sum_{i=0}^{k-3}\pr{
			\frac{\pr{1-A_{d}}^{k-i-1}}{\pc-4A_d}
			+ \frac{B_d B_\sigma \pr{1-\newrate}^{k-i-1}}{(\pc-4\newrate)(A_d-\newrate)(A_\sigma - \newrate)}
			} \E\br{V_{i}}
		\eqsp.
	\end{multline}

	\paragraph*{Upper bound on $\E\br{V_{k}}$.}

	Plugging \eqref{eq:bound:sumoneminusAsigmadlsquared:pc} and \eqref{eq:bound:sumoneminusAsigmadlsquared:pcAsigma} in \eqref{eq:bound:Vk:recursion:new:2}, we obtain
	\begin{multline}\label{eq:bound:sumoneminusAsigmadlsquared:specific:1}
		\E\br{V_{k}}
		\le \pr{1-\frac{\pc}{4}}^{k}\E\br{V_{0}}
		+ \frac{4(1-\pc)\gamma^{2}\pr{1-A_{d}}^{k}}{\pc - 4 A_{d}}\pr{C+\frac{2+\pc}{\pc}\bar{C}}\E\br{\sigma_{0}^{2}}
		\\
		+ \frac{4(1-\pc)\gamma^{2}D_{\sigma}}{\betaemptysquared A_{\sigma}(\pc-4A_{d})} \pr{C+\frac{2+\pc}{\pc}\bar{C}}
		+ \frac{4(1-\pc)\gamma^{2}}{\betaemptysquared \pc}\pr{D+\frac{2+\pc}{\pc}\bar{D}}
		+ \frac{8\pr{1-\tau}\pr{b-1}\gamma d}{b \pc}
		\\
		+ (1-\pc)\gamma^{2}\pr{B+\frac{2+\pc}{\pc}\bar{B}}
			\Bigg[
				\frac{4B_d B_\sigma \pr{S_1+\delta_{\alpha} S_{0}}\pr{1-\newrate}^{k}}{(1-A_d)(A_\sigma - A_d)(\pc - 4\newrate)(1+\delta_{\alpha})^3}
				\\
				+ \pr{\frac{4\pr{1-A_d}^{k}}{\pc - 4A_d} + \frac{4B_d B_\sigma\pr{1-\newrate}^{k}}{(A_\sigma - A_d)\prn{A_d - \newrate}(\pc - 4\newrate)}}
					\pr{\E\br{\dist_{0}^{2}} + \frac{B_{d}}{A_{\sigma}-A_{d}} \E\br{\sigma_{0}^{2}}} \\
				+ \pr{\frac{4}{\pc} + \frac{4B_d B_\sigma}{\newrate(\pc - 4\newrate)(A_\sigma - A_d)}}
					\pr{\frac{D_{d}}{A_{d}} + \frac{B_{d} D_{\sigma}}{A_{d} A_{\sigma}}}
				\\
				+ 4\pr{C_{d} + \frac{B_{d}C_{\sigma}}{A_{\sigma} - A_{d}}}\sum_{i=0}^{k-2}\pr{
					\frac{\pr{1-A_{d}}^{k-i-1}}{\pc - 4A_d}
					+ \frac{B_d B_\sigma \pr{1-\newrate}^{k-i-1}}{(A_\sigma - A_d)(A_d-\newrate)(\pc - 4\newrate)}
					} \E\br{V_{i}}	
			\Bigg]
		\\
		+ 4(1-\pc)\gamma^{2}B_{\sigma}\pr{C+\frac{2+\pc}{\pc}\bar{C}}
			\Bigg[
				\frac{B_d B_\sigma \pr{S_1+\delta_{\alpha} S_{0}}\pr{1-\newrate}^{k}}{(\pc - 4\newrate)(1-A_d)(A_\sigma - A_d)(A_\sigma - \newrate)(1+\delta_{\alpha})^3}
				\\
				+ \frac{1}{A_\sigma - A_d}\pr{\frac{\pr{1-A_d}^{k}}{\pc - 4A_{d}} + \frac{B_d B_\sigma\pr{1-\newrate}^{k}}{(\pc - 4\newrate)\prn{A_d - \newrate}(A_\sigma - \newrate)}}
					\pr{\E\br{\dist_{0}^{2}} + \frac{B_{d}}{A_{\sigma}-A_{d}} \E\br{\sigma_{0}^{2}}} \\
				+ \frac{1}{\pc}
					\pr{\frac{1}{A_\sigma} + \frac{B_d B_\sigma}{\newrate(A_\sigma - \newrate)(A_\sigma - A_d)}}
					\pr{\frac{D_{d}}{A_{d}} + \frac{B_{d} D_{\sigma}}{A_{d} A_{\sigma}}}
				\\
				+ \frac{1}{A_\sigma - A_d} \pr{C_{d} + \frac{B_{d}C_{\sigma}}{A_{\sigma} - A_{d}}}\sum_{i=0}^{k-3}\pr{
					\frac{\pr{1-A_{d}}^{k-i-1}}{\pc-4A_d}
					+ \frac{B_d B_\sigma \pr{1-\newrate}^{k-i-1}}{(\pc-4\newrate)(A_d-\newrate)(A_\sigma - \newrate)}
					} \E\br{V_{i}}
			\Bigg]
		\\
		+ \frac{4(1-\pc)\gamma^{2}C_{\sigma}}{\pc-4A_d}\pr{C+\frac{2+\pc}{\pc}\bar{C}}\sum_{l=0}^{k-2}\pr{1-A_{d}}^{k-l-1}\E\br{V_{l}}
		\eqsp.
	\end{multline}
	For any negative number $j<0$, using the convention that $\sum_{l=0}^{j}=0$ and simplifying the calculations provided by \eqref{eq:bound:sumoneminusAsigmadlsquared:specific:1}, we find that
	\begin{multline}\label{eq:bound:sumoneminusAsigmadlsquared:specific:2}
		\E\br{V_{k}}
		\le \pr{1-\frac{\pc}{4}}^{k}\E\br{V_{0}}
		+ \frac{4(1-\pc)\gamma^{2}\pr{1-A_{d}}^{k}}{\pc - 4 A_{d}}\pr{C+\frac{2+\pc}{\pc}\bar{C}}\E\br{\sigma_{0}^{2}}
		\\
		+ \frac{4(1-\pc)\gamma^{2}B_d B_\sigma \pr{S_1+\delta_{\alpha} S_{0}}\pr{1-\newrate}^{k}}{(\pc - 4\newrate)(1-A_d)(A_\sigma - A_d)(1+\delta_{\alpha})^3}
			\br{
				B+\frac{2+\pc}{\pc}\bar{B}
				+ \frac{B_{\sigma}}{A_\sigma - \newrate} \pr{C+\frac{2+\pc}{\pc}\bar{C}}
			 }
		\\
		+ \frac{4(1-\pc)\gamma^{2}D_{\sigma}}{\betaemptysquared A_{\sigma}(\pc-4A_{d})} \pr{C+\frac{2+\pc}{\pc}\bar{C}}
		+ \frac{4(1-\pc)\gamma^{2}}{\betaemptysquared \pc}\pr{D+\frac{2+\pc}{\pc}\bar{D}}
		+ \frac{8\pr{1-\tau}\pr{b-1}\gamma d}{b \pc}
		\\
		+ 4(1-\pc)\gamma^{2}
			\Bigg[
				\pr{\frac{1}{\pc} + \frac{B_d B_\sigma}{\newrate(\pc - 4\newrate)(A_\sigma - A_d)}}
					\pr{B+\frac{2+\pc}{\pc}\bar{B}} \\
				+ \frac{B_{\sigma}}{\pc}
					\pr{\frac{1}{A_\sigma} + \frac{B_d B_\sigma}{\newrate(A_\sigma - \newrate)(A_\sigma - A_d)}}
					\pr{C+\frac{2+\pc}{\pc}\bar{C}}
			\Bigg]
			\pr{\frac{D_{d}}{A_{d}} + \frac{B_{d} D_{\sigma}}{A_{d} A_{\sigma}}}
		\\
		+ \frac{4\gamma^{2}\pr{1-\pc}\pr{1-A_d}^{k}}{\pc - 4A_d}
			\br{
				B+\frac{2+\pc}{\pc}\bar{B}
				+ \frac{B_\sigma}{A_\sigma - A_d} \pr{C+\frac{2+\pc}{\pc}\bar{C}}
			}
			\pr{\E\br{\dist_{0}^{2}} + \frac{B_{d}}{A_{\sigma}-A_{d}} \E\br{\sigma_{0}^{2}}}
		\\
		+ \frac{4\gamma^{2}\pr{1-\pc}B_d B_\sigma \pr{1-\newrate}^{k}}{(\pc - 4\newrate)\pr{A_d - \newrate}(A_\sigma - A_d)}
			\br{
				B+\frac{2+\pc}{\pc}\bar{B}
				+ \frac{B_\sigma}{A_\sigma - \newrate} \pr{C+\frac{2+\pc}{\pc}\bar{C}}
			}
			\pr{\E\br{\dist_{0}^{2}} + \frac{B_{d}}{A_{\sigma}-A_{d}} \E\br{\sigma_{0}^{2}}}
		\\
		+ \frac{4\gamma^{2}\pr{1-\pc}}{\pc - 4A_d}
			\Bigg[
				C_{\sigma} \pr{C+\frac{2+\pc}{\pc}\bar{C}}
				+ \pr{C_{d} + \frac{B_{d}C_{\sigma}}{A_{\sigma} - A_{d}}}
					\pr{
					B+\frac{2+\pc}{\pc}\bar{B}
					+ \frac{B_\sigma}{A_\sigma - A_d} \pr{C+\frac{2+\pc}{\pc}\bar{C}}
					}
			\Bigg]\\
			\times\sum_{i=0}^{k-2}\pr{1-A_d}^{k-i-1} \E\br{V_{i}}
		\\
		+ \frac{4\gamma^{2}\pr{1-\pc}B_d B_\sigma}{(\pc - 4\newrate)\pr{A_d - \newrate}(A_\sigma - A_d)}
			\pr{C_{d} + \frac{B_{d}C_{\sigma}}{A_{\sigma} - A_{d}}}
			\br{
				B+\frac{2+\pc}{\pc}\bar{B}
				+ \frac{B_\sigma}{A_\sigma - \newrate} \pr{C+\frac{2+\pc}{\pc}\bar{C}}
			} \\
			\times\sum_{i=0}^{k-3}\pr{1-\newrate}^{k-i-1} \E\br{V_{i}}
		\eqsp.
	\end{multline}
	As explained in \eqref{eq:eq:deltaeq}, recall that
	\begin{align}
		&\delta_{\alpha}^{2} + \delta_{\alpha} = \frac{B_dB_\sigma}{\pr{1-A_d}\pr{A_\sigma-A_d}}\eqsp,
		&\newrate = A_d - \delta_{\alpha} (1-A_d)\eqsp.
	\end{align}
	Thus, when $B_d B_\sigma\neq 0$ then $\delta_{\alpha}\neq 0$, which implies that $A_d\neq \newrate$ and gives
	\begin{equation}
		\frac{B_dB_\sigma}{\pr{A_d-\newrate}\pr{A_\sigma-A_d}}
		= 1+\delta_{\alpha}\eqsp.
	\end{equation}
	In addition, in the proof of \Cref{lem:bound:alpha} we saw that $2\delta_{\alpha}\le A_d\le 1/2$ and also that $A_d/2\le \newrate \le A_d$. Therefore, we can regroup several terms in \eqref{eq:bound:sumoneminusAsigmadlsquared:specific:2} and write
	\begin{multline}\label{eq:bound:sumoneminusAsigmadlsquared:specific:3}
		\E\br{V_{k}}
		\le \pr{1-\frac{\pc}{4}}^{k}\E\br{V_{0}}
		+ \frac{4(1-\pc)\gamma^{2}\pr{1-A_{d}}^{k}}{\pc - 4 A_{d}}\pr{C+\frac{2+\pc}{\pc}\bar{C}}\E\br{\sigma_{0}^{2}}
		\\
		+ \frac{4(1-\pc)\gamma^{2}\delta_{\alpha} \pr{S_1+\delta_{\alpha} S_{0}}\pr{1-\newrate}^{k}}{(\pc - 4A_d)(1+\delta_{\alpha})^{2}}
			\br{
				B+\frac{2+\pc}{\pc}\bar{B}
				+ \frac{B_{\sigma}}{A_\sigma - A_d} \pr{C+\frac{2+\pc}{\pc}\bar{C}}
			 }
		\\	
		+ \frac{9\gamma^{2}\pr{1-\pc}\pr{1-\newrate}^{k}}{\pc - 4A_{d}}
			\br{
				B+\frac{2+\pc}{\pc}\bar{B}
				+ \frac{B_\sigma}{A_\sigma - A_d} \pr{C+\frac{2+\pc}{\pc}\bar{C}}
			}
			\pr{\E\br{\dist_{0}^{2}} + \frac{B_{d}}{A_{\sigma}-A_{d}} \E\br{\sigma_{0}^{2}}}
		\\
		+ \frac{4(1-\pc)\gamma^{2}D_{\sigma}}{\betaemptysquared A_{\sigma}(\pc-4A_{d})} \pr{C+\frac{2+\pc}{\pc}\bar{C}}
		+ \frac{4(1-\pc)\gamma^{2}}{\betaemptysquared \pc}\pr{D+\frac{2+\pc}{\pc}\bar{D}}
		+ \frac{8\pr{1-\tau}\pr{b-1}\gamma d}{b \pc}
		\\
		+ 4(1-\pc)\gamma^{2}
			\Bigg[
				\pr{\frac{1}{\pc} + \frac{2B_d B_\sigma}{A_d(\pc - 4A_d)(A_\sigma - A_d)}}
					\pr{B+\frac{2+\pc}{\pc}\bar{B}} \\
				+ \frac{B_{\sigma}}{\pc}
					\pr{\frac{1}{A_\sigma} + \frac{2B_d B_\sigma}{A_d(A_\sigma - A_d)^{2}}}
					\pr{C+\frac{2+\pc}{\pc}\bar{C}}
			\Bigg]
			\pr{\frac{D_{d}}{A_{d}} + \frac{B_{d} D_{\sigma}}{A_{d} A_{\sigma}}}
		\\
		+ \frac{9\gamma^{2}\pr{1-\pc}}{\pc - 4A_d}
			\Bigg[
				C_{\sigma} \pr{C+\frac{2+\pc}{\pc}\bar{C}}
				+ \pr{C_{d} + \frac{B_{d}C_{\sigma}}{A_{\sigma} - A_{d}}}
					\pr{
					B+\frac{2+\pc}{\pc}\bar{B}
					+ \frac{B_\sigma}{A_\sigma - A_d} \pr{C+\frac{2+\pc}{\pc}\bar{C}}
					}
			\Bigg]\\
			\times \sum_{i=0}^{k-2}\pr{1-\newrate}^{k-i-1} \E\br{V_{i}}
		\eqsp.
	\end{multline}
	Recall that we defined $\cte$ in \eqref{eq:def:cte} by
	\begin{equation}
		\cte = \frac{4(1-\pc)\gamma^{2}}{\pc-4A_d}\br{B+\frac{2+\pc}{\pc}\bar{B}+ \frac{B_{\sigma}}{A_\sigma - A_d} \pr{C+\frac{2+\pc}{\pc}\bar{C}}}\eqsp.
	\end{equation}
	Hence, using \eqref{eq:bound:sumoneminusAsigmadlsquared:specific:3} we get that
	\begin{multline}\label{eq:bound:sumoneminusAsigmadlsquared:specific:4}
		\E\br{V_{k}}
		\le \pr{1-\frac{\pc}{4}}^{k}\E\br{V_{0}}
		+ \frac{4(1-\pc)\gamma^{2}D_{\sigma}}{\betaemptysquared A_{\sigma}(\pc-4A_{d})} \pr{C+\frac{2+\pc}{\pc}\bar{C}}
		+ \frac{4(1-\pc)\gamma^{2}}{\betaemptysquared \pc}\pr{D+\frac{2+\pc}{\pc}\bar{D}}
		\\
		+ \frac{\cte}{A_d} \pr{1+ \frac{2B_d B_\sigma}{A_d(A_\sigma - A_d)}} \pr{D_{d} + \frac{B_{d} D_{\sigma}}{A_{\sigma}}}
		+ \frac{8\pr{1-\tau}\pr{b-1}\gamma d}{b \pc}
		\\
		+ \pr{\frac{4(1-\pc)\gamma^{2}\pr{1-A_{d}}^{k}}{\pc - 4 A_{d}}\pr{C+\frac{2+\pc}{\pc}\bar{C}} + \frac{9\cte B_{d} \pr{1-\newrate}^{k}}{4\pr{A_{\sigma}-A_{d}}}}
			\E\br{\sigma_{0}^{2}}
		\\
		+ \frac{9}{4}\cte \pr{1-\newrate}^{k} \E\br{\dist_{0}^{2}}
		+ \cte \pr{1-\newrate}^{k-2} \pr{A_d-\newrate} (1-A_d) \pr{S_1 + \frac{A_d-\newrate}{1-A_d}S_0}
		\\
		+ \br{
			\frac{9\gamma^{2}\pr{1-\pc}C_{\sigma}}{\pc - 4A_d}
				\pr{C+\frac{2+\pc}{\pc}\bar{C}}
				+ 3\cte\pr{C_{d} + \frac{B_{d}C_{\sigma}}{A_{\sigma} - A_{d}}}
			}
			\sum_{i=0}^{k-2}\pr{1-\newrate}^{k-i-1} \E\br{V_{i}}
		\eqsp.
	\end{multline}
	Finally, we conclude the proof remarking that
	\begin{multline}
		\cte\pr{1-\newrate}^{k-2}(A_d-\newrate)\br{(1-A_d)S_1 + (A_d-\newrate)S_0}\\
		\le \cte \pr{1-\newrate}^{k-2}(A_d-\newrate)\br{(2-A_d-\newrate) \E\br{d_0^2} + B_d \E\br{\sigma_0^2} + C_d\E\br{V_0} + D_d}\\
		\le \cte \pr{1-\newrate}^{k}\pr{4\E\br{d_0^2} + 2B_d \E\br{\sigma_0^2} + 2C_d\E\br{V_0} + 2D_d}\eqsp.
	\end{multline}
\end{proof}
In order to ease notation, with the definitions used in \Cref{ass:gradsto:g} and \eqref{eq:def:cte}, consider for any $\gamma\in\R_+$ the variable $\cteeps\in\R_+$ defined by
\begin{equation}\label{eq:def:cteeps}
	\cteeps = \ctev \E\br{V_0}
	+ \ctedzero \E\br{\dist_{0}^{2}}
	+ \ctesigma \E\br{\sigma_{0}^{2}}
	+ 2 D_d	
	% &\cterate = \frac{4(1-\pc)\gamma^{2}}{(p-4A_{d})(1-A_{d})}\br{B+\frac{2+\pc}{\pc}\bar{B}
	% + \pr{C+\frac{2+\pc}{\pc}\bar{C}}\pr{C_{\sigma}+\frac{B_{\sigma} C_{d}}{A_{\sigma}-A_{d}}}}\\
	% \nonumber
	% &\ctedelta = \frac{4(1-\pc)}{\pc}\br{\pr{B+\frac{2+\pc}{\pc}\bar{B}}\frac{D_{d}}{A_{d}}+\pr{C+\frac{2+\pc}{\pc}\bar{C}}\pr{\frac{D_{\sigma}}{A_{\sigma}}
	% + \frac{B_{\sigma} D_{d}}{A_{d} A_{\sigma}}} + D+\frac{2+\pc}{\pc}\bar{D}}\eqsp.
\end{equation}
In addition, with the previous notations consider
\begin{equation}\label{eq:def:delta}
	\delta = \frac{2\pr{1-\nofrac{A_d}{2}}^{-1}\cterate}{1+\sqrt{1+4\pr{1-\nofrac{A_d}{2}}^{-1}\cterate}}\eqsp
\end{equation}
and define
\begin{equation}\label{eq:def:gammaV}
	\gamma_{V}
	= \frac{\betaempty\pc^{1/2}}{(2-2\pc)^{1/2}\br{A+(1+2/\pc)\bar{A}}^{1/2}}\eqsp.
\end{equation}

% \bgroup\color{cobalt}

\begin{lemma}\label{lem:bound:Vk:expec}
	Assume \Cref{ass:dk:combination}, \Cref{ass:gradsto:g} hold with $4\cterate\le A_{d}<\min(A_{\sigma}/2, \pc/4), A_d A_\sigma\ge 8B_dB_\sigma$ and let $\gamma\in\ocint{0, \gamma_{V}}$.
	Then, for any $k\ge 1$, we have
	\begin{equation}
		\E\br{V_{k}}
		\le \pr{1-\frac{A_d}{4}}^{k} \pr{2\cteeps + \frac{4\cterate \ctedelta}{A_d}}
		+ \ctedelta
		\eqsp,
	\end{equation}
	where $V_{k}$ is defined in \eqref{eq:def:Vk}, $\cteeps,\cterate,\ctedelta$ in \eqref{eq:def:cte} and \eqref{eq:def:cteeps}.
\end{lemma}
\begin{proof}
	Let $k$ in $\N$ be fixed. Since the assumptions of \Cref{lem:bound:Vk} are satisfied, we know that
	\begin{equation}\label{eq:bound:Vk:1}
		\E\br{V_{k}}
		\le \pr{1-\alpha}^{k} \cteeps
		+ \cterate\sum_{l=0}^{k-2}\pr{1-\alpha}^{k-l-1}\E\br{V_{l}}
		+ \ctedelta
		\eqsp,
	\end{equation}
	where $\alpha$ is defined in \eqref{eq:def:newrate}. In addition, \Cref{lem:bound:alpha} shows that $A_d/2\le\alpha$. Hence, multiplying the last inequality by the weight $\omega_{k}$ defined for any $l\in\N$, by
	\begin{equation}\label{eq:def:omegak}
		\omega_{l} = \pr{1-\nofrac{A_d}{2}}^{-l}\eqsp,
	\end{equation}
	we obtain the following inequality
	\begin{equation}\label{eq:bound:wkEVk}
		\omega_{k}\E\br{V_{k}} \le \cteeps + \frac{\cterate}{1-\nofrac{A_d}{2}} \sum_{l=0}^{k-2}\omega_{l} \E\br{V_{l}} + \ctedelta \omega_{k}\eqsp.
	\end{equation}
	Applying the sharp Gr\"onwall inequality \citep{holte2009discrete}, we get
	\[
		\omega_{k}\E\br{V_{k}} \le \cteeps + \omega_k \ctedelta + \frac{\cterate}{1-\nofrac{A_d}{2}} \sum_{l=0}^{k-1} \pr{\cteeps + \omega_{l} \ctedelta}\pr{1+\frac{\cterate}{1-\nofrac{A_d}{2}}}^{k-l-1}\eqsp.
	\]
	Therefore, a calculation shows that
	\begin{equation}
		\omega_{k}\E\br{V_{k}}
		\le \cteeps + \omega_k \ctedelta + \cteeps \pr{1+\frac{\cterate}{1-\nofrac{A_d}{2}}}^{k}
		+ \frac{\cterate \ctedelta}{1-\nofrac{A_d}{2}} \sum_{l=0}^{k-1} \omega_{l} \pr{1+\frac{\cterate}{1-\nofrac{A_d}{2}}}^{k-l-1}
		\eqsp,
	\end{equation}
	and simplifying the previous inequality gives the following upper bound:
	\begin{equation}\label{eq:bound:expecVk:1}
		\E\br{V_{k}}
		\le \ctedelta + \omega_{k}^{-1}\cteeps + \cteeps \pr{1 - \frac{A_d}{2} + \cterate}^{k}
		+ {\cterate \ctedelta} \sum_{l=0}^{k-1} \pr{1 - \frac{A_d}{2} + \cterate}^{k-l-1}
		\eqsp.
	\end{equation}
	In addition, using $4\cterate<A_d<\pc/4$ implies $0<1 - \nofrac{A_d}{2} + \cterate<1$ which combined with \eqref{eq:bound:expecVk:1} gives
	\begin{equation}\label{eq:bound:expecVk:final}
		\E\br{V_{k}}
		\le \ctedelta + \omega_{k}^{-1}\cteeps + \cteeps \pr{1 - \frac{A_d}{2} + \cterate}^{k}
		+ \frac{\cterate \ctedelta}{A_d/2 - \cterate} \pr{1 - \frac{A_d}{2} + \cterate}^{k}
		\eqsp.
	\end{equation}
	Eventually, combining the last inequality with the assumption $4\cterate<A_d$ completes the proof.
\end{proof}

With the notation of the assumptions \Cref{ass:dk:combination} and \Cref{ass:gradsto:g}, we define
\begin{align}\label{def:eq:alphadsigma}
	\alpha_d = \frac{4\gamma^2}{\pc A_d}\max\ac{\pc B + 3\bar{B}, \frac{4B_\sigma}{A_\sigma}\pr{\pc C + 3\bar{C}}} \eqsp,&
	&\alpha_\sigma = \frac{4\gamma^2\pr{\pc C + 3\bar{C}}}{\pc A_\sigma}\eqsp.
\end{align}
The following lemma is used in the convergence proof of {\algoquatre} (see \Cref{lem:bound:Vk:expec:vrsaladstar}).

\begin{lemma}\label{lem:bound:Vk:new}
	Assume \Cref{ass:dk:combination}, \Cref{ass:gradsto:g} hold with
	\begin{align}
		A_d\le \min\pr{A_\sigma, \frac{\pc}{4}}\eqsp,&
		&\alpha_d C_d + \alpha_\sigma C_\sigma \le \frac{\pc}{8}\eqsp,&
		&{\alpha_d} B_d + {\gamma^{2}}\pr{C+\frac{3}{\pc}\bar{C}} \le \frac{\alpha_\sigma A_\sigma}{2}\eqsp,
	\end{align}
	and consider $\gamma\le \betaempty{\pc^{1/2}}{(2-2\pc)^{-1/2}\brn{A+(1+2/\pc)\bar{A}}^{-1/2}}$.
	Then, for any $k\in\N$, we have
	\begin{multline}\label{eq:Vk:bound:lyapunov3}
		\E\br{V_{k}}
		+ \alpha_d \E\br{\dist_{k}^2}
		+ \alpha_\sigma \E\br{\sigma_{k}^2}
		\le \pr{1-\frac{A_d}{2}}^k \pr{
			\E\br{V_{0}}
			+ \alpha_d \E\br{\dist_{0}^{2}}
			+ \alpha_\sigma \E\br{\sigma_{0}^{2}}
			}
		\\
		+ \frac{2(1-\pc)\gamma^{2}}{A_d}\pr{D+\frac{2+\pc}{\pc}\bar{D}}
		+ \frac{2\alpha_d D_d + 2\alpha_\sigma D_\sigma}{A_d}
		+ \frac{4\pr{1-\tau}\pr{b-1}\gamma d}{b A_d}
		\eqsp,
	\end{multline}
	where $V_{k}$ is defined in \eqref{eq:def:Vk}.
\end{lemma}
\begin{proof}
	Let $k\in\N^{\star}$, using for $i\in[b]$ the definitions \eqref{eq:def:Xki}, \eqref{eq:def:Xk} of $\Xlocal_{k}^{i}$ and $\Xavg_{k}$
	\begin{align}
		&\Xlocal_{k+1}^{i} = \Xlocal_{k}^{i} - \gamma G_{k}^{i} + \sqrt{2\gamma}\pr{\sqrt{\tau/b}\,\tilde{Z}_{k+1} + \sqrt{1-\tau}\,\tilde{Z}_{k+1}^{i}}\eqsp,\\
		&\Xavg_{k+1} = \Xavg_{k} -\frac{\gamma}{\betaempty b}\sum_{j=1}^{b} G_{k}^{i} + \sqrt{\frac{2\gamma\tau}{b}} \tilde{Z}_{k+1} + \frac{\sqrt{2(1-\tau)\gamma}}{b}\sum_{i=1}^{b} Z_{k+1}^{i}\eqsp.
	\end{align}
	Substracting the two above equations combined with the Jensen inequality give
	\begin{align}
		&\E\br{V_{k+1}}
		= \frac{1}{b}\sum_{i=1}^{b}\E\br{\normlr{\Xlocal_{k+1}^{i}-\Xavg_{k+1}}^{2}}\\
		&= \frac{1-\pc}{b}\sum_{i=1}^{b}\E\br{\normlr{(\Xlocal_{k}^{i}-\Xavg_{k}) - \gamma(G_{k}^{i}-G^{k})
			+ \sqrt{2(1-\tau)\gamma} Z_{k+1}^{i} - \frac{\sqrt{2(1-\tau)\gamma}}{b}\sum_{j=1}^{b} Z_{k+1}^{j}
			}^{2}}\\
		&= \frac{1-\pc}{b}\sum_{i=1}^{b}\E\br{\normlr{(\Xlocal_{k}^{i}-\Xavg_{k}) - \gamma(\bar{G}_{k}^{i}-\bar{G}^{k})}^{2}}
			+ \frac{(1-\pc)\gamma^{2}}{\betaemptysquared b}\sum_{i=1}^{b}\E\br{\normlr{(G_{k}^{i}-\bar{G}_{k}^{i}) - (G^{k}-\bar{G}^{k})}^{2}}\\
			&\qquad+ 2(1-\tau)\gamma\E\br{\normlr{Z_{k+1}^{i} - \frac{1}{b}\sum_{j=1}^{b} Z_{k+1}^{j}}^{2}}
	\end{align}
	Hence, we get 
	\begin{align}
			&\E\br{V_{k+1}} \le \frac{1-\pc}{b}\sum_{i=1}^{b}\E\br{\normlr{(\Xlocal_{k}^{i}-\Xavg_{k}) - \gamma(\bar{G}_{k}^{i}-\bar{G}^{k})}^{2}}
				+ \frac{(1-\pc)\gamma^{2}}{\betaemptysquared b}\sum_{i=1}^{b}\E\br{\normlr{G_{k}^{i}-\bar{G}_{k}^{i}}^{2}}\\
				&\qquad+ 2\pr{1-\tau}\pr{1-1 / b}\gamma d\\
			&\le \frac{(1-\pc)(1+\pc/2)}{b}\sum_{i=1}^{b}\E\br{\normlr{\Xlocal_{k}^{i}-\Xavg_{k}}^{2}}
				+ \frac{(1-\pc)\gamma^{2}}{\betaemptysquared b}\sum_{i=1}^{b}\E\br{\normlr{G_{k}^{i}-\bar{G}_{k}^{i}}^{2}}\\
				&\qquad+ \frac{(1-\pc)(1+2/\pc)\gamma^{2}}{\betaemptysquared b}\sum_{i=1}^{b}\E\br{\normlr{\bar{G}_{k}^{i}-\bar{G}^{k}}^{2}}
					+ 2\pr{1-\tau}\pr{1-1 / b}\gamma d \eqsp.
	\end{align}
	We finally obtain
	\begin{multline}
		\E\br{V_{k+1}} \le \pr{1-\pc/2}\E\br{V_{k}}
			+ \frac{(1-\pc)(2+\pc)\gamma^{2}}{\betaemptysquared \pc b}\sum_{i=1}^{b}\E\br{\normlr{\bar{G}_{k}^{i}}^{2}}\\
			+ \frac{(1-\pc)\gamma^{2}}{\betaemptysquared b}\sum_{i=1}^{b}\E\br{\normlr{G_{k}^{i}-\bar{G}_{k}^{i}}^{2}}
			+ 2\pr{1-\tau}\pr{1-\frac{1}{b}}\gamma d\eqsp.
	\end{multline}
	Combining the last inequality with \Cref{ass:gradsto:g} shows
	\begin{multline}
		\E\br{V_{k+1}}
		\le \pr{1-\frac{\pc}{2}+(1-\pc)\gamma^{2}\br{A+\frac{2+\pc}{\pc}\bar{A}}}\E\br{V_{k}}
		+ (1-\pc)\gamma^{2}\pr{D+\frac{2+\pc}{\pc}\bar{D}}\\
		+ (1-\pc)\gamma^{2}\pr{B+\frac{2+\pc}{\pc}\bar{B}}\E\br{\dist_{k}^{2}}
		+ (1-\pc)\gamma^{2}\pr{C+\frac{2+\pc}{\pc}\bar{C}}\E\br{\sigma_{k}^{2}}
		+ 2\pr{1-\tau}\pr{1-\frac{1}{b}}\gamma d\eqsp.
	\end{multline}
	Since $\gamma\le \frac{\betaempty\pc^{1/2}}{2(1-\pc)^{1/2}\br{A+(1+2/\pc)\bar{A}}^{1/2}}$, the above inequality implies that
	\begin{multline}
		\E\br{V_{k+1}}
		\le \pr{1-\frac{\pc}{4}}\E\br{V_{k}}
		+ (1-\pc)\gamma^{2}\pr{D+\frac{2+\pc}{\pc}\bar{D}}
		+ 2\pr{1-\tau}\pr{1-1 / b}\gamma d\\
		+ (1-\pc)\gamma^{2}\pr{B+\frac{2+\pc}{\pc}\bar{B}}\E\br{\dist_{k}^{2}}
		+ (1-\pc)\gamma^{2}\pr{C+\frac{2+\pc}{\pc}\bar{C}}\E\br{\sigma_{k}^{2}}\eqsp.
	\end{multline}
	The previous bound combined with \Cref{ass:dk:combination} gives that
	% \Cref{ass:dk:combination}, \eqref{eq:bound:ass:d}, \eqref{eq:bound:ass:sigma}
	% \begin{align}
	% 	&\E\br{\dist_{k+1}^{2}}
	% 	\le \pr{1-A_{d}}\E\br{\dist_{k}^{2}}
	% 	+ B_{d}\E\br{\sigma_{k}^{2}}
	% 	+ C_{d}\E\br{V_{k}}
	% 	+ D_{d}\eqsp,\\
	% 	&\E\br{\sigma_{k+1}^{2}}
	% 	\le \pr{1-A_{\sigma}}\E\br{\sigma_{k}^{2}}
	% 	+ B_{\sigma}\E\br{\dist_{k}^{2}}
	% 	+ C_{\sigma}\E\br{V_{k}}
	% 	+ D_{\sigma}\eqsp.
	% \end{align}
	\begin{multline}\label{eq:Vk:bound:lyapunov1}
		\E\br{V_{k+1}}
		+ \alpha_d \E\br{\dist_{k+1}^2}
		+ \alpha_\sigma \E\br{\sigma_{k+1}^2}
		\le \br{\pr{1-\frac{\pc}{4}}
			+ \alpha_d C_d
			+ \alpha_\sigma C_\sigma}
			\E\br{V_{k}}
		\\
		+ \br{\alpha_d (1-A_d)
			+ \alpha_\sigma B_\sigma
			+ (1-\pc)\gamma^{2}\pr{B+\frac{2+\pc}{\pc}\bar{B}}
			}
			\E\br{\dist_{k}^{2}}
		\\
		+ \br{\alpha_\sigma (1-A_\sigma)
			+ \alpha_d B_d
			+ (1-\pc)\gamma^{2}\pr{C+\frac{2+\pc}{\pc}\bar{C}}
			}
			\E\br{\sigma_{k}^{2}}
		\\
		+ (1-\pc)\gamma^{2}\pr{D+\frac{2+\pc}{\pc}\bar{D}}
			+ 2\pr{1-\tau}\frac{\pr{b-1}}{b}\gamma d
			+ \alpha_d D_d
			+ \alpha_\sigma D_\sigma
		\eqsp.
	\end{multline}
	By assumption, we have
	\begin{equation} \label{eq:bound:alphaconditions}
		\begin{aligned}
			&\alpha_d C_d
				+ \alpha_\sigma C_\sigma
			\le \frac{\pc}{8}\eqsp,
			\\
			% &{\alpha_\sigma} B_\sigma
			% 	+ {\gamma^{2}}\pr{B+\frac{3}{\pc}\bar{B}}
			% \le \frac{\alpha_d A_d}{2}\eqsp,
			% \\
			&{\alpha_d} B_d
				+ {\gamma^{2}}\pr{C+\frac{3}{\pc}\bar{C}}
			\le \frac{\alpha_\sigma A_\sigma}{2}
			\eqsp,
		\end{aligned}
	\end{equation}
	and by definition of $\alpha_d,\alpha_\sigma$ given in \eqref{def:eq:alphadsigma}, we know that ${\alpha_\sigma} B_\sigma + {\gamma^{2}}\prn{B+\nofrac{3\bar{B}}{\pc}} \le \nofrac{\alpha_d A_d}{2}$.
	In addition, since we suppose that $A_d\le \min(\pc/4, A_\sigma)$, the last inequalities combined with \eqref{eq:bound:alphaconditions} imply
	\begin{equation}\label{eq:bound:alphaconditions:2}
		\begin{aligned}
			&1-\frac{\pc}{4}
				+ \alpha_d C_d
				+ \alpha_\sigma C_\sigma
			\le 1 - \frac{A_d}{2}
			\\
			&1-A_d
				+ \frac{\alpha_\sigma}{\alpha_d} B_\sigma
				+ \frac{(1-\pc)\gamma^{2}}{\alpha_d}\pr{B+\frac{2+\pc}{\pc}\bar{B}}
			\le 1-\frac{A_d}{2}
			\\
			&1-A_\sigma
				+ \frac{\alpha_d}{\alpha_\sigma} B_d
				+ \frac{(1-\pc)\gamma^{2}}{\alpha_\sigma}\pr{C+\frac{2+\pc}{\pc}\bar{C}}
			\le 1-\frac{A_d}{2}
			\eqsp.
		\end{aligned}
	\end{equation}
	Thus, by taking up \eqref{eq:Vk:bound:lyapunov1} and using \eqref{eq:bound:alphaconditions:2}, we get
	\begin{multline}\label{eq:Vk:bound:lyapunov2}
		\E\br{V_{k+1}}
		+ \alpha_d \E\br{\dist_{k+1}^2}
		+ \alpha_\sigma \E\br{\sigma_{k+1}^2}
		\le \pr{1-\frac{A_d}{2}} \pr{
			\E\br{V_{k}}
			+ \alpha_d \E\br{\dist_{k}^{2}}
			+ \alpha_\sigma \E\br{\sigma_{k}^{2}}
			}
		\\
		+ (1-\pc)\gamma^{2}\pr{D+\frac{2+\pc}{\pc}\bar{D}}
			+ 2\pr{1-\tau}\pr{1-\frac{1}{b}}\gamma d
			+ \alpha_d D_d
			+ \alpha_\sigma D_\sigma
		\eqsp.
	\end{multline}
	Finally, the stated result follows by induction.
\end{proof}
\label{sec:transition}

\section{Main results}\label{sec:mainresults}

% !TEX root = main.tex

\Cref{sec:mainresults} is divided into four subsections in which we prove theoretical results for the {\algoun} and {\algoquatre} algorithms. These analyses are presented in \Cref{thm:bound:wass:atlernative:salad} and \Cref{thm:bound:wass:atlernative:vrsaladstar}.
The proofs are based on \Cref{lem:bound:Vk:expec} proved in \Cref{sec:Vk} to ensure that the local parameters $\{\Xlocal_k^{i}\}_{i\in[b]}$ do not deviate too much from $\Xavg_{k}$, then we apply the general result given in \Cref{sec:generalscheme} to obtain explicit upper bounds for $\wass\prn{\pi, \nug_{k}}$.

Until the end of the paper, we consider a family of independent random variables $(\xi^{i})_{i=1}^b$ distributed according to $\nu_{\xi}^{\otimes b}$, and we denote $(\gradsto^{i})_{i=1}^{b}$ a family of functions defined on $\R^d\times \mse\to \R^d$ such that for each $i\in[b], x\in\R^d$, $\gradsto^{i}(x,\xi^{i}(\cdot))$ is measurable on $(\mse,\mce)$ and satisfies the following condition:
\begin{assA}\label{ass:gradsto:lip}
	Assume there exists $\hat{L}\ge 0$, such that for any $i\in[b], x, y\in\R^d$, we have
	\begin{align}
		&\E\br{\gradsto^{i}(x,\xi^{i})} = \potential^{i}(x)\eqsp,\\
		&\E\br{\normlr{\gradsto^{i}(y,\xi^{i}) - \gradsto^{i}(x,\xi^{i})}^{2}}
		\le \hat{L}^{2}\normlr{y-x}^{2}\eqsp.
	\end{align}
\end{assA}
The assumption \Cref{ass:gradsto:lip} is equivalent to \Cref{main:ass:gradsto:lip} written in the main paper, though for clarity we prefer to replace the stochastic gradient $\maingradi$ by $\gradsto^{i}(\cdot,\xi^{i})$.
To simplify the notation, in what follows we consider the random variable $\xi = (\xi^1,\ldots,\xi^{b})$, and we denote
\begin{equation}\label{eq:def:H}
	\gradsto :
	\begin{cases}
		\R^d \times \mse^{b} \to \R^d\\
		\pr{x, z} \mapsto \sum_{i=1}^{b}\gradsto^{i}(x, z^{i})
	\end{cases}
	\eqsp.
\end{equation}
Thus, for each $x\in\R^d$, with this notation we have $\gradsto(x, \xi) = \sum_{i=1}^{b}\gradsto^{i}(x, \xi^{i})$. We also introduce the averaged versions $\barpotential, \bar{\gradsto}$ of the local potentials $\acn{\potential^{i}}_{i\in[b]}$ and the stochastic gradients $\acn{\gradsto^{i}}_{i\in[b]}$ defined by
\begin{align}\label{eq:def:barpotential}
	\barpotential(x) = \frac{1}{b}\sum_{i=1}^{b}\potential^{i}(x) \eqsp,&
	&\bar{\gradsto}(x, z) = \frac{1}{b}\sum_{i=1}^{b}\gradsto^{i}(x, z^{i})\eqsp.
\end{align}
\begin{remark}\label{rem:mini-batch:1}
	In the mini-batch scenario without replacement, the $i$th client draws a mini-batch $J_{i}\ subset [N_{i}]$ of size $n_{i}=|J_{i}|\in[N_{i}]$ among $N_{i}$ data and computes its stochastic gradient, which for $x\in\R^d$ is given by $\gradsto^{i}(x,\xi^{i})=\sum_{j\in J_{i}}\nabla\potential_{j}^{i}(x)$. Using \citet[Lemma S4]{vono2022qlsd} we know that
	\begin{align}
		\E\br{\normlr{\gradsto^{i}(y,\xi^{i}) - \gradsto^{i}(x,\xi^{i})}^{2}}
		&= \normlr{\nabla\potential^{i}(y)-\nabla\potential^{i}(x)}^{2}
		+ \var\pr{\gradsto^{i}(y,\xi^{i}) - \gradsto^{i}(x,\xi^{i})}\\
		&\le \pr{1 + \frac{n_{i}(N_{i}-n_{i})\max_{j=1}^{N_{i}}L_{j}^{i}}{N_{i}(N_{i}-1)L}} L^{2} \normlr{y-x}^{2}\eqsp.
	\end{align}
	Therefore, \Cref{ass:gradsto:lip} is satisfied for a choice of $\hat{L}>0$ such that
	\[
		\hat{L}\le L\sqrt{1+\max_{i=1}^{b}\ac{n_{i}(N_{i}-n_{i})\brn{N_{i}(N_{i}-1)}^{-1}\prn{\max_{j=1}^{N_{i}}L_{j}^{i}}L^{-1}}}\eqsp.
	\]
\end{remark}
\begin{assA}\label{ass:fij}
	For $i\in[b]$, $j\in[N_{i}]$, assume that $\potential_{j}^{i}$ is continuously differentiable, convex and there exists $L_{j}^{i} >0$ such that for any $x,y\in\R^d$,
	\begin{equation}
		\label{ass:fij:lip}
		\potential_{j}^{i}(y)
		\le \potential_{j}^{i}(x) + \ps{\nabla\potential_{j}^{i}(x)}{y-x} + \frac{L_{j}^{i}}{2} \normlr{y-x}^{2}\eqsp.
	\end{equation}
\end{assA}
\begin{assA}\label{ass:gradsto:behavior}
	Assume there exists $\constgradsto>0$ such that for any $x\in\R^d$,
	\[
		\E\br{\normlr{\gradsto(x,\xi) - \gradsto(x_{\star},\xi) - \nabla\potential(x)}^{2}} \le \constgradsto b^{2} \normlr{x - x_{\star}}^{2}\eqsp.
	\]
\end{assA}
\Cref{ass:fi} combined with \Cref{ass:gradsto:lip} implies \Cref{ass:gradsto:behavior} with $\constgradsto= 2L^2 + 2\hat{L}^2$. However, this new assumption~\Cref{ass:gradsto:behavior} is interesting because without stochastic gradient we obtain $\constgradsto=0$, which allows us to recover the classical Langevin bounds.
\begin{remark}\label{rem:mini-batch:2}
	Consider the same scenario as detailed in \Cref{rem:mini-batch:1} and define
	\begin{equation}
		\constgradsto
		= \pr{\sum_{i=1}^b\frac{n_{i}(N_{i}-n_{i})\max_{j=1}^{N_{i}}L_{j}^{i}}{b^{2} N_{i}(N_{i}-1)}}L\eqsp.
	\end{equation}
	Applying \citet[Lemma S4]{vono2022qlsd} we have the following lines
	\begin{multline}
		\E\br{\normlr{\bar{\gradsto}(x,\xi) - \bar{\gradsto}(x_{\star},\xi) - \nabla\barpotential(x)}^{2}}
		= \var\pr{\bar{\gradsto}(x,\xi) - \bar{\gradsto}(x_{\star},\xi)}\\
		=\frac{1}{b^{2}}\sum_{i=1}^b\var\pr{\gradsto^{i}(x,\xi^{i}) - \gradsto^{i}(x_{\star},\xi^{i})}
		\le \constgradsto \normlr{x-x_{\star}}^{2} \eqsp.
	\end{multline}
	Therefore, \Cref{ass:gradsto:behavior} is satisfied and in the deterministic case where all data are used to calculate the gradient, we have $\constgradsto=0$.
\end{remark}
To deal with variance reduction based algorithms, we consider the following assumption~\Cref{ass:gradsto:difflip}, which is also implied by \Cref{ass:fi}-\Cref{ass:gradsto:lip}, however the constant $\constsvrg$ vanishes with exact gradient computation.
\begin{assA}\label{ass:gradsto:difflip}
	Assume there exists $\constsvrg\ge 0$ such that for any $i\in[b]$ and $x, y\in\R^d$,
	\[
		\E\br{\normlr{\gradsto^{i}(x,\xi^{i}) - \gradsto^{i}(y,\xi^{i}) - \nabla\potential^{i}(x) + \nabla\potential^{i}(y)}^{2}}
		\le \constsvrg \normlr{x-y}^{2}\eqsp.
	\]
\end{assA}
\begin{remark}\label{rem:mini-batch:gradsto:difflip}
	In the mini-batch scenario without replacement detailed in \Cref{rem:mini-batch:1}, the use of \citet[Lemma S4]{vono2022qlsd} implies that
	\begin{multline}
		\E\br{\normlr{\gradsto^{i}(x,\xi^{i}) - \gradsto^{i}(y,\xi^{i}) - \nabla\potential^{i}(x) + \nabla\potential^{i}(y)}^{2}}
		= \var\pr{\gradsto^{i}(x,\xi^{i}) - \gradsto^{i}(y,\xi^{i})}\\
		\le \frac{n_{i}(N_{i}-n_{i})}{N_{i}(N_{i}-1)}L\max_{j=1}^{N_{i}}L_{j}^{i}\normlr{x-y}^{2}\eqsp.
	\end{multline}
	Thus, \Cref{ass:gradsto:difflip} is satisfied by setting
	\begin{equation}
		\constsvrg
		= \max_{i=1}^b\ac{\frac{n_{i}(N_{i}-n_{i})}{N_{i}(N_{i}-1)}\max_{j=1}^{N_{i}}L_{j}^{i}} L\eqsp.
	\end{equation}
	In the deterministic case, we obtain $\constsvrg=0$.
	Similarly, in the mini-batch scenario with replacement it is sufficient to set
	\begin{equation}
		\constsvrg
		= \frac{N_i - n_i}{n_i} \sum_{j=1}^{N_i}\pr{L_j^i}^2
	\end{equation}
	to ensure that \Cref{ass:gradsto:difflip} holds.
\end{remark}

% !TEX root = main.tex

\subsection{Study of \FALD{}}\label{subsec:salad}

\subsubsection{Remark on the theoretical analysis of \citet{deng2021convergence}}
\FALD{} has been proposed in \citet{deng2021convergence}, the authors develop an MCMC algorithm targeting the distribution proportional to $\exp(-b^{-1}\sum_{i=1}^b\potential^{i})$ and also establish non-asymptotic bounds.
They introduce \citep[Lemma B.2]{deng2021convergence} the stochastic processes $\{(\bar{\theta}_t^{i})_{t\ge 0}\}_{i\in[b]}$ satisfying  the Langevin stochastic differential equations for $t\ge 0$, $\rmd\bar{\theta}_t^{i}=-\nabla \potential^{i}(\bar{\theta}_t^{i})+\sqrt{2b}\,\rmd \mathsf{W}_t^{i}$ where $\{(\mathsf{W}_t^{i})_{t\ge 0}\}_{i\in[b]}$ are independent $d$-dimensional standard Brownian motion and define $\bar{\theta}_t = b^{-1}\sum_{i=1}^b \bar{\theta}_t^{i}$.
Then, it is asserted \citep[Lemma B.5]{deng2021convergence} that $(\bar{\theta}_t)$ is solution of the Langevin stochastic differential equation $\rmd\bar{\theta}_t=-b^{-1}\sum_{i=1}^b\nabla \potential^{i}(\bar{\theta}_t)+\sqrt{2}\,\rmd \mathsf{W}_t$, where $\mathsf{W}_t=b^{-1/2}\sum_{i=1}^b\mathsf{W}_t^{i}$.
However, this statement cannot hold in all generalities, and we give a counter-example. For instance, consider the Gaussian potentials $\{\potential^{i} : x\in\R^d\mapsto \Sigma_i^{-1}(x-\mathrm{m}^{i})\}_{i\in[b]}$ where $\{(\mathrm{m}^{i},\Sigma_i)\}_{i\in[b]}$ are the mean and the covariance parameters; if for $i\in[b]$, $\bar{\theta}_0^{i}$ is distributed according to $\exp(-\potential^{i})$, then $b^{-1}\sum_{i=1}^b\bar{\theta}_t^{i}$ follows $\gauss(b^{-1}\sum_{i=1}^{b}\mathrm{m}^{i}, b^{-2}\sum_{i=1}^{b}\Sigma_i)$ whereas $\exp(-b^{-1}\sum_{i=1}^b\potential^{i})$ corresponds to the density of the Gaussian $\gauss(\sum_{i=1}^{b}(\bar{\Sigma}\Sigma_i^{-1})\mathrm{m}^{i}, b\bar{\Sigma})$ where $\bar{\Sigma}=(\sum_{i=1}^{b}\Sigma_i^{-1})^{-1}$. Therefore, for any $t \ge 0$, in this case $\bar{\theta}_t$ is distributed according to $\gauss(b^{-1}\sum_{i=1}^{b}\mathrm{m}^{i}, b^{-2}\sum_{i=1}^{b}\Sigma_i)$ and thus cannot be distributed according to $\exp(-b^{-1}\sum_{i=1}^b\potential^{i})$ as crucially used in the proof of \citet[Lemma B.5]{deng2021convergence}.

\subsubsection{Theoretical analysis}
In this section, we prove the first theoretical guarantee on {\algoun} stated in \Cref{main:thm:bound:wass:atlernative:algoun}. Similar to \citet{mcmahan2017communication}, the clients update their local parameters $\acn{\Xlocal_{k}^{i}}_{i\in[b]}$ several times before transmitting them to the server with probability $\pc\in\ocint{0,1}$. Then, the server aggregates the local parameters to update its own parameter $\Xavg_{k}$ as in \eqref{eq:def:Xk}.
For all $i\in[b], k\in\N$, consider the stochastic gradients defined by
\begin{align}
	\label{eq:def:gik:salad}
	&G_{k}^{i}
	= \gradsto^{i}(\Xlocal_{k}^{i},\xi_{k+1}^{i})\eqsp,\\
	\label{eq:def:bargik:salad}
	&\bar{G}_{k}^{i}
	= \nabla\potential^{i}(\Xlocal_{k}^{i})\eqsp.
\end{align}
\begin{algorithm}[] % H = here
	\caption{Stochastic Averaging Langevin Dynamics - \algoun}
	\label{algo:fedavgLangevin:salad}
   \begin{algorithmic}
		\State {\bfseries Input:} initial vectors $(\Xavg_{0}^{i})_{i\in[b]}$, noise parameter $\tau\in\ccint{0,1}$, number of communication rounds $K$, probability $\pc$ of communication, step-size $\gamma$.
		\For{$k=0$ {\bfseries to} $K-1$}
			\State \Comment{On each client}
			\State Draw $B_{k+1}\sim\mathcal{B}(\pc), \tilde{Z}_{k+1}\sim \gauss(0_{d},\mathrm{I}_{d})$
			\State \Comment{In parallel on the $b$ clients}
			\For{$i=1$ {\bfseries to} $b$}
				\State Draw $\xi_{k+1}^{i}\sim\nu_{\xi}$ and $\tilde{Z}_{k+1}^{i}\sim \gauss(0_{d},\mathrm{I}_{d})$
				\State Compute $G_{k}^{i} = \gradsto^{i}(\Xlocal_{k}^{i}, \xi_{k+1}^{i})$
				\State Set $\Xupdate_{k+1}^{i} = \Xlocal_{k}^{i} - \gamma G_{k}^{i} + \sqrt{2\gamma}\,\prn{\sqrt{\tau/b}\,\tilde{Z}_{k+1} + \sqrt{1-\tau}\,\tilde{Z}_{k+1}^{i}}$
				\If{$B_{k+1}=1$}
				\State Broadcast $\Xupdate_{k+1}^{i}$ to the server
				\Else
				\State Update $\Xlocal_{k+1}^{i} \gets \Xupdate_{k+1}^{i}$
				\EndIf
			\EndFor
			\If{$B_{k+1}=1$}
				\State \Comment{On the central server}
				\State Update then broadcast the global parameter
						$
							\Xavg_{k+1} = \frac{1}{b}\sum_{i=1}^{b}\Xupdate_{k+1}^{i}
						$
				\State \Comment{On each client}
				\State Update the local parameter $\Xlocal_{k+1}^{i} \gets \Xavg_{k+1}$
			\EndIf
	    \EndFor
	    \State {\bfseries Output:} samples $\{\Xavg_{\ell}\}_{\{\ell\in[K]\,:\, B_{\ell}=1\}}$.
   \end{algorithmic}
\end{algorithm}
\begin{lemma}\label{lem:bound:gexplicit:salad}
	Assume \Cref{ass:fi}, \Cref{ass:gradsto:lip} and \Cref{ass:gradsto:behavior} hold. Then for any $k\in\N$, we have
	\begin{align}
		&\frac{1}{b}\sum_{i=1}^{b}\E\br{\normn{\bar{G}_{k}^{i}}^{2}}
		\le 3L^{2} \E\br{V_{k}} + 3L^{2} \E\br{\dist_{k}^{2}} + \frac{3}{b}\sum_{i=1}^{b}\normn{\nabla\potential^{i}(x_{\star})}^{2} \eqsp,\\
		&\frac{1}{b}\sum_{i=1}^{b}\E\br{\normn{G_{k}^{i} - \bar{G}_{k}^{i}}^{2}}
		\le 3\hat{L}^{2}\E\br{V_{k}} + 3\constgradsto\E\br{\dist_{k}^{2}} + 3\E\br{\normlr{\bar{\gradsto}(x_{\star}, \xi)}^{2}} \eqsp.
	\end{align}
	For any $i\in [b], k\in\N$, recall the stochastic gradients $G_{k}^{i}, \bar{G}_{k}^{i}$ are defined in \eqref{eq:def:gik:salad} and \eqref{eq:def:bargik:salad}, respectively
\end{lemma}
\begin{proof}
	Using the Young inequality combined with the Lipschitz property \Cref{ass:fi} of the gradients $(\potential^{i})_{i}^{b}$, for $k\ge 0$ we get
	\begin{align}
		\frac{1}{b}\sum_{i=1}^{b}\E\br{\normn{\bar{G}_{k}^{i}}^{2}}
		&= \frac{1}{b}\sum_{i=1}^{b}\E\br{\normn{\nabla\potential^{i}(\Xlocal_{k}^{i}) - \nabla\potential^{i}(\Xavg_{k}) + \nabla\potential^{i}(\Xavg_{k}) - \nabla\potential^{i}(x_{\star}) + \nabla\potential^{i}(x_{\star})}^{2}}\\
		&\le 3L^{2} \E\br{V_{k}} + 3L^{2} \E\br{\dist_{k}^{2}} + \frac{3}{b}\sum_{i=1}^{b}\normn{\nabla\potential^{i}(x_{\star})}^{2}\eqsp.
	\end{align}
	In addition, since the random variables $(G_{k}^{i} - \bar{G}_{k}^{i})_{i=1}^{b}$ are centered and independent, the Young and the Jensen inequality imply that
	\begin{align}
		\frac{1}{b}\sum_{i=1}^{b}\E\br{\normlr{G_{k}^{i} - \bar{G_{k}^{i}}}^{2}}
		&= \E\br{\normlr{\frac{1}{b}\sum_{i=1}^{b}\pr{G_{k}^{i} - \bar{G_{k}^{i}}}}^{2}} \\
		&=\E\Bigg[\bigg\|
			\frac{1}{b}\sum_{i=1}^{b}\gradsto^{i}(\Xlocal_{k}^{i},\xi_{k+1}^{i}) - \bar{\gradsto}(\Xavg_{k}, \xi_{k+1}) + \bar{\gradsto}(\Xavg_{k}, \xi_{k+1})
			- \bar{\gradsto}(x_{\star},\xi_{k+1}) \\
			&\qquad+ \bar{\gradsto}(x_{\star}, \xi_{k+1})
			- \nabla\barpotential(\Xavg_{k}) + \nabla\barpotential(\Xavg_{k})
			- \frac{1}{b}\sum_{i=1}^{b}\nabla\potential^{i}(\Xlocal_{k}^{i})
			\bigg\|^{2}\Bigg]\\
		&\le 3\E\br{\normlr{\frac{1}{b}\sum_{i=1}^{b}\gradsto^{i}(\Xlocal_{k}^{i},\xi_{k+1}^{i}) - \bar{\gradsto}(\Xavg_{k}, \xi_{k+1})}^{2}}\\
			&\qquad+ 3 \E\br{\normlr{\bar{\gradsto}(\Xavg_{k}, \xi_{k+1}) - \nabla\barpotential(\Xavg_{k}) - \bar{\gradsto}(x_{\star}, \xi_{k+1})}^{2}}
				+ 3 \E\br{\normlr{\bar{\gradsto}(x_{\star}, \xi)}^{2}}\\
		&\le 3\hat{L}^{2} \E\br{V_{k}} + 3\constgradsto \E\br{\dist_{k}^{2}} + 3\E\br{\normlr{\bar{\gradsto}(x_{\star}, \xi)}^{2}} \eqsp.
	\end{align}
	% Hence, using \Cref{ass:gradsto:var:salad} gives the result.
	% \begin{align}
	% 	&\frac{1}{b}\sum_{i=1}^{b}\E\br{\normlr{G_{k}^{i} - \bar{G}_{k}^{i}}^{2}}
	% 	=\E\Bigg[\bigg\|
	% 		\frac{1}{b}\sum_{i=1}^{b}\gradsto^{i}(\Xlocal_{k}^{i},\xi_{k+1}^{i}) - \bar{\gradsto}(\Xavg_{k}, \xi_{k+1}) + \bar{\gradsto}(\Xavg_{k}, \xi_{k+1}) \\
	% 		&\qquad- \bar{\gradsto}(x_{\star},\xi_{k+1}) + \bar{\gradsto}(x_{\star}, \xi_{k+1})
	% 		- \nabla\barpotential(\Xavg_{k}) + \nabla\barpotential(\Xavg_{k}) - \frac{1}{b}\sum_{i=1}^{b}\nabla\potential^{i}(\Xlocal_{k}^{i})
	% 		\bigg\|^{2}\Bigg]\\
	% 	&\le 3\E\br{\normlr{\frac{1}{b}\sum_{i=1}^{b}\gradsto^{i}(\Xlocal_{k}^{i},\xi_{k+1}^{i}) - \bar{\gradsto}(\Xavg_{k}, \xi_{k+1})}^{2}}
	% 		+ 3 \E\br{\normlr{\bar{\gradsto}(x_{\star}, \xi)}^{2}}\\
	% 		&\qquad+ 3\E\br{\normlr{\bar{\gradsto}(\Xavg_{k}, \xi_{k+1}) - \bar{\gradsto}(x_{\star}, \xi_{k+1})}^{2}}\\
	% 	&\le 3\hat{L}^{2}\E\br{V_{k}} + 3\hat{L}^{2}\E\br{\dist_{k}^{2}} + 3\E\br{\normlr{\bar{\gradsto}(x_{\star}, \xi)}^{2}} \eqsp.
	% \end{align}
\end{proof}
\begin{lemma}\label{lem:bound:dk:salad}
	Assume \Cref{ass:fi} and \Cref{ass:gradsto:lip} hold.
	Then, for any $\gamma\in(0, \betaempty\conv(6\hat{L}^{2})^{-1}]$, we have
	\[
		\E\br{\dist_{k+1}^{2}}
		\le \pr{1 - \frac{\gamma\conv}{2\betaempty}}\E\br{\dist_{k}^{2}}
		+ \frac{2\gamma L^{2}}{\betaempty\conv}\E\br{V_{k}}
		+ 3\gamma^{2}\E\br{\normn{\bar{\gradsto}(x_{\star},\xi)}^{2}}
		+ \frac{2 \gamma d}{b} \eqsp,
	\]
	where $V_{k}, \dist_{k}$ are defined in \eqref{eq:def:Vk} and \eqref{eq:def:dk}.
\end{lemma}
\begin{proof}
	Let $k$ be in $\N$. Rewriting the expression of $\Xavg_{k+1}$ defined in \eqref{eq:def:Xk}, we obtain
	\begin{align}
		\nonumber
		\E\br{\dist_{k+1}^{2}}
		&= \E\br{\normlr{\Xavg_{k+1}-x_{\star}}^{2}}\\
		\nonumber
		&= \E\br{\normlr{\Xavg_{k}-x_{\star}-\frac{\gamma}{\betaempty b}\sum_{i=1}^{b}\gradsto^{i}(\Xlocal_{k}^{i},\xi_{k+1}^{i})
		+ \sqrt{2\gamma}\pr{\sqrt{\frac{\tau}{b}}\,\tilde{Z}_{k+1} + \frac{\sqrt{1-\tau}}{b}\sum_{i=1}^{b} Z_{k+1}^{i}}}^{2}}\\
		\nonumber
		&= \E\br{\normlr{\Xavg_{k}-x_{\star}}^{2}}- 2\gamma \E\br{\ps{\Xavg_{k}-x_{\star}}{\frac{1}{\betaempty b}\sum_{i=1}^{b}\gradsto^{i}(\Xlocal_{k}^{i},\xi_{k+1}^{i})}} \\
		\label{eq:eq:boundthetak}
		&\qquad+ \gamma^{2}\E\br{\normlr{\frac{1}{\betaempty b}\sum_{i=1}^{b}\gradsto^{i}(\Xlocal_{k}^{i},\xi_{k+1}^{i})}^{2}} + \frac{2 \gamma d}{b}\eqsp.
	\end{align}
	Further, the Young inequality combined with \Cref{ass:gradsto:lip} give
	\begin{align}
		\nonumber
		\E\br{\normlr{\frac{1}{b}\sum_{i=1}^{b}\gradsto^{i}(\Xlocal_{k}^{i},\xi_{k+1}^{i})}^{2}}
		&\le \frac{3}{b}\sum_{i=1}^{b}\E\br{\normlr{\gradsto^{i}(\Xlocal_{k}^{i},\xi_{k+1}^{i})-\gradsto^{i}(\Xavg_{k},\xi_{k+1}^{i})}^{2}}
		+ 3\E\br{\normn{\bar{\gradsto}(x_{\star},\xi)}^{2}}\\
		\nonumber
		&\qquad+ 3\E\br{\normlr{\bar{\gradsto}(\Xavg_{k},\xi_{k+1})-\bar{\gradsto}(x_{\star},\xi)}^{2}}\\
		\label{eq:eq:boundthetak:2}
		&\le 3\hat{L}^{2}\E\br{V_{k}} + 3\hat{L}^{2}\E\br{\dist_{k}^{2}} + 3\E\br{\normn{\bar{\gradsto}(x_{\star},\xi)}^{2}}\eqsp.
	\end{align}
	In addition, using the fact that for any vectors $a,b\in\R^d$, $2\abs{\ps{a}{b}}\le \conv\normlr{a}^{2}+\normlr{b}^{2}/\conv$ we can upper bound the inner product derived in \eqref{eq:eq:boundthetak} as follows
	\begin{align}
		\nonumber
		- \E\br{\ps{\Xavg_{k}-x_{\star}}{\frac{1}{b}\sum_{i=1}^{b}\gradsto^{i}(\Xlocal_{k}^{i},\xi_{k+1}^{i})}}
		&= - \E\br{\ps{\Xavg_{k}-x_{\star}}{\nabla\barpotential(\Xavg_{k})}} \\
		\nonumber
		&+ \E\br{\ps{\Xavg_{k}-x_{\star}}{\frac{1}{b}\sum_{i=1}^{b}\br{\gradsto^{i}(\Xavg_{k},\xi_{k+1}^{i})-\gradsto^{i}(\Xlocal_{k}^{i},\xi_{k+1}^{i})}}}\\
		\nonumber
		&\le - \E\br{\ps{\Xavg_{k}-x_{\star}}{\nabla\barpotential(\Xavg_{k})}}
		+ \conv \E\br{\dist_{k}^{2}} / 2 + L^{2}\E\br{V_{k}} / (2\conv)\\
		\label{eq:eq:boundthetak:3}
		&\le - \conv \E\br{\dist_{k}^{2}}/2 + L^{2}\E\br{V_{k}} / (2\conv)\eqsp.
	\end{align}
	Therefore, plugging \eqref{eq:eq:boundthetak:2} and \eqref{eq:eq:boundthetak:3} in \eqref{eq:eq:boundthetak} shows
	\begin{equation}
		\E\br{\dist_{k+1}^{2}}
		\le \pr{1- \gamma\br{\conv-3\gamma\hat{L}^{2}}} \E\br{\dist_{k}^{2}}
		+ \gamma\pr{3\gamma\hat{L}^{2} + \frac{L^{2}}{\conv}}\E\br{V_{k}}
		+ 3\gamma^{2}\E\br{\normlr{\bar{\gradsto}(x_{\star},\xi)}^{2}}
		+ \frac{2 \gamma d}{b} \eqsp.
	\end{equation}
	Eventually, the assumption $\gamma\le \betaempty\conv(6\hat{L}^{2})^{-1}$ completes the proof.
\end{proof}
\begin{comment}
\begin{lemma}\label{lem:contraction:xkbis:transition:salad}
	Assume \Cref{ass:fi}, \Cref{ass:gradsto:lip} hold and let $\gamma\le\betaempty(6L)^{-1}$.
	Then, for any $k\in\N$, we have
	\begin{multline}
		\E\br{\normn{\Xcontinuous_{k+1} - \Xavg_{k+1}}^{2}}
		= \pr{1 - \frac{\gamma\conv}{2\betaempty}} \E\br{\normn{\Xcontinuous_{k\gamma} - \Xavg_{k}}^{2}}
		+ 3\gamma\pr{\frac{L^{2}}{\conv} + \gamma\hat{L}^{2}} \E\br{V_{k}}\\
		+ 3\gamma^{2}\hat{L}^{2} \E\br{\dist_{k}^{2}}
		+ \pr{\frac{2}{\betaempty\gamma\conv}\E\br{\normlr{\E^{\mathcal{F}_{k}}\br{I_{k}}}^{2}} + 3\E\br{\normlr{I_{k}}^{2}}}
		+ 3\gamma^{2} \E\br{\normlr{\bar{\gradsto}(x_{\star}, \xi)}^{2}}\eqsp. \\
	\end{multline}
	where $V_{k}, \dist_{k}, \mathcal{F}_{k}$ are defined in \eqref{eq:def:Vk}, \eqref{eq:def:dk} and \eqref{eq:def:Fk}.
\end{lemma}
%
\begin{proof}
	\Cref{ass:gunbiased} is satisfied since for any $i\in[b], x\in\R^d$ the stochastic gradient $\gradsto^{i}(x,\xi^{i})$ is unbiased. In addition, \Cref{ass:gradsto:var:salad} shows that \Cref{ass:gradsto:var} holds with $c_{d}=3\constgradsto, c_{\sigma}=0, c_{V}=3\hat{L}^{2},c=3\E\brn{\normn{\bar{\gradsto}(x_{\star}, \xi)}^{2}}$.
	Therefore, for any $k\in\N$, applying \Cref{lem:contraction:xkbis} completes the proof.
\end{proof}
\end{comment}
%
For any $\gamma\in(0, \betaempty\conv(6\hat{L}^{2})^{-1}]$, under \Cref{ass:fi}, \Cref{ass:gradsto:lip} and \Cref{ass:gradsto:behavior} using \Cref{lem:bound:gexplicit:salad} and \Cref{lem:bound:dk:salad} we have shown that \Cref{ass:dk:combination} and \Cref{ass:gradsto:g} hold with the following quantities
\begin{equation}\label{eq:def:variables:salad}
	\begin{aligned}
		&A = 3\hat{L}^{2}\eqsp,&
		&B = 3\constgradsto\eqsp,&
		&C = 0\eqsp,&
		&D = 3\E\br{\normlr{\bar{\gradsto}(x_{\star}, \xi)}^{2}}\eqsp,\\
		&\bar{A} = 3L^{2}\eqsp,&
		&\bar{B} = 3L^{2}\eqsp,&
		&\bar{C} = 0\eqsp,&
		&\bar{D} = \textstyle(\nofrac{3}{b})\sum_{i=1}^{b}\normn{\nabla\potential^{i}(x_{\star})}^{2}\eqsp,\\
		&A_{d} = \nofrac{\gamma\conv}{2\betaempty}\eqsp,&
		&B_{d} = 0 \eqsp,&
		&C_{d} = \nofrac{2\gamma L^{2}}{\betaempty\conv}\eqsp,&
		&D_{d} = 3\gamma^{2}\E\br{\normlr{\bar{\gradsto}(x_{\star},\xi)}^{2}} + \nofrac{2 \gamma d}{b}\eqsp,&\\
		&A_{\sigma} = 1\eqsp,&
		&B_{\sigma} = 0\eqsp,&
		&C_{\sigma} = 0\eqsp,&
		&D_{\sigma} = 0\eqsp.
	\end{aligned}
\end{equation}
For any $\gamma>0$, consider the following variables
\begin{equation}\label{eq:def:cte:salad}
	\begin{aligned}
		&
		\begin{aligned}
			\cte = \frac{4(1-\pc)\gamma^{2}}{\pc-4A_d}\pr{B+\frac{2+\pc}{\pc}\bar{B}}\eqsp,&
			&\cterate = 3\cte C_{d} \eqsp,&
			&\ctev = 1 + 2 C_d\cte\eqsp,
		\end{aligned}
		\\
		&
		\begin{aligned}
			\cteeps = \ctev \E\br{V_0}
				+ 7\cte \E\br{\dist_{0}^{2}}
				+ 2 D_d	\eqsp,&
			&\ctedelta = \frac{4(1-\pc)\gamma^{2}}{\pc}\pr{D+\frac{2+\pc}{\pc}\bar{D}}
				+ \frac{\cte D_{d}}{A_d}
				+ \frac{8\pr{1-\tau}\pr{b-1}\gamma d}{b \pc} \eqsp.
		\end{aligned}
	\end{aligned}
\end{equation}
% \bgroup\color{cobalt}
We also introduce $\gamma_{1}$ and $I_{\gamma}$, which are defined for any $\gamma>0$ by
\begin{align}
	\label{eq:def:overlinegamma:salad}
	&\gamma_{1} = \frac{\betaempty \pc^{1/2}}{{(2-2\pc)}^{1/2}\br{A+(1+2/\pc)\bar{A}}^{1/2}} \wedge \frac{\betaempty\conv}{6\hat{L}^{2}} \wedge \frac{\betaempty\pc}{2\conv} \wedge \frac{\qc}{\conv} \eqsp,\\
	&I_{\gamma} = \ac{\gamma\in(0,\gamma_{1}) \,:\, \gamma\conv\ge {8\cterate}}\eqsp.
\end{align}
Based on \Cref{lem:bound:Vk:expec}, we derive the following result.
\begin{lemma}\label{lem:bound:Vk:expec:salad}
	Assume \Cref{ass:fi}, \Cref{ass:gradsto:lip} and \Cref{ass:gradsto:behavior} hold.
	Then, for any $\gamma\in I_{\gamma}$ and $k\ge 1$, we have
	\begin{equation}
		\E\br{V_{k}}
		\le \pr{1-\frac{A_d}{4}}^{k} \pr{2\cteeps
			+ \frac{4\ctedelta\cterate}{A_d}
			}
		+ \ctedelta
		\eqsp.
	\end{equation}
	where $V_{k}$ is defined in \eqref{eq:def:Vk} and $\cteeps,\cterate,\ctedelta$ in \eqref{eq:def:cte:salad}.
\end{lemma}
\begin{proof}
	For any $\gamma\in I_{\gamma}$, we have $4\cterate\le A_{d}$ and moreover it is easy to check that $A_{d}<\min(A_{\sigma}/2,\pc/4)$, $A_d A_\sigma\ge 8B_dB_\sigma=0$.
	In addition, since \Cref{ass:fi}, \Cref{ass:gradsto:lip} and \Cref{ass:gradsto:behavior} are satisfied we can apply \Cref{lem:bound:gexplicit:salad} and \Cref{lem:bound:dk:salad} which show that \Cref{ass:dk:combination}, \Cref{ass:gradsto:g} hold with the variables introduced in \eqref{eq:def:variables:salad}.
	Therefore, we can use \Cref{lem:bound:Vk:expec} to complete the proof.
\end{proof}
Based on the results presented in this section, we can rewrite the upper bound on $(\E\br{V_{k}})_{k\in\N}$ given in \Cref{lem:bound:Vk:expec:salad} into the format of \Cref{ass:vk}. We consider for $\gamma >0$,
\begin{align}
	% \label{eq:def:alpha:salad}
	% &\alpha_{v} = 1-\nofrac{A_{d}}{4}\eqsp,\\
	\label{eq:def:v1:salad}
	&v_1 = 2\cteeps + \frac{4\ctedelta\cterate}{A_d}\eqsp,&
	% \label{eq:def:v2:salad}
	&v_{2} = \ctedelta\eqsp.
\end{align}
\begin{lemma}\label{lem:salad-ass-sup-general-case}
	Assume \Cref{ass:fi}, \Cref{ass:gunbiased}, \Cref{ass:gradsto:lip} hold and let $\gamma\le 2(3L)^{-1}$.
	Then for any $k\in\N$, we have
	\begin{multline}
		\E\br{\normn{\Xcontinuous_{(k+1)\gamma} - \Xavg_{k+1}}^{2}}
		\le \br{1 - \gamma\conv\pr{1-3\gamma L} + 3\gamma^2\hat{L}^2} \E\br{\normn{\Xcontinuous_{k\gamma} - \Xavg_{k}}^{2}}
		+ \gamma\pr{\frac{2 L^{2}}{\conv} + 3\gamma (L^{2} + \hat{L}^2)} \E\br{V_{k}}\\
		+ \pr{\frac{2}{\betaempty\gamma\conv}\E\br{\normlr{\E^{\mathcal{F}_{k}}\br{I_{k}}}^{2}} + 3\E\br{\normlr{I_{k}}^{2}}}
		+ \frac{3\gamma^{2}}{b^2} \int_{\R^d} \var^{\mathcal{F}_{0}}\pr{\gradsto\prn{x, \xi}} \pi(\rmd x) \eqsp.
	\end{multline}
\end{lemma}
\begin{proof}
	For any $k\in\N$, recall that $\mathcal{F}_{k}$ is defined in \eqref{eq:def:Fk} and using \Cref{lem:contraction:xkbis} we obtain
	\begin{multline}\label{eq:bound:lem:salad:maj}
		\E^{\mathcal{F}_{k}}\br{\normn{\Xcontinuous_{(k+1)\gamma} - \Xavg_{k+1}}^{2}}
		\le \br{1 - \gamma\conv\pr{1-3\gamma L}} \normn{\Xcontinuous_{k\gamma} - \Xavg_{k}}^{2}
		+ \gamma\pr{\frac{2 L^{2}}{\conv} + 3\gamma L^{2}} V_{k}\\
		+ \pr{\frac{2}{\betaempty\gamma\conv}\normlr{\E^{\mathcal{F}_{k}}\br{I_{k}}}^{2} + 3\E^{\mathcal{F}_{k}}\br{\normlr{I_{k}}^{2}}}
		+ \gamma^{2} \var^{\mathcal{F}_{k}}\pr{\frac{1}{b}\sum_{i=1}^{b} G_{k}^{i}} \eqsp.
	\end{multline}
	Since the stochastic gradients $(\gradsto^{i}(\cdot,\xi_{k+1}^{i}))_{i=1}^{b}$ are unbiased, \Cref{ass:gradsto:lip} with the Young inequality imply that
	\begin{align}
		&\var^{\mathcal{F}_{k}}\pr{\frac{1}{b}\sum_{i=1}^{b}G_{k}^{i}}
		= \E^{\mathcal{F}_{k}}\br{\normlr{\frac{1}{b}\sum_{i=1}^b\br{\gradsto^{i}(\Xlocal_{k}^{i},\xi_{k+1}^{i}) - \nabla\potential^{i}(\Xlocal_{k}^{i})}}^{2}} \\
		&= \E^{\mathcal{F}_{k}}\Bigg[\Bigg\|
			\frac{1}{b}\sum_{i=1}^b\gradsto^{i}(\Xlocal_{k}^{i},\xi_{k+1}^{i}) - \bar{\gradsto}(\Xavg_{k},\xi_{k+1})
			- \frac{1}{b}\sum_{i=1}^b\nabla\potential^{i}(\Xlocal_{k}^{i}) + \nabla\barpotential(\Xavg_{k})
			\\
			&+ \bar{\gradsto}(\Xavg_{k},\xi_{k+1}) - \bar{\gradsto}(\Xcontinuous_{k\gamma},\xi_{k+1}) - \nabla\barpotential(\Xavg_{k}) + \nabla\barpotential(\Xcontinuous_{k\gamma})
			+ \bar{\gradsto}(\Xcontinuous_{k\gamma}^{i},\xi_{k+1}) - \nabla\barpotential(\Xcontinuous_{k\gamma})
			\Bigg\|^2\Bigg]
		\\
		\label{eq:bound:diffXkYki:sgld:1}
		&\le 3\hat{L}^2 V_{k} + 3\hat{L}^2 \normn{\Xavg_{k} - \Xcontinuous_{k\gamma}}^2 + 3 \var^{\mathcal{F}_{k}}\pr{\bar{\gradsto}\prn{\Xcontinuous_{k\gamma}, \xi_{k+1}}}\eqsp.
	\end{align}
	Taking the expectation and using that $\Xcontinuous_{k\gamma}$ has distribution $\pi$ combined with \eqref{eq:bound:lem:salad:maj} complete the proof.
\end{proof}
For notational convenience, we also introduce the time step-size $\gamma_2$ defined by
\begin{equation}
	\gamma_{2}
	= \frac{\pc}{4\conv} \wedge \frac{1}{6(L+\hat{L}^2/\conv)} \wedge \frac{\pc\conv}{38(1-\pc)^{1/2}\pr{\pc\constgradsto + 3L^{2}}^{1/2} L} \eqsp.
\end{equation}
\begin{theorem}\label{thm:bound:wass:atlernative:salad}
	Assume \Cref{ass:fi}, \Cref{ass:gradsto:lip} and \Cref{ass:gradsto:behavior} hold and let $\gamma\in(0, \gamma_1\wedge\gamma_2)$.
	Then, for any initial probability measure $\mug_{0}\in\mathcal{P}_{2}(\Rd)$, $k\in\N$, we have
	\begin{multline}\label{eq:bound:wass:atlernative:salad}
		\wass^{2}\pr{\mug_k, \pi}
		\le \pr{1-\frac{\gamma\conv}{2}}^{k} \wass^{2}\pr{\mug_0, \pi}
		+ \frac{8 L^2}{\conv^2} v_1\pr{1-\frac{\gamma\conv}{8}}^{k}
		+ \frac{6 L^2}{\conv^2}v_{2}
		+ \frac{6\gamma d}{b\conv^2} \kappa_{I}
		\\
		+ \frac{6 \gamma}{b^2\conv} \int_{\Rd}\var^{\mathcal{F}_{0}}\pr{\gradsto\pr{x, \xi_{1}}} \pi(\rmd x)
		\eqsp.
	\end{multline}
	where $v_1$, $v_{2}$ are defined in \eqref{eq:def:v1:salad} and $\kappa_{I} = L^{2}(1+\gamma L^2/\conv)$. If in addition we suppose \Cref{ass:fi:ctrois}, set $\kappa_{I}=2\gamma\prn{L^3 + \nofrac{d \tilde{L}^{2}}{b}}$.
\end{theorem}
\begin{proof}
	We know that \Cref{ass:gunbiased} is satisfied since for any $i\in[b], x\in\R^d$ the stochastic gradient $\gradsto^{i}(x,\xi^{i}_1)$ is unbiased. The constraint $\gamma\le \gamma_{1}$ combined with \Cref{lem:bound:dk:salad} implies \Cref{ass:dk:combination} and plugging the expression of $A_{d}, A_{\sigma}, B_{d}, C, \bar{C}, C_{d}, C_{\sigma}$ provided in \eqref{eq:def:variables:salad} into $\cterate$ defined in \eqref{eq:def:cte:salad} gives that
	\[
		\cterate = \frac{72\gamma^{3}(1-\pc) L^{2} \pr{\constgradsto + \pr{1+\nofrac{2}{\pc}}L^{2}}}{(\pc - 2\gamma\conv)\conv} \eqsp.
	\]
	For any $\gamma\in\ocint{0,\gamma_{2}}$, we have $(\pc - 2\gamma\conv)\conv^{2} \ge 576(1-\pc)\gamma^{2} L^2 \pr{\constgradsto + \pr{1+\nofrac{2}{\pc}}L^{2}}$ which shows that $\gamma\in I_{\gamma}$.
	Thus, we can apply \Cref{lem:bound:Vk:expec:salad} which proves that \Cref{ass:vk} holds with $\qc=\gamma\conv$ and
 	$
		\alpha_{v} = 1-\nofrac{A_{d}}{4}
	$
	and $v_1, v_2$ defined in \eqref{eq:def:v1:salad}.
	Since the assumptions of \Cref{lem:salad-ass-sup-general-case} are satisfied, \Cref{ass:contraction:general:xk} holds, and therefore we can apply \Cref{thm:bound:wasserstein:general:vrsalad} with
	\begin{align}
		&
		\begin{aligned}
			(1-\qc)\alpha_0=1 - \gamma\conv\pr{1-3\gamma L} + 3\gamma^2\hat{L}^2\eqsp,&
			&\alpha_1=0\eqsp,&
			&(1-\qc)\alpha_2=\gamma\pr{\frac{2 L^{2}}{\conv} + 3\gamma (L^{2} + \hat{L}^2)}\eqsp,&
			&\alpha_3=0\eqsp,
		\end{aligned}
		\\
		&(1-\qc)\alpha_4=\pr{\frac{2}{\betaempty\gamma\conv}\E\br{\normlr{\E^{\mathcal{F}_{k}}\br{I_{k}}}^{2}} + 3\E\br{\normlr{I_{k}}^{2}}}
		+ \frac{3\gamma^{2}}{b^2} \int_{\R^d} \var^{\mathcal{F}_{0}}\pr{\gradsto\prn{x, \xi_{1}}} \pi(\rmd x)\eqsp.
	\end{align}
	Furthermore, using \Cref{lem:bound:Ik:unified} we have
	\begin{equation}\label{eq:bound:wass:atlernative:salad:4}
		\frac{2}{\betaempty\gamma\conv}\E\br{\norm{\E^{\mathcal{F}_{k}}\br{I_{k}}}^{2}} + 3\E\br{\norm{I_{k}}^{2}}
		\le \frac{3\gamma^2 d L^{2}}{b\conv} \pr{1 + \frac{19\gamma L^{2}}{36\conv}}\eqsp.
	\end{equation}
	Moreover, if we suppose \Cref{ass:fi:ctrois}, we obtain
	\begin{equation}\label{eq:bound:wass:atlernative:salad:5}
		\frac{2}{\betaempty\gamma\conv}\E\br{\norm{\E^{\mathcal{F}_{k}}\br{I_{k}}}^{2}} + 3\E\br{\norm{I_{k}}^{2}}
		\le \frac{\gamma^3d}{b\conv}\pr{5 L^3 + \frac{4d \tilde{L}^{2}}{3b}}\eqsp.
	\end{equation}
	Finally, with the notation of \Cref{thm:bound:wasserstein:general:vrsalad} we obtain $1+\delta=0$, and using $\gamma\le \prn{6(L+\conv^{-1}\hat{L}^2)}^{-1}$ combined with \eqref{eq:bound:wass:atlernative:salad:4} or \eqref{eq:bound:wass:atlernative:salad:5} if we suppose \Cref{ass:fi:ctrois} give the expected result.
\end{proof}
Now, consider the time stepsizes $\gamma_3$ and $\gamma_\star$ defined by
\begin{align}
	&\gamma_3 = \frac{\pc \conv}{3L^2 + \pc \constgradsto}\eqsp,&
	&\gamma_\star = \gamma_1 \wedge \gamma_2 \wedge \gamma_3\eqsp.
\end{align}
From the previous result, the next corollary controls the asymptotic bias obtained by \Cref{algo:fedavgLangevin:salad}.
\begin{corollary}\label{cor:wass:alternativebis:salad}
	Assume \Cref{ass:fi}, \Cref{ass:gradsto:lip} and \Cref{ass:gradsto:behavior} hold and let $\gamma\in(0,\gamma_\star)$, $\tau=1$.
	Then, for any initial probability measure $\mug_{0}\in\mathcal{P}_{2}(\Rd)$, $k\in\N$, we have
	\begin{multline}
		\frac{6^{-4} b}{\gamma d}\limsup_{k\to\infty}\wass^{2}\pr{\mug_k,\pi}
		\le \frac{\int_{\Rd}\var^{\mathcal{F}_{0}}\pr{\gradsto\pr{x, \xi_{1}}} \pi(\rmd x)}{b d \conv}
		+ \frac{\tilde{\kappa}_{I}}{\conv^2}
		\\
		+ \frac{(1-\pc) \gamma L^{2}}{\pc^2 \conv^{2}} \pr{
			\frac{1}{d} \sum_{i=1}^b\normlr{\nabla\potential^{i}(x_{\star})}^{2}
			+ \frac{\pc}{b d} \E\br{\normlr{\gradsto(x_{\star},\xi)}^{2}}
			+ \frac{L^{2} + \pc\constgradsto}{\conv}
			}
		\eqsp.
	\end{multline}
	where $\tilde{\kappa}_{I} = L^{2}$ and if we suppose \Cref{ass:fi:ctrois}, $\tilde{\kappa}_{I}=\gamma\prn{L^3 + \nofrac{d \tilde{L}^{2}}{b}}$.
\end{corollary}
\begin{proof}
	Using \Cref{thm:bound:wass:atlernative:salad} combined with $\gamma\le \gamma_1\wedge\gamma_2$ gives that
	\begin{equation}\label{eq:wass:alternativebis:sgld:1}
		\limsup_{k\to\infty}\wass^{2}\pr{\mug_k,\pi}
		\le \frac{6 \gamma}{b^2\conv} \int_{\Rd}\var^{\mathcal{F}_{0}}\pr{\gradsto\pr{x, \xi_{1}}} \pi(\rmd x)
		+ \frac{6\gamma d}{b\conv^2} \kappa_{I}
		+ \frac{6 L^2}{\conv^2}v_{2}
		\eqsp.
	\end{equation}
	Further, recall that $A_d, B, \bar{B}, D,\bar{D}, D_d$ are provided in \eqref{eq:def:variables:salad} and $\ctedelta$ is defined in \eqref{eq:def:cte:salad} by
	\begin{align}
		\ctedelta
		&= \frac{4(1-\pc)\gamma^{2}}{\pc}\pr{D+\frac{2+\pc}{\pc}\bar{D}}
			+ \frac{\cte D_{d}}{A_d}
			+ \frac{8\pr{1-\tau}\pr{b-1}\gamma d}{b \pc} \\
		&\le \frac{12(1-\pc)\gamma^{2}}{\pc} \br{1 + \frac{12\gamma}{\conv}\pr{\constgradsto+\frac{3}{\pc}L^{2}}} \E\br{\normlr{\bar{\gradsto}(x_{\star},\xi)}^{2}}
			+ \frac{8\pr{1-\tau}\pr{b-1}\gamma d}{b \pc}
			\\
		&\qquad+ \frac{36(1-\pc)\gamma^{2}}{\betaemptysquared \pc^{2} b} \sum_{i=1}^b\normlr{\nabla\potential^{i}(x_{\star})}^{2}
			+ \frac{96(1-\pc)\gamma^{2} d}{\pc b\conv}\pr{\constgradsto+\frac{3}{\pc}L^{2}}
		\\
		&\le \frac{156(1-\pc)\gamma^{2}}{\pc} \E\br{\normlr{\bar{\gradsto}(x_{\star},\xi)}^{2}}
			+ \frac{36(1-\pc)\gamma^{2}}{\betaemptysquared \pc^{2} b} \sum_{i=1}^b\normlr{\nabla\potential^{i}(x_{\star})}^{2}
			\\
		\label{eq:wass:alternativebis:sgld:22}
		&\qquad
			+ \frac{96(1-\pc)\gamma^{2} d}{\pc b\conv}\pr{\constgradsto+\frac{3}{\pc}L^{2}}
			+ \frac{8\pr{1-\tau}\pr{b-1}\gamma d}{b \pc}
			\eqsp.
	\end{align}
	Finally, setting $\tau=1$ combined with \eqref{eq:wass:alternativebis:sgld:1} and \eqref{eq:wass:alternativebis:sgld:22} show that
	\begin{multline}\label{eq:wass:alternativebis:sgld:2}
		\limsup_{k\to\infty}\wass^{2}\pr{\mug_k,\pi}
		\le \frac{6 \gamma}{b^2\conv} \int_{\Rd}\var^{\mathcal{F}_{0}}\pr{\gradsto\pr{x, \xi_{1}}} \pi(\rmd x)
		+ \frac{6\gamma d}{b\conv^2} \kappa_{I}
		\\
		+ \frac{8 (1-\pc) \gamma^2 L^{2}}{b \pc \conv^{2}} \br{
			\frac{156}{b} \E\br{\normlr{\gradsto(x_{\star},\xi)}^{2}}
			+ \frac{36}{\pc} \sum_{i=1}^b\normlr{\nabla\potential^{i}(x_{\star})}^{2}
			+ \frac{96 d}{\conv}\pr{\constgradsto+\frac{3}{\pc}L^{2}}
			}
		\eqsp.
	\end{multline}
\end{proof}

\subsection{Study of \VRFALDs{}}\label{subsec:vrsaladstar}

In this alternative of {\algoun} derived in \Cref{subsec:salad}, we introduce control variates to cope with both heterogeneity and variance in local gradients. Instead of using $\gradsto^{i}(\Xlocal_{k}^{i})$ to update the local parameter $\Xlocal_{k}^{i}$, this time the $i$th client uses the proxy $\gradsto^{i}(\Xlocal_{k}^{i},\xi_{k+1}^{i}) - \gradsto^{i}(Y_{k},\xi_{k+1}^{i}) + \nabla\potential_i(Y_k)$ based on an analog of the SVRG algorithm \citep{johnson2013accelerating,scaffold20} and where $Y_{k}$ is a global reference point updated with probability $\qc\in\ocint{0,1}$. We derive an explicit upper bound on the Wasserstein distance between the distribution of the server parameter $\Xcontinuous_{k\gamma}$ and the target distribution $\pi$. We also show how this new global control variate mitigates the effect of heterogeneity in the convergence rate.
To do so, we consider the stochastic gradients defined for any $i\in[b], k\in\N$, by
\begin{align}
	\label{eq:def:gik:vrsaladstar}
	&G_{k}^{i}
	= \gradsto^{i}(\Xlocal_{k}^{i}, \xi_{k+1}^{i}) - \gradsto^{i}(Y_{k}, \xi_{k+1}^{i}) + C_{k}\eqsp,\\
	\label{eq:def:bargik:vrsaladstar}
	&\bar{G}_{k}^{i}
	= \nabla\potential^{i}(\Xlocal_{k}^{i}) - \nabla\potential^{i}(Y_{k}) + C_{k}\eqsp
\end{align}
and denote
\begin{equation}\label{eq:def:sigmak:vrsaladstar}
	\sigma_{k}
	= \pr{\frac{1}{b}\sum_{i=1}^{b}\E^{\mathcal{F}_{k}}\br{\normlr{\gradsto^{i}(Y_{k},\xi_{k+1}^{i}) - \gradsto^{i}(x_{\star},\xi_{k+1}^{i})}^{2}}}^{1/2}\eqsp.
\end{equation}

\begin{algorithm}
	\caption{\algoquatre}
	\label{algo:fedavgLangevin:vrsaladstar}
   \begin{algorithmic}
		\State {\bfseries Input:} initial vectors $(\Xavg_{0}^{i})_{i\in[b]}$, noise parameter $\tau\in\ccint{0,1}$, number of communication rounds $K$, probability $\pc$ of communication, probability $\qc$ to update the control variates, step-size $\gamma$ and batch size $r$.
		\State Initialize $Y_{0}=(1/b)\sum_{i=1}^b \Xavg_{0}^{i}$ and $C_{0}=(1/b)\nabla\potential(Y_{0})$
		\For{$k=0$ {\bfseries to} $K-1$}
			\State \Comment{On each client}
			\State Draw $B_{k+1}\sim\mathcal{B}(\pc), \tilde{Z}_{k+1}\sim \gauss(0_{d},\mathrm{I}_{d})$
			\State \Comment{In parallel on the $b$ clients}
			\For{$i=1$ {\bfseries to} $b$}
				\State Draw $\xi_{k+1}^{i}\sim \nu_{\xi}$, $\tilde{Z}_{k+1}^{i}\sim \gauss(0_{d},\mathrm{I}_{d})$
				\State Compute $G_{k}^{i} = \gradsto^{i}(\Xlocal_{k}^{i}, \xi_{k+1}^{i}) - \gradsto^{i}(Y_{k}, \xi_{k+1}^{i}) + C_{k}$
				\State Set
					$
						\Xupdate_{k+1}^{i} = \Xlocal_{k}^{i} - \gamma G_{k}^{i} + \sqrt{2\gamma}\,\prn{\sqrt{\tau/b}\,\tilde{Z}_{k+1} + \sqrt{1-\tau}\,\tilde{Z}_{k+1}^{i}}
					$
				\If{$B_{k+1}=1$}
					\State Broadcast $\Xupdate_{k+1}^{i}$ to the server
				\Else
					\State Update $\Xlocal_{k+1}^{i} \gets \Xupdate_{k+1}^{i}$
				\EndIf
				\If{$\tilde{B}_{k+1}=1$}
					\State Broadcast $\Xlocal_{k}^{i}$ to the server
				\Else
					\State Update $Y_{k+1} \gets Y_{k}$ and $C_{k+1} \gets C_{k}$
				\EndIf
			\EndFor
			\If{$B_{k+1}=1$}
				\State \Comment{On the central server}
				\State Update then broadcast the global parameter
						$
							\Xavg_{k+1} \gets (\nofrac{1}{b})\sum_{i=1}^{b}\Xupdate_{k+1}^{i}
						$
				\State \Comment{On each client}
				\State Update the local parameter $\Xlocal_{k+1}^{i} \gets \Xavg_{k+1} $
			\EndIf
			\If{$\tilde{B}_{k+1}=1$}
				\State \Comment{On the central server}
				\State Update then broadcast $Y_{k+1} \gets (\nofrac{1}{b})\sum_{i=1}^{b} \Xlocal_{k}^{i}$ \hfill
				\State \Comment{On each client}
				\State Compute and broadcast $\nabla\potential^{i}(Y_{k+1})$ \hfill
				\State \Comment{On the central server}
				\State Update then broadcast $C_{k+1} \gets (1/b)\nabla\potential(Y_{k+1})$ \hfill
			\EndIf
	    \EndFor
	    \State {\bfseries Output:} samples $\{\Xavg_{\ell}\}_{\{\ell\in[K]\,:\, B_{\ell}=1\}}$.
   \end{algorithmic}
\end{algorithm}
\begin{lemma}\label{lem:bound:gexplicit:vrsaladstar}
	Assume \Cref{ass:fi}, \Cref{ass:gradsto:lip} and \Cref{ass:gradsto:behavior} hold. Then for any $k\in\N$, we have
	\begin{align}
		&\frac{1}{b}\sum_{i=1}^{b}\E\br{\normn{\bar{G}_{k}^{i}}^{2}}
		\le 3L^{2}\E\br{V_{k}} + 3L^{2}\E\br{\dist_{k}^{2}} + 3\E\br{\sigma_{k}^{2}} \eqsp,\\
		&\frac{1}{b}\sum_{i=1}^{b}\E\br{\normn{G_{k}^{i} - \bar{G}_{k}^{i}}^{2}}
		\le 3\hat{L}^{2} \E\br{V_{k}} + 3\constgradsto\E\br{\dist_{k}^{2}} + 3\E\br{\sigma_{k}^{2}}\eqsp.
	\end{align}
	For any $i\in [b], k\in\N$, recall the stochastic gradients $G_{k}^{i}, \bar{G}_{k}^{i}$ are defined in \eqref{eq:def:gik:vrsaladstar} and \eqref{eq:def:bargik:vrsaladstar}, respectively
\end{lemma}
\begin{proof}
	For $k\ge 0$, Lipschitz property of $\{\nabla\potential^{i}\}_{i\in[b]}$ supposed in \Cref{ass:fi} gives that
	\begin{align}
		\frac{1}{b}\sum_{i=1}^{b}\E\br{\normn{\bar{G}_{k}^{i}}^{2}}
		&= \frac{1}{b}\sum_{i=1}^{b}\E\br{\normn{\nabla\potential^{i}(\Xlocal_{k}^{i}) - \nabla\potential^{i}(Y_{k}) + \nabla\barpotential(Y_{k})}^{2}}\\
		&\le \frac{3}{b}\sum_{i=1}^{b}\E\br{\normn{\nabla\potential^{i}(\Xlocal_{k}^{i}) - \nabla\potential^{i}(\Xavg_{k})}^{2}}
		+ \frac{3}{b}\sum_{i=1}^{b}\E\br{\normn{\nabla\potential^{i}(Y_{k}) - \nabla\potential^{i}(x_{\star})}^{2}}\\
		&\qquad+ \frac{3}{b}\sum_{i=1}^{b}\E\br{\normn{\nabla\potential^{i}(\Xavg_{k}) - \nabla\potential^{i}(x_{\star})}^{2}}\\
		&\le 3L^{2}\E\br{V_{k}} + 3L^{2}\E\br{\dist_{k}^{2}} + 3\E\br{\sigma_{k}^{2}}\eqsp
	\end{align}
	and the proof is concluded by noting that \Cref{ass:gradsto:lip} gives
	\begin{align}
		\frac{1}{b}\sum_{i=1}^{b}\E\normn{G_{k}^{i} - \bar{G_{k}^{i}}}^{2}
		&=\E\br{\var^{\mathcal{F}_{k}}\pr{\frac{1}{b}\sum_{i=1}^{b}G_{k}^{i}}}\\
		&\le \E\br{\normlr{\frac{1}{b}\sum_{i=1}^{b}\gradsto^{i}(\Xlocal_{k}^{i},\xi_{k+1}^{i}) - \bar{\gradsto}(\Xavg_{k},\xi_{k+1})}^{2}}\\
		&\le 3\E\br{\normlr{\frac{1}{b}\sum_{i=1}^{b}\gradsto^{i}(\Xlocal_{k}^{i},\xi_{k+1}^{i}) - \bar{\gradsto}(\Xavg_{k},\xi_{k+1})}^{2}}
			+ 3\E\br{\normlr{\bar{\gradsto}(Y_{k},\xi_{k+1}) - \bar{\gradsto}(x_{\star},\xi_{k+1})}^{2}}\\
			&\qquad+ 3\E\br{\normlr{\bar{\gradsto}(\Xavg_{k},\xi_{k+1}) - \bar{\gradsto}(x_{\star},\xi_{k+1}) - \nabla\barpotential(\Xavg_{k})}^{2}}\eqsp.
	\end{align}
\end{proof}
%
% The result of the following lemma is similar to that of \Cref{lem:bound:dk:vrsalad}, but the control variate is now $\nabla\barpotential(Y_{k})$ instead of $b^{-1}\sum_{i=1}^{b}\nabla\potential^{i}(Y_{k}^{i})$ and therefore the expression of $\sigma_{k}$ is different.
%
\begin{lemma}\label{lem:bound:dk:vrsaladstar}
	Assume \Cref{ass:fi} and \Cref{ass:gradsto:lip} hold.
	Then, for any $\gamma\in(0, \betaempty\conv(6\hat{L}^{2})^{-1}]$, we have
	\begin{equation}
		\E\br{\dist_{k+1}^{2}}
		\le \pr{1-\frac{\gamma\conv}{2\betaempty}} \E\br{\dist_{k}^{2}}
		+ \frac{2\gamma L^{2}}{\betaempty\conv} \E\br{V_{k}}
		+ 4\gamma^{2}\E\br{\sigma_{k}^{2}}
		+ 10\gamma^{2}\E\br{\normn{\bar{\gradsto}(x_{\star},\xi)}^{2}}
		+ \frac{2 \gamma d}{b}\eqsp,
	\end{equation}
	where $V_{k}, \dist_{k}, \sigma_{k}$ are defined in \eqref{eq:def:Vk}, \eqref{eq:def:dk} and \eqref{eq:def:sigmak:vrsaladstar}.
\end{lemma}
\begin{proof}
	Let $k$ be in $\N$. Writing the expression of $\Xavg_{k+1}$ defined in \eqref{eq:def:Xk} and developing the expectation of the squared norm give
	\begin{align}
		\nonumber
		&\E\br{\dist_{k+1}^{2}}
		= \E\br{\normlr{\Xavg_{k+1}-x_{\star}}^{2}}\\
		\nonumber
		&= \E\Bigg[\bigg\|
			\Xavg_{k}-x_{\star}-\frac{\gamma}{\betaempty b}\sum_{i=1}^{b}\gradsto^{i}(\Xlocal_{k}^{i},\xi_{k+1}^{i}) + \gamma\bar{\gradsto}(Y_{k},\xi_{k+1}) - \gamma\nabla\barpotential(Y_{k})
			+ \sqrt{2\gamma}\pr{\sqrt{\frac{\tau}{b}}\,\tilde{Z}_{k+1} + \frac{\sqrt{1-\tau}}{b}\sum_{i=1}^{b} Z_{k+1}^{i}}
			\bigg\|^{2}\Bigg]\\
		\nonumber
		&= \E\br{\normlr{\Xavg_{k}-x_{\star}}^{2}}- 2\gamma \E\br{\ps{\Xavg_{k}-x_{\star}}{\frac{1}{\betaempty b}\sum_{i=1}^{b}\gradsto^{i}(\Xlocal_{k}^{i},\xi_{k+1}^{i})}}
		+ \gamma^{2}\E\br{\normlr{\frac{1}{b}\sum_{i=1}^{b}\gradsto^{i}(\Xlocal_{k}^{i},\xi_{k+1}^{i})}^{2}}\\
		\nonumber
		&\qquad- 2\gamma^{2}\E\br{\ps{\frac{1}{b}\sum_{i=1}\gradsto^{i}(\Xlocal_{k}^{i},\xi_{k+1}^{i})}{\bar{\gradsto}(Y_{k},\xi_{k+1}) - \gamma\nabla\barpotential(Y_{k})}}
		+ \gamma^{2}\E\br{\normlr{\bar{\gradsto}(Y_{k},\xi_{k+1}) - \nabla\barpotential(Y_{k})}^{2}}
		+ \frac{2 \gamma d}{b}\\
		\nonumber
		&= \E\br{\dist_{k}^{2}}- 2\gamma \E\br{\ps{\Xavg_{k}-x_{\star}}{\frac{1}{b}\sum_{i=1}^{b}\nabla\potential^{i}(\Xlocal_{k}^{i})}}
		+ 2\gamma^{2}\E\br{\normlr{\frac{1}{b}\sum_{i=1}^{b}\gradsto^{i}(\Xlocal_{k}^{i},\xi_{k+1}^{i})}^{2}}\\
		\label{eq:eq:bound:dk:vrsaladstar}
		&\qquad+ 2\gamma^{2}\E\br{\normlr{\bar{\gradsto}(Y_{k},\xi_{k+1}) - \nabla\barpotential(Y_{k})}^{2}}
		+ \frac{2 \gamma d}{b}\eqsp.
	\end{align}
	Using the Young inequality combined with \Cref{ass:gradsto:lip} show
	\begin{align}
		\nonumber
		\E\br{\normlr{\frac{1}{b}\sum_{i=1}^{b}\gradsto^{i}(\Xlocal_{k}^{i},\xi_{k+1}^{i})}^{2}}
		&\le \frac{3}{b}\sum_{i=1}^{b}\E\br{\normlr{\gradsto^{i}(\Xlocal_{k}^{i},\xi_{k+1}^{i})-\gradsto^{i}(\Xavg_{k},\xi_{k+1}^{i})}^{2}}\\
		\nonumber
		&\qquad+ 3\E\br{\normlr{\bar{\gradsto}(\Xavg_{k},\xi_{k+1})-\bar{\gradsto}(x_{\star},\xi)}^{2}} + 3\E\br{\normn{\bar{\gradsto}(x_{\star},\xi)}^{2}}\\
		\label{eq:eq:bound:dk:1:vrsaladstar}
		&\le 3\hat{L}^{2}\E\br{V_{k}} + 3\hat{L}^{2}\E\br{\dist_{k}^{2}} + 3\E\br{\normn{\bar{\gradsto}(x_{\star},\xi)}^{2}}\eqsp.
	\end{align}
	We also have that
	\begin{align}
		\nonumber
		\E\br{\normlr{\bar{\gradsto}(Y_{k},\xi_{k+1}) - \nabla\barpotential(Y_{k})}^{2}}
		&\le 2\E\br{\normlr{\bar{\gradsto}(Y_{k},\xi_{k+1}) - \bar{\gradsto}(x_{\star},\xi_{k+1})}^{2}}\\
		\nonumber
		&\qquad+ 2\E\br{\normn{\bar{\gradsto}(x_{\star},\xi)}^{2}}\\
		\label{eq:eq:bound:dk:2:vrsaladstar}
		&\le 2\E\br{\sigma_{k}^{2}} + 2\E\br{\normn{\bar{\gradsto}(x_{\star},\xi)}^{2}}\eqsp.
	\end{align}
	In addition, using the fact that for any vectors $a,b\in\R^d$, $2\abs{\ps{a}{b}}\le\conv\normlr{a}^{2}+\normlr{b}^{2}/\conv$, we can upper bound the inner product derived in \eqref{eq:eq:bound:dk:vrsaladstar} as follows
	\begin{align}
		\nonumber
		- \E\br{\ps{\Xavg_{k}-x_{\star}}{\frac{1}{b}\sum_{i=1}^{b}\nabla\potential^{i}(\Xlocal_{k}^{i})}}
		&= - \E\br{\ps{\Xavg_{k}-x_{\star}}{\nabla\barpotential(\Xavg_{k})}} \\
		\nonumber
		&+ \E\br{\ps{\Xavg_{k}-x_{\star}}{\frac{1}{b}\sum_{i=1}^{b}\br{\gradsto^{i}(\Xavg_{k},\xi_{k+1}^{i})-\gradsto^{i}(\Xlocal_{k}^{i},\xi_{k+1}^{i})}}}\\
		\nonumber
		&\le - \E\br{\ps{\Xavg_{k}-x_{\star}}{\nabla\barpotential(\Xavg_{k})}}
		+ \conv \E\br{\dist_{k}^{2}} / 2 + L^{2}\E\br{V_{k}} / (2\conv)\\
		\label{eq:eq:bound:dk:3:vrsaladstar}
		&\le - \conv \E\br{\dist_{k}^{2}}/2 + L^{2}\E\br{V_{k}} / (2\conv)\eqsp.
	\end{align}
	Hence, combining \eqref{eq:eq:bound:dk:vrsaladstar}, \eqref{eq:eq:bound:dk:1:vrsaladstar}, \eqref{eq:eq:bound:dk:2:vrsaladstar} and \eqref{eq:eq:bound:dk:3:vrsaladstar} implies that
	\begin{equation}
		\E\br{\dist_{k+1}^{2}}
		\le \pr{1-\gamma\conv+6\gamma^{2}\hat{L}^{2}} \E\br{\dist_{k}^{2}}
		+ \pr{\frac{\gamma L^{2}}{\betaempty\conv} + 6\gamma^{2}\hat{L}^{2}} \E\br{V_{k}}
		+ 4\gamma^{2}\E\br{\sigma_{k}^{2}}
		+ 10\gamma^{2}\E\br{\normn{\bar{\gradsto}(x_{\star},\xi)}^{2}}
		+ \frac{2 \gamma d}{b}\eqsp.
	\end{equation}
	Using the assumption on $\gamma$ completes the proof.
\end{proof}
\begin{lemma}\label{lem:bound:sigmak:vrsaladstar}
	Assume the $L$-smoothness of the potentials $\{\potential^{i}\}_{i\in[b]}$ and \Cref{ass:gradsto:lip} hold.
	Then, for any $k\in\N$, we have
	\[
		\E\br{\sigma_{k+1}^{2}}
		\le (1-\qc)\E\br{\sigma_{k}^{2}} + 2q\hat{L}^{2} \E\br{\dist_{k}^{2}} + 2q\hat{L}^{2} \E\br{V_{k}}\eqsp,
	\]
	where $V_{k}, \dist_{k}, \sigma_{k}$ are defined in \eqref{eq:def:Vk}, \eqref{eq:def:dk} and \eqref{eq:def:sigmak:vrsaladstar}.
\end{lemma}
\begin{proof}
	Let's consider $k\ge 0$, using \Cref{ass:gradsto:lip} implies that
	\begin{align}
		&\E\br{\sigma_{k+1}^{2}}
		= \frac{1}{b}\sum_{i=1}^{b}\E\br{\normlr{\gradsto^{i}(Y_{k+1}^{i},\xi_{k+1}^{i}) - \gradsto^{i}(x_{\star},\xi_{k+1}^{i})}^{2}} \\
		&= \frac{1-\qc}{b}\sum_{i=1}^{b}\E\br{\normlr{\gradsto^{i}(Y_{k}^{i},\xi_{k+1}^{i}) - \gradsto^{i}(x_{\star},\xi_{k+1}^{i})}^{2}}
			+ \frac{\qc}{b}\sum_{i=1}^{b}\E\br{\normlr{\gradsto^{i}(\Xlocal_{k}^{i},\xi_{k+1}^{i}) - \gradsto^{i}(x_{\star},\xi_{k+1}^{i})}^{2}} \\
		&= (1-\qc)\E\br{\sigma_{k}^{2}} + \frac{2q}{b}\sum_{i=1}^{b}\E\br{\normlr{\gradsto^{i}(\Xlocal_{k}^{i},\xi_{k+1}^{i}) - \gradsto^{i}(\Xavg_{k},\xi_{k+1}^{i})}^{2}
			+ \normlr{\gradsto^{i}(\Xavg_{k},\xi_{k+1}^{i}) - \gradsto^{i}(x_{\star},\xi_{k+1}^{i})}^{2}}\\
		&\le (1-\qc)\E\br{\sigma_{k}^{2}} + 2q\hat{L}^{2} \E\br{\dist_{k}^{2}} + 2q\hat{L}^{2} \E\br{V_{k}} \eqsp.
	\end{align}
	Which shows the expected result.
\end{proof}
%
% \begin{lemma}\label{lem:contraction:xkbis:transition:vrsaladstar}
% 	Assume \Cref{ass:fi}, \Cref{ass:gradsto:lip} hold and let $\gamma\le\betaempty(6L)^{-1}$.
% 	Then, for any $k\in\N$, we have
% 	\begin{multline}
% 		\E\br{\normn{\Xcontinuous_{k+1} - \Xavg_{k+1}}^{2}}
% 		= \pr{1 - \frac{\gamma\conv}{2\betaempty}} \E\br{\normn{\Xcontinuous_{k\gamma} - \Xavg_{k}}^{2}}
% 		+ 3\gamma\pr{\frac{L^{2}}{\conv} + \gamma\hat{L}^{2}} \E\br{V_{k}}\\
% 		+ 3\gamma^{2}\hat{L}^{2} \E\br{\dist_{k}^{2}}
% 		+ 3\gamma^{2} \E\br{\sigma_{k}^{2}}
% 		+ \pr{\frac{2}{\betaempty\gamma\conv}\E\br{\normlr{\E^{\mathcal{F}_{k}}\br{I_{k}}}^{2}} + 3\E\br{\normlr{I_{k}}^{2}}}\eqsp.
% 	\end{multline}
% 	where $\mathcal{F}_{k}, V_{k}, \dist_{k}, \sigma_{k}$ are defined in \eqref{eq:def:Fk}, \eqref{eq:def:Vk}, \eqref{eq:def:dk} and \eqref{eq:def:sigmak:vrsaladstar}.
% \end{lemma}
% %
% \begin{proof}
% 	Using \Cref{ass:gradsto:var:vrsaladstar}, we can prove the expected result by following the same lines as derived in \Cref{lem:contraction:xkbis:transition:vrsalad}.
% \end{proof}
%
For any $\gamma\in(0, \betaempty\conv(6\hat{L}^{2})^{-1}]$, under \Cref{ass:fi}, \Cref{ass:gradsto:lip} and \Cref{ass:gradsto:behavior} we have shown that \Cref{lem:bound:gexplicit:vrsaladstar} and \Cref{lem:bound:dk:vrsaladstar} imply \Cref{ass:dk:combination} and \Cref{ass:gradsto:g} with
\begin{equation}\label{eq:def:variables:vrsaladstar}
	\begin{aligned}
		&A = c_{V} = 3\hat{L}^{2}\eqsp,&
		&B = c_{d} = 3\constgradsto\eqsp,&
		&C = c_{\sigma} = 3\eqsp,&
		&D = c = 0\eqsp,\\
		&\bar{A} = 3L^{2}\eqsp,&
		&\bar{B} = 3L^{2}\eqsp,&
		&\bar{C} = 3\eqsp,&
		&\bar{D} = 0\eqsp,\\
		&A_{d} = \nofrac{\gamma\conv}{2\betaempty}\eqsp,&
		&B_{d} = 4\gamma^{2} \eqsp,&
		&C_{d} = \nofrac{2\gamma L^{2}}{\betaempty\conv}\eqsp,&
		&D_{d} = (10\gamma^{2})\E\br{\normlr{\bar{\gradsto}(x_{\star},\xi)}^{2}} + \nofrac{2\gamma d}{b}\eqsp,&\\
		&A_{\sigma} = q\eqsp,&
		&B_{\sigma} = 2q\hat{L}^{2}\eqsp,&
		&C_{\sigma} = 2q\hat{L}^{2}\eqsp,&
		&D_{\sigma} = 0\eqsp.
	\end{aligned}
\end{equation}
For any $\gamma>0$, consider the following variables
\begin{align}\label{def:eq:alphadsigma:vrsaladstar}
	\alpha_d = \frac{4\gamma^2}{\pc A_d}\max\ac{\pc B + 3\bar{B}, \frac{4B_\sigma}{A_\sigma}\pr{\pc C + 3\bar{C}}} \eqsp,&
	&\alpha_\sigma = \frac{4\gamma^2\pr{\pc C + 3\bar{C}}}{\pc A_\sigma}\eqsp.
\end{align}

\begin{lemma}\label{lem:bound:Vk:vrsaladstar:lyapunov:vrsaladstar}
	Assume \Cref{ass:fi}, \Cref{ass:gradsto:lip} and \Cref{ass:gradsto:behavior} hold with
	\begin{align}
		A_d\le \min\pr{A_\sigma, \frac{\pc}{4}}\eqsp,&
		&\alpha_d C_d + \alpha_\sigma C_\sigma \le \frac{\pc}{8}\eqsp,&
		&{\alpha_d} B_d + {\gamma^{2}}\pr{C+\frac{3}{\pc}\bar{C}} \le \frac{\alpha_\sigma A_\sigma}{2}\eqsp,
	\end{align}
	and consider $\gamma\le \betaempty\conv(6\hat{L}^{2})^{-1} \wedge \betaempty{\pc^{1/2}}{(2-2\pc)^{-1/2}\brn{A+(1+2/\pc)\bar{A}}^{-1/2}}$.
	Then, for any $k\in\N$, we have
	\begin{equation}\label{eq:Vk:vrsaladstar:bound:lyapunov:vrsaladstar}
		\E\br{V_{k}}
		\le \pr{1-\frac{A_d}{2}}^k \pr{
			\E\br{V_{0}}
			+ \alpha_d \E\br{\dist_{0}^{2}}
			+ \alpha_\sigma \E\br{\sigma_{0}^{2}}
			}
		+ \frac{2\alpha_d D_d}{A_d}
		+ \frac{4\pr{1-\tau}\pr{b-1}\gamma d}{b A_d}
		\eqsp,
	\end{equation}
	where $V_{k}$ is defined in \eqref{eq:def:Vk}.
\end{lemma}
\begin{proof}
	Applying \Cref{lem:bound:Vk:new} with the variables provided in \eqref{eq:def:variables:vrsaladstar} gives the result.
\end{proof}
Let's introduce $\gamma_{1}>0$ such that
\begin{multline}
	\label{eq:def:overlinegamma:vrsaladstar}
	\gamma_1
	\le \frac{\betaempty\conv}{128\hat{L}^{2}} \wedge \frac{\conv}{8\max\pr{3L^2+\pc\constgradsto, 24\hat{L}^2}} \wedge \frac{2 q}{\conv}
	\wedge \frac{\pc}{2\conv} \wedge \frac{\pc}{\br{2(1-\pc)(\pc A + 3 \bar{A})}^{1/2}} \\
	\wedge \frac{\pc}{8\br{6\pr{\frac{L^2}{\conv^2}\max\pr{3L^2+\pc\constgradsto, 24\hat{L^2}}} + \frac{2}{\qc}}^{1/2}}	
	\eqsp.
\end{multline}
Under \Cref{ass:fi}, \Cref{ass:gradsto:lip} and \Cref{ass:gradsto:behavior}, for all $\gamma\in\ocint{0,\gamma_1}$ the assumptions of \Cref{lem:bound:Vk:vrsaladstar:lyapunov:vrsaladstar} are satisfied. The upper bound on $(\E\br{V_{k}})_{k\in\N}$ derived in \Cref{lem:bound:Vk:vrsaladstar:lyapunov:vrsaladstar} can be rewritten into the format of \Cref{ass:vk} by considering
\begin{align}
	\label{eq:def:tilde:vrsaladstar}
	% &\alpha_{v} = 1-\nofrac{A_{d}}{4}\eqsp,\\
	% \label{eq:def:v1:vrsaladstar}
	\tilde{v}_1 = \E\br{V_{0}} + \alpha_d \E\br{\dist_{0}^{2}} + \alpha_\sigma \E\br{\sigma_{0}^{2}} \eqsp,&
	% \label{eq:def:v2:vrsaladstar}
	&\tilde{v}_2 = \frac{2\alpha_d D_d}{A_d} + \frac{4\pr{1-\tau}\pr{b-1}\gamma d}{b A_d}\eqsp.
\end{align}
In addition, for any $\gamma>0$, consider the following variables
% \begin{equation}\label{eq:def:cte:vrsaladstar}
% 	\begin{aligned}
% 		&\cte = (1-\pc)\gamma^{2}\br{\pr{B+\frac{2+\pc}{\pc}\bar{B}} + \frac{4B_{\sigma}}{\pc-4A_{d}}\pr{C+\frac{2+\pc}{\pc}\bar{C}}}\eqsp,\\
% 		&\ctesigma = \frac{4(1-\pc)\gamma^{2}\pr{1-A_{d}}}{\pc - 4 A_{d}}\pr{C+\frac{2+\pc}{\pc}\bar{C}} + \frac{2B_{d}\cte}{A_{d}(A_{\sigma}-A_{d})}\eqsp,\\
% 		&\cterate = \frac{4(1-\pc)\gamma^{2}C_{\sigma}}{\pc-4A_d}\pr{C+\frac{2+\pc}{\pc}\bar{C}} + \frac{2\cte}{A_{d}\pr{1-A_{d}}}\pr{C_{d} + \frac{B_{d}C_{\sigma}}{A_{\sigma} - A_{d}}}\eqsp,\\
% 		&\ctedelta = \frac{2\cte D_{d}}{A_{d}^{2}}
% 		+ \frac{8\pr{1-\tau}\pr{b-1}\gamma d}{b \pc}\eqsp.
% 	\end{aligned}
% \end{equation}
\begin{equation}\label{eq:def:cte:vrsaladstar}
	\begin{aligned}
		&\cte = \frac{4(1-\pc)\gamma^{2}}{\pc-4A_d}\br{B+\frac{2+\pc}{\pc}\bar{B}+ \frac{B_{\sigma}}{A_\sigma - A_d} \pr{C+\frac{2+\pc}{\pc}\bar{C}}}\eqsp, \\
		&\cterate = \frac{9\gamma^{2}\pr{1-\pc}C_{\sigma}}{\pc - 4A_d} \pr{C+\frac{2+\pc}{\pc}\bar{C}} + 3\cte\pr{C_{d} + \frac{B_{d}C_{\sigma}}{A_{\sigma} - A_{d}}} \eqsp,\\
		&\ctesigma = \frac{4(1-\pc)\gamma^{2}}{\pc - 4 A_{d}}\pr{C+\frac{2+\pc}{\pc}\bar{C}} + \cte B_d \pr{2 + \frac{3}{A_{\sigma}-A_{d}}}\eqsp,\\
		&
		\begin{aligned}
			\ctedzero = 7\cte \eqsp,&
			&\qquad\ctev = 1 + 2 \cte C_d\eqsp,
		\end{aligned}
		\\
		% \ctedone = \frac{\cte B_d B_\sigma}{(1-A_d)^{2}(A_\sigma - A_d)}\eqsp,\\
		&\ctedelta = \frac{\cte D_{d}}{A_d} \pr{1+ \frac{2B_d B_\sigma}{A_d(A_\sigma - A_d)}}
		+ \frac{8\pr{1-\tau}\pr{b-1}\gamma d}{b \pc}\eqsp,\\
		&\cteeps = \ctev \E\br{V_0} + \ctedzero \E\br{\dist_{0}^{2}} + \ctesigma \E\br{\sigma_{0}^{2}} + 2 D_d	\eqsp.
	\end{aligned}
\end{equation}
Based on \Cref{lem:bound:Vk}, we derive the following result.
\begin{lemma}\label{lem:bound:Vk:expec:vrsaladstar}
	Assume \Cref{ass:fi}, \Cref{ass:gradsto:lip} and \Cref{ass:gradsto:behavior} hold and consider $\gamma\in\ocint{0,\gamma_1}$.
	Then, for any $k\in\N$, we have
	\begin{equation}\label{eq:Vk:vrsaladstar:bound:lyapunov:vrsaladstar:combination}
		\E\br{V_{k}}
		\le \pr{1-\frac{A_d}{4}}^k
			\pr{
				\cteeps + \frac{4\cterate \tilde{v}_1}{A_d}
			}
		+ \frac{2\cterate \tilde{v}_2}{A_d}
		+ \ctedelta
		\eqsp,
	\end{equation}
	where $V_{k}$ is defined in \eqref{eq:def:Vk} and $\cteeps,\cterate,\ctedelta$ in \eqref{eq:def:cte:vrsaladstar}.
\end{lemma}
\begin{proof}
	Since we suppose \Cref{ass:fi}, \Cref{ass:gradsto:lip} and \Cref{ass:gradsto:behavior} hold with $\gamma\le \gamma_1$, the assumptions of \Cref{lem:bound:Vk:vrsaladstar:lyapunov:vrsaladstar} are satisfied. Therefore, for any $l\in\N$, we obtain 
	\begin{equation}\label{eq:bound:lem:vk:vrsaladstar:1}
		\E\br{V_{l}}
		\le \pr{1-\frac{A_d}{2}}^{l} \tilde{v}_1 + \tilde{v}_2\eqsp.
	\end{equation}
	Moreover, the condition $\gamma\le \nofrac{\conv}{128\hat{L}^{2}}$ ensures that $A_d A_\sigma=q\gamma\conv/2\ge 8B_dB_\sigma=64q\gamma^{2}\hat{L}^{2}$, hence we can apply \Cref{lem:bound:Vk}. Then, plugging \eqref{eq:bound:lem:vk:vrsaladstar:1} in the bound derived in \Cref{lem:bound:Vk} gives
	\begin{equation}\label{eq:bound:lem:vk:vrsaladstar:0}
		\E\br{V_{k}}
		\le \pr{1-\newrate}^{k} \cteeps
		+ \cterate
			\sum_{i=0}^{k-2}\pr{1-\newrate}^{k-i-1} \E\br{V_{i}}
		+ \ctedelta
		\eqsp,
	\end{equation}
	where $\newrate$ is defined in \eqref{eq:def:newrate} by
	\begin{equation}
		\newrate
		= A_d - \frac{2(A_\sigma - A_d)^{-1}B_dB_\sigma}{1 + \sqrt{1 + 4 (1-A_d)^{-1}(A_\sigma - A_d)^{-1} B_d B_\sigma}} \eqsp.
	\end{equation}
	Using \Cref{lem:bound:alpha}, we know that $A_d/2<\newrate\le A_d$ and combining this bound with \eqref{eq:bound:lem:vk:vrsaladstar:1} and \eqref{eq:bound:lem:vk:vrsaladstar:0} leads to
	\begin{equation}\label{eq:bound:lem:vk:vrsaladstar:2}
		\E\br{V_{k}}
		\le \pr{1-\frac{A_d}{4}}^k
			\pr{
				\cteeps + \frac{4\cterate \tilde{v}_1}{A_d}
			}
		+ \frac{2\cterate \tilde{v}_2}{A_d}
		+ \ctedelta
		\eqsp.
	\end{equation}
	% Finally, using $2A_d<\pc$ concludes the proof.
\end{proof}
In order to rewrite the upper bound on $(\E\br{V_{k}})_{k\in\N}$ given in \Cref{lem:bound:Vk:expec:vrsaladstar} in the format of \Cref{ass:vk}, we consider for $\gamma >0$
\begin{align}
	% \label{eq:def:alpha:vrsaladstar}
	% &\alpha_{v} = 1-\nofrac{A_{d}}{4}\eqsp,\\
	\label{eq:def:v1:vrsaladstar}
	&v_1 = \cteeps + \frac{4\cterate \tilde{v}_1}{A_d}\eqsp,&
	% \label{eq:def:v2:vrsaladstar}
	&v_{2} = \frac{2\cterate \tilde{v}_2}{A_d} + \ctedelta\eqsp.
\end{align}

\begin{lemma}\label{lem:bound:diffXkYki:vrsaladstar}
	Assume \Cref{ass:fi}, \Cref{ass:gradsto:difflip}, \Cref{ass:gunbiased} and hold  and let $\gamma\le \betaempty(6L)^{-1}$. Using the convention that $\sum_{0}^{-1}=0$, then for any $k\in\N$, we have
	\begin{multline}\label{eq:bound:diffXkYki:vrsaladstar:10}
		\E\br{\normn{\Xcontinuous_{(k+1)\gamma} - \Xavg_{k+1}}^{2}}
		\le \br{1 - \gamma\conv + \gamma^{2}\pr{3\conv L + 4 \constsvrg}} \E\br{\normn{\Xcontinuous_{k\gamma} - \Xavg_{k}}^{2}}\\
		+ 4\gamma^{2}\constsvrg \qc\sum_{l=0}^{k-1}(1-\qc)^{k-l-1} \E\br{\normn{\Xcontinuous_{l\gamma} - \Xavg_{l}}^{2}}
		+ \gamma\pr{\frac{2 L^{2}}{\conv} + 3\gamma L^{2} + 4\gamma\constsvrg} \E\br{V_{k}}
		\\
		+ \pr{\frac{2}{\gamma\conv}\E\br{\normlr{\E^{\mathcal{F}_{k}}\br{I_{k}}}^{2}} + 3\E\br{\normlr{I_{k}}^{2}}}
		+ \frac{16\gamma^3 \constsvrg d}{b \qc} \pr{1 + \frac{\gamma L}{\qc}}
		\eqsp.
	\end{multline}
\end{lemma}
\begin{proof}
	For $k\in\N$, using the independence of $(\xi_{k+1}^{i})_{i\in[b]}$ combined with \Cref{ass:gunbiased} and \Cref{ass:gradsto:difflip}, we obtain
	\begin{align}
		\var^{\mathcal{F}_{k}}\pr{\frac{1}{b}\sum_{i=1}^{b}G_{k}^{i}}
		&= \E^{\mathcal{F}_{k}}\br{\normlr{\frac{1}{b}\sum_{i=1}^b\br{\nabla\potential^{i}(\Xlocal_{k}^{i}) - \nabla\potential^{i}(Y_{k}) - \gradsto^{i}(\Xlocal_{k}^{i},\xi_{k+1}^{i}) + \gradsto^{i}(Y_{k},\xi_{k+1}^{i})}}^{2}} \\
		&= \frac{1}{b}\sum_{i=1}^b\E^{\mathcal{F}_{k}}\br{\normlr{\nabla\potential^{i}(\Xlocal_{k}^{i}) - \nabla\potential^{i}(Y_{k}) - \gradsto^{i}(\Xlocal_{k}^{i},\xi_{k+1}^{i}) + \gradsto^{i}(Y_{k},\xi_{k+1}^{i})}^{2}} \\
		\label{eq:bound:Vk:vrsaladstar}
		&\le \frac{\constsvrg}{b}\sum_{i=1}^b \normlr{\Xlocal_{k}^{i} - Y_{k}}^{2}\eqsp.
	\end{align}
	Denote $t_{k}\in\N$ the time when the reference point of the control variate is updated, therefore we have
	\begin{align}\label{eq:def:tk}
		t_{k} =
		\begin{cases}
			0\eqsp, \qquad\text{if $k=0$}\\
			\max\ac{l\in\ac{0,\ldots, k-1} \,:\, Y_{k} = b^{-1}\sum_{i=1}^b \Xlocal_{k}^{i}}\eqsp, \qquad\text{if $k\ge 1$}
		\end{cases}\eqsp.
	\end{align}
	Hence, for any $i\in[b], k\ge 0$, we have
	\begin{equation}\label{eq:bound:diffXkYki:vrsaladstar}
		\Xlocal_{k}^{i} - Y_{k} = \prn{\Xlocal_{k}^{i} - \Xavg_{k}} + \prn{\Xavg_{k} - \Xcontinuous_{k\gamma}} + \prn{\Xcontinuous_{k\gamma} - \Xcontinuous_{t_{k}\gamma}} + \prn{\Xcontinuous_{t_{k}\gamma} - Y_{k}}\eqsp.
	\end{equation}
	Thus for $k\ge 0$, combining the previous line with Young's inequality, it yields that
	\begin{equation}\label{eq:bound:diffXkYki:vrsaladstar:2}
		\frac{1}{b}\sum_{i=1}^b \E\br{\normlr{\Xlocal_{k}^{i} - Y_{k}}^{2}}
		\le 4\E\br{V_{k}}
		+ 4\E\br{\normn{\Xavg_{k} - \Xcontinuous_{k\gamma}}^{2}}
		+ 4\E\br{\normn{\Xcontinuous_{k\gamma} - \Xcontinuous_{t_{k}\gamma}}^{2}}
		+ 4\E\br{\normn{\Xcontinuous_{t_{k}\gamma} - Y_{k}}^{2}}
		\eqsp.
	\end{equation}
	For $k\ge 1$, by definition of $t_{k}$, we have
	\begin{equation}\label{eq:bound:diffXkYki:vrsaladstar:3}
		\E\br{V_{t_{k}}}
		= \sum_{l=0}^{k-1}\prob\pr{t_{k}=l} \E\br{V_{l}}
		= q\sum_{l=0}^{k-1}(1-\qc)^{k-l-1} \E\br{V_{l}}
		\eqsp.
	\end{equation}
	Moreover, for $k\ge 1$ we get
	\begin{align}
		\E\br{\normlr{\Xcontinuous_{k\gamma} - \Xcontinuous_{t_{k}\gamma}}^{2}}
		&= \sum_{l=0}^{k-1}\prob\pr{t_{k}=l} \E\br{\normlr{\Xcontinuous_{k\gamma} - \Xcontinuous_{l\gamma}}^{2}} \\
		&= q\sum_{l=0}^{k-1}(1-\qc)^{k-l-1} \E\br{\normlr{-\int_{l\gamma}^{k\gamma}\nabla\barpotential(\Xcontinuous_s) \rmd s + \sqrt{\frac{2}{b}}\pr{\mathsf{W}_{k\gamma} - \mathsf{W}_{l\gamma}}}^{2}} \\
		\label{eq:bound:diffXkYki:vrsaladstar:4}
		&\le 2\gamma q\sum_{l=0}^{k-1}(k-l)(1-\qc)^{k-l-1} \pr{\int_{l\gamma}^{k\gamma}\E\br{\normlr{\nabla\barpotential(\Xcontinuous_s)}^{2}} \rmd s + \frac{2d}{b}}
		\eqsp.
	\end{align}
	Using \citet[Lemma 2]{dalalyan2017further} with $s\in\R_+$, we obtain
	\[
		\E\br{\normn{\nabla\barpotential(\Xcontinuous_s)}^{2}}
		\le \nofrac{dL}{b}\eqsp.
	\]
	Using by convention that $\sum_{l=1}^{0}=0$, for any $k\in\N$ and $x\ne 1$ we have
	\begin{equation}\label{eq:bound:diffXkYki:vrsaladstar:5}
		\sum_{l=1}^k l^{2} x^{l-1}
		= (1-x)^{-3}\pr{1+x - x^k\br{2x + k x(1-x) + (k+1)(1+k(1-x))(1-x)}}\eqsp.
	\end{equation}
	Thus, setting $x=1-q$ inside the last shows that
	\[
		\sum_{l=1}^k l^{2} (1-\qc)^{l-1} \le \nofrac{2}{\qc^3} \eqsp.
	\]
	Hence, the above line combined with $\sum_{l=1}^{k}l(1-\qc)^{l-1}  = q^{-2}\br{1 - (1+kq)(1-\qc)^k}$ and \eqref{eq:bound:diffXkYki:vrsaladstar:4} yield the following upper bound
	\begin{align}
		\E\br{\normlr{\Xcontinuous_{k\gamma} - \Xcontinuous_{t_{k}\gamma}}^{2}}
		&\le \frac{2\gamma dq}{b}\sum_{l=0}^{k-1}\br{(k-l)(1-\qc)^{k-l-1} \pr{2 + (k-l)\gamma L}} \\
		&\le \frac{2\gamma dq}{b}\sum_{l=0}^{k-1}\br{(k-l)(1-\qc)^{k-l-1} \pr{2 + (k-l)\gamma L}} \\
		\label{eq:bound:diffXkYki:vrsaladstar:6}
		&\le \frac{4\gamma d}{b \qc} \pr{1 + \frac{\gamma L}{\qc}}
		\eqsp.
	\end{align}
	In addition, by definition~\eqref{eq:def:tk} of $t_{k}$, we immediately get for any $k\ge 1$, that
	\begin{align}
		\E\br{\normlr{\Xcontinuous_{t_{k}\gamma} - \Xavg_{t_{k}}}^{2}}
		&= \sum_{l=0}^{k-1}\prob\pr{t_{k}=l} \E\br{\normlr{\Xcontinuous_{l\gamma} - \Xavg_{l}}^{2}} \\
		\label{eq:bound:diffXkYki:vrsaladstar:7}
		&= q\sum_{l=0}^{k-1}(1-\qc)^{k-l-1} \E\br{\normlr{\Xcontinuous_{l\gamma} - \Xavg_{l}}^{2}}
		\eqsp.
	\end{align}
	% In addition, by definition~\eqref{eq:def:tk} of $t_{k}$, we immediately get for any $k\ge 1$, that
	% \begin{align}
	% 	\E\br{\normlr{\Xcontinuous_{t_{k}\gamma} - \Xavg_{t_{k}}}^{2}}
		% &= \sum_{l=0}^{k-1}\prob\pr{t_{k}=l} \E\br{\normlr{\Xcontinuous_{l\gamma} - \Xavg_{l}}^{2}} \\
	% 	\label{eq:bound:diffXkYki:vrsaladstar:7}
	% 	&= q\sum_{l=0}^{k-1}(1-\qc)^{k-l-1} \E\br{\normlr{\Xcontinuous_{l\gamma} - \Xavg_{l}}^{2}}
	% 	\eqsp.
	% \end{align}
	Combining \eqref{eq:bound:Vk:vrsaladstar}, \eqref{eq:bound:diffXkYki:vrsaladstar:2} with \eqref{eq:bound:diffXkYki:vrsaladstar:6}, for any $k\ge 1$ we obtain
	\begin{multline}\label{eq:bound:diffXkYki:vrsaladstar:8}
		\E\br{\var^{\mathcal{F}_{k}}\pr{\frac{1}{b}\sum_{i=1}^{b}G_{k}^{i}}}
		\le 4\constsvrg\E\br{\normn{\Xavg_{k} - \Xcontinuous_{k\gamma}}^{2}}
		+ 4\constsvrg \qc\sum_{l=0}^{k-1}(1-\qc)^{k-l-1} \E\br{\normlr{\Xcontinuous_{l\gamma} - \Xavg_{l}}^{2}} \\
		+ 4\constsvrg\E\br{V_{k}}
		+ \frac{16\gamma \constsvrg d}{b \qc} \pr{1 + \frac{\gamma L}{\qc}} \eqsp.
	\end{multline}
	Since $Y_{0} = b^{-1}\sum_{i=1}\Xavg_{0}^{i}$, we have $\var^{\mathcal{F}_{k}}\prn{b^{-1}\sum_{i=1}^{b}G_{k}^{i}}\le\constsvrg V_{k}$ and therefore the above inequality also holds for $k=0$.
	Lastly, using \Cref{lem:contraction:xkbis} gives
	\begin{multline}\label{eq:bound:diffXkYki:vrsaladstar:9}
		\E^{\mathcal{F}_{k}}\br{\normn{\Xcontinuous_{(k+1)\gamma} - \Xavg_{k+1}}^{2}}
		\le \br{1 - \gamma\conv\pr{1-3\gamma L}} \normn{\Xcontinuous_{k\gamma} - \Xavg_{k}}^{2}
		+ \gamma\pr{\frac{2 L^{2}}{\conv} + 3\gamma L^{2}} V_{k}\\
		+ \pr{\frac{2}{\gamma\conv}\normlr{\E^{\mathcal{F}_{k}}\br{I_{k}}}^{2} + 3\E^{\mathcal{F}_{k}}\br{\normlr{I_{k}}^{2}}}
		+ \gamma^{2}\var^{\mathcal{F}_{k}}\pr{\frac{1}{b}\sum_{i=1}^{b}G_{k}^{i}}
		\eqsp.
	\end{multline}
	Hence, plugging \eqref{eq:bound:diffXkYki:vrsaladstar:8} in the above inequality yields the expected result.
\end{proof}
Based on \Cref{lem:bound:diffXkYki:vrsaladstar}, for any $\gamma>0$ introduce the following notations
\begin{align}\label{eq:def:alpha:diffXkYki:vrsaladstar}
	&\alpha_{0} = \pr{1-\qc}^{-1} \br{1 - \gamma\conv + \gamma^{2}\pr{3\conv L + 4 \constsvrg}} \eqsp,\qquad
	\alpha_1 = \frac{4\gamma^{2} \constsvrg q}{\pr{1-\qc}^{2}} \eqsp, \\
	&\alpha_2 = \frac{\gamma}{1-\qc} \pr{\frac{2 L^{2}}{\conv} + 3\gamma L^{2} + 4\gamma\constsvrg} \eqsp,\qquad
	\alpha_3 = 0 \eqsp,\\
	&\alpha_4 = (1-\qc)^{-1} \pr{\frac{2\sup_{l\in\N}}{\gamma\conv}\E\br{\normlr{\E^{\mathcal{F}_{l}}\br{I_{l}}}^{2}} + 3\sup_{l\in\N}\E\br{\normlr{I_{l}}^{2}}
		+ \frac{16\gamma^3 \constsvrg d}{b \qc} \pr{1 + \frac{\gamma L}{\qc}}} \eqsp.
\end{align}
For ease of reading, we also introduce the time step-size $\gamma_2$ defined by
\begin{equation}
	\gamma_2
	\le \frac{\qc}{L} \wedge \frac{\qc}{2\conv} \wedge \frac{1}{6(L+4\conv^{-1}\constsvrg)}
	\eqsp.
\end{equation}

\begin{theorem}\label{thm:bound:wass:atlernative:vrsaladstar}
	Assume \Cref{ass:fi}, \Cref{ass:gradsto:lip}, \Cref{ass:gradsto:behavior}, \Cref{ass:gradsto:difflip} and let $\gamma\in(0, \gamma_1 \wedge \gamma_2)$.
	Then, for any initial probability measure $\mugvrs_0\in\mathcal{P}_{2}(\Rd)$, $k\in\N$, we have
	\begin{equation}
		\wass^{2}\pr{\mugvrs_k, \pi}
		\le \pr{1 - \frac{\gamma\conv}{2}}^k \wass^2\pr{\mugvrs_0,\pi}
			+ \pr{1-\frac{\gamma\conv}{8}}^{k} \frac{3 L^2}{\conv^2} v_1
		+ \frac{6 L^2}{\conv^2} v_{2}
		+ \frac{6\gamma d}{b\conv^{2}} \kappa_{I}
		+ \frac{32\gamma^{2} \constsvrg d}{b\conv q}
		\eqsp,
	\end{equation}
	where $v_1$, $v_{2}$ are defined in \eqref{eq:def:v1:vrsaladstar} and $\kappa_{I} = L^2\prn{1 + \nofrac{\gamma L^{2}}{\conv}}$. If in addition we suppose \Cref{ass:fi:ctrois}, set $\kappa_{I} = 2\gamma (L^3 + d\tilde{L}^{2}/b)$.
\end{theorem}
\begin{proof}
	We know that \Cref{ass:gunbiased} is satisfied since for any $i\in[b], x\in\R^d$ the stochastic gradient $\gradsto^{i}(x,\xi^{i})$ is unbiased.
	% \Cref{lem:bound:gexplicit:vrsaladstar} implies \Cref{ass:gradsto:g} with $A,\bar{A},B,\bar{B},C,\bar{C},D,\bar{D}$ given in \eqref{eq:def:variables:vrsaladstar}. Moreover, the constraint $\gamma\le \gamma_{1}$ combined with \Cref{lem:bound:dk:vrsaladstar} and \Cref{lem:bound:sigmak:vrsaladstar} imply \Cref{ass:dk:combination}.
	% Since \Cref{ass:dk:combination}, \Cref{ass:gradsto:g} hold with $\gamma\le \gamma_1$, thus we can apply
	\Cref{lem:bound:Vk:expec:vrsaladstar} proves that \Cref{ass:vk} holds with
	$
	   \alpha_{v} = 1-\nofrac{A_{d}}{4}
	$
	and $v_1, v_{2}$ defined in \eqref{eq:def:v1:vrsaladstar}.
	\Cref{lem:bound:diffXkYki:vrsaladstar} implies that \Cref{ass:contraction:general:xk} holds with the choice of $(\alpha_{i})_{i=0}^4$ detailed in \eqref{eq:def:alpha:diffXkYki:vrsaladstar}.
	Finally, since \Cref{ass:contraction:general:xk} and \Cref{ass:vk} hold, we can apply \Cref{thm:bound:wasserstein:general:vrsalad} to show that
	\begin{multline}\label{eq:bound:wass:atlernative:vrsaladstar:2}
		\wass^2\pr{\mugvrs_k,\pi}
		\le \pr{1+\alpha_{0}+\delta}^{k} \pr{1 - \qc}^{k} \wass^2\pr{\mugvrs_0,\pi} \\
		+ (1-\qc)v_1\pr{\alpha_2 + \frac{\alpha_3}{\alpha_{0}+\delta}} \frac{\alpha_{v}^k - \pr{1+\alpha_{0}+\delta}^{k} \pr{1 - \qc}^k}{\alpha_{v} - \pr{1+\alpha_{0}+\delta} \pr{1 - \qc}} \\
		+ \frac{1-\qc}{\qc - (1-\qc)(\alpha_{0}+\delta)} \br{\pr{\alpha_2 + \frac{\alpha_3}{\alpha_{0}+\delta}} v_{2} + \alpha_4} \eqsp,
	\end{multline}
	where $\delta = 2^{-1}\prn{\sqrt{(\alpha_{0}-1)^{2} + 4\alpha_1}-1-\alpha_{0}}$ is defined in \eqref{eq:def:delta:diffXkgeneral}.
	Using for any $a>0,b\ge 0$, that $\sqrt{a+b}\le \sqrt{a}+b/(2\sqrt{a})$, we obtain
	\begin{multline}
		\alpha_{0} + \sqrt{(\alpha_{0}-1)^{2} + 4\alpha_1}
		= 1 + \pr{\alpha_{0} - 1} \pr{1 + \sqrt{1 + \frac{4\alpha_1}{\pr{\alpha_{0}-1}^2}}} \\
		\le 1 + 2 \pr{\alpha_{0} - 1} \pr{1 + \frac{\alpha_1}{\pr{\alpha_{0}-1}^2}}
		= 2 \alpha_{0} - 1  + \frac{2\alpha_1}{\alpha_{0}-1}
		\eqsp.
	\end{multline}
	Since $\gamma\le \gamma_2\le q(2\conv)^{-1}\wedge \acn{6(L+4\conv^{-1}\constsvrg)}^{-1}$, the previous line implies that
	\begin{align}
		2 \pr{1 - \qc} \pr{1+\alpha_{0}+\delta}
		&= (1-\qc)\pr{1+\alpha_{0} + \sqrt{(\alpha_{0}-1)^{2} + 4\alpha_1}} \\
		&\le 2(1-\qc)\pr{\alpha_{0} + \frac{\alpha_1}{\alpha_0 - 1}} \\
		&= 2\pr{1 - \gamma\conv + \gamma^{2}\pr{3\conv L + 4 \constsvrg
			+ \frac{4q \constsvrg}{\qc - \gamma\conv + \gamma^{2}\pr{3\conv L + 4 \constsvrg}}}}\\
		\label{eq:bound:1plusalphaplusdelta:vrsaladstar}
		&\le 2\pr{1 - \gamma\conv/2}\eqsp.
	\end{align}
	This upper bound gives that
	\begin{equation}
		(1-\qc)(\alpha_{0}+\delta)
		= (1-\qc) \pr{1+\alpha_{0}+\delta} + q - 1
		\le q - \gamma\conv/2\eqsp.
	\end{equation}
	Thus, we deduce that
	\begin{equation}\label{eq:bound:wasserstein:vrsaladstar:complique:2}
		\frac{1}{\qc - (1-\qc)(\alpha_{0}+\delta)}
		\le \frac{2}{\gamma\conv}\eqsp.
	\end{equation}
	Further, using $\gamma\le \gamma_2$ combined with the definitions of $\alpha_0,\alpha_2,\alpha_3,\alpha_{v}$ and $\delta$ show that
	\begin{equation}\label{eq:bound:wasserstein:vrsaladstar:complique}
		\begin{aligned}
			&\frac{\alpha_{v}^k - \pr{1+\alpha_{0}+\delta}^{k} \pr{1 - \qc}^k}{\alpha_{v} - \pr{1+\alpha_{0}+\delta} \pr{1 - \qc}}
			\le \frac{8}{3\gamma\conv} \pr{1-\frac{\gamma\conv}{8}}^{k}\eqsp,\\
			&\alpha_2 + \frac{\alpha_3}{\alpha_{0}+\delta} = \frac{\gamma}{1-\qc} \pr{\frac{2 L^{2}}{\conv} + 3\gamma L^{2} + 4\gamma\constsvrg}
			\le \frac{3\gamma L^2}{(1-\qc)\conv} \eqsp.
		\end{aligned}
	\end{equation}
	Lastly, plugging \eqref{eq:bound:1plusalphaplusdelta:vrsaladstar}, \eqref{eq:bound:wasserstein:vrsaladstar:complique:2} and \eqref{eq:bound:wasserstein:vrsaladstar:complique} in \eqref{eq:bound:wass:atlernative:vrsaladstar:2} yields
	\begin{equation}\label{eq:bound:wass:atlernative:vrsaladstar:3}
		\wass^{2}\pr{\mugvrs_k, \pi}
		\le \pr{1 - \frac{\gamma\conv}{2}}^k \wass^2\pr{\mugvrs_0,\pi}
			+ \pr{1-\frac{\gamma\conv}{8}}^{k} \frac{3 L^2}{\conv^2} v_1
		+ \frac{6 L^2}{\conv^2} v_{2} + \frac{2(1-\qc) \alpha_4}{\gamma \conv} \eqsp.
	\end{equation}
	In addition, following the lines provided in the proof of \Cref{thm:bound:wass:atlernative:salad}, we deduce
	% Furthermore, using \Cref{lem:Ik:C2} we have
	% \begin{equation}
	% 	\E\br{\normlr{I_{k}}^{2}}
	% 	\le \frac{\gamma^3d L^{2}}{b}\pr{1 + \frac{\gamma L^{2}}{2\conv} + \frac{\gamma^{2} L^{2}}{12\betaemptysquared }}\eqsp.
	% 	% \le 3d \gamma^3 L^{2}\eqsp.
	% \end{equation}
	% % In this case, we have $\kappa_c\le \gamma^{2}\pr{\nofrac{9dL^{2}}{\betaempty\conv} + c}$
	% Moreover, if we suppose an additional regularity assumption of order three combined with a regularity assumption on the hessian of the potentials $(\potential^{i})_{i=1}^b$ as given by \Cref{ass:fi:ctrois}, we obtain a sharper upper bound on $\E\brn{\normn{\E^{\mathcal{F}_{k}}\brn{I_{k}}}^{2}}$. Indeed, we show in \Cref{lem:Ik:C3} that
	% \begin{equation}\label{eq:bound:wass:atlernative:vrsaladstar:5}
	% 	\frac{2}{\betaempty\gamma\conv}\E\br{\normlr{\E^{\mathcal{F}_{k}}\br{I_{k}}}^{2}}
	% 	\le \frac{4\gamma^3 d}{3b\conv} \pr{L^3 + \frac{d \tilde{L}^{2}}{b}}\eqsp.
	% \end{equation}
	% Therefore, we deduce
	\begin{equation}\label{eq:bound:wass:atlernative:vrsaladstar:4}
		\frac{2(1-\qc) \alpha_4}{\gamma \conv}
		\le \frac{6\gamma d L^{2}}{b\conv^2} \pr{1 + \frac{19\gamma L^{2}}{36\conv}}
		+ \frac{32\gamma^{2} \constsvrg d}{b\conv q} \eqsp.
	\end{equation}
	If in addition we suppose \Cref{ass:fi:ctrois}, then we obtain
	\begin{equation}\label{eq:bound:wass:atlernative:vrsaladstar:5}
		\frac{2(1-\qc) \alpha_4}{\gamma \conv}
		\le \gamma \conv L^{2}\pr{1 + \frac{\gamma L^{2}}{2\conv} + \frac{\gamma^{2} L^{2}}{12}} + \frac{4\gamma}{9} \pr{L^3 + \frac{d\tilde{L}^{2}}{b}}
		+ \frac{32\gamma^{2} \constsvrg d}{b\conv q} \eqsp.
	\end{equation}
	Finally, plugging \eqref{eq:bound:wass:atlernative:vrsaladstar:4} or \eqref{eq:bound:wass:atlernative:vrsaladstar:5} if \Cref{ass:fi:ctrois} holds inside \eqref{eq:bound:wass:atlernative:vrsaladstar:3} combined with $\gamma\le q L^{-1}$ lead to the expected result.
	% Since $\var^{\mathcal{F}_{k}}\prn{b^{-1}\sum_{i=1}^{b}G_{k}^{i}}=0$,
	% \begin{multline}\label{}
	% 	\E^{\mathcal{F}_{0}}\br{\normn{\Xcontinuous_{\gamma} - \Xavg_{1}}^{2}}
	% 	\le \pr{1 - \gamma\conv + 3\gamma^{2}\conv L} \normn{\Xcontinuous_{0} - \Xavg_{0}}^{2}
	% 	+ \gamma\pr{\frac{2 L^{2}}{\conv} + 3\gamma L^{2}} V_{0}\\
	% 	+ \pr{\frac{2}{\gamma\conv}\normlr{\E^{\mathcal{F}_{0}}\br{I_{0}}}^{2} + 3\E^{\mathcal{F}_{0}}\br{\normlr{I_{0}}^{2}}}
	% 	\eqsp.
	% \end{multline}
\end{proof}
Now, consider the time stepsizes $\gamma_3$ and $\gamma_\star$ defined by
\begin{align}
	&\gamma_3 = \frac{\pc\conv}{3L^2 + 16\hat{L}^2 + \pc\constgradsto} \eqsp,&
	&\gamma_\star = \gamma_1 \wedge \gamma_2 \wedge \gamma_3\eqsp.
\end{align}
From the previous result, the next corollary controls the asymptotic bias obtained by \Cref{algo:fedavgLangevin:vrsaladstar}.
\begin{corollary}\label{cor:wass:alternative:vrsaladstar}
	Assume \Cref{ass:fi}, \Cref{ass:gradsto:lip}, \Cref{ass:gradsto:behavior}, \Cref{ass:gradsto:difflip} and let $\gamma\in(0, \gamma_\star)$ with $\tau=1$.
	Then, for any initial probability measure $\mugvrs_0\in\mathcal{P}_{2}(\Rd)$, $k\in\N$, we have
	\begin{multline}
		\frac{9^{-9}b}{\gamma d}\limsup_{k\to\infty}\wass^{2}\pr{\mugvrs_k,\pi}
		\le \frac{\kappa_{I}}{\conv^{2}}
		+ \frac{\gamma \constsvrg}{\conv q} \\
		+ \frac{(1-\pc)\gamma L^ 2}{\pc^2\conv^5} \pr{L^{2} + \hat{L}^{2} + \pc \constgradsto} \pr{1 + \frac{\gamma}{b d} \E\br{\normlr{\gradsto(x_{\star},\xi)}^{2}}} \pr{L^2 + \frac{\qc}{\pc}\hat{L}^2}
		\eqsp,
	\end{multline}
	where $\tilde{\kappa}_{I} = L^{2}(1 + \gamma L^{2}\conv^{-1})$ and if we suppose \Cref{ass:fi:ctrois}, $\tilde{\kappa}_{I} = \gamma (L^3 + d\tilde{L}^{2} b^{-1})$.
\end{corollary}
\begin{proof}
	Applying \Cref{thm:bound:wass:atlernative:vrsaladstar} with $\gamma\in(0,\gamma_1\wedge\gamma_2)$ shows that
	\begin{align}
		\limsup_{k\to\infty}\wass^{2}\pr{\mugvrs_k,\pi}
		&\le \frac{6 L^2}{\conv^2} v_{2}
		+ \frac{6\gamma d}{b\conv^{2}} \kappa_{I}
		+ \frac{32\gamma^{2} \constsvrg d}{b\conv q} \\
		\label{eq:bound:cor:wass:atlernative:vrsaladstar:3}
		&\le \frac{6 L^2 \ctedelta}{\conv^2}
			+ \frac{12 L^2 \cterate \tilde{v}_2}{A_d \conv^2}
			+ \frac{6\gamma d}{b\conv^{2}} \kappa_{I}
			+ \frac{32\gamma^{2} \constsvrg d}{b\conv q}
		\eqsp.
	\end{align}
	Plugging the definitions of $\tilde{v}_1,\tilde{v}_2$ provided in \eqref{eq:def:tilde:vrsaladstar} combined with the previous inequality, we obtain
	\begin{equation} \label{eq:bound:cor:wass:atlernative:vrsaladstar:new:0}
		\limsup_{k\to\infty}\wass^{2}\pr{\mugvrs_k,\pi}
		\le \frac{6 L^2 \ctedelta}{\conv^2}
			+ \frac{24 L^2 \cterate \alpha_d D_d}{A_d^2 \conv^2}
			+ \frac{48 L^2 \cterate \pr{1-\tau}\pr{b-1}\gamma d}{b A_d^2 \conv^2}
			+ \frac{6\gamma d}{b\conv^{2}} \kappa_{I}
			+ \frac{32\gamma^{2} \constsvrg d}{b\conv q}
		\eqsp.
	\end{equation}

	Further, recall that $A_d, B, \bar{B}, D,\bar{D}, D_d$ are provided in \eqref{eq:def:variables:vrsaladstar} and $\alpha_d$ is defined in \eqref{def:eq:alphadsigma:vrsaladstar} by
	\begin{align}
		\alpha_d
		&= \frac{4\gamma^2}{\pc A_d}\max\ac{\pc B + 3\bar{B}, \frac{4B_\sigma}{A_\sigma}\pr{\pc C + 3\bar{C}}}\\
		&= \frac{24\gamma}{\pc\conv}\max\ac{3L^2+\pc\constgradsto, 8(\pc+3) \hat{L}^2}
		\le \frac{768\gamma}{\pc\conv}\pr{L^2 + \hat{L}^2 + \pc\constgradsto}
		\eqsp.
	\end{align}
	Moreover, $\ctedelta, \cterate$ are defined in \eqref{eq:def:cte:vrsaladstar} by
	\begin{align}
		\ctedelta
		&= \frac{\cte D_{d}}{A_d} \pr{1+ \frac{2B_d B_\sigma}{A_d(A_\sigma - A_d)}}
			+ \frac{8\pr{1-\tau}\pr{b-1}\gamma d}{b \pc} \\
		&= \frac{10\cte}{\conv} \pr{1+ \frac{64\gamma q\hat{L}^{2}}{(2q - \gamma\conv)\conv}} \pr{5\gamma \E\br{\normlr{\bar{\gradsto}(x_{\star},\xi)}^{2}}+ \frac{d}{b}}
			+ \frac{8\pr{1-\tau}\pr{b-1}\gamma d}{b \pc} \\
		&\le \frac{360(1-\pc)\gamma^{2}}{\conv\pc^{2}} \pr{3 L^{2} + 11\hat{L}^{2} + \pc\constgradsto} \pr{5\gamma \E\br{\normlr{\bar{\gradsto}(x_{\star},\xi)}^{2}}+ \frac{d}{b}}
			+ \frac{8\pr{1-\tau}\pr{b-1}\gamma d}{b \pc} \label{eq:bound:cor:wass:atlernative:vrsaladstar:4}
		\eqsp, \\
		\cterate
		&= \frac{9\gamma^{2}\pr{1-\pc}C_{\sigma}}{\pc - 4A_d} \pr{C+\frac{2+\pc}{\pc}\bar{C}} + 3\cte\pr{C_{d} + \frac{B_{d}C_{\sigma}}{A_{\sigma} - A_{d}}}\\
		&\le \frac{144\gamma^2(1-\pc)}{\pc^2}\br{3q \hat{L}^2 + \gamma\pr{\frac{L^2}{\conv} + 8\gamma \hat{L}^2}\pr{\pc \constgradsto + 3 L^2 + 16\hat{L}^2}}
		\le \frac{432\gamma^2(1-\pc)}{\pc^2}\pr{\pc L^2 + q \hat{L}^2}
	\end{align}
	Eventually, for the specific choice $\tau=1$ combined with \eqref{eq:bound:cor:wass:atlernative:vrsaladstar:3} and \eqref{eq:bound:cor:wass:atlernative:vrsaladstar:4}, it yields that
	\begin{multline}\label{eq:bound:cor:wass:atlernative:vrsaladstar:5}
		\limsup_{k\to\infty}\wass^{2}\pr{\mugvrs_k,\pi}
		\le \frac{6\gamma d}{b\conv^{2}} \kappa_{I}
		+ \frac{32\gamma^{2} \constsvrg d}{b\conv q}
		+ \frac{18432 \gamma \cterate D_d L^2}{A_d^2 \conv^3 \pc} \pr{L^2 + \hat{L}^2 + \pc \constgradsto}
		 \\
		+ \frac{2160(1-\pc)\gamma^{2} L^2}{\pc^{2}\conv^3} \pr{3 L^{2} + 11\hat{L}^{2} + \pc\constgradsto} \pr{5\gamma \E\br{\normlr{\bar{\gradsto}(x_{\star},\xi)}^{2}}+ \frac{d}{b}}
		\eqsp.
	\end{multline}
	Therefore, using \eqref{eq:bound:cor:wass:atlernative:vrsaladstar:4} and \eqref{eq:bound:cor:wass:atlernative:vrsaladstar:5} we can finally conclude that
	\begin{multline}
		9^9\limsup_{k\to\infty}\wass^{2}\pr{\mugvrs_k,\pi}
		\le \frac{\gamma d}{b\conv^{2}} \kappa_{I}
		+ \frac{\gamma^{2} \constsvrg d}{b\conv q}
		 \\
		+ \frac{(1-\pc)\gamma^{2} L^2}{\pc^{2}\conv^5} \pr{L^{2} + \hat{L}^{2} + \pc\constgradsto} \pr{\gamma \E\br{\normlr{\bar{\gradsto}(x_{\star},\xi)}^{2}}+ \frac{d}{b}} \pr{L^2 + \frac{\qc}{\pc}\hat{L}^2}
		\eqsp.
	\end{multline}
\end{proof}

The single client case corresponds to $b=\pc=1$ and leads for $k\ge 0$ to $V_k=0$. Moreover, the assumption \Cref{ass:vk} holds with $v_1=v_{2}=0$. Thus, we obtain a convergence bound for SVRG-LD from \Cref{thm:bound:wass:atlernative:vrsaladstar}.
\begin{theorem}\label{thm:bound:wass:svrgLangevin}
	Assume \Cref{ass:fi}, \Cref{ass:gradsto:lip}, \Cref{ass:gradsto:behavior}, \Cref{ass:gradsto:difflip} and let $\gamma\in(0, \gamma_1 \wedge \gamma_2)$.
	Then, for any initial probability measure $\mugvrs_0\in\mathcal{P}_{2}(\Rd)$, $k\in\N$, we have
	\begin{equation}
		\wass^{2}\pr{\mugvrs_k, \pi}
		\le \pr{1 - \frac{\gamma\conv}{2}}^k \wass^2\pr{\mugvrs_0,\pi}
		+ \frac{6\gamma d}{b\conv^{2}} \kappa_{I}
		+ \frac{32\gamma^{2} \constsvrg d}{\conv q}
		\eqsp,
	\end{equation}
	where $\kappa_{I} = L^{2}(1 + \nofrac{\gamma L^{2}}{2\conv} + \nofrac{\gamma^{2} L^{2}}{12\betaemptysquared })$. If in addition we suppose \Cref{ass:fi:ctrois}, set $\kappa_{I} = 3\gamma (L^3 + d\tilde{L}^{2}/b)$.
\end{theorem}

\begin{remark} \label{rem:vrsalad}\noindent
	\begin{itemize}
		\item The constants obtained in this result can be refined by directly using that $\E[V_k]=0$ in the proof of \Cref{lem:bound:diffXkYki:vrsaladstar} and by simplifying the calculations detailed in \Cref{thm:bound:wass:atlernative:vrsaladstar}.
		\item The proof given in \citet[Theorem 4.2-Option 2]{chatterji2018theory} on the convergence of SVRG-LD seems to have some gaps since the authors use Grönwall's inequality \citep{clark1987short} as if $\spadesuit = \tau^2\pr{8\delta d + 4M\delta^2 d + 4\delta^2 M\Omega_1}$ were constant, which is not the case because $\Omega_1 = \ps{\nabla f(y_k) - \nabla f(x_k)}{y_k - x_k}$ depends on the iteration $k$.
		If we denote $\spadesuit_k$ instead of $\spadesuit$ and adopt their other notation (we also correct a typography in the right-hand term), we obtain
		\begin{equation}\label{eq:bound:SVRGLD}
			\E\br{\norm{x_k-\tilde{x}}^2_2} \le\spadesuit_k + \sum_{j=\tau s}^{k-1}\E\br{\norm{x_j-\tilde{x}}^2_2}\eqsp.
		\end{equation}
  		Then, it is claimed in the proof of \citet[Theorem 4.2-Option 2]{chatterji2018theory}  that \eqref{eq:bound:SVRGLD} implies $\E\brn{\norm{x_k-\Xcontinuous_k}^2}\le \spadesuit_k\exp(\tau\rho)$. But this inequality cannot hold in all generalities, for example if we consider : $\tau s=0$, for $j<k$, $\spadesuit_j=1$, $x_j = \tilde{x} + \sqrt{2^{j} /d} \cdot\boldsymbol{1}$ and $\spadesuit_k=0$, $x_k = \tilde{x} + \boldsymbol{1} / \sqrt{d}$, then \eqref{eq:bound:SVRGLD} holds for $j\in[k]$ but $\E\brn{\norm{x_k-\Xcontinuous_k}^2}=1$ whereas $\spadesuit_k\exp(\tau\rho)=0$.
	\end{itemize}
\end{remark}

\section{Lower bound on the heterogeneity in a Gaussian case}\label{sec:gaussian:counterex}

% !TEX root = main.tex

In this section, we want to illustrate the heterogeneity problem by lower bounding the Wasserstein distance $\wass$ in a simple case. To simplify the calculations, we assume that each client performs $2$ local iterations following the {\algoun} update before communicating its local parameter to the central server.
More specifically, take $(\mu_1,\mu_2, \sigma_1,\sigma_2)\in \R^{2}\times(\R_+^*)^2$ and define the potentials $\potential^1:x\in\R^d\mapsto {\sigma_1^{-2}}{\pr{x-\mu_1}^2}$, $\potential^2:x\in\R^d\mapsto {\sigma_2^{-2}}{\pr{x-\mu_2}^2}$.
Thus, the global posterior distribution $\pi$ is Gaussian with mean $\bar{\mathrm{m}}$ and variance $\bar{\sigma}^2$ given by
% The clients $i\in\{1,2\}$ can locally target the Gaussian distribution $\pi_i\propto\exp(-\potential^{i})$ whereas they want to target the Gaussian distribution 
% $\pi = \gauss(\bar{\mathrm{m}},\bar{\sigma}^2)\propto \exp(-\potential^1-\potential^2)$ with parameters given by
\begin{align}\label{eq:def:paramGaussian}
    \bar{\mathrm{m}} = \frac{\mu_1\sigma_2^2 + \mu_2\sigma_1^2}{\sigma_1^2 + \sigma_2^2}&
    &\bar{\sigma} = \pr{\frac{1}{\sigma_1^{2}}+\frac{1}{\sigma_2^{2}}}^{-1/2}\eqsp.
\end{align}
The objective is to illustrate the problem of heterogeneity in the basic version of {\algoun}. To do so, we first show that this algorithm generates samples targeting a distribution $\pi_\gamma\in\mathcal{P}_2(\Rd)$ where the distance $\wass(\pi,\pi_\gamma)$ is lower bounded by a heterogeneity term.
To this end, we introduce the Markov kernel, which for each $\gamma > 0,\msb\in\mathcal{B}(\R^d)$ is given by
\begin{equation}\label{eq:def:Pgamma}
    P_\gamma(x, \msb)
    = \int_{\msb}\exp\pr{- \frac{\normlr{x' - \pr{1-\frac{\gamma}{\bar{\sigma}^2}+\frac{\gamma^2}{2}\pr{\frac{1}{\sigma_1^4}+\frac{1}{\sigma_2^4}}}x - \frac{\gamma\bar{\mathrm{m}}}{\bar{\sigma}^2} +\frac{\gamma^2}{2}\pr{\frac{\mu_1}{\sigma_1^4}+\frac{\mu_2}{\sigma_2^4}}}^2}{2\gamma\pr{1+\pr{1-\frac{\gamma}{2\bar{\sigma}^2}}^2}}}\frac{\rmd x'}{(2\pi)^{d/2}}
    \eqsp,
\end{equation}
and we define the stochastic processes $(A_k,\tilde{A}_k)_{k\ge 0}$ on $\R^d\times \R^d$ starting from $(X_0,X_0) = (x,\tilde{x})$ and following the recursion for $k\ge 0$,
\begin{equation}\label{eq:def:couplingdiscret-discret}
    \begin{aligned}
        &A_{k+1}=A_k - \frac{\gamma}{\bar{\sigma}^2}\pr{A_k - \bar{\mathrm{m}}} + \frac{\gamma^2}{2}\pr{
            \frac{A_k - \mu_1}{\sigma_1^4}
            + \frac{A_k - \mu_2}{\sigma_2^4}
            }+ \sqrt{\gamma}\br{\pr{1-\frac{\gamma}{2\bar{\sigma}^2}}Z_{k+1} + Z_{k+2}}\eqsp,\\
        &\tilde{A}_{k+1}=\tilde{A}_k - \frac{\gamma}{\bar{\sigma}^2}\pr{\tilde{A}_k - \bar{\mathrm{m}}} + \frac{\gamma^2}{2}\pr{
            \frac{\tilde{A}_k - \mu_1}{\sigma_1^4}
            + \frac{\tilde{A}_k - \mu_2}{\sigma_2^4}
            }+ \sqrt{\gamma}\br{\pr{1-\frac{\gamma}{2\bar{\sigma}^2}}Z_{k+1} + Z_{k+2}}\eqsp.
    \end{aligned}
\end{equation}
It is possible to verify that $(A_{k},\tilde{A}_{k})$ is distributed according to $(\delta_{x}P_\gamma^k,\delta_{\tilde{x}}P_\gamma^k)$.

\begin{lemma}\label{lem:gaussconterex:1}
    Let $\gamma\in\ooint{0, 2(\sigma_1\sigma_2)^4\brn{\bar{\sigma}^2(\sigma_1^4+\sigma_2^4)}^{-1}}$.
    Then, there exists $\pi_\gamma\in\mathcal{P}_2(\R^d)$ such that for any distribution $\pi^0\in\mathcal{P}_2(\R^d)$, the sequence $(\pi^0 P_\gamma^k)_{k\in\N}$ converges to $\pi_\gamma$ in $\mathcal{P}_2(\R^d)$.
\end{lemma}

\begin{proof}
    Let $k\in\N$ and consider the stochastic processes $(A_{l},\tilde{A}_{l})_{l\in\N}$ defined in \eqref{eq:def:couplingdiscret-discret}, subtracting the two recursions we obtain
    \begin{align}
        A_{k+1} - \tilde{A}_{k+1} = \pr{1-\frac{\gamma}{\bar{\sigma}^2} + \frac{\gamma^2}{2}\pr{\frac{1}{\sigma_1^4}+\frac{1}{\sigma_2^4}}} \pr{A_{k} - \tilde{A}_{k}}\eqsp.
    \end{align}
    Since $0<\gamma<2(\sigma_1\sigma_2)^4\brn{\bar{\sigma}^2(\sigma_1^4+\sigma_2^4)}^{-1}$, taking the norm in the previous inequality implies that
    \begin{equation}\label{eq:eq:diffAtildeA}
        \norm{A_{k+1} - \tilde{A}_{k+1}} = \pr{1-\frac{\gamma}{\bar{\sigma}^2} + \frac{\gamma^2}{2}\pr{\frac{1}{\sigma_1^4}+\frac{1}{\sigma_2^4}}} \norm{A_{k} - \tilde{A}_{k}}\eqsp.
    \end{equation}
    Finally, combining \eqref{eq:eq:diffAtildeA} with \citet[Lemma 20.3.2]{douc2018markov}, we deduce that the $c$-Dobrushin coefficient of $P_\gamma$ is upper bounded by $1-\nofrac{\gamma}{\bar{\sigma}^2} + \nofrac{\gamma^2}{2}\pr{\nofrac{1}{\sigma_1^4}+\nofrac{1}{\sigma_2^4}}$. Hence, applying \citet[Theorem 20.3.4]{douc2018markov} we deduce the existence and uniqueness of a stationary distribution $\pi_\gamma\in\mathcal{P}_2(\R^d)$ for the Markov Kernel $P_\gamma$ such that $\wass(\pi^0 P_\gamma^k,\pi)\le \pr{1-\nofrac{\gamma}{\bar{\sigma}^2} + \nofrac{\gamma^2}{2}\pr{\nofrac{1}{\sigma_1^4}+\nofrac{1}{\sigma_2^4}}}^k \wass(\pi^0,\pi_\gamma)$.
\end{proof}

\Cref{lem:gaussconterex:1} shows the existence of a invariant distribution $\pi_\gamma\in\mathcal{P}_2(\R^d)$ for $P_\gamma$ and the next lemma specifies this distribution of $\pi_\gamma$.

\begin{lemma}\label{lem:gaussconterex:2}
    Assume $\gamma\in\ooint{0, 2(\sigma_1\sigma_2)^4\brn{\bar{\sigma}^2(\sigma_1^4+\sigma_2^4)}^{-1}}$.
    Then, the stationarity distribution $\pi_\gamma$ is Gaussian with parameters given by
    \begin{align}
        \mathrm{m}_{(\gamma)} = \frac{\bar{\mathrm{m}} - \frac{\gamma\bar{\sigma}^2}{2}\pr{\frac{\mu_1}{\sigma_1^4} + \frac{\mu_2}{\sigma_2^4}}}{1 - \frac{\gamma\bar{\sigma}^2}{2}\pr{\frac{1}{\sigma_1^4} + \frac{1}{\sigma_2^4}}}\eqsp,&
        &\sigma_{(\gamma)}^2
        = \frac{\bar{\sigma}^2 - \frac{\gamma}{2} + \frac{\gamma^2}{8\sigma^2}}{1 - \frac{\gamma}{2}\pr{\frac{\bar{\sigma}^2}{\sigma_1^4}+\frac{\bar{\sigma}^2}{\sigma_2^4}}
        - \frac{\gamma}{2}\pr{\frac{1}{\bar{\sigma}} -\frac{\gamma}{2} \pr{\frac{\bar{\sigma}}{\sigma_1^4}+\frac{\bar{\sigma}}{\sigma_2^4}}}^2} \eqsp.
    \end{align}
\end{lemma}

\begin{proof}
    First, let $k\in\N$ be fixed and introduce
    \begin{align}
        &\alpha = 1 - \frac{\gamma}{\bar{\sigma}^2} + \frac{\gamma^2}{2}\pr{\frac{1}{\sigma_1^4}+\frac{1}{\sigma_2^4}}\eqsp,&
        &\beta = \frac{\gamma\bar{\mathrm{m}}}{\bar{\sigma}^2}-\frac{\gamma^2}{2}\pr{\frac{\mu_1}{\sigma_1^4}+\frac{\mu_2}{\sigma_2^4}}\eqsp,\\
        &\tilde{Z}_{k} = \pr{1-\frac{\gamma}{2\bar{\sigma}^2}}Z_{2k-1} + Z_{2k}\eqsp.
    \end{align}
    Moreover, consider $(A_{l})_{l\in\N}$ the stochastic process following \eqref{eq:def:couplingdiscret-discret} and initialized at $\pi_\gamma$.
    % By calculations, we obtain
    % \begin{align}
    %     A_{k+1}
    %         &= A_k - \gamma \nabla\barpotential(A_k)
    %         - \frac{\gamma}{2}\bigg[
    %         \nabla\potential^1\pr{A_k - \gamma\nabla\potential^1(A_k) + \sqrt{\gamma} Z_{k+1}}\\
    %         &+ \nabla\potential^2\pr{A_k - \gamma\nabla\potential^2(A_k) + \sqrt{\gamma} Z_{k+1}}
    %         \bigg]
    %         + \sqrt{\gamma}\pr{Z_{k+1} + Z_{k+2}} \\
    %         &= A_k - \frac{\gamma}{2\bar{\sigma}^2}\pr{A_k - \bar{\mathrm{m}}}
    %         - \frac{\gamma}{2}\bigg[
    %         \frac{1}{\sigma_1^2}\pr{A_k - \mu_1 - \frac{\gamma}{\sigma_1^2}\pr{A_k - \mu_1} + \sqrt{\gamma} Z_{k+1}} \\
    %         &+ \frac{1}{\sigma_2^2}\pr{A_k - \mu_2 - \frac{\gamma}{\sigma_2^2}\pr{A_k - \mu_2} + \sqrt{\gamma} Z_{k+1}}
    %         \bigg]
    %         + \sqrt{\gamma}\pr{Z_{k+1} + Z_{k+2}} \\
    %         &= A_k - \frac{\gamma}{\bar{\sigma}^2}\pr{A_k - \bar{\mathrm{m}}} + \frac{\gamma^2}{2}\pr{
    %             \frac{A_k - \mu_1}{\sigma_1^4}
    %             + \frac{A_k - \mu_2}{\sigma_2^4}
    %             }\\
    %         &+ \sqrt{\gamma}\br{\pr{1-\frac{\gamma}{2\bar{\sigma}^2}}Z_{k+1} + Z_{k+2}}\eqsp.
    %         \label{eq:eq:Xstationarity}
    % \end{align}
    By induction, we know that
    \begin{equation}\label{eq:eq:Xtwok}
        A_{k} = \alpha^k A_0 + \beta \sum_{l=0}^{k-1}\alpha^l + \sqrt{\gamma} \sum_{l=0}^{k-1}\alpha^{k-l-1} \tilde{Z}_l\eqsp.
    \end{equation}
    Since $A_{k}$ is distributed according to $\pi_\gamma P_\gamma^k$, we have that $A_{k}$ follows $\pi_\gamma$. Denote $\nu_\gamma^k$ the distribution of $\txts\sqrt{\gamma}\sum_{l=0}^{k-1}\alpha^{k-l-1} \tilde{Z}_l - \beta \sum_{l=0}^{k-1}\alpha^l$, combining \eqref{eq:eq:Xtwok} with the definition of the Wasserstein, we have
    \begin{equation}\label{eq:bound:wass:pigammanugamma}
        \wass^2\pr{\pi_\gamma, \nu_\gamma^k}
        \le \E\br{\normlr{A_{k} - \sqrt{\gamma}\sum_{l=0}^{k-1}\alpha^{k-l-1} \tilde{Z}_l - \beta \sum_{l=0}^{k-1}\alpha^l}^2}
        = \alpha^{2k} \E\br{\norm{A_0}^2}\eqsp.
    \end{equation}
    Since $A_0$ is distributed according to $\pi_\gamma$ belonging to $\mathcal{P}_2(\R^d)$, we deduce that $\E\brn{\normn{A_0}^2} < \infty$. Consequently, \eqref{eq:bound:wass:pigammanugamma} implies that $(\nu_\gamma^k)_{k\in\N}$ converges to $\pi_\gamma$, but using the fact that $(\nu_\gamma^k)_{k\in\N}$ converges to a Gaussian distribution, we obtain by uniqueness of the limit in metric space $(\mathcal{P}_2(\R^d), \wass)$ that $\pi_\gamma$ is a Gaussian distribution.
    Recalling that $\mathrm{m}_{(\gamma)}$ denotes the expectation of the random variable distributed according to $\pi_\gamma$, using \eqref{eq:def:couplingdiscret-discret} at stationarity yields
    \begin{equation}
        \mathrm{m}_{(\gamma)} = \mathrm{m}_{(\gamma)} - \frac{\gamma}{\bar{\sigma}^2}\pr{\mathrm{m}_{(\gamma)}-\bar{\mathrm{m}}} + \frac{\gamma^2}{2}\pr{\frac{\mathrm{m}_{(\gamma)}-\mu_1}{\sigma_1^4} -\frac{\mathrm{m}_{(\gamma)}-\mu_2}{\sigma_2^4}}
    \end{equation}
    Thus, we deduce that
    \begin{equation}
        \mathrm{m}_{(\gamma)} = \frac{\bar{\mathrm{m}} - (\gamma\bar{\sigma}^2/2)\pr{\mu_1/\sigma_1^4 + \mu_2/\sigma_2^4}}{1 - (\gamma\bar{\sigma}^2/2)\pr{1/\sigma_1^4 + 1/\sigma_2^4}}\eqsp.
    \end{equation}
    In addition, we can obtain the standard deviation $\sigma_{(\gamma)}$ of $\pi_\gamma$ since we have
    \begin{align}
        \var\pr{\beta \sum_{l=0}^{k-1}\alpha^l + \sqrt{\gamma} \sum_{l=0}^{k-1}\alpha^{k-l-1} \tilde{Z}_l}
        &= \gamma \var\pr{\sum_{l=0}^{k-1}\alpha^{k-l-1} \tilde{Z}_l}
        = \frac{\gamma(1-\alpha^{2k})}{1-\alpha^2}\var\prn{\tilde{Z}_0} \\
        &\xrightarrow[k\to\infty]{} \frac{\gamma\var\prn{\tilde{Z}_0}}{1-\alpha^2} \\
        &= \frac{\gamma\pr{2 - \frac{\gamma}{\bar{\sigma}^2} + \frac{\gamma^2}{4\bar{\sigma}^4}}}{1-\pr{1 - \frac{\gamma}{\bar{\sigma}^2} + \frac{\gamma^2}{2}\pr{\frac{1}{\sigma_1^4}+\frac{1}{\sigma_2^4}}}^2} \\
        &= \frac{1 - \frac{\gamma}{2\bar{\sigma}^2} + \frac{\gamma^2}{8\bar{\sigma}^4}}{\frac{1}{\bar{\sigma}^2} - \frac{\gamma}{2}\pr{\frac{1}{\sigma_1^4}+\frac{1}{\sigma_2^4}}
            - \frac{\gamma}{2}\pr{\frac{1}{\bar{\sigma}^2}-\frac{\gamma}{2}\pr{\frac{1}{\sigma_1^4}+\frac{1}{\sigma_2^4}}}^2} \eqsp.
    \end{align}
\end{proof}

\begin{theorem}\label{thm:gaussconterex:3}
    Assume $\gamma\in\ooint{0, 2(\sigma_1\sigma_2)^4\brn{\bar{\sigma}^2(\sigma_1^4+\sigma_2^4)}^{-1}}$.
    Then, the Wasserstein distance between the stationnary distribution $\pi_\gamma$ and the target $\pi$ of {\algoun} is lower bounded as
    \[
        \wass\pr{\pi_\gamma,\pi} \ge \frac{\gamma}{2}\abs{\mu_1 - \mu_2}\abs{\frac{\bar{\sigma}^2}{\sigma_1^2} - \frac{\bar{\sigma}^2}{\sigma_2^2}}\eqsp.
    \]
\end{theorem}

\begin{proof}
    % Calculations show
    % \begin{equation}
    %     \sigma_{(\gamma)}^2 = \sigma^2 + \frac{\gamma}{2}\pr{\frac{\sigma^4}{\sigma_1^4} + {\frac{\sigma^4}{\sigma_2^4}}} + \Oh(\gamma^2)\eqsp.
    % \end{equation}
    % \begin{equation}
    %     \pr{\sigma_{(\gamma)}-\sigma}^2 = \frac{\gamma^2}{16}\pr{\frac{\sigma^4}{\sigma_1^4} + {\frac{\sigma^4}{\sigma_2^4}}}^4 + \Oh(\gamma^3)\eqsp.
    % \end{equation}
    Based on \Cref{lem:gaussconterex:2}, we know that $\pi_\gamma$ is Gaussian with parameters $(\mathrm{m}_{(\gamma)},\sigma_{(\gamma)}^2)$ and using that $\pi$ is Gaussian too with parameters $(\bar{\mathrm{m}},\bar{\sigma}^2)$ given in \eqref{eq:def:paramGaussian}, we have that
    \begin{align}
        \wass^2\pr{\pi_\gamma,\pi} &= \pr{\mathrm{m}_{(\gamma)} - \bar{\mathrm{m}}}^2 + \pr{\sigma_{(\gamma)} - \bar{\sigma}}^2
        \ge \frac{\gamma^2\bar{\sigma}^4}{4}\abs{\pr{\frac{1}{\sigma_1^4} + \frac{1}{\sigma_2^4}}\bar{\mathrm{m}} - \frac{\mu_1}{\sigma_1^4} - \frac{\mu_2}{\sigma_2^4}}^2\\
        &= \frac{\gamma^2\bar{\sigma}^4 \pr{\mu_1 - \mu_2}^2}{4}\pr{\frac{1}{\sigma_1^2} - \frac{1}{\sigma_2^2}}^2\eqsp.
    \end{align}
\end{proof} 

% !TEX root = main.tex

\section{Analysis of the complexity and communication cost}\label{sec:analysis-costs}

In this section, we study the optimal choices of $k,\gamma$ when $\pc$ is fixed. For $c_0, c_1, c_2\ge 0$ fixed, we consider the following optimization problem:
\begin{equation}
    \begin{cases}
        &\min_{k\in\N^{\star},\gamma>0} \ac{k} \\
        &\text{Subject to } \ac{c_0 \exp\pr{-\nofrac{8k\gamma}{\conv}} + c_1 \gamma + c_2 \gamma^2 \le \epsilon^2}
    \end{cases}\eqsp.
\end{equation}
Using that the constraint must be saturated at the optimum (which can be proved), we can write $k$ as a function of $\gamma$. Hence, the problem becomes
\begin{equation}\label{eq:mt:0}
    \begin{cases}
        &\min_{k,\gamma} \ac{\frac{8}{\gamma\conv}\log\pr{\frac{c_0}{\epsilon^2 - c_1 \gamma - c_2 \gamma^2}}} \\
        &\text{Subject to } 0<\gamma \text{ and } \epsilon^2 - c_1 \gamma - c_2 \gamma^2 > 0
    \end{cases}\eqsp.
\end{equation}
Let us introduce $x\in\R_+^*$, defined by $x = \epsilon^{-2} \gamma$ and let $\tilde{c}_2=\epsilon^2 c_2$. We can rewrite \eqref{eq:mt:0} as
\begin{equation}\label{eq:mt:1}
    \begin{cases}
        &\min_{k,x} \ac{\frac{8}{\epsilon^2\conv x}\log\pr{\frac{c_0}{\epsilon^2 (1 - c_1 x - \tilde{c}_2 x^2)}}} \\
        &\text{Subject to } 0<x \text{ and } 1 - c_1 x - \tilde{c}_2 x^2 > 0
    \end{cases}\eqsp.
\end{equation}
Consider $\mu=-\nofrac{c_1}{(2 \tilde{c}_2)}$, $\sigma=\sqrt{\nofrac{c_1^2}{(4\tilde{c}_2^2)} + \nofrac{1}{\tilde{c}_2}}$, and denote $z=\nofrac{(x-\mu)}{\sigma}$.
Since $x=\mu + z\sigma$, we can verify that $1 - c_1 x - \tilde{c}_2 x^2 = \tilde{c}_2\sigma^2 (1-z^2)$. Hence, \eqref{eq:mt:1} is equivalent to
\begin{equation}\label{eq:mt:1}
    \begin{cases}
        &\min_{k,\gamma} \ac{\frac{8}{\epsilon^2\conv (\mu + z\sigma)}\log\pr{\frac{c_0}{\epsilon^2 \tilde{c}_2 \sigma^2 (1 - z^2)}}} \\
        &\text{Subject to } -\nofrac{\mu}{\sigma} < z < 1
    \end{cases}\eqsp.
\end{equation}
According to the intermediate value theorem, we have the existence of $z_\epsilon$ (not necessarily unique, but we can consider one of the solutions) such that
\begin{equation}
    z_\epsilon=\argmax_{-\nofrac{\mu}{\sigma} < z < 1} \ac{\frac{\log(1 - z^2)}{\mu + z\sigma}} \eqsp.
\end{equation}
Thus, the solution is
\begin{align}\label{eq:ref:gKepsilon}
    &\gamma_\epsilon = \epsilon^2 \times \frac{z_\epsilon^2 + (4\epsilon^2 c_2)^{-1} (z_\epsilon^2 - 1) c_1^2}{c_1/2 + z_\epsilon\sqrt{4^{-1} c_1^2 + \epsilon^2 c_2}}\eqsp, \\
    &K_\epsilon = \frac{8 \prbig{c_1/2 + z_\epsilon\sqrt{4^{-1} c_1^2 + \epsilon^2 c_2}}}{\epsilon^2\conv \pr{z_\epsilon^2 + (4\epsilon^2 c_2)^{-1} (z_\epsilon^2 - 1) c_1^2}}\log\pr{\frac{c_0}{\epsilon^2 \prn{c_1^2/4 + \epsilon^2 c_2}^{1/2} (1 - z_\epsilon^2)}} \eqsp.
\end{align}

\paragraph{\FALD{}.}

According to the \Cref{main:thm:bound:wass:atlernative:algoun}, we have
\begin{equation}
    \begin{cases}
        &c_0 = \initconst(\mu_0) \\
        &c_1 = \varconst_\pi + \pr{1 - 1_{\text{\Cref{ass:fi:ctrois}}}} \nofrac{\langevinrest}{b} + \nofrac{(1-\tau)(1-b^{-1}) d}{\pc} \\
        &c_2 = 1_{\text{\Cref{ass:fi:ctrois}}} \nofrac{\langevinrest}{b} + (1-\pc) \ac{\heterogeneity + \pc \varconst_\epsilon + \nofrac{d}{b}} / \pc^2
    \end{cases}\eqsp.
\end{equation}
If $c_1 > 0$, define $w=\epsilon^2c_2/c_1^2$. For $\epsilon\in(0, c_1/\sqrt{2c_2}]$, we have $0< w\le 1/2$. Consider $z = 1-w$, we get that
\begin{align}
    \pr{\frac{\mu}{\sigma}}^2
    = \frac{1}{1 + 4\epsilon^2c_2/c_1^2}
    < \frac{1}{1+2w}
    \le 1 - w
    \le 1- 2w + w^2 = z^2
    < 1 \eqsp.
\end{align}
Hence, the previous inequalities show that $-\mu/\sigma<z<1$, and for this choice
\begin{align}
    \frac{c_1/2 + z\sqrt{4^{-1}c_1^2 + \epsilon^2 c_2}}{z^2 + (4\epsilon^2 c_2)^{-1} (z^2 - 1) c_1^2}
    \le \frac{c_1 + \epsilon (1-w) \sqrt{c_2}}{7/8 + (w - 2 + 1/64) w} \eqsp.
\end{align}
Thus, for any $\epsilon\in(0, c_1 (2\sqrt{c_2})^{-1}]$, we deduce that $w<1/4$. Therefore, we have shown that $K_\epsilon=\tilde{\Oh}((\epsilon^2\conv)^{-1}(c_1 + \epsilon \sqrt{c_2}))$.
Moreover, this result is immediately valid when $c_1=0$ since $z_\epsilon=\argmax_{0<z<1}\acn{z^{-1}\log(1-z^2)}$.
Furthermore, when $\pce\downarrow 0^+$, $\pce K_\epsilon=\tilde{\Oh}((\epsilon\conv)^{-1}\sqrt{b^{-1}\langevinrest})$ as stressed in the main paper.

\paragraph{\VRFALDs{}.}

Using \Cref{main:thm:bound:wass:atlernative:vrsalads}, we obtain
\begin{equation}
    \begin{cases}
        &c_0 = \initvrsconst(\mu_0) \\
        &c_1 = \pr{1 - 1_{\text{\Cref{ass:fi:ctrois}}}} \nofrac{\langevinrest}{b} + \nofrac{(1-\tau)(1-b^{-1}) d}{\pc} \\
        &c_2 = 1_{\text{\Cref{ass:fi:ctrois}}} \nofrac{\langevinrest}{b} + (1-\pc) \ac{\pc \varconst_\epsilon + \nofrac{d}{b}} / \pc^2
    \end{cases}\eqsp.
\end{equation}
When assuming \Cref{ass:fi:ctrois} and $\tau=1$, we have $c_1=0$. Hence, $z_\epsilon=\argmax_{0<z<1}\acn{z^{-1}\log(1-z^2)}$ and therefore
\begin{equation}
    K_\epsilon = \frac{8 \sqrt{c_2}}{\epsilon\conv z_\epsilon}\log\pr{\frac{c_0}{\epsilon^3 \sqrt{c_2} (1 - z_\epsilon^2)}} \eqsp.
\end{equation}
When $\pce \downarrow 0^+$, the minimum number of communications becomes $\pce K_\epsilon = \tilde{O}(\epsilon^{-1}\sqrt{b^{-1}d})$.
% We retrieve the claims, taking $\pce\approx 1$ is the optimal choice in order to minimize the number of rounds $k_\epsilon$.
Finally, setting $\pce=1$ gives $K_\epsilon=\tilde{O}(\epsilon^{-1}\sqrt{b^{-1}\langevinrest+b^{-1}\omega d})$.

% \paragraph{Societal impact.}

% Our contributions are in improving algorithms for uncertainty management, rather than any specific application. We expect the broader impacts of our work to be similar to those of prior works, and do not see any negative societal impacts specific to our work.

% Our work is of theoretical and methodological nature., and we do not anticipate that it will have negative societal impacts.

\section{Numerical experiments}

% !TEX root = main.tex

\subsection{Gaussian example}
\label{subsec:gaussian}

In this first experiment, we consider $b=100$ clients associated with potentials: $\forall i\in[b]$, $\potential^{i}: x\in\R^{d} \mapsto (1/2) (x - \mu_i)^{\top} \Sigma_i^{-1}(x - \mu_i)$ in dimension $d=20$. In this particular case, we know, that the posterior distribution $\pi\propto\exp(-\sum_{i=1}^{b}\potential^{i})$ is Gaussian with mean $x_{\star}=\sum_{i=1}^b (\Sigma_\star\Sigma_i^{-1}\mu_i)$ and covariance $\Sigma_\star=(\sum_{i=1}^b\Sigma_i^{-1})^{-1}$.
Also, we have a close formula to calculate $\int \norm{x-x_{\star}}^2 \rmd \pi(x)$, since this quantity is equal to $\trace(\Sigma_\star)$. To speed up the calculations, we initialize all chains at $x_{\star}$, we discard the first 10\% of the samples and keep all others.
Moreover, we consider the step size $\bar{\gamma}=2[\lambda_{\text{min}}(\Sigma_\star^{-1})+\lambda_{\text{max}}(\Sigma_\star^{-1})]^{-1}$ for Langevin Monte Carlo \citep{dalalyan2019user,durmus2019high}, and we run the algorithms for the step sizes $\gamma\in\{\frac{\pc\bar{\gamma}}{2}, \frac{\pc\bar{\gamma}}{5}, \frac{\pc\bar{\gamma}}{10}\}$ associated with $\pc\in\{\frac{1}{5},\frac{1}{10},\frac{1}{20}\}$. We set the probability of updating the control variates $\qc=\pc$ so as not to increase the communication cost too much.
We also consider the two extreme values of the parameter $\tau\in\{0,1\}$ to determine whether it is preferable to have independent Gaussian noise on each client or if it is better to have a common one.

\subsection{Bayesian Logistic Regression}
\label{subsec:bayesian-logistic}

The second experiment is performed on the Titanic dataset, which is in the public domain and licensed under the Commons Public Domain Dedication License (PDDL-1.0). We distribute this dataset heterogeneously across $b=10$ clients by drawing a Dirichlet random variable for each label on the standard $b-1$ simplex. Since the sum of the coordinates of these random variables equals $1$, each coordinate indicates the fraction of labels to be distributed to each client. To have access to ground truth, we also implement Langevin Stochastic Dynamics (LSD).
We compute $K=250000$ iterations, each time considering a burn-in period of length 10\% initialized with a warm start provided by SGD.
The $i$th client uses its local dataset $\{(z_{ij}, o_{ij})\in\R^{4}\times \acn{0,1}: j\in[N_i]\}$ to calculate the local potential $U^{i}(x)=\sum_{j=1}^{N_i}[o_{ij}\log(1+\exp(-z_{ij}^T x)) + (1-o_{ij})\log(1+\exp(z_{ij}^T x))] + \lambda \norm{x}^2$, where $\lambda=1$ is associated with the Gaussian prior.
Denote $\mat{Z}_{\text{train}}$ the matrix whose lines are the covariates $z_{ij}^T$, and write $\Sigma=\mat{Z}_{\text{train}}^T \mat{Z}_{\text{train}}$. We run the algorithms with mini-batches of size $n_i=1$; a step size $\gamma=2[\lambda_{\text{min}}(\Sigma)+\lambda_{\text{max}}(\Sigma)]^{-1}$ for {\algoun}, {\algoquatre} and equal to $\gamma/b$ for {LSD} with thinning inversely proportional to the step size. Moreover, we consider a communication probability of $\pc = 1/20$ and clients update their control variables with probability $\qc=\pc$. 
Finally, to evaluate the obtained results, we consider the accuracy, agreement, and total variation, as well as the calibration results such as ECE, BS, and NLL, which are described below.

\paragraph{Accuracy.}

Based on samples from the approximate posterior distribution, we compute the minimum mean squared estimator (\emph{i.e.}, which corresponds to the posterior mean) and use it to make predictions for the test dataset. The \emph{Accuracy} metric corresponds to the percentage of well-predicted labels.

\paragraph{Agreement.}

Let $p_{\mathrm{ref}}$ and $p$ denote the predictive densities associated with {HMC} and an approximate simulation-based algorithm, respectively.
Similar to \citet{izmailov2021bayesian}, we define the agreement between $p_{\mathrm{ref}}$ and $p$ as the proportion of test data points for which the top-1 predictions of $p_{\mathrm{ref}}$ and $p$, \emph{i.e.}
\[
  \mathrm{agreement}(p_{\mathrm{ref}},p) = \frac{1}{|\mathrm{D}_{\mathrm{test}}|} \sum_{x \in \mathrm{D}_{\mathrm{test}}} \mathbf{1}\bbr{\argmax_{y'} p_{\mathrm{ref}}(y'\mid x) = \argmax_{y'} p(y'\mid x)}\eqsp.
\]

\paragraph{Total variation (TV).}

By denoting $\mathcal{Y}$ as the set of possible labels, we consider the total variation metric between $p_{\mathrm{ref}}$ and $p$, \emph{i.e.}
\[
  \mathrm{TV}(p_{\mathrm{ref}},p) = \frac{1}{2|\mathrm{D}_{\mathrm{test}}|} \sum_{x \in \mathrm{D}_{\mathrm{test}}} \sum_{y'\in\mathcal{Y}} \left|p_{\mathrm{ref}}(y'\mid x) - p(y'\mid x)\right|\eqsp.
\]

\paragraph{Expected Calibration Error (ECE).}

To measure the difference between the accuracy and confidence of the predictions, we group the data into $M\ge 1$ buckets defined for each $m \in [M]$ by $\mathrm{B_m}=\{(x,y)\in\mathrm{D}_{\mathrm{test}}: p(y_{\mathrm{pred}}(x)| x)\in\left]\nofrac{(m-1)}{M}, \nofrac{m}{M}\right]\}$. As in the previous work of \citet{ovadia2019can}, we denote the model accuracy on $\mathrm{B_{m}}$ by
\[
  \mathrm{acc}\left(\mathrm{B_{m}}\right)=\frac{1}{\left|\mathrm{B_{m}}\right|} \sum_{(x,y)\in \mathrm{B_{m}}} \mathbf{1}_{y_{\mathrm{pred}}(x)=y}
\]
and define the confidence on $\mathrm{B_m}$ by
\[
  \mathrm{conf}\left(\mathrm{B_m}\right)=\frac{1}{\left|\mathrm{B_m}\right|} \sum_{(x,y) \in \mathrm{B_m}} p(y_{\mathrm{pred}}(x)| x)\eqsp.
\]
As emphasized in \citet{guo2017calibration}, for any $m\in[M]$ the accuracy $\mathrm{acc}\left(\mathrm{B_m}\right)$ is an unbiased and consistent estimator of $\PP\left(y_{\mathrm{pred}}(x)=y \mid (m-1)/M<p(y_{\mathrm{pred}}(x)| x)\le \nofrac{m}{M}\right)$.
Therefore, the ECE is defined by
\[
  \mathrm{ECE}=\sum_{m=1}^{M} \frac{\left|\mathrm{B_m}\right|}{\left|\mathrm{D}_{\mathrm{test}}\right|}\left|\mathrm{acc}\left(\mathrm{B_m}\right)-\mathrm{conf}\left(\mathrm{B_m}\right)\right|
\]
and is an estimator of
\[
  \E_{(x,y)}\Big[\big|\ PP \pr{y_{\mathrm{pred}}(x)=y \mid p(y_{\mathrm{pred}}(x)| x)}-p(y_{\mathrm{pred}}(x) | x)\big|\Big].
\]
Thus, the ECE measures the absolute difference between the confidence level of a prediction and its accuracy.

\paragraph{Brier Score (BS).}

The BS is a proper scoring rule (see for example \citet{dawid2014theory}) that can only evaluate random variables taking a finite number of values. Denote by $\mathcal{Y}$ the finite set of possible labels, the BS measures the confidence of the model in its predictions and is defined by
\[
  \mathrm{BS} = \frac{1}{|\mathrm{D}_{\mathrm{test}}|}\sum_{(x,y)\in\mathrm{D}_{\mathrm{test}}}\sum_{c\in\mathcal{Y}}(p(y=c| x) - \mathbf{1}_{y=c})^2 \eqsp.
\]

\paragraph{Normalized Negative Log Likelihood (nNLL).}

This classical score defined by
\[
  \mathrm{nNLL} = -\frac{1}{|\mathrm{D}_{\mathrm{test}}|}\sum_{(x,y)\in\mathrm{D}_{\mathrm{test}}}\log p(y| x)
\]
measures the ability of the model to predict good labels with high probability.

\paragraph{Highest posterior density (HPD).}

Under the Bayesian paradigm, we are interested in quantifying uncertainty by estimating the regions of high probability.
For all $\alpha\in(0,1)$, we run each algorithm to estimate $\eta_\alpha^{\text{algo}} > 0$ such that $\int_{R_\alpha}\pi(x)\rmd x=1-\alpha$, where $\mathcal{R}_\alpha=\acn{x\in\R^d: \pi(x)\ge \exp(-\eta_\alpha^{\text{algo}})}$.
Then we define the relative HPD error as $\absn{\eta_\alpha^{\text{algo}}/\eta_\alpha^{\textrm{LSD}}-1}$, where $\eta_\alpha^{\textrm{LSD}}$ is estimated based on the samples drawn with the Langevin Stochastic Dynamics method.

\subsection{Bayesian Neural Network: MNIST}
\label{subsec:MNIST}

To investigate the behavior of the proposed algorithms in a highly non-convex setting, we perform a first Deep Learning experiment on the MNIST dataset \citep{deng2012mnist}, which can be publicly downloaded using the torchvision package and is available under the Creative Commons Attribution-Share Alike 3.0 license. To this end, we distribute the entire dataset across $b=20$ clients in a highly heterogeneous manner to train the LeNet5 neural network \citep{lecun1998gradient}.
The MNIST real-world dataset consists of $70000$ grayscale images of size $28\times 28$ associated with the $10$ digits. This dataset is divided into two subsets: the training set, which contains $60000$ images, and the test set, which consists of the remaining $10000$ images. We report the median of the scores with their associated hyperparameters in \Cref{suppl:table:mnist-comparison}. The burn-in corresponds to the number of steps performed before we start storing the samples, and the thinning is the frequency with which we keep the samples. We also consider a Gaussian prior which corresponds to a squared norm regularizer with weight decay.
We initialized FSGLD \citep{el2021federated} with a global SGD warm start combined with local SWAG \citep{maddox2019simple} to learn Gaussian conducive gradients.

\begin{table*}[]
    \centering
        \begin{tabular}{lccccccc}
        \toprule
        \textsc{Method}           & {\SGLD}         & {p\SGLD}       & {\algoun}      & {\algoquatre}  & {FSGLD} \\
        \midrule
        \texttt{Accuracy}         & $99.1\pm 0.1$   & $99.2\pm 0.1$  & $99.1\pm 0.1$  & $99.2\pm 0.1$  & $98.5\pm 0.2$ \\
        $10^3\times$\texttt{ECE}  & $6.88\pm 27.07$ & $21.6\pm 11.1$ & $4.07\pm 0.80$ & $4.34\pm 1.26$ & $6.34\pm 1.90$ \\
        $10^2\times$\texttt{BS}   & $1.66\pm 1.76$  & $1.45\pm 0.12$ & $1.47\pm 0.45$ & $1.39\pm 0.07$ & $2.39\pm 1.72$ \\
        $10^2\times$\texttt{nNLL} & $3.53\pm 5.08$  & $4.24\pm 1.14$ & $3.06\pm 0.43$ & $3.43\pm 0.37$ & $4.87\pm 0.51$ \\        
        Weight Decay & 5 & 5 & 5 & 5 & 5 \\
        Batch Size & 64 & 64 & 8 & 8 & 64 \\
        Learning rate & 1e-07 & 1e-08 & 1e-07 & 1e-07 & 1e-08 \\
        Local steps & N/A & N/A & $20$ & $20$ & $20$ \\
        Burn-in & 100epch. & 100epch. & 1e04 & 1e04 & 1e04 \\
        Thinning & 1 & 1 & 1e03 & 1e03 & 1e03 \\
        Training & 1e03epch. & 1e03epch. & 1e05it. & 1e05it. & 1e05it. \\
        \bottomrule
        \end{tabular}
    \caption{Performance of Bayesian FL algorithms on MNIST. \label{suppl:table:mnist-comparison}}
\end{table*}

\subsection{Bayesian Neural Network: CIFAR10}
\label{subsec:CIFAR}
In this last experiment, we consider the more challenging dataset CIFAR10 \citep{CIFAR10}, which is available under license MIT and contains images of size $(3, 32, 32)$. We used different approaches to sample the weights for the ResNet-20 model \citep{he2016deep}, which is publicly available in the pytorchcv library. We initialized the algorithms with 10 different parameters using SGD (400 epochs) trained with a OneCycleLR scheduler \citep{smith2019super}, and we also use data augmentation with a mini-batch of size 128 and a learning rate of 2e-7.
Based on these initializations, we ran 10 chains in parallel for {\SGLD}, {\algoun}, and {\algoquatre} with step sizes of 1e-7, 2e-8, 1e-8. We considered 1e4 iterations with only one stored sample every 1e3 iterations (we did not keep the initial weights obtained by SGD to make the predictions). For each chain, we can see that Bayesian model averaging increases the accuracy.
To compare the behavior of the mentioned algorithms, we compute the \texttt{accuracy}, the \texttt{agreement}, i.e., the percentage of time the top-1 prediction of an algorithm matches that given by the {HMC}, and the total variation (\texttt{TV}) between the predictive distribution given by an algorithm with the one associated with the {HMC} sampler.
We also give some classical calibration scores \citep{guo2017calibration}, such as the expected calibration error (\texttt{ECE}), the Brier score (\texttt{BS}), and the negative log-likelihood (\texttt{nNLL}).

\begin{table}[]
  \centering
  \begin{tabular}{lccccccc}
      \toprule
      \textsc{Method}          & {HMC}           & {SGD}            & \textsc{Deep Ens.} & {\SGLD}                & \algoun{}                 & \algoquatre{}  \\
      \midrule
      \texttt{Accuracy}        & $89.6 \pm 0.25$ & $91.57 \pm 0.34$ & $91.68 \pm 0.17$ & $89.96 \pm 0.72$         & $\textbf{92.54} \pm 0.04$ & $92.03 \pm 0.09$ \\
      \texttt{Agreement}       & $94.0 \pm 0.25$ & $90.99 \pm 0.35$ & $91.03 \pm 0.43$ & $\textbf{92.43} \pm 0.03$ & $91.53 \pm 0.39$          & $91.12 \pm 0.39$ \\
      $10\times$ \texttt{TV}   & $0.74 \pm 0.03$ & $1.45 \pm 0.05$  & $1.49 \pm 0.05$  & $\textbf{1.03} \pm 0.03$ & $1.42 \pm 0.01$           & $1.39 \pm 0.01$ \\
      $10^2\times$\texttt{ECE} & $5.9 \pm$NA     & $4.71 \pm 1.35$  & $5.44 \pm 0.67$  & $4.41 \pm 0.37$          & $3.79 \pm 0.11$           & $\textbf{3.26}\pm 0.09$ \\
      $10\times$\texttt{BS}    & $1.4 \pm$NA     & $1.69 \pm 0.11$  & $1.45 \pm 0.10$  & $1.53 \pm 0.10$          & $\textbf{1.16} \pm 0.03$  & $1.20 \pm 0.03$ \\
      $10\times$\texttt{nNLL}  & $3.07 \pm$NA    & $3.35 \pm 0.70$  & $3.81 \pm 0.51$  & $3.15 \pm 0.21$          & $2.75 \pm 0.04$           & $\textbf{2.63} \pm 0.04$ \\
      \bottomrule
  \end{tabular}
  \caption{Performance of Bayesian FL algorithms on CIFAR10.}
  \label{table:cifar10-comparison}
\end{table}

\end{document}